\documentclass[PhD,binding=0.6cm]{sapthesis}

\usepackage{microtype}
\usepackage{comment}
\usepackage{color}
\usepackage{amsmath}
\usepackage{amsthm}
\usepackage{amssymb}
\usepackage{latexsym}
\usepackage{wasysym}

\usepackage{alltt}

\usepackage[usenames,dvipsnames]{xcolor}
\usepackage{xspace}
\usepackage{multirow}
\usepackage[inline]{enumitem}
\usepackage{caption}

\usepackage{graphicx}
\usepackage{subfigure}

\usepackage{esvect}
\usepackage[ruled,vlined]{algorithm2e}
\usepackage{url}
\usepackage{rotating}

\definecolor{light-gray}{gray}{0.95}

\usepackage{hyperref}
\hypersetup{colorlinks=true,
	linkcolor=blue,
	anchorcolor=blue,
	citecolor=blue,
	urlcolor=blue,
	pdftitle={PhD Thesis of Andrea Marrella},
	pdfauthor={Andrea Marrella}}

\title{\textsf{SmartPM}: Automatic Adaptation of \\Dynamic Processes at Run-Time}
\author{Andrea Marrella}
\IDnumber{796069}
\course[Engineering in Computer Science]{Ingegneria Informatica} 
\courseorganizer{Dipartimento di Ingegneria Informatica, Automatica e Gestionale A. Ruberti,
Universit\'a di Roma ``\texttt{La Sapienza}''}
\cycle{XXV}
\submitdate{October 2013}
\copyyear{2013}
\advisor{Prof. Massimo Mecella (Tutor)\\Prof. Daniele Nardi\\Prof. Luca Iocchi\\Prof. Umberto Nanni\\Prof. Stavros Vassos\\Prof. Marco Aiello (Ext. Reviewer)\\Prof. S. Rinderle Ma (Ext. Reviewer)}
\authoremail{marrella@dis.uniroma1.it}

\website{http://www.dis.uniroma1.it/$\sim$marrella}
\begin{document}

\frontmatter

\maketitle

\dedication{Dedicated to my grandmother, with love.}

\begin{acknowledgments}[Ringraziamenti]
\\
Questo lavoro di Tesi è il risultato di un percorso di ricerca durato 3 anni e mezzo. Provo emozioni contrastanti se ripenso all'esperienza di Dottorato che ho appena concluso. Ho vissuto momenti difficili, talvolta frustranti, ma l'impegno e la passione per la ricerca sono stati più forti delle difficoltà. In definitiva, non cambierei nulla della mia esperienza di Dottorato. Penso che mi abbia reso una persona migliore, dal punto di vista accademico ed umano.

Non sarei mai riuscito a completare questo percorso senza l'apporto di persone speciali. A questo proposito, vorrei principalmente ringraziare Chiara, che non ha mai mancato di incoraggiarmi e mi è stata vicina soprattutto nei momenti più difficili, e la mia famiglia, che mi ha sempre mostrato un sostegno impagabile in ogni situazione. Ringrazio anche i miei colleghi di Dottorato, in particolare Alessandro, Claudio, Francesco, Donatella, Mario, Riccardo, Paolo. Mi ritengo fortunato ad aver condiviso una parte della mia vita con loro.

Un ringraziamento speciale va a Massimo, grazie al quale ho acquisito un'esperienza lavorativa ed umana più unica che rara, e a quei ricercatori con cui ho lavorato e condiviso idee di ricerca. Una menzione particolare va al Prof. Yves Lesperance, grazie al quale ho potuto intraprendere una fantastica esperienza di ricerca e di vita in Canada, che mi ha arricchito culturalmente e professionalmente.
\\
\\
Grazie di cuore a tutti.

\end{acknowledgments}

\tableofcontents

\begin{sloppypar}
\newtheorem{mydef}{Definition}
\newtheorem{definition}{Definition}
\newtheorem{theorem}{Theorem}

\newtheorem{example}{Example}[section]

\newcommand {\myi} {\emph{(i)}~}
\newcommand {\myii} {\emph{(ii)}~}
\newcommand {\myiii} {\emph{(iii)}~}
\newcommand {\myiv} {\emph{(iv)}~}
\newcommand {\myv} {\emph{(v)}~}

\newcommand{\DD}{\mbox{$\cal D$}}                     
\newcommand{\CC}{\mbox{$\cal C$}}                     
\newcommand{\RR}{\mbox{$\cal R$}}                     

\newcommand{\hilight}[1] {\colorbox{yellow}{#1}}

\newcommand {\smartpm} {\textsf{SmartPM}\xspace}
\newcommand {\smartpmML} {\textsf{SmartPM Modeling Language}\xspace}
\newcommand {\smartML} {\textsf{SmartML}\xspace}
\newcommand {\indigolog} {\textsf{IndiGolog}\xspace}
\newcommand {\ConGolog} {\textsf{ConGolog}\xspace}
\newcommand {\Golog} {\textsf{Golog}\xspace}
\newcommand {\golog} {\textsf{Golog}\xspace}
\newcommand {\swiprolog} {\textsf{SWI-Prolog}\xspace}
\newcommand {\SWIProlog} {\textsf{SWI-Prolog}\xspace}
\newcommand {\sitcalc} {Situation Calculus\xspace}
\newcommand{\Prol}[1]{\texttt{\mbox{#1}}}
\newcommand{\kTrue}{\texttt{kTrue}}
\newcommand{\mTrue}{\texttt{mTrue}}

\newcommand{\arrow}[1]{\,\hbox{$\stackrel{\small \mbox{#1}}{\rightarrow}$}\,}
\newcommand{\arrowstar}{\,\hbox{$\stackrel{\mbox{$\ast$}}{\rightarrow}$}\,}
\newcommand{\isdef}{\ \hbox{$\stackrel{\mbox{\tiny def}}{=}$}\ }

\newcommand{\fontFluents}[1]{\mbox{\textit{#1}}}
\newcommand{\fontActions}[1]{\mbox{\textsl{#1}}}

\newcommand{\Poss}{\fontFluents{Poss}}
\newcommand{\Trans}{\fontFluents{Trans}}
\newcommand{\Final}{\fontFluents{Final}}
\newcommand{\Do}{\fontFluents{Do}}

\newcommand{\aUnlock}{\fontFluents{unlock}}
\newcommand{\aOpenDoor}{\fontFluents{open}}
\newcommand{\aCloseDoor}{\fontActions{close}}
\newcommand{\fDoorOpen}{\fontFluents{DoorOpen}}
\newcommand{\fLocked}{\fontFluents{Locked}}
\newcommand{\fColor}{\fontFluents{Color}}
\newcommand{\fFree}{\fontFluents{Free}}
\newcommand{\fReserved}{\fontFluents{Reserved}}
\newcommand{\fAssigned}{\fontFluents{Assigned}}
\newcommand{\fIdentifier}{\fontFluents{Identifier}}

\def\prparallel{\mathrel{\rangle\!\rangle}}
\newcommand{\mif}{\mbox{\bf if}}
\newcommand{\mwhile}{\mbox{\bf while}}
\newcommand{\mreturn}{\mbox{\bf return}}
\newcommand{\mthen}{\mbox{\bf then}}
\newcommand{\melse}{\mbox{\bf else}}
\newcommand{\mdo}{\mbox{\bf do}}
\newcommand{\mnoOp}{\mbox{\bf noOp}}
\newcommand{\mproc}{\mbox{\bf proc}}
\newcommand{\mend}{\mbox{\bf end}}
\newcommand{\mendproc}{\mbox{\bf endProc}}
\newcommand{\mendif}{\mbox{\bf endIf}}
\newcommand{\mendwhile}{\mbox{\bf endWhile}}
\newcommand{\mendfor}{\mbox{\bf endFor}}
\newcommand{\mfor}{\mbox{\bf for}}
\def\supparallel{\mathord{|\!|}}
\newcommand{\tuple}[1]{\ensuremath{\langle #1 \rangle}}     
\newcommand{\search}{\mbox{$\Sigma$}}
\newcommand{\nil}{\mbox{\it nil}}
\newcommand{\false}{\mathtt{false}}
\newcommand{\False}{\mathtt{false}}
\newcommand{\true}{\mathtt{true}}
\newcommand{\True}{\mathtt{True}}
\newcommand{\xdo}{\mbox{\it do}}

\newcommand{\aAssign}{\fontActions{assign}}

\newcommand{\recovered}{\fontFluents{Recovered}}
\newcommand{\Monitor}{\fontFluents{Monitor}}
\newcommand{\Relevant}{\fontFluents{Relevant}}
\newcommand{\fRelevant}{\fontFluents{Relevant}}
\newcommand{\SameConfig}{\fontFluents{SameConfig}}
\newcommand{\SameState}{\fontFluents{SameState}}
\newcommand{\Linear}{\fontFluents{Linear}}
\newcommand{\Recovery}{\fontFluents{Recovery}}
\newcommand{\fRealityChanged}{\fontFluents{RealityChanged}}

\newcommand{\fService}{\fontFluents{Service}}
\newcommand{\workitem}{\fontFluents{Workitem}}
\newcommand{\fTask}{\fontFluents{Task}}
\newcommand{\fCapability}{\fontFluents{Capability}}
\newcommand{\fProvides}{\fontFluents{Provides}}
\newcommand{\fRequires}{\fontFluents{Requires}}
\newcommand{\fNeigh}{\fontFluents{Neigh}}
\newcommand{\fCovered}{\fontFluents{Covered}}
\newcommand{\fCapable}{\fontFluents{Capable}}
\newcommand{\fAvailable}{\fontFluents{Available}}
\newcommand{\fListElem}{\fontFluents{ListElem}}
\newcommand{\fVar}{\fontFluents{Var}}
\newcommand{\fX}{\fontFluents{X}}
\newcommand{\fY}{\fontFluents{Y}}
\newcommand{\fEnabled}{\fontFluents{Reserved}}
\newcommand{\DomainOfX}{\fontFluents{DomainOfX}}

\newcommand{\aStart}{\fontActions{start}}
\newcommand{\aAck}{\fontActions{ackCompl}}
\newcommand{\aRelease}{\fontActions{release}}
\newcommand{\aReady}{\fontActions{readyToStart}}
\newcommand{\aFinish}{\fontActions{finishedTask}}
\newcommand{\aGo}{\fontActions{go}}
\newcommand{\aMove}{\fontActions{move}}
\newcommand{\aTakePhoto}{\fontActions{takephoto}}
\newcommand{\aEvacuate}{\fontActions{evacuate}}
\newcommand{\aRemoveDebris}{\fontActions{removedebris}}
\newcommand{\aExtinguishFire}{\fontActions{extinguishfire}}
\newcommand{\aUpdateStatus}{\fontActions{updatestatus}}
\newcommand{\aChargeBattery}{\fontActions{chargeBattery}}

\newcommand{\aAdaptStart}{\fontActions{adaptStart}}
\newcommand{\aAdaptFinish}{\fontActions{adaptFinish}}
\newcommand{\aResetReality}{\fontActions{resetReality}}

\newcommand{\fAt}{\fontFluents{At}}
\newcommand{\fAtRobot}{\fontFluents{AtRobot}}
\newcommand{\fEvacuated}{\fontFluents{Evacuated}}
\newcommand{\fStatus}{\fontFluents{Status}}
\newcommand{\fPhotoTaken}{\fontFluents{PhotoTaken}}

\newcommand{\fMoveStep}{\fontFluents{MoveStep}}
\newcommand{\fDebrisStep}{\fontFluents{DebrisStep}}
\newcommand{\fBatteryRecharging}{\fontFluents{BatteryRecharging}}
\newcommand{\fGeneralBattery}{\fontFluents{GeneralBattery}}

\newcommand{\send}{\fontActions{send}}
\newcommand{\alignRealities}{\fontActions{alignRealities}}
\newcommand{\receive}{\fontActions{receive}}

\newcommand{\isConnected}{\fontFluents{isConnected}}
\newcommand{\isRobotConnected}{\fontFluents{isRobotConnected}}

\newcommand{\fBatteryLevel}{\fontFluents{BatteryLevel}}

\newcommand{\fPhotoOK}{\fontFluents{PhotoOK}}
\newcommand{\fRescueOK}{\fontFluents{RescueOK}}
\newcommand{\fSurveyOK}{\fontFluents{SurveyOK}}
\newcommand{\aPhoto}{\fontActions{photo}}
\newcommand{\aSurvey}{\fontActions{survey}}
\newcommand{\aRescue}{\fontActions{rescue}}

\newcommand{\photoLost}{\fontActions{photoLost}}
\newcommand{\rockSlide}{\fontActions{rockSlide}}
\newcommand{\fireRisk}{\fontActions{fireRisk}}

\newcommand{\fAdapting}{\fontFluents{Adapting}}

\newcommand {\ps} {\emph{Planning Service}\xspace}
\newcommand {\planlet} {\textsc{Planlet}\xspace}
\newcommand {\planlets} {\textsc{Planlets}\xspace}

\newcommand{\proc}[1]{\textbf{\normalfont \textbf{#1}}}
\newcommand{\fExog}{\fontFluents{Exogenous}}
\newcommand{\aResetExog}{\fontActions{resetExog}}
\newcommand{\fIsExog}{\fontFluents{ExogAction}}
\newcommand{\fIsFinished}{\fontActions{Finished}}
\newcommand{\finish}{\fontActions{finish}}
\newcommand{\wait}{\fontActions{wait}} 

\chapter{Extended Abstract}
\label{ch_extended_abstract}

Business Process Management~\cite{WeskeBook2007} (a.k.a. BPM) is a ``hot topic'' because it is highly relevant from a practical point of view while at the same it offers many challenges for computer scientists and researchers. It is based on the observation that each product that a company provides to the market is the outcome of a number of activities performed. \emph{Business processes} are the key instruments to organizing these activities and to improving the understanding of their interrelationships. BPM addresses the topic of process support in a broader perspective by incorporating different types of analysis (e.g., simulation, verification, and process mining) and linking processes to business and social aspects. Moreover, the current interest in BPM is fueled by technological developments (e.g., service oriented architectures) triggering standardization efforts. Several research issues have been addressed about the definition of models for describing processes, some of them more targeted towards domain-specific business designers (e.g., UML Activity Diagrams~\cite{umlactivitydiagrams}, BPMN -- Business Process Modeling Notation~\cite{BPMN20}), and others more targeted to formal definitions of processes, in order to enable verification over process schemas (e.g., workflow nets~\cite{vanderAllst@WorkflowNets1998} -- a variant of Petri Nets~\cite{murata1989petri,petriNet} targeted to describing processes, YAWL~\cite{YAWLBook2009}, etc.).

A \emph{Process Management System}~\cite{vanderAllstBook2004} (a.k.a. PMS) is a generic software system that is driven by explicit process representations (also called \emph{process models}) to coordinate the enactment of business processes, aiming at increasing the efficiency and effectiveness in their execution.
A process model, which is always built at design-time, is in charge of organizing the execution order of the activities of the business process. The basic constituents of a process model are \emph{tasks} that describe an activity to be performed by an automated service (e.g., within a service-oriented architecture) or a human (e.g., an employee). The procedural rules to control the tasks are usually described by routing constructs like sequences, loops, parallel and alternative branches that form the \emph{control flow} of the process~\cite{Minor@IS2012}.

The core of a PMS is the engine that takes in input a process model and manages the process routing by deciding which tasks are enabled for execution, taking into account the control flow, the value of process variables and tasks constraints. The representation of a single execution of a process model within the engine of the PMS is called a \emph{process instance}~\cite{Hollingsworth@WfMC1995}. Once a task is ready for being assigned, the engine is also in charge of assigning it to proper participants; this step is performed by considering the participant ``skills'' required by the single task: a task will be assigned to the participant that provides all of the skills required. Participants are provided with a client application, part of the PMS, named \emph{Task Handler}. It is aimed at receiving notifications of task assignments. Participants can, then, use this application to pick the next task to work on. Current technologies exist on the market which concretely allow the enactment of processes (e.g., the YAWL Engine~\cite{YAWLBook2009} and the jBPM orchestration engine~\cite{jBPMBook2007} based on the WS-BPEL~\cite{bpel} specification, etc.).

Traditionally, PMSs have focused on the support of predictable and repetitive business processes, which can be fully pre-specified in terms of formal process models. All possible paths through those processes are well-understood, and the process participants usually do not need to make a decision about what to do next since the path is completely determined by their data entry or other attributes of the process. This kind of highly-structured work includes mainly production and administrative processes (such as financial services, manufacturing, etc.)~\cite{LeymannBook2000}. However, current maturity of process management methodologies has led to the application of process-oriented approaches in new challenging knowledge-intensive scenarios~\cite{KiBP2012}, such as healthcare~\cite{Reichert-healthcare,Reichert-healthcare-processes_2007} or home automation (e.g., domotics~\cite{Helal2005}). In such working environments, changes in the operational context and in other heterogeneous contextual information may occur unpredictably and at any time, requiring the ability to react to those changes and properly \emph{adapt} and \emph{modify} process behavior. This has led to the need to provide support for flexible and adaptive process management~\cite{Weske@HICSS2001,ReichertBook2012}, by reconsidering the trade-off between flexibility and support provided by existing PMSs~\cite{Pesic@RD2009}. According to~\cite{Sadiq@ER2001}, \emph{Process Adaptation} can be seen as the ability of a process to react to exceptional circumstances (that may or may not be foreseen) and to adapt/modify its structure accordingly.

The current-day leading commercial PMS products~\cite{COSA,TIBCO,WebSphere,SAP} and research prototypes~\cite{Alonso@TSE2000,ADOME,YAWLBook2009} provide some techniques to react to exceptions and adapt process instances to mitigate their effects. Specifically, they provide the support for the handling of \emph{expected exceptions}. The process models are designed in order to cope with potential exceptions, i.e., for each kind of exception that is envisioned to occur, a specific contingency process (a.k.a. \emph{exception handler} or compensation flow) is defined. These approaches perform well with predictable and repetitive business processes, where all possible exceptions and deviations that can be encountered are predictable and defined in advance, along with the specific handling logic. However, in knowledge-intensive scenarios, the process usually \emph{dynamically evolves}, i.e., it strongly depends on user decisions made during process execution. For those processes, it is not possible to anticipate all real-world exceptions and to capture their handling in a process model at design-time.

In this direction, since the last Nineties, a new class of \emph{Adaptive Process Management Systems} is emerged, by facilitating structural changes of processes at run-time~\cite{Dadam@JIIS1998,adaptflow,agentwork,Reichert@ICDE2005,adeptcbr,Weber@ACBR2004}. When something goes wrong during the process execution, structural changes apply directly to process elements, and the adaptation is carried out by deleting, adding, or modifying one or several process elements. For example, an adaptive PMS like ADEPT2~\cite{Reichert@ICDE2005} is able to support the handling of unanticipated exceptions, by enabling different kinds of ad-hoc deviations from the pre-modeled process instance at run-time, according to the structural process change patterns defined in~\cite{Weber@DKE2008}. New process models can be created and tailored for a particular demand or business case, and process instances can be adapted after they have been started if some unforeseen events occur. Currently, adaptive PMSs have reached such a level of maturity that they are about to be transferred into practice~\cite{Minor@IS2012}.

The majority of the above approaches face the challenge to provide flexibility and adaptation and to offer process support at the same time. Traditional PMSs deal with expected exceptions at design-time by automatically providing an exception handler at run-time, while adaptive PMSs offers structural process change at run-time for unanticipated exceptions, but they do not automate the adaptation; a manual intervention of a domain expert is always required for adapting a faulty process instance at run-time. However, in the last years, the widespread availability of mobile computing platforms has led to the application of process-oriented approaches in \emph{pervasive and highly dynamic scenarios}~\cite{IEEE2008,EMSOA2010,ISCRAM2010,IJISCRAM2011}. An interesting example comes from the emergency management domain. During the management of complex emergency scenarios, teams of first responders act in disaster locations with the main purpose of achieving specific goals, including assisting potential victims and assessing and stabilizing the situation. The set of activities and procedures that collectively define an emergency response plan are characterized for being as complex as typical business processes and for involving teams of many members. Emergency response operators can benefit from the use of mobile devices and wireless communication technologies, as well as from the adoption of a process-oriented approach for team coordination~\cite{Rome4EU}. A response plan encoded as a business process and executed by a PMS deployed on mobile devices can help coordinate the activities of emergency operators equipped with PDAs and smartphones and supported by mobile networks. In this dynamic context the environment may change continuously and processes can be easily invalidated because of exogenous events and of tasks not executed as expected. This means that \myi it is not possible to predict all possible exceptions at design-time and that \myii adaptation ought to be as \emph{automatic} as possible and to require \emph{minimum manual human intervention} at run-time; in fact, in emergency management, saving minutes could result in saving injured people, preventing buildings from collapses, and so on. 


\section*{Research Contributions}
\label{sec:research_contributions}

The main focus of the author's research activity is to devise an approach for \emph{run-time automatic adaptation} of \emph{dynamic processes}. Dynamic processes are a particular kind of processes for which there is not a clear, anticipated correlation between a change in the context and corresponding process changes. Usually, the structure of a dynamic process can be completely captured with a procedural process model that explicitly defines the tasks and their execution constraints. Examples of dynamic processes are processes for emergency management (an extensive case study is presented in Section~\ref{sec:introduction-case_study}) and military forces deployment plans. A dynamic process is thought to be enacted in pervasive and highly dynamic scenarios, where exceptions and exogenous events ``are not the exception but the rule''. Dynamic context changes or undesirable outcome of some activities may often cause abnormal termination of the process activities and prevent the achievement of the business goals.

Traditional approaches that try to anticipate how the work will happen by solving each problem at design time~\cite{COSA,TIBCO,WebSphere,SAP,Alonso@TSE2000,ADOME,YAWLBook2009}, as well as approaches that allow to manually change the process structure at run time~\cite{Dadam@JIIS1998,adaptflow,agentwork,Reichert@ICDE2005,adeptcbr,Weber@ACBR2004}, are often ineffective or not applicable in rapidly evolving contexts. The design-time specification of all possible compensation actions requires an extensive manual effort for the process designer, that has to anticipate all potential problems and ways to overcome them in advance, in an attempt to deal with the unpredictable nature of dynamic processes. Moreover, the designer often lacks the needed knowledge to model all the possible contingencies, or this knowledge can become obsolete as process instances are executed and evolve, by making useless his/her initial effort. Although the exploitation of current adaptive PMSs to support the enactment of processes in pervasive and mobile scenarios represents a promising and helpful approach, dynamic processes demand a more agile approach recognizing the fact that in dynamic environments process models quickly become outdated and hence require closer interweaving of modeling and execution~\cite{ReichertBook2012}.

With respect to these needs, the author's research main target is to devise \emph{intelligent failure handling mechanisms} that allow to monitor running process instances and to react to tasks failures or to the occurrence of exogenous events that may put at risk process executions, by providing \emph{automatic adaptation} to identify a suitable compensation and recovery strategy. The idea is to provide an adaptation framework that is able to deal automatically with \emph{unanticipated exceptions at run-time}, without explicitly defining any handler/policy to recover from exceptions at design-time, and without the intervention of domain experts.

If compared with classical business processes, the execution and adaptation of dynamic processes demand two special requirements:
\begin{itemize}[itemsep=1pt,parsep=0.5pt]
\item A dynamic process must be \emph{adaptable to the context}. This means there is the need to explicitly define the contextual data describing the scenario in which the process will be enacted;
\item A PMS that executes a dynamic process must be able to \emph{automatically detect} exceptional situations, to \emph{derive} and to correctly \emph{apply} the recovery procedures necessary to handle them.
\end{itemize}
To this end, we use a specialized version of the concept of adaptation from the field of agent-oriented programming~\cite{GiacomoRS98}. Specifically, we consider adaptation as reducing the gap between the \emph{expected reality}, the (idealized) model of reality that is used by a PMS to reason on the dynamic process under execution, and the \emph{physical reality}, the real world with the actual values of conditions and outcomes. A misalignment of the two realities often stems from errors in the tasks outcomes (e.g., incorrect data values) or is the result of exogenous events coming from the environment.

Our approach to process adaptation is mainly based on well-established techniques and frameworks from Artificial
Intelligence (a.k.a. AI), such as \emph{\sitcalc}\cite{ReiterBook}, \indigolog~\cite{Indigolog:2009} and \emph{automatic planning}~\cite{TraversoBook2004}. \sitcalc is a logical language specifically designed for representing dynamically changing worlds in which all changes are the result of named actions. We used \sitcalc for providing a declarative specification of the domain (i.e., available tasks, contextual properties, tasks preconditions and effects, what is known about the initial state) in which the dynamic process has to be executed. On top of \sitcalc, we used the \indigolog language for formalizing the structure and the control flow of our dynamic process. \indigolog is a logic-based programming language used for robot and agent programming.

Then, we propose a PMS realization, namely \smartpm~\cite{COMA2008,SMARTPM2011}, which is based on the \indigolog\ interpreter\footnote{http://sourceforge.net/projects/indigolog/} developed at University of Toronto and RMIT University, Melbourne. The \indigolog interpreter reasons about the preconditions and effects of the process tasks to find a legal terminating execution of the process. \indigolog programs are executed online together with sensing the environment and monitoring for events. When an exception or an exogenous event is sensed, it results in a discrepancy between the physical and expected reality, and the \indigolog interpreter is in charge to determine if such event is able to invalidate the execution of the dynamic process under execution. If so, \smartpm allows the synthesis of a recovery procedure at run-time by invoking an external \emph{state-space planner}~\cite{LPG}. In AI, planning systems are problem-solving algorithms that operate on explicit representations of an initial state (the faulty process state that reflects the physical reality), a goal condition (the process state reflecting the expected reality) and actions (the set of tasks executable in the contextual scenario under observation). A state-space planner explores only strictly linear sequences of actions directly
connected from the initial state to the goal, i.e., in our case, it searches for a plan that may turn the physical reality into the expected reality. If the recovery plan exists, it will be executed by \smartpm for adapting the faulty process instance. Since the adaptation mechanism deployed on \smartpm is \emph{blocking} (i.e., the execution of the main process is stopped during the synthesis/enactment of the recovery procedure), we also propose a \emph{non-blocking} repairing mechanism based on \emph{continuous planning} techniques~\cite{BPMDS2011}, that avoids to stop directly any task in the main process during the computation of the recovery process.

\bigskip

In the second part of the thesis, we analyze a problem that often involves the design-time specification of a dynamic process. Since resources of a dynamic process are usually shared by the process participants, it is difficult to foresee all the potential tasks interactions in advance and there is the risk that concurrent tasks could not be independent from one from another (e.g., they could operate on the same data at the same time), resulting in incorrect outcomes. We address this issue and proposing an approach~\cite{BPMDS2013} that exploits \emph{partial-order planning algorithms}~\cite{TraversoBook2004,Weld@AImag1994} for building automatically a library of process template definitions.  Partial-order planning differs from classical state-space planning algorithms, that explore only strictly linear sequences of actions directly connected to the start or goal, by devising totally ordered plans. On the contrary, partial-order planning is based on the least commitment principle~\cite{Weld@AImag1994}, whose main advantage is that decisions about action ordering are postponed until a decision is forced, thus guaranteing flexibility in the execution of the plan and by possibly permitting actions to run concurrently. The strength of the approach is that resulting templates are reusable in a variety of partially-known contextual environments, and all concurrent tasks composing the templates are effectively independent one from another.

\bigskip

Finally, although \smartpm is born as a PMS for supporting first responders in emergency management scenarios, the planning-based adaptation approach employed is enough general for being used on top of existing PMSs. Specifically, in the third part of this Thesis, we concretize our approach to automatic process adaptation on top of the well-known YAWL modeling language and execution environment~\cite{YAWLBook2009}. To this end, we show a concrete design and implementation proposal of how the YAWL architecture can be extended to integrate planning capabilities and to support the handling of unanticipated exception at run-time~\cite{BPMDEMO2011,CollCOM2011,CoopIS2012}.




In summary, the main contributions and results of the author's research activity can be summarized as follow:
\begin{itemize}[itemsep=1pt,parsep=1pt]
\item[\textbf{R1}] First of all, in~\cite{ISCRAM2007} we collected and analyzed the general requirements for the application of a process-oriented approach in emergency management scenarios. Starting from those requirements, in~\cite{EMSOA2010} we focus on the challenges related to the design and implementation of a PMS able to support first responders on the field, analyzing the core components of the overall architectural model.
\item[\textbf{R2}] In~\cite{IEEE2008} we present the high level architecture for emergency management systems devised in the European research project WORKPAD\footnote{http://www.dis.uniroma1.it/~workpad/}. Such an architecture is specifically tailored in supporting collaborative work of human operators during disaster scenarios, where different teams, belonging to different organizations, need to collaborate in order to reach a common goal. The overall approach, user-centered methodology and achieved results in the design, implementation and validation of the WORKPAD architecture are detailed in~\cite{SPR2009,HCI2009,ISCRAM2008,IJISCRAM2011}. Moreover, a qualitative and quantitative evaluation of main outcomes, as a result of on-the-field validation and showcase activities of the project are shown in~\cite{ISCRAM2010}.
\item[\textbf{R3}] Starting from the experience gained in the area and lessons learned, in~\cite{ISCRAM2011} we provide a general set of guidelines, suggestions and possible research directions on how to effectively design mobile information systems for supporting on-the-field collaboration of emergency operators.
\item[\textbf{R4}] In~\cite{CBMS2012} we present the design and prototype of a PMS for improving operational support to clinicians during their daily activities in hospital wards, on the basis of clinical guidelines.
\item[\textbf{R5}] In~\cite{KiBP2012}, starting from three different dynamic real world scenarios, we present a critical and comparative analysis of the existing approaches used for supporting, modeling and adapting dynamic and knowledge-intensive processes.
\item[\textbf{R6}] In~\cite{COMA2008,SMARTPM2011,BPMDS2011} we present and formalize \smartpm, our AI-based PMS that deals automatically with unanticipated exceptions at run-time. Specifically, in~\cite{COMA2008,SMARTPM2011} we exploit a built-in adaptation mechanism offered by the \indigolog interpreter (basically, a simple planner based on a breadth-first search algorithm) for the synthesis of the recovery procedure. In order to improve drastically the time needed for finding the recovery plan, in~\cite{BPMDS2011} we delegate every aspect of adaptation to an external state-of-the-art planner. This allows to separate the planning phase with the modeling and execution phase, and to introduce a \emph{non-blocking} repairing mechanisms, based on \emph{continuous planning} techniques.
\item[\textbf{R7}] In~\cite{BPMDEMO2011} we contextualize and demonstrate our approach in a service-oriented environment, as an application of the architectural solutions for the integration of different modeling approaches to achieve flexibility~\cite{faas}. We show how the YAWL environment (and its imperative modeling approach) can be complemented with the \smartpm execution environment~\cite{SMARTPM2011}.
\item[\textbf{R8}] In~\cite{CollCOM2011} we propose a general approach and a conceptual architecture to automatic process adaptation, based on the concept of declarative modeling of processes and the use of continuous planning techniques; we show the feasibility of the proposed approach by discussing its deployment on top of the YAWL system~\cite{YAWLBook2009}.
\item[\textbf{R9}] In~\cite{CoopIS2012}, we introduce and define \textsc{Planlets}, as self-contained YAWL specifications where process tasks are annotated at design-time with pre-conditions, desired effects and post-conditions. The role of declarative task annotations is twofold: \myi pre- and post-conditions enable run-time process execution monitoring and exception detection: they are checked respectively before and after task executions, and the violation of a pre- or post-condition results in an exception to be handled; \myii along with the input/output parameters consumed/produced by the task, pre-conditions and effects provide a complete specification of the task: this allows the task to be represented as an action in a planning domain description and used for solving a planning problem built to handle an exception. In the presence of an exception, this approach allows delegating to an external planner the automatic run-time synthesis of a suitable recovery procedure by contextually selecting the compensation tasks from a specific repository linked to the \textsc{Planlet} under execution.
\item[\textbf{R10}] Since the design time specification of dynamic processes can be time-consuming and error-prone, due to the high number of tasks involved and their context-dependent nature, such processes frequently suffer from potential interference among their constituents, since resources are usually shared by the process participants and it is difficult to foresee all the potential tasks interactions in advance. Concurrent tasks may not be independent from one from another (e.g., they could operate on the same data at the same time), resulting in incorrect outcomes.  To address these issues, in~\cite{BPMDS2013} we propose an approach that exploits partial-order planning algorithms for automatically synthesizing a library of process template definitions for different contextual cases. The resulting templates guarantee sound concurrency in the execution of their activities and are reusable in a variety of partially-known contextual environments.
\end{itemize}
During the realization of this thesis, the following publications have been produced:

\begin{small}
\begin{itemize}[itemsep=1pt,parsep=1pt]
\item[\cite{BPMDS2013}] \textbf{A. Marrella, Y. Lespérance}\\
    \emph{Synthesizing a Library of Process Templates through Partial-Order Planning Algorithms}.
    14th International Working Conference on Business Process Modeling, Development and Support (BPMDS 2013), Valencia, Spain, 17-18 June 2013.
\item[\cite{CoopIS2012}] \textbf{A. Marrella, A. Russo, M. Mecella}\\
    \emph{Planlets: Automatically Recovering Dynamic Processes in YAWL}.
    20th International Conference on Cooperative Information Systems (CoopIS 2012) - OTM Conferences (1), Rome, Italy, 10-14 September 2012.
\item[\cite{CBMS2012}] \textbf{F. Cossu, A. Marrella, M. Mecella, A. Russo, G. Bertazzoni, M. Suppa, F. Grasso}\\
    \emph{Improving Operational Support in Hospital Wards through Vocal Interfaces and Process-Awareness}.
    25th IEEE International Symposium on Computer-Based Medical Systems (CBMS 2012), Rome, Italy, 20-22 June 2012.
\item[\cite{KiBP2012}] \textbf{C. Di Ciccio, A. Marrella, A. Russo}\\
    \emph{Knowledge-intensive Processes: An Overview of Contemporary Approaches}.
    1st International Workshop on Knowledge-intensive Business Processes (KiBP 2012)), Rome, Italy, 15 June 2012.
\item[\cite{CollCOM2011}] \textbf{A. Marrella, M. Mecella, A. Russo}\\
    \emph{Featuring Automatic Adaptivity through Workflow Enactment and Planning}.
    7th International Conference on Collaborative Computing: Networking, Applications and Worksharing (CollaborateCom 2011), Orlando, Florida, USA, 15-18 October 2011.
\item[\cite{BPMDEMO2011}] \textbf{A. Marrella, M. Mecella, A. Russo, A.H.M. ter Hofstede, S. Sardina}\\
    \emph{Making YAWL and SmartPM Interoperate: Managing Highly Dynamic Processes by Exploiting Automatic Adaptation Features}.
    9th International Conference on Business Process Management (BPM 2011), Demonstration Track, Clermont-Ferrand, France, 28 August - 02 September 2011.
\item[\cite{SMARTPM2011}] \textbf{M. de Leoni, A. Marrella, M. Mecella, S. Sardina}\\
    \emph{SmartPM - Featuring Automatic Adaptation to Unplanned Exceptions}.
    Technical Report of Dipartimento di Informatica e Sistemistica ANTONIO RUBERTI, SAPIENZA - Università di Roma. June 2011.
\item[\cite{BPMDS2011}] \textbf{A. Marrella, M. Mecella}\\
    \emph{Continuous Planning for solving Business Process Adaptivity}.
    12th International Working Conference on Business Process Modeling, Development and Support (BPMDS 2011), London, UK, 20-21 June 2011.
\item[\cite{IJISCRAM2011}] \textbf{T. Catarci, M. de Leoni, A. Marrella, M. Mecella, A. Russo, M. Bortenschlager, R. Steinmann}\\
    \emph{WORKPAD : Process Management and Geo-Collaboration Help Disaster Response}.
    International Journal of Information Systems for Crisis Response and Management (IJISCRAM), Volume 3, Issue 1, pp. 32--49, 2011.
\item[\cite{ISCRAM2011}] \textbf{A. Marrella, M. Mecella, A. Russo}\\
    \emph{Collaboration On-the-field: Suggestions and Beyond}.
    8th International Conference on Information Systems for Crisis Response and Management (ISCRAM 2011), Lisbon, Portugal, 8-11 May 2011.
\item[\cite{EMSOA2010}] \textbf{M. de Leoni, A. Marrella, A. Russo}\\
    \emph{Process-aware Information Systems for Emergency Management}.
    International Workshop on Emergency Management through Service Oriented Architectures (EMSOA) co-located with the ServiceWave 2010 Conference, Ghent, Belgium, 13 December 2010.
\item[\cite{ISCRAM2010}] \textbf{T. Catarci, M. de Leoni, A. Marrella, M. Mecella, M. Bortenschlager, R. Steinmann}\\
    \emph{The WORKPAD Project Experience: Improving the Disaster Response through Process Management and Geo Collaboration}.
    7th International Conference on Information Systems for Crisis Response and Management (ISCRAM 2010), Seattle, USA, 2-5 May 2010.
\item[\cite{HCI2009}] \textbf{S. R. Humayoun, T. Catarci, M. de Leoni, A. Marrella, M. Mecella, M. Bortenschlager, R. Steinmann}\\
    \emph{The WORKPAD User Interface and Methodology: Developing Smart and Effective Mobile Applications for Emergency Operators}.
    13th International Conference on Human-Computer Interaction (HCI International 2009), Session ``Designing for Mobile Computing'', San Diego, USA, 19-24 July 2009.
\item[\cite{SPR2009}] \textbf{S. R. Humayoun, T. Catarci, M. de Leoni, A. Marrella, M. Mecella, M. Bortenschlager, R. Steinmann}\\
    \emph{Designing Mobile Systems in Highly Dynamic Scenarios. The WORKPAD Methodology}.
    Springer's International Journal on Knowledge, Technology and Policy, Volume 22, Number 1 - March 2009.
\item[\cite{COMA2008}] \textbf{M. de Leoni, A. Marrella, M. Mecella, S. Valentini, S. Sardina}\\
    \emph{Coordinating Mobile Actors in Pervasive and Mobile Scenarios: An AI-based Approach}.
    2nd IEEE International Workshop on Interdisciplinary Aspects of Coordination Applied to Pervasive Environments: Models and Applications (COMA 2008) at WETICE 08, Rome, Italy, 23-25 June 2008.
\item[\cite{ISCRAM2008}] \textbf{A. Capata, A. Marrella, R. Russo, M. Bortenschlager, H. Rieser}\\
    \emph{A Geo-based Application for the Management of Mobile Actors during Crisis Situations}.
    5th International Conference on Information Systems for Crisis Response and Management (ISCRAM 2008), Washington DC, USA, 4-7 May 2008.
\item[\cite{IEEE2008}] \textbf{T. Catarci, M. de Leoni, A. Marrella, M. Mecella, B. Salvatore, G. Vetere,S. Dustdar, L. Juszczyk, A. Manzoor, Hong-Linh Truong}\\
    \emph{Pervasive and Peer-to-Peer Software Environments for Supporting Disaster Responses}.
    IEEE Internet Computing Journal - Special Issue on Crisis Management - Volume 12, Number 1 - January-February 2008.
\item[\cite{ISCRAM2007}] \textbf{M. de Leoni, A. Marrella, M. Mecella, F. De Rosa, A. Poggi, A. Krek, F. Manti}\\
\emph{Emergency Management: from User Requirements to a Flexible P2P Architecture}.
4th International Conference on Information Systems for Crisis Response and Management (ISCRAM 2007), Delft, the Netherlands, 13-16 May 2007.
\end{itemize}
\end{small}
The work~\cite{BPMDS2013}, presented thoroughly in Chapter~\ref{ch:templates}, is the result of a research internship of the author at the Department of Computer Science and Engineering at York University in Toronto (Ontario, Canada), under the supervision of Prof. Yves Lespérance.

The implementation of the \indigolog based PMS has been developed in cooperation with Dr. Sebastian Sardina, research fellow at the Intelligent Agent Group of the RMIT University in Melbourne (Australia) and Dr. Massimiliano de Leoni, research assistant at the Eindhoven University of Technology (the Netherlands).

The works~\cite{BPMDEMO2011} and~\cite{CoopIS2012} are the result of a research collaboration with Prof. Arthur H. M. ter Hofstede, the co-leader of the BPM Group of the Faculty of Information Technology of Queensland University of Technology, Brisbane (Australia).

The author has also co-chaired a workshop on Knowledge-intensive Business Processes (KiBP 2012) held in Rome on June 15th, 2012, co-located with the 13th International Conference on Principles of Knowledge Representation and Reasoning (KR 2012)\footnote{Web site: \url{http://www.dis.uniroma1.it/~kibp2012/}}. The main focus of the workshop was to discuss about how the use of techniques that came from different fields, such as Artificial Intelligence (AI), Knowledge Representation (KR), Business Process Management (BPM), Service Oriented Computing (SOC), etc., can be used jointly for improving the modeling and the enactment phase of a knowledge-intensive process. The purpose was to devise interesting approaches that can still achieve the goals of understanding, visibility and control of these emergent processes. The KiBP 2012 proceedings are available online at \url{http://ceur-ws.org/Vol-861/}.

\section*{Thesis Outline}
\label{sec:extended_abstract-thesis_outline}

\begin{itemize}[itemsep=1pt,parsep=1pt]
\item Chapter~\ref{ch:introduction} introduces background concepts and definitions related to process flexibility and process adaptation, and provides a systematic view of the different approaches and methodologies that have emerged to support classes of processes with different requirements. This serves as the basis for positioning the performed work. Moreover, a case study based on a real emergency management scenario is presented.
\item Chapter~\ref{ch:state_of_the_art} analyzes the state of the art concerning process adaptation and process recovery. 
    Specifically, we first analyze existing process adaptation and exception handling techniques.
    Then, we discuss on the degree of adaptation/flexibility provided by several commercial PMSs and academic prototypes. Finally, we directly compare \smartpm with other research works that exploit AI techniques for improving the degree of process adaptation.
\item Chapter~\ref{ch:approach} focuses on our general approach to automatic process adaptation, that involves formalizing processes in \sitcalc and \indigolog. We clearly define the execution semantic provided by \smartpm for the enactment of dynamic processes, together with the logic used for monitoring running process instances and adapting them when needed through the use of classical planning techniques.
\item Chapter~\ref{ch:framework} discusses the overall architecture of \smartpm, our AI-based PMS, including technical details about the system implementation, the \indigolog interpreter and the Task Handler used for assigning tasks to process participants. Then, we describe the \smartpmML (a.k.a. \smartML), which combines a modeling formalism for representing the information of the contextual scenario linked to a specific dynamic process, and a graphical tool (specifically, Eclipse BPMN\footnote{\url{http://www.eclipse.org/modeling/mdt/?project=bpmn2}}) for designing the control flow of the process. We also show how a dynamic process formalized through \smartML is automatically translatable in \sitcalc and \indigolog readable formats and is therefore ready for being executed by \smartpm. Finally, we analyze the algorithms for converting a \smartML domain theory in a PDDL planning domain and for generating a PDDL planning problem when process adaptation is required.
\item Chapter~\ref{ch:validation} reports on performance evaluation and system validation activities. Specifically, we first report on experimental evaluation results, in terms of time needed for automatically adapting the dynamic process taken from our case study when exceptions of growing complexity arise. Then we measure the effectiveness of \smartpm in finding recovery procedures by simulating the execution of thousands of processes instances having different control-flows in different contextual environments.
\item Chapter~\ref{ch:templates} proposes an approach that exploits partial-order planning algorithms for synthesizing automatically a library of process template definitions starting from a declarative specification of process tasks. A template can be seen as the ``closest thing'' to a completely defined process model. It guarantees sound concurrency in the execution of its activities and is reusable in a variety of partially-known contextual environments.
\item Chapter~\ref{ch:planlets} provides an in-depth discussion, a concrete design and implementation proposal of how the YAWL architecture can be extended to integrate planning capabilities. For this aim, we propose the approach of \textsc{Planlets}, self-contained YAWL specifications with recovery features, based on modeling of pre- and post-conditions of tasks and the use of planning techniques.
\item Chapter~\ref{ch:conclusion} concludes the thesis by discussing limitations and future developments of the planning-based adaptation approach we have proposed. Moreover, we show ongoing and future research activities we are currently investigating.
\end{itemize}


\mainmatter

\chapter{Introduction}
\label{ch:introduction}

Business process management (BPM) solutions have been prevalent in both industry products and academic prototypes since the late 1990s~\cite{Sadiq@BIS2007}. A classical \emph{business process} reflects a ``preferred work practice'', i.e., a set of one or more connected activities which collectively realize a particular business goal~\cite{WeskeBook2007}. Usually, a business process is linked to an organizational structure defining functional roles and organizational relationships. Examples of business processes include insurance claim processing, order handling and personnel recruitment~\cite{ReichertBook2012}.

In order to improve their business processes, enterprises are increasingly interested in aligning their information systems in a process-centered way offering the right business functions to the right users at the right time~\cite{PAISBook2005,ReichertBook2012}. This need has evolved primarily from the desire to understand, organize, and automate the processes upon which a business is based. For this purpose, during the last decade, a new generation of information systems, called \emph{Process Management Systems} (a.k.a. PMSs, or more generally, Process-Aware Information Systems, a.k.a. PAISs) have become increasingly popular to effectively support the business processes of a company at an operational level. A PMS is a software system that manages and executes operational processes involving people, applications and information sources on the basis of process models~\cite{PAISBook2005}.

PMSs hold the promise of facilitating the everyday operation of many enterprises and work environments, by supporting
business processes in all the steps of their life-cycle~\cite{PAISBook2005}. As shown in Figure~\ref{fig:fig_introduction_spectrum_life_cycle}(a), the ``life'' of a business process is organized in 4 main stages. In the \emph{design} phase, starting from a requirements analysis, \emph{process models} are designed using a suitable modeling language. A process model is a formal representation of work procedures that controls the sequence of performed tasks and the allocation of resources to them~\cite{Oberweis@PAIS2005}. In the \emph{configuration} phase process models are implemented by configuring a PMS in order to support process enactment via an execution engine. In the \emph{enactment} phase \emph{process instances} are then initiated, executed and monitored by the run-time environment. The execution engine drives and monitors the work of the involved entities, and performed tasks generating execution traces are tracked and logged. After process execution, in the \emph{diagnosis} phase process logs are evaluated and mined to identify problems and possible improvements, potentially resulting in process re-design and evolution.

Traditionally, PMSs have focused on the management of ``administrative'' processes characterized by clear and well-defined structures. However, the use of business processes for supporting work in highly dynamic contexts (such as emergency management or health care) has become a reality, thanks also to the growing use of mobile devices in everyday life, which offer a simple way of picking up and executing tasks. Those processes are usually subject to an higher frequency of unexpected contingencies than classical scenarios; therefore, a certain degree of flexibility is needed to support dynamic process adaptation in case of exceptions.

In this chapter, that serves as the basis for positioning the performed author's research activities, we aim at discussing the flexibility requirements of actual business processes, ranging from pre-specified processes to dynamic and knowledge-intensive processes. For this purpose, in Section~\ref{sec:introduction-flexibility} we analyze concepts and definitions related to process flexibility and in Section~\ref{sec:introduction-spectrum} we classify business processes on the basis of their degree of structure by providing a systematic view of the different approaches and methodologies that have emerged to support classes of processes with different requirements. Specifically, we will focus on addressing fundamental adaptation needs for supporting dynamic processes in mobile settings for pervasive scenarios. To this end, in Section~\ref{sec:introduction-case_study} we present a case study based on a real dynamic process thought to be enacted in emergency management scenarios.

\section{Flexibility Issues in Process Management Systems}
\label{sec:introduction-flexibility}

The notion of flexibility has emerged as a main research topic in BPM over the last years~\cite{Soffer@CAISE2005,Bider@CAISE2005,Schmidt@CAISE2005}. The need for flexibility stems from the observation that organizations often face continuous and unprecedented changes in their respective business environments. Such disturbances and perturbations of business routines need to be reflected within the business processes in the sense that processes need to be able to adapt to such change. Research on process flexibility has traditionally explored alternative ways of considering flexibility during the design of a business process. The focus typically has been on ways of how the demand for process flexibility can be satisfied by advanced process modeling techniques, i.e., issues intrinsic to the processes~\cite{Rosemann@BPMDS2006}.

In~\cite{Sadiq@ER2001}, the authors define flexibility as ``the ability of the process to execute on the basis of a loosely, or partially specified model, where the full specification of the model is made at runtime, and may be unique to each instance''. Moreover, they advocate an approach that aims at making the process of change part of the business process itself. Such an approach, called \emph{Pockets of Flexibility}, relies on a pre-specified process model with placeholder activities. For each placeholder activity, a constraint-based process model (i.e., activities and constraints) can be specified. During process enactment, placeholder activities are refined, meaning that users
define a process fragment that has to satisfy the constraints and which substitutes the placeholder activity.


\begin{figure}[t]
\centering{
 \includegraphics[width=0.95\columnwidth]{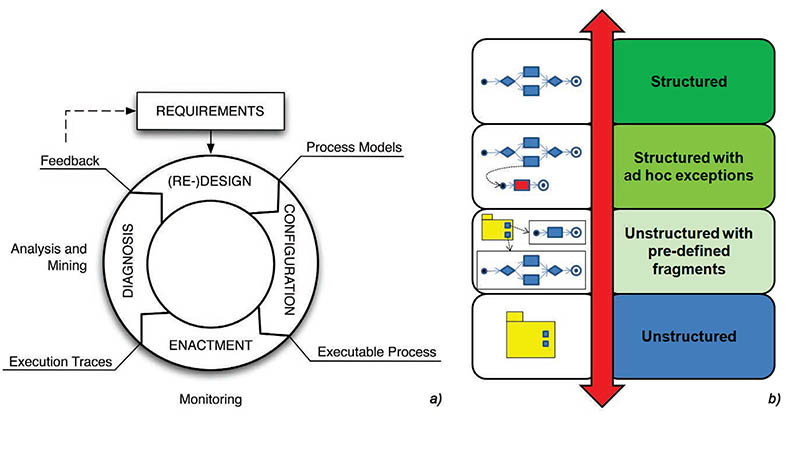}
 } \caption{Process life-cycle (a) and Spectrum (b).}
 \label{fig:fig_introduction_spectrum_life_cycle}
\end{figure}

An interesting look into the various ways in which flexibility can be achieved is made in~\cite{Schonenberg@CIAO2008}. Here an extensive taxonomy of process flexibility is proposed. In particular, it is presented a comprehensive description of four distinct approaches that can be taken to facilitate flexibility within a process. \emph{Flexibility by Design} is the ability to incorporate alternative execution paths within a process model at design time. The selection of the most appropriate execution path is made at run-time for each process instance. \emph{Flexibility by Deviation} is presented as the ability for a process instance to deviate at run-time from the execution path prescribed by the original process without altering its process model. \emph{Flexibility by Underspecification} is the ability to execute an incomplete process model at run-time, i.e., the model needs to be completed by providing a concrete realization for the undefined parts at run-time. Finally, \emph{Flexibility by Change} is the ability to modify a process model at run-time such that one or all of the currently executing process instances are migrated to a new process model. All of these strategies try to improve the ability of business processes to respond to changes in their operating environment without necessitating a complete redesign of the underlying process model, however they differ in the timing and manner in which they are applied. Moreover they are intended to operate independently of each other.

The above works concentrated on the control-flow perspective of a business process, while other perspectives addressing data and resources used in a process are also subject to flexibility requirements. A complete analysis that incorporates also these perspectives has been performed in~\cite{ReichertBook2012}, where flexible process support is characterized with four major flexibility needs, namely support for \myi \emph{variability}, \myii \emph{looseness}, \myiii \emph{adaptation}, and \myiv \emph{evolution}.
\begin{itemize}[itemsep=1pt,parsep=1pt]
\item \emph{Process variability} requires processes to be handled differently - resulting in different \emph{process variants} - on the basis of a given context. Starting from a fixed core process model, the course of actions may vary from variant to variant~\cite{Hallerbach@JSME2009}. Usually, there exists a multitude of variants of a particular process model, whereby each of these variants is valid in a specific scenario; i.e., the configuration of a particular process variant depends on requirements of the process context. Variability can be usually introduced due to different groups of customers involved in the process enactment or differences in regulations found in different countries~\cite{Reichert@ProcessVariants2010}.

    \vskip 0.5em
    \noindent\colorbox{light-gray}{\begin{minipage}{0.9\textwidth}
    \begin{example}
    \emph{A typical example of a real procedure that needs flexible support is the process for the organization of the study plans for master students in some European universities. The procedure for application, review and acceptance of study plans is managed by an on-line system and is generally well defined. After the enrollment, all the master students must submit a study plan from the on-line system. Let us suppose that the courses for the study plan can be chosen starting from 3 specialization options (Computer Networks, Software and Services for the Information Society, Distributed Systems and Architectures) with pre-selected combination of courses. If a student chooses to include in its study plan a pre-selected set of courses, the approval of the study plan is immediate (and the review phase is not required).}

    \emph{Variability can be introduced if a student decide to choose freely which courses to include in the study plan, by combining the single courses taken from each of the specialization. In this last case, the study plan must be reviewed by a university delegate, modified (if necessary), and, finally, approved or rejected. Another variant of the procedure is enacted when a foreign student wants to apply for a study plan. These kinds of students can not apply their study plan on-line, but they need to be received by the university delegate, and the study plan is built ad-hoc for the single cases.}
\end{example}
\end{minipage}
}\vskip 0.5em
\item \emph{Looseness} is a characteristic of \emph{Knowledge-intensive Processes}~\cite{KiBP2012}, that are processes characterized by being \emph{non repeatable} (the models of two process instances may differ one another), \emph{non predictable} (the course of the actions depends on context-specific parameters, whose values are not known a priori and may change during process execution) and \emph{emergent} (the course of the actions only emerges during process execution, when more information is available)~\cite{ReichertBook2012}. For those processes, only the goal and the modeling of the loose process are known a priori, meaning that those processes can not be fully pre-specified at design-time.

    \vskip 0.5em
    \noindent\colorbox{light-gray}{\begin{minipage}{0.9\textwidth}
    \begin{example}
    \emph{Typically, for being enacted, health-care processes requires a loose specification. For example, Patient Treatment Processes are rarely identical and the course of actions is unpredictable, since it depends on the specific patient's case (e.g., health status of the patient, allergies, examination results, etc.).}
    \end{example}
    \end{minipage}
    }\vskip 0.5em

\item \emph{Process Adaptation} is the ability of a process to react to exceptional circumstances (that may or may not be foreseen) and to adapt/modify its structure accordingly~\cite{Sadiq@ER2001}. Exceptions can be either \emph{expected} or \emph{unanticipated}~\cite{ReichertBook2012}.

    An expected exception can be planned at design-time, i.e., a process designer can provide an \emph{exception handler} which is invoked during run-time to cope with the expected exception. Therefore, if during run-time the PMS detects an expected exception, it immediately invokes a suitable exception handler for dealing with the exception itself. A list of \emph{exception handling patterns} that can be applied when exceptional situations arise is shown in~\cite{ex_handl_patterns}. Those patterns cover typical strategies to be used when defining exception handlers for a particular process model.
    \vskip 0.5em
    \noindent\colorbox{light-gray}{\begin{minipage}{0.9\textwidth}
    \begin{example}
    \emph{If we consider the procedure for managing study plans, a ``classical'' expected exception is captured when a student forgets to compile all the mandatory fields (e.g., the student ID code or the mobile phone number) needed for submitting the study plan on-line. In such a case, the on-line system
    notifies that some required information is missing, and asks to the student to compile correctly all the mandatory fields. The system will allow to submit the study plan only when all required information has been correctly provided.}
    \end{example}
    \end{minipage}
    }\vskip 0.5em

    Even thought the handling of expected exceptions is fundamental for every PMS, in many cases the number of possible exceptions may be too large, and requires an extensive manual effort for the process designer, that has to anticipate all potential problems and ways to overcome them in advance. Moreover, for many knowledge-intensive and dynamic scenarios, expected exceptions cover only partially relevant situations~\cite{Strong@ACMTIS1995}, and it is not realistic to assume that all exceptional situations, as well as required exception handlers, can be anticipated at design-time and thus incorporated into the process model~\cite{ReichertBook2012}. 
    This means that PMSs should provide the support for the handling of \emph{unanticipated exceptions} at run-time. Such exceptions can be detected during the execution of a process instance, when a mismatch between the computerized version of the process and the corresponding real-world business process occurs. To cope with those exceptions, a PMS is required to allow \emph{structural adaptation} of its corresponding process model at run-time; structural changes apply directly to process elements, and the adaptation is carried out by deleting, adding, or modifying one or several process elements.

    \vskip 0.5em
    \noindent\colorbox{light-gray}{\begin{minipage}{0.9\textwidth}
    \begin{example}
    \emph{In many universities, it may happen that a student passes a course that is not listed in the her/his study plan. Such an exception can not be managed at design-time, since the university delegate, which supervises the procedure for managing study plans, may decide:
    \begin{itemize}[itemsep=0.2pt,parsep=0.2pt,topsep=0.2pt]
    \item to include the passed course in the student's study plan as an ``excess course'';
    \item to insert the passed course in the student's study plan by substituting it with another similar course already included in the study plan;
    \item to refuse to insert the passed course in the student's study plan.
    \end{itemize}
    Each decision requires a certain knowledge about the student's academic situation. For example, if a student of Computer Science Engineering passes an exam of Anatomy at the Faculty of Medicine, it is obvious that the university delegate will refuse to update the student's study plan by inserting the respective course (that is out of the scope of the student's course of studies). On the contrary, if a student passes a exam during a period spent in a foreign university, and the associate course is not included in her/his study plan but is very similar to another course included in her/his study plan, the university delegate may decide to substitute the planned course with the one passed by the student when s/he was abroad.}
   \end{example}
   \end{minipage}}\vskip 0.5em

\item \emph{Process Evolution} is the ability of an implemented process to change when the corresponding business process evolves~\cite{Casati@DKE1998}. It is often driven by changes in the business, the technological environment, and the legal context~\cite{vanderAllst@ISF2001}. The evolution may be \emph{incremental} as for process improvements~\cite{Dadam@EDS2009} (i.e., only small changes are required to the implemented process), or \emph{drastic} as for process innovation or process re-engineering~\cite{HammerBook1995} (i.e., if radical changes are required). The biggest problem here concerns the handling of active process instances, which were initiated in an old model, but need to comply with a new specification. Achieving compliance for these affected instances may involve loss of work and therefore has to be carefully planned~\cite{Sadiq@ER2001}. 
    \vskip 0.5em
    \noindent\colorbox{light-gray}{\begin{minipage}{0.9\textwidth}
    \begin{example}
    \emph{Let us consider again the procedure for managing study plans, and suppose that the University revises the admission procedure for master students, requiring all applicants to submit a statement of purpose together with their application for the study plan. To implement this change, there can be two options available; one is to flush all existing applications, and apply the change to new applications only. Thus all existing applications will continue to be processed according to the old process model. The second option to implement the change is to migrate to the new process. It may be decided that all applicants, existing and new, will be affected by the change. Thus all admission applications, which were initiated under the old rules, now have to migrate to the new process.}
    \end{example}
    \end{minipage}
    }\vskip 0.5em
\end{itemize}

With respect to these needs, our research activities have broadly focused on the problem of \emph{process adaptation} defined as the ability of a PMS to adapt the process and its structure (i.e., pre-specified model) to emerging events. While several approaches~\cite{COSA,TIBCO,WebSphere,SAP,Alonso@TSE2000,ADOME,YAWLBook2009} have been proposed and implemented for dealing with \emph{expected} exceptions via exception handlers typically pre-specified by process designers at design-time, we tackled the problem of \emph{unanticipated} exceptions and their handling through structural process changes at run-time. Specifically, we focus our attention on \emph{dynamic processes}, a particular kind of processes thought to be enacted in pervasive and highly dynamic scenarios, and for which there is not a clear, anticipated correlation between a change in the context and corresponding process changes. Examples of dynamic processes are processes for emergency management (a case study is presented in Section~\ref{sec:introduction-case_study}) and military forces deployment plans.

Some systems supporting structural changes of processes at run-time exist and are well supported by the research community~\cite{Dadam@JIIS1998,adaptflow,agentwork,Reichert@ICDE2005,adeptcbr,Weber@ACBR2004}, but they do not automate the adaptation; a manual intervention of a domain expert is always required for adapting a faulty process instance at run-time. However, dynamic processes demand a more agile approach recognizing the fact that in dynamic environments process models quickly become outdated and hence require closer interweaving of modeling and execution~\cite{ReichertBook2012}. To this end, our approach, named the \smartpm approach, will allow to \emph{adapt automatically dynamic processes at run-time} when \emph{unanticipated exceptions} occur, without the need to the define any recovery policy at design-time.

In the following section we show which classes of business processes currently exist on the market, we classify them on the basis of their degree of structure and we provide a systematic view of the different approaches and methodologies that have emerged to support different classes of processes. Moreover, we make clear the main requirements for modeling, executing and adapting dynamic processes.

\section{The Spectrum of Process Management and Modeling Paradigms}
\label{sec:introduction-spectrum}

When realizing a PMS based on executable process models, there is a variety of processes showing different characteristics and needs. On one hand, there exist well-structured and highly repetitive processes whose behavior can be fully pre-specified; on the other hand, many processes are knowledge-intensive and highly dynamic: typically, they can not be fully pre-specified and require loosely specified models~\cite{ReichertBook2012}. In this section, we will analyze how operational and business processes can be classified on the basis of the degree of structuring they exhibit~\cite{Kemsley@BPM2011}, which directly influences the level of automation, control, support and flexibility that they can provide.

\subsection{Structured Processes}
\label{subsec:introduction-spectrum-structured}

Figure~\ref{fig:fig_introduction_spectrum_life_cycle}(b) shows how processes can be classified on the basis of their ``degree of structure''~\cite{Kemsley@BPM2011}. Traditional PMSs perform well with \emph{structured processes} and controlled interactions between participants. Structured processes are characterized by a well defined structure in terms of activities to be executed and relations among them. They reflect highly repeatable routine work with low flexibility requirements (such as back-office financial transactions, manufacturing, production and administrative processes) that can be easily standardized and automated to increase efficiency~\cite{LeymannBook2000}. A major assumption is that such processes, after having been modeled, can be repeatedly instantiated and executed in a predictable and controlled manner. All possible options and decisions (alternative paths) that can be made during process enactment are statically pre-defined at design time.

\emph{Structured processes with ad hoc exceptions} have similar characteristics, but events and exceptions can occur that
make the structure of the process less rigid and require process adaptation strategies. In the presence of expected
exceptions, possible events and deviations that can be encountered are predictable and defined in advance, along with the specific handling logic, whereas the handling of unanticipated exceptions typically require (manual) structural process changes at run-time~\cite{Reichert@ICDE2005}.

Structured processes can be completely captured by procedural process models that explicitly define the tasks and their execution constraints, participants, roles and input/output data. Most of classical process management environments deal with structured processes and are driven by imperative languages and procedural models, such as XPDL~\cite{xpdl}, WS-BPEL~\cite{bpel}, Event-driven Process Chains (EPCs)~\cite{epc}, BPMN~\cite{BPMN_book}, UML Activity Diagrams (UML)~\cite{umlactivitydiagrams} and YAWL nets~\cite{YAWLBook2009}. All these languages mainly focus on the control-flow perspective and are widely used in research prototypes (e.g., YAWL~\cite{YAWLBook2009}) and in open-source/free (e.g., jBPM~\cite{jBPMBook2007}, Apache ODE) and commercial products (e.g., Tibco Staffware Process Suite~\cite{TIBCO}, Oracle BPEL Process Manager, IBM Process Manager~\cite{WebSphere}). While YAWL was directly defined starting from the Petri Nets formalism~\cite{murata1989petri,petriNet}, most of the aforementioned languages have been mapped to (variants of) Petri Nets (e.g., in~\cite{umltocpn,bpmntopn}) in order to provide a formal semantics and enable different verification techniques.

\subsection{Loosely Structured Processes}
\label{subsec:introduction-spectrum-loosely}

In many application domains, pre-specifying the entire process model is not possible. A wide range of processes exhibit a loosely structured or semi-structured behavior (cf. unstructured processes with pre-defined fragments in Figure~\ref{fig:fig_introduction_spectrum_life_cycle}(b)), for which it is important to balance between flexibility and support (such as clinical guidelines and medical treatment procedures) and events and exceptions can occur that make the structure for the process significantly less rigid. While parts of the process logic are known at design-time, other parts are undefined or uncertain and can only be specified at run-time. For loosely specified processes, decisions regarding the specification of (parts of) the process have to be deferred to run-time.

Looseness is a characteristic of \emph{Knowledge-intensive Processes}, that are processes characterized by being \emph{non repeatable} (the models of two process instances may differ one another), \emph{non predictable} (the course of the actions
depends on context-specific parameters, whose values are not known a priori and may change during process execution) and \emph{emergent} (the course of the actions only emerges during process execution, when more information is available)~\cite{ReichertBook2012}. To some extent, knowledge-intensive processes and the looseness of their execution can be supported exploiting constraint-based process models. However, similarly to structured processes, loosely specified processes are still \emph{activity-centric}, i.e., they focus on a set of activities that may be performed during process execution. Procedural models are not able to provide the degree of flexibility required in these settings as they may unnecessarily limit possible execution behaviors, with either over-specified or over-constrained models~\cite{montali}. Process models should define tasks and their relationships in a less rigid manner, so that activities can be executed in multiple orders (or even multiple times) until the intended goals are achieved. Different modeling languages and management systems have been proposed and developed to meet these requirements, such as ADEPT2~\cite{Reichert@ICDE2005}, Flower~\cite{Berens_Flower@2005}, as well as service-oriented architectural solutions for the integration of different modeling approaches and PMS technologies~\cite{faas}.

Recent and ongoing work shows that declarative languages and models can be effectively used to increase the degree of flexibility for process specifications, still allowing to provide a good level of support. Languages such as ConDec~\cite{condec} and DecSerFlow~\cite{decserflow}, supported by the Declare tool~\cite{declare}, propose a declarative constraint-based approach for modeling, enacting and monitoring business processes. Instead of strictly and rigidly defining the control-flow of process tasks using a procedural language, they exploit the concept of control-flow \emph{constraints}, defined as Linear Temporal Logic (LTL) formulae~\cite{Vardi_LTL}, for the specification of relationships among tasks (generally classified as existence, choice, relation, negation and branching constraints over process activities). Constraints implicitly define possible execution alternatives by prohibiting undesired execution behavior, and they reflect policies and business rules to be satisfied and followed in order to successfully perform a process and achieve the intended goals.

\subsection{Unstructured Processes}
\label{subsec:introduction-spectrum-unstructured}

Unstructured processes are characterized by a low level of structuring and an high degree of flexibility. Process participants decide on the activities to be executed as well as their execution order, and the structure of a process thus dynamically evolves and strongly depends on user decisions made during process execution. These processes reflect knowledge work and collaboration activities driven by rules and events, for which no predefined models can be specified and little automation can be provided. Knowledge workers rely on their experience and capabilities to perform ad hoc tasks on a per-case basis and handle unexpected events and changes in the operational context.

For processes with these characteristics only their goal is known a priory and they can not be fully pre-specified at a fine-grained level at design time. In practice, approaches that focus on the role of \emph{data} as main driver for process execution and activity coordination are required. In this direction, Adaptive Case Management (ACM)~\cite{acm} has emerged as a way for supporting unstructured, unpredictable and unrepeatable business cases. ACM adopts a \emph{data-centric} (rather than an \emph{activity-centric}) approach, focusing on the concept of \emph{case} (an insurance claim, a customer purchase request, patient medical records, etc.) as primary object of interest, and the progress of the case itself is driven by the availability, values, changes and evolution of data objects and their dependencies. Each execution of a case management process involves a particular situation (the case) and a desired outcome (or goal) for that case, and the determination of actions to take in each case involves the exercise of human judgement and decision-making at run-time. In 2009, the Object Management Group (OMG) issued a Request for Proposal~\cite{omgrfp} for a meta-model extension to BPMN 2.0 to support modeling of case management processes but, to date, there exist no standards for supporting case management process modeling and execution.

Another approach for dealing with unstructured processes is the \emph{object-aware approach} proposed in~\cite{objectsimpda}. For object-aware processes, the information perspective is predominant and captures object types, their attributes, and their interrelations, which together form a data structure or \emph{information model}. In accordance to data modeling, the modeling and execution of processes can be based on two levels of granularity: object behavior and object interactions. At run-time, the different object types comprise a varying number of inter-related object instances, that evolve according to their specified behavior and interaction models. Process execution and possible actions are thus related to objects and their states~\cite{philharmonicflows}: the enabling of a process step or action does not directly depend on the completion of preceding steps (i.e., on the control-flow as for activity-centric approaches), but rather on the changes and evolution of object states and relations.

Similarly, artifact-centric models~\cite{artifact1,artifact2} aim at providing a declarative data-centric modeling approach where, rather than prescribing control-flow constraints between tasks (as in process-centric models), process specification and enactment are driven by data dependencies and evolutions of business entities.

\subsection{Dynamic Processes}
\label{subsec:introduction-spectrum-dynamic}

Both structured and loosely specified processes are activity-centric; i.e., they are based on a set of activities that may be performed during process execution. However, the class of \emph{dynamic processes} is transversal with respect to the classification proposed in~\cite{Kemsley@BPM2011}. These processes represent activities in highly dynamic situations and unforeseen exceptions (e.g., emergency management scenarios) and that are executed in a world with little structure and possibly imperfect information. The scenario dictates who should be involved and who is the right person to execute a particular step of the process, and collaborative interactions among the users typically is a major part of such processes.

When trying to support and implement dynamic processes through a PMS, a complete integration of processes, data and users has to be achieved. Usually, the structure of a dynamic process can be completely captured with a \textbf{\emph{procedural process model}} that explicitly defines the tasks and their execution constraints. However, a dynamic process is thought to be enacted in pervasive and highly dynamic scenarios, where exceptions and exogenous events ``are not the exception but the rule''. Therefore, there is the need to explicitly represent the \emph{\textbf{contextual data}} describing the scenario in which the process will be enacted, \textbf{\emph{constraints}} that evaluate if a particular activity can be applied in a specific state of the contextual scenario and \emph{\textbf{dependencies}} to specify execution dependencies between activities.

The execution of process tasks may affect the data associated to the contextual scenario, and it may happen that the \textbf{\emph{undesirable outcome of some activities}} causes changes in the contextual environment that may prevent the achievement of the business goals. The same risk comes from \textbf{\emph{exogenous events}}, that represent external events that come from the environment and correspond to changes in the data associated to the contextual scenario. In both cases, in order to make a dynamic process adaptable to the contextual environment, a PMS implementation should automatically \myi \textbf{\emph{detect exceptional situations}}, \myii \textbf{\emph{derive at run-time}} and \myiii \textbf{\emph{correctly apply}} the recovery procedures necessary to handle them, by involving some form of reasoning on the available process tasks and contextual data. In particular, the third point imposes a strong requirement: \textbf{\emph{any structural change made to the process instance at run-time must not violate the process model correctness and process instance execution}}.

In order to meet these requirements, in Chapters~\ref{ch:approach} and~\ref{ch:framework} we propose a modeling language, an approach and a PMS realization named \smartpm that allow to perform automatic adaptation of those processes at run-time. In the meanwhile, an extensive case study involving a real dynamic process is shown in the following section.

\section{Case Study}
\label{sec:introduction-case_study}

The \smartpm system has been validated in laboratory tests through the use of a real process of the Italian Railway company (that is ``Reti Ferroviarie Italiane''). Specifically, let us consider the emergency management scenario described in Fig.~\ref{fig:fig_introduction-case_study-context_1}(b). It concerns a train derailment and depicts a map of the area (as a 4x4 grid of locations) where the disaster happened. For the sake of simplicity, we suppose that the train is composed of a locomotive (located in \emph{loc33}) and two coaches (located in \emph{loc32} and \emph{loc31} respectively). The goal of an incident response plan defined for such a context is to evacuate people from the coaches located in \emph{loc32} and \emph{loc31} and to take pictures for evaluating possible damages to the locomotive, located in \emph{loc33}.

\begin{figure}[t]
\centering{
 \includegraphics[width=0.9\columnwidth]{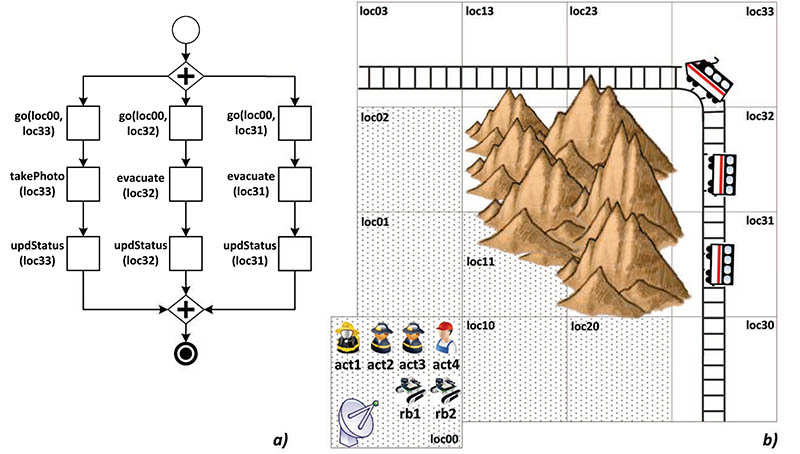}
 } \caption{Area (and context) of the intervention.}
 \label{fig:fig_introduction-case_study-context_1}
\end{figure}

Thus, a response team can be sent to the derailment scene. The team is composed of four first responders (in the remainder, we refer to them as \emph{actors}) and two robots, initially located in \emph{loc00}. We assume that actors are equipped with mobile devices (for picking up and executing tasks) and provide specific capabilities. For example, actor $act\emph{1}$ is able to extinguish fire and take pictures, while $act\emph{2}$ and $act\emph{3}$ can evacuate people from train coaches. The two robots, instead, may remove debris from specific locations. Each robot has a battery and each action consumes a given amount of battery charge. When the battery of a robot is discharged, actor $act\emph{4}$ can charge it. Moreover, the battery charge consumption amounts are provided for each action (e.g., taking pictures in a given location consumes 1 unit of battery charge, removing debris consumes 3 units, etc.).

In order to carry on the overall process, all the actors/robots need to be continually inter-connected. The connection between mobile devices is supported by a network provided by a fixed antenna (whose range is limited to the dotted squares in Fig.~\ref{fig:fig_introduction-case_study-context_1}(b)), and the robots \emph{rb1} and \emph{rb2} can act as wireless routers for extending the network range in the area. A robot provides a connection limited to the locations adjacent (in any direction) to its position. Each robot can move in the area, but it is constrained to be always connected to the main network. This is guaranteed if the intersection between the squares covered by the main network and the squares covered by the robot connection is not empty. A robot connected to the main network can act as a ``bridge'', allowing the other robot to be connected through it to the main network.

\begin{figure}[t]
\centering{
 \includegraphics[width=0.9\columnwidth]{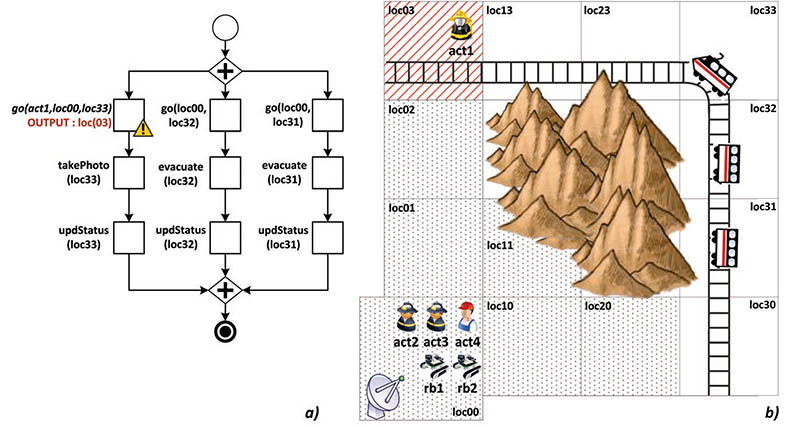}
 } \caption{Effects of a task with an outcome different from the one expected.}
 \label{fig:fig_introduction-case_study-context_2}
\end{figure}

Collected information is used for defining and configuring at run-time an incident response plan, defined by a contextually and dynamically selected set of activities to be executed on the field by first responders. A possible concrete realization of the incident response plan is shown in Figure~\ref{fig:fig_introduction-case_study-context_1}(a). The process is composed by three parallel branches with tasks that instruct first responders to act for evacuating people from train coaches, to take pictures and to assess the gravity of the accident. Despite the simple structure of the incident response plan, the high dynamism of the operating environment can lead to a wide range of exceptions. In general, for dynamic processes there is not a clear, anticipated correlation between a change in the context and corresponding process changes. Suppose, for example, that the task \emph{go(loc00,loc33)} is assigned to actor \emph{act\emph{1}} (cf. Fig.~\ref{fig:fig_introduction-case_study-context_2}(a)), which reaches instead the location \emph{loc03} (cf. Fig.~\ref{fig:fig_introduction-case_study-context_2}(b)). This means that $act\emph{1}$ is now located in a different position than the desired one, and s/he is out of the optimal network range. Since all the actors/robots need to be continually inter-connected to execute the process, the PMS has to find a recovery procedure that first instructs the robots to move in specific positions for maintaining the network connection, and then re-assign the task \emph{go(loc03,loc33)} to \emph{act1}.

\begin{figure}[t]
\centering{
 \includegraphics[width=0.9\columnwidth]{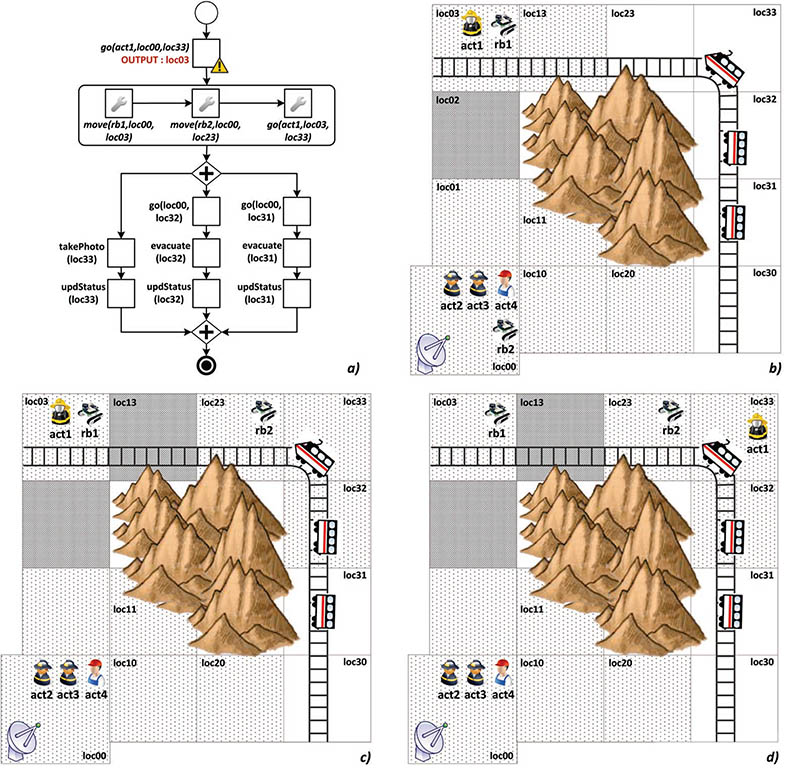}
 } \caption{Synthesis and enactment of the recovery process after an undesirable outcome.}
 \label{fig:fig_introduction-case_study-context_3}
\end{figure}

Even this very simple example shows that the same failure may require significantly different adaptation activities depending on the current context. For example, if robots \emph{rb1} and \emph{rb2} have enough battery charge, the PMS can instruct first \emph{rb1} to move in \emph{loc03} (cf. Fig.~\ref{fig:fig_introduction-case_study-context_3}(a) and Fig.~\ref{fig:fig_introduction-case_study-context_3}(b)), in order to re-establish the connection of actor \emph{act1} with the network. Then, robot \emph{rb2} can be instructed to reach \emph{loc23} (cf. Fig.~\ref{fig:fig_introduction-case_study-context_3}(a) and Fig.~\ref{fig:fig_introduction-case_study-context_3}(c)) and to broaden the network range and make \emph{loc33} (the expected destination of \emph{act1}) covered by the network connection. Finally, the task \emph{go(loc03,loc33)} can be effectively reassigned to \emph{act1} (cf. Fig.~\ref{fig:fig_introduction-case_study-context_3}(a) and Fig.~\ref{fig:fig_introduction-case_study-context_3}(d)).

After having executed the recovery procedure, the enactment of the main process can be resumed to its normal flow. The execution of a dynamic process can be also prevented by the occurrence of external events (a.k.a. exogenous events). These are events coming from the external environment that change asynchronously some contextual properties of the scenario in which the process is under execution. For example, in Fig.~\ref{fig:fig_introduction-case_study-context_4}(a) we are supposing that a rock slide has collapsed on location \emph{loc31}. The presence of a rock slide modifies the contextual properties of the scenario in a way not expected when the dynamic process was designed, by possibly jeopardizing the correctness of the dynamic process itself. Since a dynamic process - by definition - has to be \emph{adaptable to the context}, the PMS needs to find a recovery procedure that allows to remove the rock slide from \emph{loc31} by maintaining all the process participants inter-connected. A possible solution is shown in Fig.~\ref{fig:fig_introduction-case_study-context_4}(b) and Fig.~\ref{fig:fig_introduction-case_study-context_4}(c), and consists in moving the robot \emph{rb1} in \emph{loc31} and letting it remove debris.

It is unrealistic to assume that the process designer can pre-define all possible compensation activities for dealing with these exceptions (apparently simple), since the process may be different every time it runs and the recovery procedure strictly depends on the actual contextual information (the positions of actors/robots, the range of the main network, the battery level of each robot, etc.). In the worst case, the number of recovery processes to pre-define may depend to all the possible combinations of contextual information. For the same reason, it is also difficult to manually define an ad-hoc recovery procedure at run-time, as the correctness of the process execution is highly constrained by the values (or combination of values) of contextual data.

The main purpose of the \smartpm system is to develop a PMS providing automatic mechanisms that, starting from a process model, are able to adapt the process instance without explicitly defining handlers/policies to recover from exceptions and exogenous events and without the intervention of domain experts.

\begin{figure}[t]
\centering{
 \includegraphics[width=0.9\columnwidth]{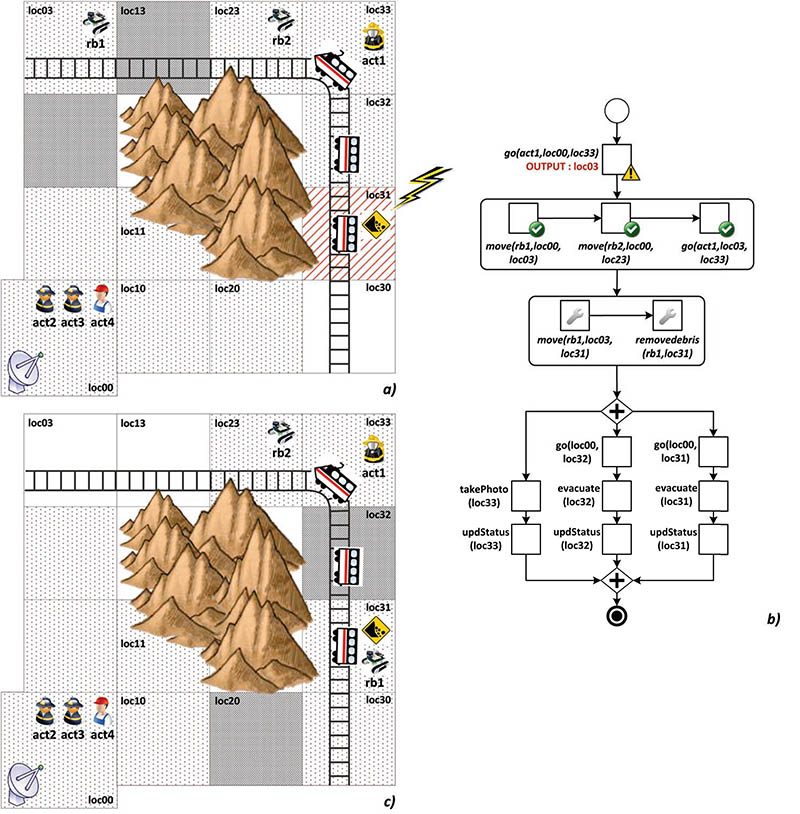}
 } \caption{Synthesis and enactment of the recovery process after the occurrence of an exogenous event.}
 \label{fig:fig_introduction-case_study-context_4}
\end{figure}

\chapter{State of the Art}
\label{ch:state_of_the_art}

Process adaptation refers to the ability of a PMS to modify its behavior according to environmental changes and exceptions that may occur during process execution. If not detected and handled effectively, exceptions can result in severe impacts on the cost and schedule performance of PMSs~\cite{SpIss_AdaptiveWfMSs}.

Nowadays, PMSs provide wide support for different modeling styles and for all phases of the process life-cycle, from the specification and enactment to the verification, monitoring and analysis of process models~\cite{PAISBook2005}. In addition, PMSs provide tools for modelling business processes that are predictable and repetitive. However, in many real world scenarios, enabling process adaptation is crucial for any PMS, as implemented processes may have to be adapted to deal with changing environments and evolving needs.

In this chapter, we focus on the \emph{exception handling} perspective provided by current state-of-the-art PMSs, by first showing some of the best known techniques for exception handling and then by reviewing some of the most interesting current commercial and academic prototype PMSs in relation to their approaches to process adaptation.
Finally, in Section~\ref{sec:state_of_the_art-adaptation_planning} we analyze a number of techniques from the field of Artificial Intelligence (AI) that were applied to BPM, with the purpose to facilitate automatic adaptation of a business process at run-time.

\section{Process Adaptation}
\label{sec:state_of_the_art-adaptation}

Over the last years, there was a trend in providing PMSs with a growing support for adapting business processes. Proposed approaches can be analyzed considering to what extent users are involved in the process of defining exception conditions and handling policies, and the degree of automation provided in the exception resolution and process adaptation stages.



This section is thought for discussing the current state-of-the-art in process adaptation and exception handling (cf. Section ~\ref{subsec:state_of_the_art-ex_handling_techniques}) and for reviewing some of the most interesting current commercial and academic prototype PMSs in relation to their approaches to process adaptation (cf.~\ref{subsec:state_of_the_art-adaptation-PMSs}). In Section~\ref{subsec:state_of_the_art-adaptation-discussion} we provide a comparative analysis between the existing techniques and the \smartpm approach we are going to present in this thesis.


\subsection{Exception Handling Techniques}
\label{subsec:state_of_the_art-ex_handling_techniques}

One of the first attempts to define and categorize exceptions in PMSs was proposed in the nineties in~\cite{E_L@COOPIS1996}, where possible exceptions were classified as \emph{basic failures}, \emph{application failures}, \emph{expected exceptions} and \emph{unanticipated exceptions}. While basic failures (which reflect a failure at system level, e.g., DBMS or network failure) and application failures (corresponding to the failure of an application implementing a given task) are generally handled at the system and application levels, PMSs are in charge to provide support for \emph{exception handling}.

An \emph{\textbf{expected exception}} is an exception that can be planned at design-time, i.e., a process designer can provide an exception handler which is invoked during run-time to cope with the exception itself. In~\cite{Casati_2001}, Casati identifies 4 potential sources for expected exceptions:
\begin{itemize}[itemsep=1pt,parsep=1pt,topsep=1pt]
\item \emph{Process Exceptions}. They can occur during the process enactment, and are strictly related to \emph{task failures}. More in detail, a task can fail due to an abnormal termination of the invoked application or web service implementing the specific task, or because of a negative termination of the task itself (which returns an output different from the expected one).
\item \emph{Constraint Violations}. They refer to violations of constraints over data (e.g., data required for task execution might be missing), tasks (e.g., in terms of pre/post-condition of a task not satisfied before/after task execution) and resources (e.g., unavailability of resources during process execution).
\item \emph{Temporal Exceptions}. They can be associated with deadlines, and upon deadline expiration an exception is launched.
\item \emph{External Exceptions}. They happen when an external event may affect the control/data flow of the process under execution.
\end{itemize}

To enable an expected exception to be detected at run-time, a PMS has to notice its occurrence during process enactment. To this end, a process designer needs to associate pre-specified process models with exception handlers at design-time. Along this line, in~\cite{ex_handl_patterns} the authors propose a patterns-based approach to exception systematization, defining a general classification framework and language for exception handling. Specifically, exception handlers are modeled as alternative branches of the process model that add or cancel behavior to the normal flow of the process instance under execution.

Among early attempts to address the problem of exception handling, it is worth mentioning the transaction oriented Workflow Activity Model WAMO~\cite{WAMO}. The model allows the designer to define expected exceptions which may arise during process execution and to derive compensation tasks for exception handling. The specification of compensation activities avoids the designer to explicitly model all possible process execution alternatives in presence of exceptions, and allows the system to automatically control the reliable execution of exceptions and failures. Handling capabilities depend on task properties and range from simple re-execution of failed activities to the execution of complex compensation strategies.

In~\cite{sarn} the authors proposed Self-Adaptive Recovery Nets (SARNs), an extended Petri net model for specifying exceptional behavior in PMSs at design time. SARNs can adapt the structure of the underlying Petri net model at run time to handle pre-specified exceptions. Process adaptation is achieved by applying high-level recovery policies (such as skipping, redoing or compensating tasks when specific events occur) defined at design time for single tasks or sets of tasks.

In~\cite{Casati_1999}, authors argue that the majority of exceptions are asynchronous (they can occur at any time during the process execution, rather than before/after a task execution), and that they can not be suitably represented within the process model. Therefore, they advocate that the only way to handle exceptions is programmatically via Event-Condition-Action (ECA) rules. ECA rules have the form ``on \emph{event} if \emph{condition} do \emph{action}'' and specify to execute the \emph{action} (i.e., the exception handler) automatically when the \emph{event} happens (i.e., when the exception is caught), provided the a specific \emph{condition} holds.

ECA rules represent a good way for separating the graphical representation of the process with the ``exception handling flow''. A similar principle has been applied in YAWL~\cite{YAWLBook2009}. In YAWL, for each exception that can be anticipated, it is possible to define an exception handling process, named \emph{exlet}, which includes a number of exception handling primitives (for removing, suspending, continuing, completing, failing and restarting a workitem/case) and one or more compensatory processes in the form of \emph{worklets} (i.e., self-contained YAWL specifications executed as a replacement for a workitem or as compensatory processes). Exlets are linked to specifications by defining specific
rules (through the Rules Editor graphical tool), in the shape of \emph{Ripple Down Rules} specified as ``if condition then conclusion'', where the condition defines the exception triggering condition and the conclusion defines the exlet.

The above approaches use to deal with exceptions by attempting to anticipate all possible failures. This greatly complicates the process models and thereby obscures the main ``preferred'' process if the number of exceptions to define is too large~\cite{SpIss_AdaptiveWfMSs}. Moreover, in many cases, exceptions can only be modeled at the technical level, although their handling requires organizational knowledge and thus should be expressed in semantic process models as well. Hence,
authorized users should be allowed to situationally adapt single process instances during run-time to cope with unanticipated exceptions~\cite{ReichertBook2012}.

An \textbf{\emph{unanticipated exception}} can be detected during the execution of a process instance, when a mismatch between the computerized version of the process and the corresponding real-world business process occurs. To cope with those exceptions, a PMS is required to allow structural adaptation of its corresponding process model at run-time.

In~\cite{ReichertBook2012} is presented a basic taxonomy and respective patterns for applying structural changes to running process instances in case of unanticipated exceptions. Specifically, a structural change can be applied on the \emph{state of the process instance} or/and on the \emph{structure of its process model} and concerns the application of \emph{change primitives} to the process model (i.e., by adding or deleting single tasks or transitions) or of \emph{high-level change operations} (i.e., by inserting, deleting or moving activities or entire process fragments). Another dimension concerns the \emph{degree of automation} in applying a change. The majority of the approaches dealing with unanticipated exceptions provide a \emph{manual} definition of a structural change~\cite{Breeze,Weske@HICSS2001,bandinelli1993software,adept2a,Reichert@ICDE2005}; an expert user suspends the process instance, loads it into a process editor and manually adjusts its structure.

Since this procedure is quite complex for non expert end-users, a growing number of systems~\cite{adeptcbr,Weber@ACBR2004,minor2008agile} is providing a \emph{semi-automatic} way for assisting end-users in defining structural changes at run-time by exploiting the \emph{Case-Based Reasoning} (CBR) approach~\cite{CBR_Book}. CBR is a paradigm based on reasoning and machine learning, and solves an emerging problem by drawing on past experiences and by adapting respective solutions to the new problem situation. Problems and their solutions are described as \emph{cases} and stored in \emph{case-bases}. Each problem that has been solved provides a new case that is stored for future reuse, making learning a natural side-effect of the reasoning process. When using CBR, the end-user is required to provide a complete problem specification at design-time for case retrieval.

By the way, even systems that do support manual and semi-automatic dynamic process model modification at run-time do not help determine the best response to a given exception, leaving this often difficult decision to end-users~\cite{SpIss_AdaptiveWfMSs}. In order to relieve end-users from building complex adaptation tasks, there are few systems~\cite{Minor@IS2012,EPOS1993,Ferreira@IJCIS2006,Weske@ADBIS2004} that provide approaches for \emph{automatically} adapting a process in presence of unanticipated exceptions. While we will present the goal-based approaches~\cite{EPOS1993,Ferreira@IJCIS2006,Weske@ADBIS2004} in Section~\ref{sec:state_of_the_art-adaptation_planning}, it is interesting to say a few words about the work~\cite{Minor@IS2012}. It concerns a recent framework dealing with automatic adaptation of processes through CBR. The method employs process adaptation cases that record adaptation episodes from the past. The recorded changes can be automatically transferred to a new process that is in a similar situation of change. The case-based adaptation method includes the so-called \emph{anchor mapping algorithm} which identifies the parts of the target process where to automatically apply the changes.



\subsection{Analysis of Existing PMSs}
\label{subsec:state_of_the_art-adaptation-PMSs}

Generally, existing commercial PMSs provide only basic support for handling exceptions in a proprietary manner~\cite{zur2004workflow}. They typically require the process model to be fully defined before it can be instantiated, and changes must be incorporated by modifying the model statically:
\begin{itemize}[itemsep=1pt,parsep=1pt,topsep=1pt]
\item \emph{COSA}~\cite{COSA} allows to associate sub-processes to external ``triggers'' and events. All events, sub-processes and adaptation policies must be defined at design-time, although models can be modified at run-time (but only for future instantiations). If a task fails, the associated activity can be rolled back or restarted.
    When a deadline expires, a compensating activity is triggered. COSA also allows manual ad-hoc run-time adaptation by using change patterns~\cite{Weber@DKE2008} such as reordering, skipping, repeating, postponing or terminating steps.
\item \emph{Tibco iProcess Suite}~\cite{TIBCO} provides constructs called ``event nodes''. For each event node, a separate pre-defined exception handler can be activated when an exception occurs at that point. If no handler exists for the specific exception, the identified exception is forwarded into a ``default exception queue'' where it may be manually handled. Moreover, the designer can decide to manually skip tasks at run-time.
\item \emph{WebSphere MQ Workflow}~\cite{WebSphere} supports deadlines and, when they expire, some pre-defined exception handlers are invoked and/or notification messages are sent to the process administrator, which can manually suspend, restart or terminate processes, as well as s/he can reallocate tasks.
\item \emph{SAP Workflow}~\cite{SAP} allows to provide exception events for checking tasks pre- and post-constraints and for waiting until an external trigger occurs. Exception handlers are pre-defined at design-time and, when an exception occurs at run-time, the process administrator can manually select the most appropriate recovery procedure from the pre-specified ones. When an exception handler is found, all tasks in the block where the exception is caught are cancelled.
\end{itemize}
The main issue with the above approaches is that they provide solid support for transaction-level exceptions and very few support for exceptions at process or instance level. Moreover, exceptions are mainly thrown when some deadline expires, and also the behavior of those systems after a deadline is occurred depends mainly by a process administrator~\cite{AdamsPhDThesis}. Research in process adaptation has born mainly for overcoming the above limitations, with the purpose to increase the degree of automatic adaption of business processes. There have been a number of academic prototypes developed in the last decade, and we limit the discussion to the more popular and recent systems and prototypes\footnote{A more detailed analysis can be found in~\cite{AdamsPhDThesis}.}:
\begin{itemize}[itemsep=1pt,parsep=1pt,topsep=1pt]
\item The \emph{OPERA} prototype~\cite{Alonso@TSE2000} borrows its main ideas from a combination of programming language concepts and transaction processing techniques, adapting them to the special characteristics of PMSs. OPERA was one of the first systems that integrate language primitives for exception handling into PMSs. It has a modular structure in which activities are nested. When a task fails, its execution is stopped and a pre-specified exception handler is launched for dealing with that kind of exception. If the handler cannot solve the problem, it propagates the exception up the activity tree; if no handler is found the entire process instance aborts.
\item In the \emph{ADOME} system~\cite{ADOME}, the problem of exception handling is addressed adopting an integrated event-driven approach that covers exception detection, exception resolution, reuse of exception handlers, and automated resolution of expected exceptions. As exceptions are identified by specific events, Event-Condition-Action (ECA) rules are used to dynamically bind exception handlers to exception conditions at design time, and enable automatic exception detection and resolution at run-time.
\item An interesting approach for dealing with process adaptation is provided by the \emph{YAWL} system~\cite{YAWLBook2009}. Exception handling capabilities provided by YAWL were designed and implemented starting from the conceptual framework for workflow exception handling presented in~\cite{WorkflowExceptionPatterns2006}. For each exception that can be anticipated, it is possible to define at design time exception handling processes (\emph{exlets}) linked to process specifications by defining triggering conditions and rules in the form of Ripple Down Rules~\cite{yawlex}. Handling processes, which include exception handling primitives (for removing, suspending, continuing, completing, failing and restarting a work item/case) and one or more compensatory processes, are dynamically selected and incorporated at run-time in process instances. More details about YAWL are given in Chapter~\ref{ch:planlets}.
\item Strong support for adaptive process management and exception handling is provided by the \emph{ADEPT} system. $ADEPT_{flex}$ offers modeling capabilities to explicitly define pre-specified exceptions, and supported changes of process instances to enable different kinds of ad-hoc deviations from the pre-modeled process schemas in order to deal with run-time exceptions~\cite{Dadam@JIIS1998}. These features have been extended an improved in \emph{ADEPT2}~\cite{adept2a,Reichert@ICDE2005}, which provides full support for the structural process change patterns defined in~\cite{Weber@DKE2008}, and in \emph{ProCycle}, which combines ADEPT2 with conversational case-based reasoning (CCBR) methodologies~\cite{adeptcbr}. CCBR is an extension od CBR, where the user is not required to provide a complete a priori specification for case retrieval. Specifically, process instance changes are manually performed and stored as \emph{cases} in a case-base, to be retrieved and reused to perform similar changes in the future. Users are supported in finding relevant cases and adapting processes by taking into account how similar events were previously handled and reusing information about similar changes applied to previously executed instances.
\item The \emph{AdaptFlow} prototype~\cite{adaptflow} supports ECA rules-based detection of exceptions and the dynamic adaptation of process instances, although each adaptation must be confirmed manually by a process administrator before it is applied. If the exception handler does not satisfy the process administrator, s/he can manually deals with the exception. The prototype has been designed as an overlay to the ADEPT system, providing dynamic extensions.
\item Similarly, \emph{AgentWork}~\cite{agentwork} relies on ADEPT and exploits a temporal ECA rule model to automatically detect logical failures and enable both reactive and predictive process adaptation of control- and data-flow elements. Here, exception handling is limited to single tasks failures, and the possibility exists for conflicting rules to generate incompatible actions, which requires manual intervention and resolution.
\item Conversational CBR has been also applied in the \emph{CBRFlow} system~\cite{Weber@ACBR2004} for supporting adaptation of pre-defined process models to changing circumstances by allowing (manual) annotation of business rules during run-time via incremental evaluation by the user. Thus users must be actively involved in the inference process during each case.
\end{itemize}

\subsection{Discussion}
\label{subsec:state_of_the_art-adaptation-discussion}

All the approaches analyzed so far share a common background. At design-time the process modeler identifies possible exceptions that can occur, explicitly defines exception conditions, triggering events and handling policies, and integrates such information into the the process model. At run-time exceptions are detected and handled either manually or (semi-)automatically, adapting affected processes at the instance level.

A good way for understanding and comparing the kind of adaptation we are proposing with the other existing approaches is to consider the \texttt{try-catch} approach used in some programming languages such as Java\footnote{\url{http://www.java.com/en/}}. If the \texttt{try} block encloses the code that might throw an exception, the \texttt{catch} is the definition of the possible exceptions and the specification, defined at design-time by the process modeler, of how to deal with them.

For expected exceptions, both the \texttt{try} and the \texttt{catch} blocks need to be completely pre-specified, i.e., for each kind of exception that is envisioned to occur, a specific exception handler is defined. Systems like~\cite{COSA,TIBCO,WebSphere,SAP,Alonso@TSE2000,ADOME,YAWLBook2009} provide mainly this kind of support to process adaptation.

On the contrary, for unanticipated exceptions, the \texttt{try} is automatically derived as the situation in which the PMS does not adequately reflect the real-world process anymore. As a consequence, one or several process instances have to be adapted, and the \texttt{catch} block should include those recovery procedures required for realigning the computerized processes with the real-world ones. The common strategy used by the current adaptive PMSs~\cite{Dadam@JIIS1998,adaptflow,agentwork,Reichert@ICDE2005,adeptcbr,Weber@ACBR2004} is to \emph{manually} or \emph{semi-automatically} define at run-time the \texttt{catch} block, by devising an ad-hoc recovery procedure that deals with the specific exception captured.

Our proposal - instead - aims at \emph{synthesizing at run-time} the \texttt{catch} block, without the need of any manual intervention at run-time. This because the design-time specification of all possible compensation actions requires an extensive manual effort for the process designer, which has to anticipate all potential problems along with possible ways to overcome them. This is particularly true in real-world and dynamic scenarios, where the process designer often lacks the needed knowledge to model all the possible contingencies at design-time, or this knowledge can become obsolete as process instances are executed and evolve, by making useless his/her initial effort. For the same reason, analyzing and defining these adaptations ``manually'' becomes time-demanding and error-prone in real-world scenarios. Indeed, the designer should have a global vision of the application and its context to define appropriate recovery actions, which becomes complicated when the number of relevant context features and their interleaving increases.



\section{AI-based Process Adaptation}
\label{sec:state_of_the_art-adaptation_planning}


The Artificial Intelligence (AI) community has been involved with research on process management for several decades~\cite{Myers@1998}. While BPM has concentrated on business and manufacturing processes, the AI community has been motivated primarily by domains that involve active control of computational entities and physical devices (e.g., robots, antennas, satellites, computer networks, software agents, etc.). Despite their differing origins and emphases, there is much overlap between the objectives, requirements, and approaches for process management within these two communities. Specifically, AI technologies can play an important role in the construction of the PMS engines that manage complex processes, while remaining robust, reactive, and adaptive in the face of both environmental and tasking changes. An interesting report that describes how techniques from the AI community could be leveraged to provide several of the advanced process management capabilities envisioned by the BPM community is shown in~\cite{Myers@1998}.

The advantage of integrating AI planning techniques~\cite{TraversoBook2004} for several applications in the field of BPM has long been acknowledged. For example, in~\cite{Jarvis@AAAIwrk1999} the authors take a broad view of the problem of adaptive PMSs, and show that there is a strong mapping between the requirements of such systems and the capabilities offered by AI techniques. The work describes how planning can be interleaved with process execution and plan refinement, and investigates plan patching and plan repair as means to enhance flexibility and responsiveness.
The approach presented in~\cite{Moreno@KBSJ2002} highlights the improvements that a legacy workflow application can gain by incorporating planning techniques into its day-to-day operation. The use of contingency planning~\cite{contingency_planning} to deal with uncertainty increases system flexibility, but it does suffer from a number of problems. Specifically, contingency planning is often highly time-consuming and does not guarantee a correct execution under all possible circumstances.

Also the use of planning techniques for process adaptation purposes is not strictly new. One of the first work dealing with this research challenge is~\cite{Beckstein1999}. It discusses the use of an intelligent assistant based on planning techniques, whose purpose is to suggest compensation procedures or the re-execution of activities if some failure arises during the process execution. The intelligent assistant makes use of a meta-level knowledge incorporated into the process model.

A number of \emph{goal-based approaches} exist for enabling automatic process instance change in case of emerging exceptions. Those approaches~\cite{EPOS1993,Ferreira@IJCIS2006,Weske@ADBIS2004} explicitly formalize an output for the process to be executed in terms of a process \emph{goal}. Starting from the process goal, they are able to automatically derive a process model (i.e., the activities to be performed and their execution order) that complies with the goal itself. In addition, if an activity failure occurs at run-time and leads to a goal violation, the process instance is adapted accordingly. Specifically:
\begin{itemize}[itemsep=1pt,parsep=1pt,topsep=1pt]
\item EPOS~\cite{EPOS1993} presents a general, incremental replanning algorithm, which allows to automatically adapting process instances when process goals change. The adaptation consists in synthesizing a new process specification that reflects the changes in the process goal.
\item A similar approach is proposed in~\cite{Weske@ADBIS2004}, where the authors present a concept for dynamic and automated workflow re-planning that allows recovering from task failures. To handle the situation of a partially executed process, a multi-step procedure is proposed that includes the termination of failed activities, the sound suspension of the process, the generation of a new complete process definition and the adequate process resumption.
\item A new life cycle for workflow management based on the continuous interplay between learning and planning is proposed in~\cite{Ferreira@IJCIS2006}. The approach is based on learning business activities as planning operators and feeding them to a planner that generates a candidate process model that is able of achieving some business goals. If an activity fails during the process execution at run-time, an alternative candidate plan is provided on the same business goals.
\end{itemize}

Planning techniques are also used in~\cite{Pernici@TSE2010} to define a self-healing approach for handling exceptions in service-based processes and repairing faulty activities. The purpose of~\cite{Pernici@TSE2010} is to provide a tool that analyzes the repairability of service-based processes and generates repair plans based on a set of design-time repair actions. At design time, the process description is augmented with information about: \myi the data dependencies within the
process, \myii the desired process results, and \myii available repair actions for each process activity. In addition, an heuristic-based approach is used for reasoning about the repairability of the process schema, by checking if it includes some design flaw that may possibly affects the repairability of its activities. 
When a failure is detected during the process execution, a \emph{diagnosis tool} determines what caused the failure and where (i.e., in which point of the process) the error took place. Such knowledge, obtained by the diagnosis tool or - if needed - with the support of a human process modeler, is used together with the repairability information inferred at design time for generating a repair plan dealing with the exception. Repair is specified as a planning problem whose goal is to build a plan consisting of recovery actions that, after being executed, should recover the faulty process instance. The repair plan is generated by taking into account constraints posed by the process structure and by applying or deleting repair actions taken from a given \emph{generic repair plan}. The generic repair plan is built by considering all possible applications of the recovery actions in a specific state.

The work~\cite{Bucchiarone@SOCA2011} proposes a goal-driven approach to business process adaptation in service-based applications. The authors provide an approach that allows to automatically adapt business processes to run-time context changes that impede achievement of a business goal. The execution of the business process and the evolution of its context are continuously monitored against the desired model of the application, which is based on the notion of goal and on service annotations describing how services contribute to goal achievement. The model encodes also the business policies defined over the elements of the domain; each contextual property is modeled with a state transition system capturing all possible values of the specific property, and their evolution after a precondition violation or the occurrence of an exogenous event. The adaptation mechanism is based on service composition via automated planning techniques~\cite{TraversoBook2004}: whenever a deviation is detected (that results in a violation of some business policy), a specific planner will be in charge to generate a composition of available services that achieves the adaptation goals in compliance with the business policies. In this approach, compensation activities that compose a recovery procedure are not required to be explicitly represented at design-time, but are automatically derived from the various aspects of the environment (the contextual properties, the business policies, the available services, etc.).




Another interesting work dealing with process interference is~\cite{vanBeest@ICSOC2010}. Process interference is a situation that happens when several concurrent business processes that depend on some common data are executed in a highly distributed environment. During the processes execution, it may happen that some of these data are modified causing unexpected or wrong business outcomes. To overcome this limitation, the work~\cite{vanBeest@ICSOC2010} proposes a run-time mechanism which uses \emph{Depepency Scopes} - for identifying critical parts of the processes whose correct execution depends on some shared variables - and \emph{Intervention Processes} for solving the potential inconsistencies generated from the interference.
In~\cite{vanBeest@KiBP2012}, the authors present an algorithm for automating the discovery of critical sections and the generation of dependency scopes, whereas intervention processes need to be manually defined at design time. One of the future work claimed by the authors concerns to devise a technique for the automatic synthesis of intervention processes through a domain independent planner based on CSP techniques. Details about the planner and the algorithm used for synthesizing a plan are shown in~\cite{Aiello@AAAI11}.

\subsection{Discussion}
\label{subsec:state_of_the_art-adaptation_planning-discussion}


If compared with the above works, our approach provides some unique features in dealing with unanticipated exceptions for dynamic processes:
\begin{itemize}[itemsep=1pt,parsep=1pt,topsep=1pt]
\item The goal-based approaches~\cite{EPOS1993,Ferreira@IJCIS2006,Weske@ADBIS2004} apply planning techniques to automatically derive a process model from business goals, and to repair it during run-time if goals change by replanning the process specification. However, the structure of a dynamic process usually derives from a pre-defined high-level procedure, and it is difficult to extract a goal representing the correct execution of the process. Moreover, current planning methods do not still cover adequately all relevant modeling aspects (e.g., treatment of loops, appropriate handling of data flow, etc.). But the major issue lies in the replanning stage used for adapting a faulty process instance. In fact, dynamic processes usually involve real human participants acting in pervasive scenarios, and to re-define completely the process specification at run-time when the process goal changes (due to some activity failure) means to completely revolutionize the work-list of tasks assigned to the process participants. On the contrary, our approach adapts a running process instance by modifying only those parts of the process that need to be changed/adapted by keeping other parts stable.
\item In the work~\cite{Pernici@TSE2010}, it is required to explicitly define repair activities at design-time. But, for a dynamic process, it is not an easy task to predict how the process instances will unfold and to understand which recovery actions effectively will be required for adapting the processes. Instead, in our approach, the recovery procedure will be built by composing tasks stored in a specific repository. The repository contains both tasks used for defining the specific process instance under execution and other tasks built on the same contextual scenario and possibly used in past executions of the process. This means we don't explicitly define at design time any recovery action, by letting an external planner to decide which tasks are really required for building the recovery procedure. By the way, the real bottleneck of the approach presented in~\cite{Pernici@TSE2010} consists in the presence of a generic recovery plan from which an ad-hoc repair plan is derived when an exception occurs. The authors exploit the disjunctive logic programming system DLV~\cite{leone2006DLV} to search for successful repair plans starting from the generic one. However, given n repair actions, there exist, in the worst case, at least O(\emph{$n^{n}$}) possible repair plans, since all possible permutations of repair actions have to be considered. For avoiding such a complexity authors apply heuristics that limit the search space. The drawback is that some successful repair plans may be possibly not computed.
\item The work~\cite{Bucchiarone@SOCA2011} presents the interesting feature of providing services with non-deterministic outcomes. 
    However, in this approach, service operations as well as context properties are modeled through state transition systems, by requiring a considerable manual effort at design-time. Furthermore, the approach does not support numeric variables and would suffer of state explosion and poor performance if monitored variables range over large domains. The main difference between the \smartpm approach and the one presented in~\cite{Bucchiarone@SOCA2011} is that in \smartpm the recovery procedure is synthesized at runtime, without the need to define any recovery policy at design-time. On the contrary, in~\cite{Bucchiarone@SOCA2011}, the process modeler is required to define some ``business policies'' for detecting the exceptions.
\item The work presented in~\cite{vanBeest@ICSOC2010} is the closest to the \smartpm approach, but it does not still provide automatic mechanisms for the building of the recovery procedure (however, it has been claimed to be one of the main future direction of this work) and it only considers external events as a source of failure (we focus on both external events and tasks failures).
\end{itemize}

\chapter{The \smartpm Approach}
\label{ch:approach}

Approaches enabling automated process instance adaptation at run-time aim at reducing error-prone and costly manual ad-hoc changes, and thus at relieving users from complex adaptations tasks~\cite{ReichertBook2012}. As a prerequisite for such
automated changes, a PMS must be able to:
\begin{itemize}[itemsep=1pt,parsep=1pt]
\item automatically detect exceptional situations;
\item derive at run-time a recovery procedure to handle them;
\item correctly apply the recovery procedure to the process instance involved in the specific exception.
\end{itemize}

To address this research issue, in the following two chapters we present \smartpm (that is the acronym for \emph{Smart Process Management}), which is a model and a prototype PMS featuring a set of techniques providing support for automatic adaptation of processes at run-time. Such techniques are able to automatically adapt process instances without explicitly defining handlers/policies to recover from tasks failures and exogenous events at run-time and without the intervention of domain experts. While in Chapter~\ref{ch:framework} we will describe in detail the architecture and the technological aspects of the \smartpm software prototype, in this chapter we focus on the theoretical aspects of our general approach for adaptation of dynamic processes.

To accomplish this, we make use of well-established techniques and frameworks from Artificial
Intelligence, such as \emph{\sitcalc}\cite{ReiterBook}, \indigolog~\cite{Indigolog:2009} and \emph{automatic planning}~\cite{TraversoBook2004}. After presenting in Section~\ref{sec:approach-overview} an overview of our general approach, in Section~\ref{sec:approach-preliminaries} we will discuss some preliminary notion around these challenging topics. Then, in Section~\ref{sec:approach-formalization} we show how we explicitly formalize processes in situation calculus and \indigolog and in Section~\ref{sec:approach-monitor} we discuss how this formalization can be used for devising a technique that automatically detect and recover from failures. Finally, in Section~\ref{sec:approach-adaptation} we present three different adaptation mechanisms developed within the \smartpm prototype.

\section{Overview of the Approach}
\label{sec:approach-overview}

\smartpm adopts a \emph{service-based} approach to process management, that is, tasks are executed by services (that could be software applications, human actors, robots). In \smartpm a process model is defined as a set of $n$ task definitions, where each task $t_{i}$ can be considered as a single step that consumes input data and produces output data. Data are represented through a set $F$ of \emph{fluents} $f_{j}$ whose definition depends strictly on the specific process domain of interest. In AI, a fluent is a condition that can change over time. Such fluents can be used to constrain the task assignment (in terms of \emph{task preconditions}), to assess the outcome of a task (in terms of \emph{task effects}) and as guards into the expressions at decision points (e.g., for cycles or conditional statements).

Choosing the fluents that are used to describe each activity falls into the general problem of \emph{knowledge representation}. To this end, the environment, services and tasks are grounded in domain theories described in \sitcalc~\cite{ReiterBook}. Situation Calculus is specifically designed for representing dynamically changing worlds in which all changes are the result of the tasks' execution. Processes are represented as \indigolog programs. \indigolog~\cite{Indigolog:2009} allows for the definition of programs with cycles, concurrency, conditional branching and interrupts that rely on program steps that are actions of some domain theory expressed in Situation Calculus. The dynamic world of \smartpm is modeled as progressing through a series of \emph{situations}. Each situation is the result of various tasks being performed so far. Fluents may be thought of as ``properties'' of the world whose values may vary across situations.

\bigskip

\smartpm provides mechanisms for adapting process schemas that require no pre-defined handlers. To this end, a specialized version of the concept of adaptation from the field of agent-oriented programming~\cite{GiacomoRS98} is used. Specifically, adaptation in \smartpm can be seen as reducing the gap between the \emph{expected reality}, the (idealized) model of reality that is used by the PMS to reason, and the \emph{physical reality}, the real world with the actual values of conditions and outcomes. The physical reality $\Phi_{s}$ reflects the concept of ``now'', i.e., what is happening in the real environment whilst the process is under execution. In general, a task can only be performed in a given physical reality $\Phi_{s}$ if and only if that reality satisfies the \emph{preconditions} $Pre_{i}$ of that task. Moreover, each task has also a set of \emph{effects} $Eff_{i}$ that change the current physical reality $\Phi_{s}$ into a new physical reality $\Phi_{s+1}$.

A PMS that takes in input such a process specification should guarantee that each task is executed correctly, i.e., with an output that satisfies the process specification itself. In fact, at execution time, the process can be easily invalidated because of task failures or since the environment may change due to some external event. For this purpose, the concept of \emph{expected reality}\ $\Psi_{s}$ is given.  A recovery procedure is needed if the two realities are different from each other. An execution monitor is responsible for detecting whether the gap between the expected and physical realities is such that the original process $\delta_0$ cannot progress its execution. In that case, the PMS has to find a recovery process $\delta_a$ that repairs $\delta_0$ and removes the gap between the two kinds of reality.

\bigskip


We developed three different mechanisms for deriving a recovery procedure $\delta_a$.
Currently, the adaptation algorithm deployed in \smartpm synthesizes a linear process $\delta_a$ (i.e., a process consisting of a sequence of tasks) and inserts it at a given point of the original process - specifically, that point of the process where the deviation was first noted. To provide more details, let us assume that the current process is $\delta_0$ = ($\delta_1;\delta_2$) in which $\delta_1$ is the part of the process already executed and $\delta_2$ is the part of the process which remains to be executed when a deviation is identified. The adapted process is $\delta'_0$ = $(\delta_1;\delta_a;\delta_2)$. However, whenever a process needs to be adapted, every running task is interrupted, since the ``repair'' sequence of tasks $\delta_a = [t_1, \ldots, t_n]$ is placed before them. Thus, active branches can only resume their execution after the repair sequence has been executed. This last requirement is fundamental to avoid the risk of introducing data inconsistencies during the repair phase.

\subsection{Representing Tasks in \smartpm}
\label{subsec:approach-overview-representing tasks}

An end user that wants to interact with the \smartpm prototype is not required to model a process and contextual data through complex languages such as \sitcalc and \indigolog. We use those languages for reasoning on the process under execution and for building automatically recovery procedures at run-time. In the following chapter, we describe how the end-user can build an admissible process specification for \smartpm with the \smartpm Definition Tool, that allows to graphically represent the control flow of a process through the BPMN language~\cite{BPMN_book} and to annotate in a descriptive way (through the \smartML language) the tasks composing the process. Before to execute a process defined with BPMN and annotated with \smartML, the \smartpm system automatically translates those specifications in situation calculus and \indigolog readable formats.

In \smartpm, tasks are collected in a specific repository, and each task can be considered as a single step that consumes input data and produces output data. For example, if we consider our case study introduce in Section~\ref{sec:introduction-case_study}, the task \aGo \ can be described with \smartML as follows:

\begin{footnotesize}
\begin{verbatim}
<task>
    <name>go</name>
    <parameters>
        <arg>from - Location_type</arg>
        <arg>to - Location_type</arg>
    </parameters>
    <precondition>at[SRVC] == from AND
                  isConnected[SRVC]
    </precondition>
    <effects>
        <supposed>at[SRVC] = to</supposed>
    </effects>
</task>
\end{verbatim}
\end{footnotesize}
The task \aGo \ involves two input parameters $from$ and $to$ of type $Location$, representing the starting and arrival locations, while $SRVC$ represents the service (i.e., the process participant) that will execute the task at run-time. An instance of this task can be executed only if the service $SRVC$ is currently at the starting location $from$ and is connected to the team's network. The supposed consequence of task execution is that $SRVC$ moves from the starting to the arrival location, and this is reflected by assigning to $at[SRVC]$ the value $to$ in the effect. But, in the real world, the execution of \aGo \ can lead to a different outcome from the expected one. If $SRVC$ reaches a location $to'$ different from $to$, it means that the expected reality changes as the task was executed correctly, while the physical reality continues to reflect the real world, by assuming that $SRVC$ has reached $to'$. In those cases, \smartpm will be able to adapt the process for re-aligning the two realities without any pre-defined policy.

Let us consider that not every task has an outcome that requires the \smartpm system to monitor the two realities. This is the case of the task \aMove:

\begin{footnotesize}
\begin{verbatim}
<task>
    <name>move</name>
    <parameters>
        <arg>from - Location_type</arg>
        <arg>to - Location_type</arg>
    </parameters>
    <precondition>atRobot[SRVC] == from AND
                  batteryLevel[SRVC]>=moveStep[] AND
                  isRobotConnected[SRVC]
    </precondition>
    <effects>
        <supposed>atRobot[SRVC] = to</supposed>
        <automatic>batteryLevel[SRVC] -= moveStep[]</automatic>
    </effects>
</task>
\end{verbatim}
\end{footnotesize}
A service $SRVC$ that execute the task \aMove \ (this task has been thought for representing the robots movement) should reach the location $to$, but the supposed effect involves a contextual property (i.e., \emph{atRobot[SRVC]}, used for recording the current position of a robot) flagged at design-time as \emph{not relevant}, meaning that we need only to update the physical reality with the real output returned by the task, but no update to the expected reality is required. The same is true for the $automatic$ effect that decreases the battery level of the robot of a fixed quantity equal to $moveStep[]$. An automatic effect is automatically applied when the task completes without the need to consider the task outcomes.

A thorough description of the \smartML language is given in Section~\ref{sec:framework-smartpm_definition_tool-smartml}. In this chapter, we concentrate on modeling processes, data and tasks through situation calculus and \indigolog, that allow us to perform some reasoning for automatically adapting faulty process instances ar run-time.

\subsection{Resource Model and Task Life-Cycle in \smartpm}
\label{subsec:approach-overview-resource_perspective}


When allocating tasks to the resources involved in business process execution, PMSs typically adopt a \emph{pull-based} approach where the system offers each task to one or more resources qualified for it (e.g., using a \emph{role-based} distribution approach) and a resource chooses among the offered items and commits to undertake the execution. This approach is motivated by the fact that, in business settings, a PMS has to leverage the interests, priorities, needs and constraints of single participants, and the overall revenue of the organizations they belong to.

The \smartpm prototype was instead initially developed for coordinating emergency management processes. Those dynamic processes are highly critical and time demanding, as well as they often need to be carried out within strictly specified deadlines. Therefore, it is inapplicable to use a \emph{pull} mechanism for task assignment, as the risk is to have some task(s) waiting indefinitely for being chosen and executed. In emergency management, the personal benefit of single operators is unimportant with respect to the overall effectiveness of rescue operations. Consequently, it is preferable a \emph{push-based} approach, where the system dynamically selects a resource qualified for executing a given task (e.g., using a \emph{role-based} or a more specific \emph{capability-based} distribution approach) and directly allocates the work item to the selected resource. According to this approach, in \smartpm each task is assigned \textbf{to only one resource}, and moreover, each resource gets assigned \textbf{at most one task}. In our model, resources involved in process executions are classified as either \emph{human} (e.g., first responders that act on the field), or \emph{non-human} (e.g., robots whose purpose is to maintain the network connection). However, we abstract all possible resources as \emph{services}, i.e., entities capable of executing specific tasks~\cite{resource-patterns}.

\begin{figure}[t]
\centering{
 \includegraphics[width=0.9\columnwidth]{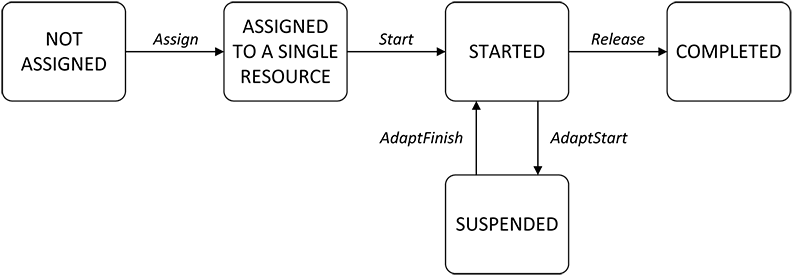}
 } \caption{The task-life cycle in \smartpm.}
 \label{fig:fig_introduction_task_life_cycle}
\end{figure}

The allocation of tasks to services and their execution during process enactment, as well as performed recovery and adaptation procedures, determine the life-cycle of a task, as illustrated in the state-transition diagram in Fig.~\ref{fig:fig_introduction_task_life_cycle}. In \smartpm, each task making part of the process control flow that was never been assigned to any service lies in a \emph{Not Assigned} state. A task becomes \emph{Assigned} after the PMS assigns it to a specific service qualified for executing the task and upon fulfillment of control-flow preconditions. With respect to the general task life-cycle defined in~\cite{resource-patterns}, in our model a task is never \emph{offered} to one or more services and is directly \emph{assigned} to a service for execution. When a service is ready for executing the task, it notifies the \smartpm system through its \emph{Task Handler} about the starting of the task itself; as a consequence the task moves to the \emph{started} state.  A started task may then be temporarily \emph{suspended}, specifically when the system needs to restructure and modify the process for adaptation purposes, and successfully \emph{completed}. A task failure due to the impossibility of the task complete its execution is seen by \smartpm as the task has completed with an ``undefined'' outcome.


\section{Preliminaries}
\label{sec:approach-preliminaries}

In this section we introduce some preliminary notions, namely \emph{\sitcalc} and \indigolog, that are used as proper formalisms to reason about processes and exogenous events. Then, we will give some insights around the use of \emph{classical planning techniques}. This section is not meant to give an all-comprehensive and very formal introduction of the notions. It aims mostly at giving an overall insight to those who are not very expert on such topics.


\subsection{Situation Calculus}
\label{subsec:approach-preliminaries-sitcalc}

The \sitcalc~\cite{McCarthy69-SitCalc,ReiterBook} is a logical language specifically designed for representing dynamically changing worlds in which all changes are the result of named \emph{actions}. A possible history of actions is represented by a so-called \emph{situation}, a first-order term in the language.  The constant $S_0$ denotes the initial situation, where no actions have yet been performed. Sequences of actions are built using the special function symbol $do(a,s)$, that denotes the successor situation resulting from performing action $a$ in situation $s$.

\vskip 0.5em \noindent\colorbox{light-gray}{\begin{minipage}{0.98\textwidth}
\begin{example}
\emph{Let us suppose we defined an action \aOpenDoor(x) concerning the opening of a door $x$. The situation term $do(\aOpenDoor(d_2),do(\aOpenDoor(d_1),S_0))$ denotes the situation resulting from first opening door $d_1$ in $S_0$ and then opening door $d_2$.}
\end{example}
\end{minipage}
}\vskip 0.5em
In general, the \sitcalc can be seen as a ``dialect'' of the first-order logic in which situations and actions are explicitly taken to be objects in the domain. Note that constants and function symbols for actions are completely application dependent. Properties that hold in a situation are called \emph{fluents}. Technically, these are predicates taking a situation term as their last argument.

\vskip 0.5em \noindent\colorbox{light-gray}{\begin{minipage}{0.98\textwidth}
\begin{example}
\emph{The fluent $\fDoorOpen(x,s)$ may denote that door $x$ is open in situation $s$. The formula $\neg\fDoorOpen(d_1,s) \land \fDoorOpen(d_1,do(\aOpenDoor(d_1),s))$ says that the door $d_1$ is closed in some situation $s$, but it will be opened in the situation that results from opening it in $s$.}
\end{example}
\end{minipage}
}\vskip 0.5em
A distinguished predicate $\Poss(a,s)$ is used to state that action $a$ is executable in $s$.

\vskip 0.5em \noindent\colorbox{light-gray}{\begin{minipage}{0.98\textwidth}
\begin{example}
\emph{The predicate $\Poss(\aOpenDoor(x),S_0)$ says that is possible to open the door $x$ in the initial situation $S_0$.}
\end{example}
\end{minipage}
}\vskip 0.5em

To reason about a changing world, it is necessary to have beliefs not only about what is true initially, but also about how the world changes as result of actions. Actions typically have \emph{preconditions}, that is, conditions that need to be true for the action to occur. Formulas like the one in the Example \ref{ex:Preliminaries:Poss} are called \emph{preconditions axioms}.

\vskip 0.5em \noindent\colorbox{light-gray}{\begin{minipage}{0.98\textwidth}
\begin{example}
\label{ex:Preliminaries:Poss}
\emph{The action of opening a door in situation $s$ is possible only if the door is closed in $s$:
\[
\begin{array}{l}
\Poss(\aOpenDoor(x),s)\equiv \neg \fDoorOpen(x,s).
\end{array}
\]}
\end{example}
\end{minipage}
}\vskip 0.5em

Actions typically have also \emph{effects}, that is, fluents that are changed as a result of performing the action.

\vskip 0.5em \noindent\colorbox{light-gray}{\begin{minipage}{0.98\textwidth}
\begin{example}
\label{ex_positive_effect}
\emph{If a door $x$ is closed and not locked, the action \aOpenDoor(x) can be used to open it:
\[
\begin{array}{l}
(\neg\fDoorOpen(x,s) \land \neg\fLocked(x,s)) \supset \fDoorOpen(x,do(\aOpenDoor(x),s)).
\end{array}
\]}
\end{example}
\end{minipage}
}\vskip 0.5em

\vskip 0.5em \noindent\colorbox{light-gray}{\begin{minipage}{0.98\textwidth}
\begin{example}
\label{ex_negative_effect}
\emph{To close a door causes it to be no more opened:
\[
\begin{array}{l}
\neg\fDoorOpen(x,do(\aCloseDoor(x),s)).
\end{array}
\]}
\end{example}
\end{minipage}
}\vskip 0.5em

These formulas are called \emph{effects axioms}; they can be \emph{positive} (cf. Example~\ref{ex_positive_effect}) if they describe when a fluent becomes true, and \emph{negative}  (cf. Example~\ref{ex_negative_effect}) otherwise. In general, for any fluent $F(\vec{x},s)$ we can rewrite all positive effects axioms as a single formula of the form:
\begin{equation}
\label{eq_positive_effects}
P_F(\vec{x},a,s) \supset F(\vec{x},do(a,s))
\end{equation}
and all the negative effect axioms as a single formula of the form:
\begin{equation}
\label{eq_negative_effects}
N_F(\vec{x},a,s) \supset \neg F(\vec{x},do(a,s))
\end{equation}
If we assume that formulas \ref{eq_positive_effects} and \ref{eq_negative_effects} characterize all the conditions under which an action $a$ changes the value of a fluent $F$ (this is called the \emph{completeness assumption}, cf.~\cite{Brachman:2004}), and if we state that no action $a$ satisfies the condition of making the fluent $F$ both true and false at the same time, then all the effects axioms that are linked to a specific fluent $F$ can be combined in a single \emph{successor state axiom} that completely characterizes the value of $F$ in the successor state resulting from performing action $a$ in situation $s$.
\begin{equation}
F(\vec{x},do(a,s)) \equiv P_F(\vec{x},a,s) \lor (F(\vec{x},s) \land \neg N_F(\vec{x},a,s))
\end{equation}
Specifically, $F$ is true after doing $a$ if and only if before doing $a$, $P_F$ (the positive effect condition for $F$) was true, or both $F$ and $\neg N_F$ (the negative effect condition for F) were true.

\vskip 0.5em \noindent\colorbox{light-gray}{\begin{minipage}{0.98\textwidth}
\begin{example}
\emph{The successor state axiom for fluent $\fDoorOpen(x,s)$ is as follows:
\[
\begin{array}{l}
	\fDoorOpen(x,do(a,s)) \equiv {}\\
	\qquad (a=\aOpenDoor(x) \land \neg \fLocked(x,s) \land \neg\fDoorOpen(x,s)) \lor \\
	\qquad  \fDoorOpen(x,s) \land a\not=\aCloseDoor(x);
\end{array}
\]
That is, a door is open after an action $a$ has been performed in $s$ iff $a$ denotes
the action of opening that door and the door is closed and not locked, or the door is already 
open in $s$ and the action $a$ is not that one of closing it.}
\end{example}
\end{minipage}
}\vskip 0.5em

The use of of successor state axioms for capturing how a specific fluent changes after performing some action solves automatically the so-called \emph{frame problem}~\cite{Brachman:2004}. In fact, to really know how the world can change, it should be also necessary to know what fluents are \emph{unaffected} by performing an action.
\vskip 0.5em \noindent\colorbox{light-gray}{\begin{minipage}{0.98\textwidth}
\begin{example}
\emph{To open a door $x$ does not change its color $c$:
\[
\begin{array}{l}
\neg\fColor(x,c,s) \longrightarrow \fColor(x,c,do(\aOpenDoor(x),s)).
\end{array}
\]}
\end{example}
\end{minipage}
}\vskip 0.5em
These kinds of formulas are called \emph{frame axioms}, because they limit or frame the effects of an action, and they would be required to fully capture the dynamism of a situation. Usually, it is expected that only a very small number of actions affects the value of a specific fluent; the rest leave it invariant (e.g., the color of a door is unaffected by opening that same door). The \emph{frame problem} forces to know and reason with a large number of frame axioms; if $A$ is the number of actions and $F$ the number of fluents, we should reason on about 2x$A$x$F$ facts that do not change when we are performing an action. A simple solution to the frame problem has been proposed in~\cite{ReiterBook,Brachman:2004}, and consists of using:
\begin{itemize}[itemsep=1pt,parsep=1pt]
\item precondition axioms, one per action (specifying when the action is executable);
\item successor state axioms, one per fluent (capturing the effects and non-effects of actions);
\item \emph{initial state axioms} describing what is true initially (i.e., what is true in the initial situation $S_0$);
\item unique name axioms for actions (stating that the only action terms that can be equal are two identical actions with identical arguments):
\begin{itemize}
\item $A(\vec{x}) = A(\vec{y}) \supset (x_1 = y_1) \land ... \land (x_n = y_n)$
\item $A(\vec{x}) \neq B(\vec{y})$ where $A$ and $B$ are distinct action names.
\end{itemize}
\end{itemize}
The completeness assumption for the effects of actions allows to conclude that action that are not mentioned explicitly in the effect axioms leave the fluent invariant (so, we do not need anymore to represent explicitly frame axioms).

\subsubsection{Using the Situation Calculus}
\label{subsubsec:approach-preliminaries-sitcalc-using}

Given a knowledge base KB containing facts expressed in the situation calculus, there are various sorts of reasoning tasks we can consider~\cite{ReiterBook,Brachman:2004}. Two basic reasoning tasks are \emph{projection} and \emph{legality testing}.

Given a sequence of actions and some initial situation $S_0$, the \emph{projection task} aims at determining what would be true if those actions were performed starting in that initial situation. In~\cite{Brachman:2004}, the problem is formalized as follows:

\begin{description}
\item [Projection Task] Suppose that $\xi(s)$ is a formula with a single free variable $s$ of the situation sort, and that $\vec{a}$ is a sequence of actions <$a_1,...,a_n$>. To find out if $\xi(s)$ would be true after performing $\vec{a}$ starting in the initial situation $S_0$, we determine whether or not KB $\models \xi(do(\vec{a},S_0))$, where $do(\vec{a}, S_0)$
is an abbreviation for $do(a_n, do(a_{n-1},..., do(a_2, do(a_1, S_0))...))$, and for
$S_0$ itself when n = 0.
\end{description}
\vskip 0.5em \noindent\colorbox{light-gray}{\begin{minipage}{0.98\textwidth}
\begin{example}
\emph{Let us suppose that in $S_0$ the doors $d_1$ and $d_2$ are closed and locked:}
\[
\begin{array}{l}
\fLocked(d_1,S_0) \land \fLocked(d_2,S_0) \land \neg\fDoorOpen(d_1,S_0) \land \neg\fDoorOpen(d_2,S_0)
\end{array}
\]
\emph{and $\xi(s) \equiv (\fDoorOpen(d_1,s) \land \fDoorOpen(d_2,s))$. The sequence of actions}:
\[
\begin{array}{l}
<\aUnlock(d_1),\aOpenDoor(d_1),\aUnlock(d_2),\aOpenDoor(d_2)>
\end{array}
\]
\emph{makes the formula $\xi(s)$ equal to true. In other words, $\xi(s)$ holds in the situation:}
\[
\begin{array}{l}
s = do(\aOpenDoor(d_2),do(\aUnlock(d_2),do(\aOpenDoor(d_1),do(\aUnlock(d_1),S_0))))
\end{array}
\]
\end{example}
\end{minipage}
}\vskip 0.5em

Another interesting task performable through situation calculus is the \emph{legality testing task}, that determines whether a sequence of actions leads to a legal situation. In~\cite{Brachman:2004}, the problem is formalized as follows:
\begin{description}
\item [Legality Task] Suppose that $\vec{a}$ is a sequence of actions <$a_1,...,a_n$>. To find out if $\vec{a}$ can be legally performed starting in the initial situation $S_0$, we determine whether or not KB $\models \Poss(a_i,do($<$a_1,...a_{i-1}$>$,S_0))$ for every $i$ such that $1 \leq i \leq n$.
\end{description}
\vskip 0.5em \noindent\colorbox{light-gray}{\begin{minipage}{0.98\textwidth}
\begin{example}
\emph{For example, given the initial situation:}
\[
\begin{array}{l}
\fLocked(d_1,S_0) \land \fLocked(d_2,S_0) \land \neg\fDoorOpen(d_1,S_0) \land \neg\fDoorOpen(d_2,S_0)
\end{array}
\]
\emph{although the situation term} $do(\aOpenDoor(d_2),do(\aUnlock(d_1),S_0))$ \emph{is well-formed, it is not a legal situation, because the precondition for opening $d_2$ states that $d_2$ has to be unlocked, while the first action in the situation term unlocks door $d_1$. The situation term:}
\[
\begin{array}{l}
s = do(\aOpenDoor(d_2),do(\aUnlock(d_2),do(\aOpenDoor(d_1),do(\aUnlock(d_1),S_0))))
\end{array}
\]
\emph{is instead legal, since the sequence of actions can be legally performed starting from $S_0.$}
\end{example}
\end{minipage}
}\vskip 0.5em


\subsection{Indigolog}
\label{subsec:approach-preliminaries-indigolog}

On top of Situation Calculus action theories, one can define complex control behaviors by means of high-level programs expressed in \golog-like programming languages. Specifically we focus on \indigolog~\cite{Indigolog:2009}, which provides a set of programming constructs sufficient for defining every well-structured process as defined in~\cite{workflow-patterns}.

\indigolog is a programming language for autonomous agents that sense their environment and do planning as they operate.  The programmer provides a high-level nondeterministic program involving domain-specific actions and tests to perform the agent's tasks. The \indigolog interpreter then reasons about the preconditions and effects of the actions in the program to find
a legal terminating execution. To support this, the programmer provides a declarative specification of the domain (i.e., primitive actions, preconditions and effects, what is known about the initial state) in the \sitcalc. \indigolog programs are executed \textbf{online} together with sensing the environment and monitoring for events.

\indigolog derives from \ConGolog\ to which it adds basically the lookahead \emph{search operator}. Such operator allows to simulate the execution of a process with the aim of searching for a successful termination before actually performing
the program in the real world. In its, turn \ConGolog~\cite{Congolog:2000} extends the original \Golog~\cite{Golog:1997} by introducing construct for current execution of different operations. Table~\ref{tab:indigolog_constructs} summarizes the constructs of \indigolog used in this thesis.

\begin{table}[t]
\caption{\indigolog\ constructs.}
\centering
\begin{footnotesize}
\begin{tabular}{|l|p{0.70\columnwidth}|
}
  \hline
  \textbf{Construct} & \textbf{Meaning}
  \\\hline
  $a$ & A primitive action.
  \\\hline
  $\sigma?$ & Wait while the $\sigma$ condition holds.
  \\\hline
  $(\delta_1;\delta_2)$ & Sequence of two sub-programs $\delta_1$ and $\delta_2$.
  \\\hline
  $proc~P(\overrightarrow{v})~\delta$ & Invocation of a \indigolog\ procedure $\delta$
  passing a vector $\overrightarrow{v}$ of parameters.
   \\\hline
  $(\delta_1 | \delta_2)$ & Non-deterministic choice among (sub-)programs $\delta_1$ and $\delta_2$.
  \\\hline
  $if~\sigma~then~\delta_1~else~\delta_2$ & Conditional statement: if $\sigma$ holds, $\delta_1$ is executed; otherwise $\delta_2$.
  \\\hline
  $while~\sigma~do~\delta$ & Iterative invocation of $\delta$.
  \\\hline
  $(\delta_1\parallel\delta_2)$ & Concurrent execution.
  \\\hline
  $(\delta_1\prparallel\delta_2)$ & Prioritized concurrent execution.
  \\\hline
  $\delta^*$ & Non-deterministic iteration of program execution.
  \\\hline
  $\pi a.\delta$ & Non-deterministic choice of argument $a$ followed by the execution of $\delta$.
  \\\hline
  $\tuple{\sigma \rightarrow \delta}$ & $\delta$ is repeatedly executed until
  $\sigma$ becomes false, releasing control to anyone
  else able to execute.\\\hline
  $\Sigma(\delta)$ & search operator
  \\\hline
  $send(\Upsilon,\overrightarrow{v})$ & a vector $\overrightarrow{v}$ of parameters is passed to an external program $\Upsilon$.\\\hline
  $receive(\Upsilon,\overrightarrow{v})$ & a vector $\overrightarrow{v}$ of parameters is received by an external program $\Upsilon$.\\\hline
\end{tabular}
\end{footnotesize}
  \label{tab:indigolog_constructs}
\end{table}

In the first line, $a$ stands for a situation calculus action term whereas, in the second line, $\sigma$? stands for a formula over situation calculus predicates and fluents that needs to be evaluated when reached.

There are some constructs for concurrent programming. In particular ($\delta_1\parallel\delta_2$) expresses the concurrent execution (interpreted as interleaving) of the programs $\delta_1$ and $\delta_2$. 
Another concurrent programming construct is $(\delta_1 \prparallel \delta_2)$, where $\delta_1$ has higher priority than $\delta_2$, and $\delta_2$ may only execute when $\delta_1$ is done or blocked.

The list in Table~\ref{tab:indigolog_constructs} includes also some nondeterministic constructs. For example, $(\delta_1 \mid \delta_2)$ nondeterministically chooses between programs $\delta_1$ and $\delta_2$. Test actions $\sigma$? can be used to control which branches may be executed, e.g., $[(\sigma?;  \delta_1)  |  (\neg\sigma?; \delta_1)]$ will perform $\delta_1$ when $\sigma$ is \emph{true} and $\delta_2$ when $\sigma$ is \emph{false} (we use [...] and (...) interchangeably to disambiguate structure in programs). The construct $\pi\,a.\, \delta$, nondeterministically picks a binding for the variable $a$ and performs the program $\delta$ for this binding of $a$. Finally, $\delta^*$, performs $\delta$ zero or more times. $\pi\,a_1,\ldots,a_n.\, \delta$ is an abbreviation for $\pi\,a_1.\ldots.\pi\, a_n\, \delta$.

The constructs $\mif\ \sigma\ \mthen\ \delta_1\ \melse\ \delta_2$ and $\mwhile\ \sigma\ \mdo\ \delta$ are the synchronized versions of the usual if-then-else and while-loop. They are synchronized in the sense that testing the condition $\sigma$ does not involve a transition per se: the evaluation of the condition and the first action of the branch chosen are executed as an atomic unit. So these constructs behave in a similar way to the test-and-set atomic instructions used to build semaphores in concurrent programming.

Let's focus on the interrupt construct:
\[\begin{array}{l}
\langle\ \sigma \rightarrow
\delta\ \rangle\ \isdef\
\mwhile\ Interrupts\_running\ \mdo\\
\mif\ \sigma\ \mthen\
\delta\ \melse\
false\
\mendif\\
\mendwhile
\end{array}\]
To see how this works, first assume that the special fluent $Interrupts\_running$ is identically true. When an interrupt $\tuple{\sigma \rightarrow \delta}$ gets control from higher priority processes, suspending any lower priority processes that may have been advancing, it repeatedly executes $\delta$ until $\sigma$ becomes false. Once the interrupt body $\delta$ completes its execution, the suspended lower priority processes may resume. The control release also occurs if $\sigma$ cannot progress (e.g., since no action meets its precondition).


Since \indigolog\ provides flexible mechanisms for interfacing with other programming languages such as \texttt{Java} or \texttt{C} and for socket communication (cf. Section~\ref{sec:framework-indigolog_platform}), for our convenience we have defined here two more abstract constructs to send/receive parameters as well as values with external programs, defined out of the range of the \indigolog\ process. 

From a formal point of view, two predicates are introduced to specify program transitions:
\begin{itemize}[itemsep=1pt,parsep=1pt]
    \item $Trans(\delta',s',\delta'',s'')$, given a program $\delta'$
    and a situation $s'$, returns \myi a new situation $s''$ resulting from executing
    a single step of
    $\delta'$, and \myii $\delta''$ which is the remaining program to be executed.
    \item $Final(\delta',s')$ returns true when the program $\delta'$
    can be considered successfully completed in situation $s'$, i.e., $\delta'$ is \emph{legally terminated} in situation $s'$.
\end{itemize}

The predicates $\Trans$ and $Finals$ for programs 
are characterized by the following set of axioms:
\begin{enumerate}[itemsep=1.5pt,parsep=1pt]

\item Empty program:
\[
    \Trans(\nil,s,\delta',s') ~\Leftrightarrow~ \False.
\]
\[
    \Final(\nil,s) \equiv \True.
\]

\item Primitive actions:

\[
\Trans(a,s,\delta',s') ~\Leftrightarrow~
    \Poss(a[s],s) \land \delta'=\nil \land s'=\xdo(a[s],s).
\]
\[
    \Final(a,s) \equiv \False.
\]

\item Test/wait actions:
\[
\Trans(\sigma?,s,\delta',s') ~\Leftrightarrow~
    \sigma[s] \land \delta'=\nil \land s'=s.
\]
\[
    \Final(\sigma?,s) \equiv \False.
\]

\item Sequence:
\begin{eqnarray*}
\lefteqn{\Trans(\delta_1;\delta_2,s,\delta',s') ~\Leftrightarrow}\\
    & \phantom{xxx} &
    \exists\gamma.\delta'=(\gamma;\delta_2) \land
    \Trans(\delta_1,s,\gamma,s') \lor
    \Final(\delta_1,s) \land \Trans(\delta_2,s,\delta',s').
\end{eqnarray*}
\[
    \Final(\delta_1;\delta_2,s) \equiv \Final(\delta_1,s) \land \Final(\delta_2,s).
\]

\item Nondeterministic branch:
\[
\Trans(\delta_1 \mid \delta_2,s,\delta',s') ~\Leftrightarrow~
    \Trans(\delta_1,s,\delta',s') \lor \Trans(\delta_2,s,\delta',s').
\]
\[
    \Final(\delta_1 | \delta_2,s) \equiv \Final(\delta_1,s) \lor \Final(\delta_2,s).
\]

\item Nondeterministic choice of argument:
\[
\Trans(\pi{v}.\delta,s,\delta',s')
       ~\Leftrightarrow~
    \exists x.\Trans({\delta}^{v}_{x},s,\delta',s').
\]
\[
    \Final(\pi{v}.\delta,s) \equiv \exists x.\Final(\delta^{v}_{x},s).
\]

\item Nondeterministic iteration:
\[
\Trans(\delta^*,s,\delta',s') ~\Leftrightarrow~
   \exists\gamma.(\delta'=\gamma;\delta^*) \land \Trans(\delta,s,\gamma,s').
\]
\[
    \Final(\delta^*,s) \equiv \True.
\]

\item Synchronized conditional:
\begin{eqnarray*}
\lefteqn{\Trans(\mif\ \sigma\ \mthen\ \delta_1\ \melse\ \delta_2\
\mendif,s,\delta',s')~~\Leftrightarrow~~{}}\\
    & \phantom{xxx} &
    \sigma[s] \land \Trans(\delta_1,s,\delta',s') \lor
    \lnot\sigma[s] \land \Trans(\delta_2,s,\delta',s').
\end{eqnarray*}
\begin{eqnarray*}
    \Final(\mif\ \sigma\ \mthen\ \delta_1\ \melse\ \delta_2\ \mendif,s) \equiv \qquad \qquad \qquad\\
    \qquad \qquad \qquad \sigma[s] \ \land \ \Final(\delta_1,s) \ \lor \ \neg\sigma[s] \ \land \ \Final(\delta_2,s).
\end{eqnarray*}

\item Synchronized loop:
\begin{eqnarray*}
\lefteqn{\Trans(\mwhile\ \sigma\ \mdo\ \delta\
\mendwhile,\,s,\delta',s')
    ~\Leftrightarrow~}\\
    & \phantom{xxx} &
    \exists\gamma. (\delta'= \gamma;\mwhile\ \sigma\ \mdo\ \delta)
         \land \sigma[s] \land \Trans(\delta,s,\gamma,s').
\end{eqnarray*}
\[
    \Final(\mwhile\ \sigma\ \mdo\ \delta\ \mendwhile,\,s) \equiv \neg\sigma[s] \ \lor \ \Final(\delta,s).
\]

\item Concurrent execution:
\begin{eqnarray*}
\lefteqn{\Trans(\delta_1\parallel\delta_2,s,\delta',s') ~\Leftrightarrow~}\\
    &  &
    \exists\gamma. \delta'= (\gamma\parallel\delta_2) \land
        \Trans(\delta_1,s,\gamma,s') \lor
    \exists\gamma. \delta'=(\delta_1\parallel\gamma) \land
        \Trans(\delta_2,s,\gamma,s').
\end{eqnarray*}
\[
    \Final(\delta_1\parallel\delta_2,s) \equiv \Final(\delta_1,s) \land \Final(\delta_2,s).
\]

\item Prioritized concurrency:
\begin{eqnarray*}
\lefteqn{\Trans(\delta_1\prparallel\delta_2,s,\delta',s') ~~\Leftrightarrow~~{}}\\
    &\phantom{xxx} &
    \exists\gamma. \delta'=(\gamma\prparallel\delta_2) \land
        \Trans(\delta_1,s,\gamma,s') ~\lor~{}\\
    &\phantom{xxx} &
    \exists\gamma. \delta'=(\delta_1\prparallel\gamma) \land
        \Trans(\delta_2,s,\gamma,s') \land
    \lnot\exists\zeta,s''.\Trans(\delta_1,s,\zeta,s'').
\end{eqnarray*}
\[
    \Final(\delta_1\prparallel\delta_2,s) \equiv \Final(\delta_1,s) \land \Final(\delta_2,s).
\]

\item Concurrent iteration:
\[
\Trans(\delta^{\supparallel},s,\delta',s') ~\Leftrightarrow~
    \exists{\gamma}. \delta'=(\gamma\parallel\delta^{\supparallel})
    \land \Trans(\delta,s,\gamma,s').
\]
\[
    \Final(\delta^{\supparallel},s) \equiv \True.
\]

\end{enumerate}

The \textbf{off-line} execution of programs, which is the kind of execution originally proposed for \Golog~\cite{Golog:1997} and \ConGolog~\cite{Congolog:2000} is characterized using the $Do(\delta,s,s')$ predicate, which means that there is an execution of program $\delta$ that starts in situation $s$ and terminates in situation $s'$:
\begin{displaymath}
Do(\delta,s,s')\Leftrightarrow\exists\delta'.Trans^*(\delta,s,\delta',s')
\wedge Final(\delta',s')
\end{displaymath}
\noindent where $Trans^*$ is the definition of the reflective and
transitive closure of \emph{Trans}. Thus there is an execution of program $\delta$ that starts in situation $s$ and terminates in situation $s'$ if and only if we can perform 0 or more transitions from program $\delta$ in
situation $s$ to reach situation $s'$ with program $\delta'$ remaining, at which point
one may legally terminate. Notice that there may be more
than one resulting situation $s'$ since \indigolog programs can be
non-deterministic (e.g., due to concurrency).

The off-line execution model of \Golog and \ConGolog requires the executor to search over the whole program to find a complete execution before performing any action. This is obviously problematic for agents that need to sense their environment as they operate. On the contrary, the strength of \indigolog is that it provides an \textbf{online} execution model that allows to execute actions on the real world, to update its knowledge after each action execution and to monitor for possible exogenous events or actions not executed as expected.

Finally, to cope with the impossibility of backtracking actions executed in
the real world, \indigolog\ incorporates a new programming
construct, namely the {\em search operator} $\search(\delta)$, which is used to specify that lookahead should be performed over the (nondeterministic) program $\delta$ to ensure that nondeterministic choices are resolved in a way that guarantees its successful completion. More precisely, let $\delta$ be any
\indigolog\ program, which provides different alternative executable
actions. When the interpreter encounters program $\Sigma(\delta)$,
before choosing among alternative executable actions of $\delta$, it
performs reasoning in order to decide for a step which still allows
the rest of $\delta$ to terminate successfully. Formally,
according to \cite{degiacomo:1999}, the semantics of the
search operator is:
\[
Trans(\Sigma(\delta),s,\Sigma(\delta'),s')\,\,\Leftrightarrow\,\,
Trans(\delta,s,\delta',s')\,\land\,\exists{s^*}.Do(\delta',s',s^*).
\]
If $\delta$ is the entire program under consideration,
$\Sigma(\delta)$ emulates completely its off-line execution as if it is executed on-line.

\subsection{Classical Planning}
\label{subsec:approach-preliminaries-planning}

Automated Planning~\cite{TraversoBook2004} is the branch of AI that consists in the deliberation process of building plans, i.e., organized actions, in order to fulfill some pre-stated objectives. Typically, these plans are executed by intelligent agents, and the solution amounts to synthesize agent's plans satisfying a goal specification.

There exist several forms of planning. We can distinguish between \emph{classical planning} domains, \emph{conditional planning} and \emph{conformant planning}. \emph{Classical planning} domains are fully observable, static, and deterministic, in which plans can be computed in advance and then applied unconditionally. Therefore a \emph{plan} is a simple sequence $a_0,a_1,\ldots,a_n\in Act^*$ of actions, for a fixed action alphabet $Act$. \emph{Conditional} (or \emph{contingency}) \emph{planning} deals instead with bounded indeterminacy by constructing a conditional plan with different branches for the different contingencies that may arise, and even though plans are pre-computed, the agent finds out which part of the plan to execute by including sensing actions in the plan to test for the appropriate conditions. Therefore, conditional plans can be thought of as tree-like structures, in contrast with sequential plans that are instead action sequences. Finally, \emph{conformant planning}~\cite{smith1998conformant} aims to construct standard, sequential plans that are executed in partially-observable settings, without perception. Namely, they are required to achieve the goal in all possible circumstances, regardless of the true initial state and the actual action outcomes.


In this thesis, we focus on classical planning techniques. Classical planning has made huge advances in the last twenty years, leading to solvers able to create plans with thousands of actions for problems described by hundreds of propositions. The standard representation language for classical planners is known as the Planning Domain Definition Language (PDDL (cf.~\cite{PDDL})); it allows one to formulate a problem $\textsf{PR}$ through the description of the initial state of the world $init_{\textsf{PR}}$, the description of the desired goal condition $goal_{\textsf{PR}}$ and a set of possible actions. An action definition defines the conditions under which an action can be executed, called \emph{preconditions}, and its \emph{effects} on the state of the world. The set of all action definitions $\Omega$ is the \emph{domain} $\textsf{PD}$ of the planning problem. Each action $a \in \Omega$ has a precondition list and an effect list, denoted respectively as $Pre_a$ and $Eff_a$. For example, let us consider the following (trivial) planning domain named ``example'':

\begin{footnotesize}
\begin{verbatim}
(define (domain example)
(:predicates
  (x) (y) (z)
  (k) (v) (s)
)
 (:action t1
    :precondition (x)
    :effect (and (not(x) (z) (k))))
 (:action t2
    :precondition (z)
    :effect (y))
 (:action t3
    :precondition (v)
    :effect (s))
 (:action t4
    :precondition (k)
    :effect (x))
)
\end{verbatim}
\end{footnotesize}

The meaning is straightforward. There are 6 predicates and 4 planning actions, and for each action some preconditions and effects are defined. For example, the first action definition states that for executing $t_1$, the predicate $x$ must hold. Then, it states that a successful execution of $t_1$ guarantees that predicates $\neg x$, $z$ and $k$ will hold together. Given a specific planning domain, it is possible to define a proper planning problem:

\begin{footnotesize}
\begin{verbatim}
(define (problem pr) (:domain example)
(:init
(x) (v)
)
(:goal
(and
(y) (s)
))
)
\end{verbatim}
\end{footnotesize}

It states that in $init_{\textsf{PR}}$ only predicates $x$ and $v$ are known to be $true$, and the goal condition $goal_{\textsf{PR}}$ is a formula where the conjunction of $y$ and $s$ is \emph{true}. Note that one of the main assumption we rely in building planning problem is the well known \emph{closed-world assumption}~\cite{Reiter@NMR1987}. Basically, it states that all unmentioned literals in a planning problem are considered to be \emph{false}. Hence, for example, in the planning problem defined above we are supposing that predicates $z$, $k$, $y$ and $s$ are initially $false$.

A planner that works on such inputs generates a sequence of actions (the \emph{plan}) that corresponds to a path from the initial state to a state meeting the goal condition. A simple plan satisfying the above planning problem consists first in executing action $t3$ (since predicate $v$ is \emph{true} in the initial state), that makes $s$ equal to \emph{true}, and then on executing $t1$ and $t2$ in sequence, for turning the value of $y$ to $true$. This simple sequence composed by three actions is an \emph{optimal solution}, since it requires the minimum number of steps for satisfying the goal condition. Other possible plans that satisfy the goal condition but require more steps if compared with the optimal solution, are called \emph{suboptimal}.

In this chapter, we represent planning domains and planning problems making use of PDDL version 2.2~\cite{PDDL22}. PDDL 2.2 is characterized for enabling the representation of realistic planning domains, which include (in particular) actions and goals involving numerical expressions, operators with universally quantified effects or existentially quantified preconditions, operators with disjunctive or implicative preconditions, derived predicates and plan metrics. However, currently, our formalism does not allow to represent conditional and universally quantified effects.

Specifically, we synthesize our recovery plan through the LPG-td planner~\cite{LPG} (Local search for Planning Graphs). It is a planner based on local search and planning graphs that handles PDDL 2.2 domains. The basic search scheme of LPG-td is inspired by Walksat~\cite{WALKSAT}, an efficient procedure to solve SAT-problems. The search space of LPG consists of ``action graphs'', particular subgraphs of the planning graph representing partial plans. The search steps are certain graph modifications transforming an action graph into another one.  LPG-td exploits a compact representation of the planning graph to define the search neighborhood and to evaluate its elements using a parametrized function, where the parameters weight different types of inconsistencies in the current partial plan, and are dynamically evaluated during search using discrete Lagrange multipliers. The evaluation function uses some heuristics to estimate the ``search cost'' and the ``execution cost'' of achieving a (possibly numeric) precondition. The planner can produce good quality plans in terms of one or more criteria. This is achieved by an anytime process producing a sequence of plans, each of which is an improvement of the previous ones in terms of its quality. More details on the search algorithm and heuristics devised for this planner can be found at~\cite{LPG,LPGth}.  

\section{Formalizing processes in \indigolog}
\label{sec:approach-formalization}

Our approach to process management relies on the \indigolog platform~\cite{Indigolog:2009}, that is able to reason about a dynamic changing world by updating its knowledge after each action execution. The PMS provided by the \smartpm system has been developed on top of the \indigolog platform (in the rest of this thesis, when we refer to the \indigolog PMS, or to the \indigolog engine, we always mean the PMS developed through the \indigolog platform). The \indigolog PMS takes in input a theory of actions - specified in the situation calculus - representing the contextual environment in which the process operates, and an \indigolog program (specified through the constructs shown in Table~\ref{tab:indigolog_constructs}) that reflects the control flow of the dynamic process to be executed. To this end, in this section, we describe how processes can be formalized in situation calculus and \indigolog. We refer to our case study introduces in Section~\ref{sec:introduction-case_study}.

To denote the various objects of interest, we make use of the following domain-independent predicates (that is, non-fluent rigid predicates):

\begin{itemize}[itemsep=1pt,parsep=1pt,topsep=0.5pt]
    \item $\fService(c)$: $c$ is a service (i.e., a process participant);
    \item $\fTask(t)$: $t$ is a task;
    \item $\fCapability(b)$: $b$ is a capability;
    \item $\fProvides(c,b)$: service $c$ provides the capability $b$;
    \item $\fRequires(t,b)$: task $t$ requires the capability $b$.
\end{itemize}

\vskip 0.5em \noindent\colorbox{light-gray}{\begin{minipage}{0.98\textwidth}
\begin{example}
\emph{Let us consider the case study defined in Section~\ref{sec:introduction-case_study}. We need to define 6 services (4 human actors and 2 robots), a repository of emergency-management tasks and a number of capabilities to be associated both to services and tasks.}\\
\\
Service(act1), Service(act2), Service(act3), Service(act4), Service(rb1), Service(rb2).\\
\\
Capability(hatchet), Capability(extinguisher), Capability(camera), Capability(battery), Capability(movement), Capability(gprs), Capability(digger), Capability(powerpack).\\
\\
Task(chargebattery), Task(go), Task(move), Task(extinguishfire), Task(evacuate), Task(removedebris), Task(takephoto), Task(updatestatus).\\
\\
Provides(act1,movement), Provides(act1,gprs), Provides(act1,extinguisher), Provides(act1,camera), Provides(act2,movement), Provides(act2,gprs),\\
Provides(act2,hatchet), Provides(act3,movement), Provides(act3,gprs),\\
Provides(act3,hatchet), Provides(act4,movement), Provides(act4,powerpack),\\
Provides(act4,gprs), Provides(rb1,battery), Provides(rb1,digger),\\
Provides(rb2,battery), Provides(rb2,digger).\\
\\
Requires(go,movement), Requires(evacuate,hatchet), Requires(takephoto,camera),
\\Requires(updatestatus,gprs), Requires(extinguishfire,extinguisher),
\\Requires(move,battery), Requires(removedebris,digger),\\Requires(chargebattery,powerpack).
\end{example}
\end{minipage}
}\vskip 0.5em

To refer to the ability of a service $c$ to perform a certain task $t$, we introduce the following abbreviation:
\begin{equation} \label{eq:eq_capable}
\begin{array}{l}
\fCapable(c,t) \isdef \forall b.\fRequires(t,b) \Rightarrow \fProvides(c,b).
\end{array}
\end{equation}
That is, service $c$ can carry out a certain task $t$ if and only if $c$ provides all capabilities required by the task $t$.

The life-cycle of a task involves the execution of four basic actions:
\begin{itemize}[itemsep=1pt,parsep=1pt]
    \item $\aAssign(c,id,t,\vec{i},\vec{p})$: a task $t$ with input $\vec{i} = [i_1,...,i_z]$ is assigned to a service $c$. With $\vec{p} = [p_1,...,p_w]$ we denote the list of \emph{expected outputs} that $t$ is supposed to return if its execution is successful;
    \item $\aStart(c,id,t)$: service $c$ is notified to start task $t$;
    \item $\aAck(c,id,t)$: service $c$ acknowledges of the completion of task $t$;
    \item $\aRelease(c,id,t,\vec{i},\vec{p},\vec{q})$: service $c$ releases task $t$, that was executed with the list of inputs $\vec{i}$ and expected outputs $\vec{p}$, and returns a list of \emph{physical outputs} $\vec{q} = [q_1,...,q_w]$.
\end{itemize}

Every running task will be associated with a unique identifier $id$ (to this end, we provide a list of admissible identifiers and a further predicate $\fIdentifier(id)$ that is $true$ if $id$ represents a valid identifier) and with a number of input and output parameters. In this sense, the terms $\vec{i}$, $\vec{p}$ and $\vec{q}$ denote arbitrary sets of input/output, which depend on the specific task. The special constant $\emptyset$ denotes empty input or output. Note that we suppose to work with domains in which services, tasks, input and output parameters are finite.

The actions performed by the process need to be ``complemented'' by other actions executed by the services themselves. The following are used to inform the PMS engine about how tasks execution is progressing:
\begin{itemize}[itemsep=1pt,parsep=1pt]
    \item $\aReady(c,id,t)$: service $c$ declares to be ready to start performing $t$;
    \item $\aFinish(c,id,t,\vec{q})$: service $c$ declares to have completed the execution of task $t$ with a list of physical outputs denoted by $\vec{q}$.
\end{itemize}

\begin{figure}[t]
\centering{
 \includegraphics[width=0.85\columnwidth]{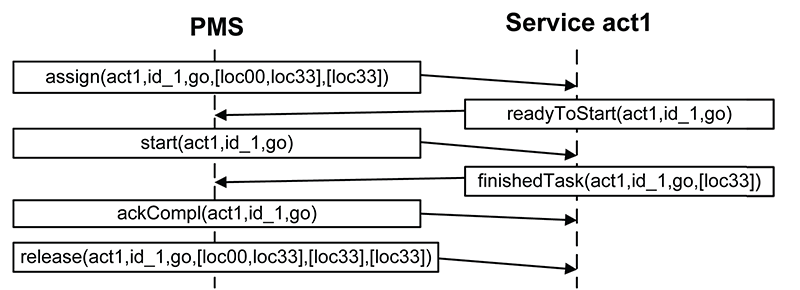}
 } \caption{The protocol for task assignment and execution.}
 \label{fig:fig_approach-formalization_seqDiagram}
\end{figure}

The protocol for a successful execution of a specific task $t_j$ can be described as follows. After the PMS assigns $t_j$  to a service $c$ through the action $\aAssign(c,id,t_j,\vec{i},\vec{p})$, the service $c$ has to report to be ready to execute the task itself by launching the action $\aReady(c,id,t_j)$. When it does happen, the PMS can eventually perform the action $\aStart(c,id,t_j)$, meaning that $c$ has been formally allowed to start executing $t_j$. When $c$ completes the execution of $t_j$, it invokes the action $\aFinish(c,id,t_j,\vec{q})$, with $\vec{q}$ representing the real outcomes (we also call them ``physical'' outcomes) of $t_j$ applied on the contextual environment. At this point the PMS can update the properties (i.e., the situation calculus fluents, see later) reflecting the evolution of the contextual scenario in which the process is under execution. In addition, the PMS first acknowledges the completion of $t_j$ through the action $\aAck(c,id,t_j)$, and then releases the service $c$ from the task $t_j$. After the execution of $\aRelease(c,id,t_j,\vec{i},\vec{p},\vec{q})$, service $c$ is again ready for a new task assignment. In Fig.~\ref{fig:fig_approach-formalization_seqDiagram} we describe the protocol for assigning the task \aGo \ to service \emph{act1} and executing it properly. In fact, when \emph{act1} launches the \aFinish \ action (meaning the task has been completed by \emph{act1}) it returns an expected output equal to the desired destination expected at design time. On the contrary, note that the scenario described in  Fig.~\ref{fig:fig_introduction-case_study-context_2} reflects a wrong execution of \aGo, with an output different from the one expected.

In order to describe the contextual scenario in which the process will be executed and to bound task inputs and outputs, there is the need to define some basic data types. The \indigolog language already provides the $Boolean\_type$ (for representing boolean values) and the $Integer\_type$ (for representing integers from 0 to a specific customizable maximum value - the default is 30). Furthermore, to clearly delineating our case study, we also need to define a data type $Location\_type$ that represents locations in the area and a data type $Status\_type$ for describing if a specific location is on fire or buried by debris.

\vskip 0.5em \noindent\colorbox{light-gray}{\begin{minipage}{0.98\textwidth}
\begin{example}
\emph{The list of required data types for our case study is as follows:}\\
\\Boolean\_type(true).\\
Boolean\_type(false).\\
Integer\_type(0).\\
...\\
Integer\_type(30).\\
Status\_type(ok).\\
Status\_type(fire).\\
Status\_type(debris).\\
Location\_type(loc00).\\
...\\
Location\_type(loc33).
\end{example}
\end{minipage}
}\vskip 0.5em

Data types can be associated to tasks for bounding input and output parameters. For this aim, we introduce the concept of \emph{workitem}, that reflects the run-time container of a task, including the task identifier and its input/output parameters. This means that the same task can be assigned multiple times to services if and only if it is associated with different identifiers, i.e, if it is contained in different workitems. Before to execute the process, we need to clearly state which workitems are admissible for being executed by the PMS.
\vskip 0.5em \noindent\colorbox{light-gray}{\begin{minipage}{0.98\textwidth}
\begin{example} \emph{The task \aGo \ can be executed by the PMS if and only if it is associated to a couple of inputs \emph{from} and \emph{to}, that represent starting and arrival locations, and to an expected output \emph{to}, that corresponds to the expected destination to be reached after the task execution. Note that in this specific case the expected output corresponds to the second argument given as input to the task \aGo.}
\begin{equation*}
\begin{array}{l}
\forall \ (from,to,id) \ s.t. \ Identifier(id) \ \land \ Location\_type(from) \ \land \\
\qquad Location\_type(to) \Rightarrow workitem(\aGo,id,[from,to],[to]).
\end{array}
\end{equation*}
\end{example}
\end{minipage}
}\vskip 0.5em

\bigskip
The formalization of processes in \indigolog\ requires two distinct sets of fluents. A first set includes those \emph{domain-independent fluents} that the PMS uses to manage the task life-cycle and the resource perspective of a process. Domain-independent fluents do not change across different domains. The second set concerns those fluents that we use to denote the data needed by a process instance and the properties of the contextual scenario in which the process is under execution. We call them \emph{data fluents}, and their definition depends strictly on the specific process domain of interest.

\vskip 0.5em
\noindent{\textbf{Representing domain-independent fluents.}}

Task assignment is driven by the fluent $\fFree(c,s)$, which denotes whether a service $c$ is available for task assignments in situation $s$. Basically, $\fFree(c,s)$ holds if the service $c$ is free from any assignment in situation $s$. The corresponding successor state axiom can be defined as follows:
\begin{equation}\label{eq:freeAxiom}
\begin{array}{l}
\fFree(c,do(a,s)) \Leftrightarrow {}\\
\qquad\big(\exists \ t,id,\vec{i},\vec{p},\vec{q} \ \ s.t. \ a = \aRelease(c,id,t,\vec{i},\vec{p},\vec{q}) \big) \ \vee
{}\\
\qquad\big(\fFree(c,s) \wedge \nexists \ t,id,\vec{i},\vec{p} \ \ s.t. \ a = \aAssign(c,id,t,\vec{i},\vec{p})\big)
\end{array}
\end{equation}
Therefore, service $c$ is considered as free in the current situation if and only if it has just been released from a task assignment or it was free in the previous situation and no task has been assigned to it.

The domain-independent fluent $\fAssigned(c,id,t,s)$ aims at representing if a task $t$ with identifier $id$ has been assigned to a service $c$ in situation $s$.
\begin{equation}\label{eq:assignedAxiom}
\begin{array}{l}
\fAssigned(c,id,t,do(a,s)) \Leftrightarrow \\
\qquad\big(\exists \ \vec{i},\vec{p} \ \ s.t. \ a = \aAssign(c,id,t,\vec{i},\vec{p}) \big)  \ \vee {}\\
\qquad\big( \fAssigned(c,id,t,s)
\wedge \ \nexists \ \vec{i},\vec{p},\vec{q} \ \ s.t. \ a = \aRelease(c,id,t,\vec{i},\vec{p},\vec{q}) \big)
\end{array}
\end{equation}
The fluent holds when a task $t$ with identifier $id$ is assigned to a service $c$,
or when $t$ has already been assigned to $c$, and no release action concerning $t$ and $c$ is launched in the current situation $s$.

A third fluent used for constraining the protocol for task assignment and execution is \fReserved. The corresponding successor-state axiom can be defined as follows:
\begin{equation}\label{eq:reservedAxiom}
\begin{array}{l}
\fReserved(c,id,t,do(a,s)) \Leftrightarrow \\
\qquad\big( \fAssigned(c,id,t,s) \wedge a = \aReady(c,id,t) \big)  \ \vee {}\\
\qquad\big( \fReserved(c,id,t,s)
\wedge \ \nexists \ \ \vec{q} \ \ s.t. \ a = \aFinish(c,id,t,\vec{q}) \big)
\end{array}
\end{equation}
This fluent holds when a service $c$ is ready for start executing a task $t$ with identifier $id$ that was already assigned to it, or if $c$ is still executing $t$ and no \aFinish \ action has been launched from $c$.

The $\fReserved$ fluent can be complemented with the \emph{precondition axioms} defined for the $\aStart$ and $\aAck$ actions:
\begin{equation}\label{eq:possStartStop}
\begin{array}{l}
Poss(\aStart(c,id,t),s) \Leftrightarrow \fReserved(c,id,t,s)\\
Poss(\aAck(c,id,t),s) \Leftrightarrow \neg \fReserved(c,id,t,s)\\
\end{array}
\end{equation}
The meaning of those axioms is straightforward: $\aStart(c,id,t)$ can be launched by the PMS only if $\fReserved(c,id,t,s)$ is true, therefore after $c$ has performed the action $\aReady(c,id,t)$ (cf. Equation~\ref{eq:reservedAxiom}). The action $\aAck(c,id,t)$ can be instead invoked when $\fReserved(c,id,t,s)$ does not hold anymore, so when $c$ launches $\aFinish(c,id,t,\vec{q})$ (cf. again Equation~\ref{eq:reservedAxiom}). The combination of precondition axioms defined in Equation~\ref{eq:possStartStop} and of the state-successor axiom represented in Equation~\ref{eq:reservedAxiom} states that $\aAck(c,id,t)$ never comes before $\aStart(c,id,t)$, by guaranteeing the synchronization of the single steps provided by the protocol for task assignment and execution.

\vskip 0.5em
\noindent{\textbf{Representing data fluents.}}

Data Fluents need to be customized for every domain, and they can be used to constrain the task assignment and as guards into the expressions at decision points (e.g., for cycles, conditional statements).  However, a data fluent $\fX_\varphi$\footnote{Sometimes we use \emph{arguments-suppressed} formulas; these are uniform formulas with all arguments suppressed (e.g. $\fX_\varphi$ denotes the arguments-suppressed expression for $\fX_\varphi(\vec{o},s))$.} is mainly meant to capture one of the outcomes of a (specific) task $T$.
The general way for expressing a successor-state axiom of a data fluent is as follows:
\begin{equation} \label{eq:eq_datafluent}
\begin{array}{l}
\fX_\varphi(\vec{o},do(a,s)) = q_j \equiv{} \\
\quad \big(\exists c,id,\vec{i},\vec{p},\vec{q} \ \ s.t. \ \ a=\aRelease(c,id,T,\vec{i},\vec{p},\vec{q})\big) \ \land \ \vec{q} = [q_0,...,q_j,...,q_w] \ \lor{} \\
\quad\big(\fX_\varphi(\vec{o},s)=q_j\ \land \\
\quad \neg \exists c,id,\vec{i},\vec{p},\vec{q'} \ s.t. \ \ a=\aRelease(c,id,T,\vec{i},\vec{p},\vec{q'}) \land \vec{q'} = [q'_{0},...,q'_{j},...,q'_{w}] \land (q'_{j} \neq q_j)  \big).
\end{array}
\end{equation}
The value of $\fX_\varphi$ is changed to value $q_j$ when the task $T$ finishes with output $\vec{q} = [q_0,...,q_j,...,q_w]$ or, more formally, when the PMS launches the action $\aRelease(c,id,T,\vec{i},\vec{p},\vec{q})$, with $q_j \in \vec{q}$. With $\vec{o}$ we denote the list (possibly empty) of arguments of a data fluent, which can be customized in according to the process needs. For example, one can require some fluents defined for each service $c$ in the formalization:
\vskip 0.5em \noindent\colorbox{light-gray}{\begin{minipage}{0.98\textwidth}
\begin{example}
\emph{The task \aGo \ is used for instructing a service (specifically, a human actor) $act$ to move from a location \emph{from} to a location \emph{to}. When $act$ terminates the execution of the task \aGo \ that was assigned to her/him with identifier $id_\emph{1}$, its final position $q_j$ in situation $s$ is stored in a specific data fluent $\fAt_\varphi$.}
\begin{equation} \label{eq:eq_atRobot}
\begin{array}{l}
\fAt_\varphi(act,do(a,s)) = q_j \equiv{} \\
\quad \big(\exists \ q_j \ s.t. \ Location\_type(q_j) \ \land \\
\quad a=\aRelease(act,id_\emph{1},\aGo,[from,to],[to],[q_j])\big) \lor{} \\
\quad\big(\fAt_\varphi(act,s) = q_j\ \land \\
\quad \neg \exists \ q_j' \ s.t. Location\_type(q_j')\ \land \\
\quad a=\aRelease(act,id_1,\aGo,[from,to],[to],[q_j']) \land (q_j' \neq q_j)  \big).
\end{array}
\end{equation}
\end{example}
\end{minipage}
}\vskip 0.5em
\vskip 0.5em \noindent\colorbox{light-gray}{\begin{minipage}{0.98\textwidth}
\begin{example}
\emph{The task \aMove \ is used for instructing a robot $rb$ to move from a location \emph{from} to a location \emph{to}. When $rb$ terminates the execution of the task \aMove \ that was assigned to it with identifier $id_\emph{2}$, its final position $q_j$ in situation $s$ is stored in a specific data fluent $\fAtRobot_\varphi$.}
\end{example}
\end{minipage}
}

\noindent\colorbox{light-gray}{\begin{minipage}{0.98\textwidth}
\begin{equation} \label{eq:eq_atRobot}
\begin{array}{l}
\fAtRobot_\varphi(rb,do(a,s)) = q_j \equiv{} \\
\quad \big(\exists \ q_j \ s.t. \ Location\_type(q_j) \ \land \\
\quad a=\aRelease(rb,id_2,\aMove,[from,to],[to],[q_j])\big) \lor{} \\
\quad\big(\fAtRobot_\varphi(rb,s) = q_j\ \land \\
\quad \neg \exists \ q_j' \ s.t. \ Location\_type(q_j') \ \land \\
\quad a=\aRelease(rb,id_2,\aMove,[from,to],[to],[q_j']) \land (q_j' \neq q_j)  \big).
\end{array}
\end{equation}
\end{minipage}
}\vskip 0.5em

Data fluents are mainly used for recording the effects of a task. If we analyze the descriptions of the tasks \aGo \ and  \aMove \ provided in Section~\ref{subsec:approach-overview-representing tasks}, we can see that the content of the <$effects$> tag coincides with the successor state axioms described above.

When a task returns some real-world outcome after its completion, we define that outcome as \emph{supposed}, since its physical value may be different from the expected one as thought at design-time. However, sometimes it may happen that a task effect is \emph{automatic}, i.e., it is applied every time a task completes its execution, independently by the outcomes returned by the task itself.

\vskip 0.5em \noindent\colorbox{light-gray}{\begin{minipage}{0.98\textwidth}
\begin{example}
\emph{After a robot $rb$ completes the execution of the task \aMove \ with identifier $id_\emph{2}$, its battery level decreases of a fixed quantity equal to $\fMoveStep_\varphi(s)$. The battery charge level of a robot can be represented with a fluent $\fBatteryLevel_\varphi$:}
\begin{equation} \label{eq:eq_BatteryLevel}
\begin{array}{l}
\fBatteryLevel_\varphi(rb,do(a,s)) = v \equiv{} \\
\quad \big(\exists \ \vec{i},\vec{p},\vec{q} \ s.t. \ a=\aRelease(rb,id_2,move,\vec{i},\vec{p},\vec{q}) \ \land \\
\quad v = \fBatteryLevel_\varphi(rb,s) - \fMoveStep_\varphi(s) \land Integer\_type(v)\big) \
\lor{} \\
\quad\big(\fBatteryLevel_\varphi(rb,s) = v\ \land \\
\quad \neg \exists \ \vec{i},\vec{p},\vec{q} \ s.t. \ \ a=\aRelease(rb,id_2,move,\vec{i},\vec{p},\vec{q})\big).
\end{array}
\end{equation}
\emph{The fluent holds every time a $\aRelease$ action involving the task $move$ and the service $rb$ is launched by the PMS, and it results in decreasing the level of the battery of the fixed quantity $\fMoveStep_\varphi(s)$ by obtaining a new battery charge level equal to $v$.}
\end{example}
\end{minipage}
}\vskip 0.5em

The representation of a dynamic scenario may require to define some \emph{fixed} data fluents that reflect the contextual properties of the scenario itself. For example, our case study requires to define a data fluent $\fNeigh(loc1,loc2,s)$ that expresses the neighborhood property between two locations $loc1$ and $loc2$, and a fluent $\fCovered(loc,s)$ that indicates which locations $loc$ are directly covered by the radio range provided by the main antenna. Moreover, some data fluents can be used for recording constant values. For example, the fluent $\fMoveStep(s)$ indicates the fixed quantity of battery that will be consumed by a robot after the execution of a \aMove \ task. Those kinds of fluents, in general, never change their values during process execution.

\vskip 0.5em
\noindent{\textbf{Representing complex formulae and exogenous events.}}


The contextual scenario presented in the case study in Section~\ref{sec:introduction-case_study} requires also to define some properties that help to understand if a service is connected to the network provided by the main antenna. To this end, we can express some situation calculus \emph{abbreviations} for verifying the connection status of each service in every possible situation $s$. An abbreviation is a complex formula that does not depend directly from tasks effects, and that can be evaluated in each situation $s$. Specifically, given a human service \emph{act}, the abbreviation  $\isConnected(\emph{act},s)$ holds in situation $s$ if the service $act$ is situated in a location covered by the main network. Otherwise, if $act$ is outside the main network, it results as being connected only if it is close (cf. fluent \fNeigh) to a location where a robot $rb$ is situated, and the robot is - in its turn - connected to the network.
\begin{equation} \label{eq:eq_possconnected}
    \begin{array}{l}
    \isConnected(act,s) \equiv \\
     \quad provides(act,movement) \ \land \ \fAt_\varphi(act,s)=loc \ \land \ Location\_type(loc) \ \land\\
     \quad \big(\fCovered(loc,s) \ \ \lor\\
     \quad (\exists \ rb \ s.t. \ provides(rb,battery) \ \land \ \fAtRobot(rb,s)=loc_{rb} \ \land \\
     \quad Location\_type(loc_{rb}) \ \land \ \fNeigh(loc,loc_{rb},s) \ \land \ \isRobotConnected(rb,s))\big).
    \end{array}
\end{equation}
We omit here the definition of the second abbreviation \isRobotConnected \ (interesting readers can found it in the appendix), that holds if a robot is connected to the main network on the basis of the connection policy described in Section~\ref{sec:introduction-case_study}.

Abbreviations can be used together with the data fluents for defining the \emph{preconditions} of each task. In \indigolog, we make explicit task preconditions through the axiom \Poss \ applied on the \aAssign \ basic action:
\begin{equation*} \label{eq:eq_possassign}
\begin{array}{l}
\Poss(\aAssign(c,id,t,\vec{i},\vec{p}),s) \Leftrightarrow workitem(t,id,\vec{i},\vec{p})\ \land \\
\qquad (\fX_{\varphi,1} \land ... \land \fX_{\varphi,m}) \land (...list \ of \ abbreviations...).
\end{array}
\end{equation*}
A task $t$ with input $\vec{i}$ and expected outputs $\vec{p}$ can be assigned to a service $c$ iff the task $t$ combined with the associated inputs/outputs represents an admissible workitem. Moreover, values of data fluents $\fX_{\varphi,y}$ (where y ranges over \{1..m\})\ and abbreviations possibly included in the axiom should be evaluated.

\vskip 0.5em \noindent\colorbox{light-gray}{\begin{minipage}{0.98\textwidth}
\begin{example}
\emph{A task \aGo \ with inputs \emph{from} and \emph{to} can be assigned to a service \emph{act} in situation $s$ if and only if \emph{act} is located in $from$ and is connected to the network.}
\begin{equation} \label{eq:eq_possassignGO}
\begin{array}{l}
\Poss(\aAssign(act,id,\aGo,[from,to],[to]),s) \Leftrightarrow \\
\quad workitem(\aGo,id,[from,to],[to])\ \land \ \fAt_{\varphi}(act,s) = from\ \land \isConnected(act,s).
\end{array}
\end{equation}
\end{example}
\end{minipage}
}\vskip 0.5em

Finally, we conclude the section by pointing out that in dynamic domains it is typical that variables asynchronously change their value in an impredicative fashion. In order to represent this, we define \emph{exogenous events} as external actions coming from the environment that may change the values of data fluents. If we consider our case study, we have defined three different exogenous events :
\begin{itemize}[itemsep=1pt,parsep=1pt]
\item \photoLost(\emph{loc}) indicates that all the pictures taken in location $loc$ have been lost;
\item \fireRisk(\emph{loc}) alerts about a fire that is broken out in location $loc$;
\item \rockSlide(\emph{loc}) alerts about a rock slide collapsed in location $loc$.
\end{itemize}
When an exogenous event is detected, it possibly causes an update in the values of some data fluent. This means that interested fluents has to capture the possible exogenous event modification through their successor state axiom.
\vskip 0.5em \noindent\colorbox{light-gray}{\begin{minipage}{0.98\textwidth}
\begin{example}
\emph{The task \aTakePhoto \ is used for instructing a service (specifically, a human actor) $act$ to move in a location \emph{loc} in order to take some pictures. When $act$ completes the execution of the task \aTakePhoto \ that was assigned to her/him, the physical outcome $q_j$ of the task is stored in a specific data fluent $\fPhotoTaken_\varphi$. Note that if an exogenous event $\photoLost$ is captured in $s$, the value of the fluent $\fPhotoTaken_\varphi$ (that holds in $s$) switches to false in the situation \emph{do(a,s)}.}
\begin{equation} \label{eq:eq_photoTaken}
\begin{array}{l}
\fPhotoTaken_\varphi(loc,do(a,s)) = true \equiv{} \\
\quad \big(\exists \ q_j \ s.t. \ a=\aRelease(act,id,\aTakePhoto,[loc],[true],[q_j]) \ \land (q_j = true)\big) \lor{} \\
\quad\big(\fPhotoTaken_\varphi(loc,s) = true \ \land \\
\quad \big((a\neq\photoLost(loc)) \ \lor \\
\quad \neg \exists \ q_j' \ s.t. \ (a=\aRelease(act,id,\aTakePhoto,[loc],[true],[q_j']) \land (q_j' = false))\big).
\end{array}
\end{equation}
\end{example}
\end{minipage}
}\vskip 0.5em

The list of all predicates, fluents, exogenous events and abbreviations defined through \indigolog can be named as the \emph{\textbf{SitCalc Theory}}. Remember that, given a SitCalc Theory, before to execute a process linked to the SitCalc Theory through the \indigolog PMS, there is the need to clearly specify the values of each data fluent in the initial situation $S_0$, that represents the situation in which no actions have yet occurred. In the following example, we show how we set the values of data fluents in $S_0$ for representing the starting state of our contextual scenario.
\vskip 0.5em \noindent\colorbox{light-gray}{\begin{minipage}{0.98\textwidth}
\begin{example}
\emph{We complete the formalization of the scenario described in our case study by providing the whole list of data fluents (all the corresponding successor state axioms are shown in the Appendix), workitems and precondition axioms:}

\begin{itemize}[itemsep=1pt,parsep=1pt]
\item $\fAt_\varphi(c,s) = loc$ \emph{records the location $loc$ in which the human service $c$ is situated in situation $s$. In the initial situation $S_0$, each human service is located in \emph{loc00}.}
\item $\fAtRobot_\varphi(c,s) = loc$ \emph{records the location $loc$ in which the robot service $c$ is situated in situation $s$. In the initial situation $S_0$, each robot service is located in \emph{loc00}.}
\item $\fEvacuated_\varphi(l,s)$ \emph{is true if the location $l$ is proven to be evacuated in situation $s$. In $S_0$, no location has been already evacuated.}
\item $\fStatus_\varphi(l,s) = st$ \emph{records the status of a location $l$ in situation $s$, that can be equal to '\emph{ok}' (meaning that nothing wrong still happened), or to '\emph{fire}' (meaning that a fire has broken out in location $l$) or to '\emph{debris}' (meaning that a rock slide has collapsed in $l$). Asynchronously, in any moment, an exogenous event \fireRisk(l) can change the value of $\fStatus_\varphi(l,s)$ to '\emph{fire}'. At the same way, an exogenous event \rockSlide(l) can turn the value of $\fStatus_\varphi(l,s)$ to '\emph{debris}'. In $S_0$, $\forall \ loc \ s.t \ Location\_type(loc) \ \Rightarrow \ \fStatus_\varphi(l,S_0) = ok$.}
\item $\fPhotoTaken_\varphi(l,s)$ \emph{is true if in situation $s$ some pictures have been taken in location $l$. If an exogenous event \photoLost(l) is captured, the value of $\fPhotoTaken_\varphi(l,s)$ is turned asynchronously to $false$. In $S_0$, no picture has been captured in any location $l$.}
\end{itemize}
\end{example}
\end{minipage}
}
\vskip 0.5em \noindent\colorbox{light-gray}{\begin{minipage}{0.98\textwidth}
\begin{itemize}[itemsep=1pt,parsep=1pt]
\item $\fBatteryLevel_\varphi(c,s)$ \emph{stores the battery charge level of each robot $c$ in situation $s$. The level of the robot battery may decrease of a fixed quantity (equal to the value of $\fMoveStep_\varphi(s)$ or to $\fDebrisStep_\varphi(s)$) after $c$ completes a \aMove \ task or a \aRemoveDebris \ task. A human service can recharge the battery of a robot by increasing its level of a fixed quantity equal to the value of $\fBatteryRecharging_\varphi(s)$. For each robot service $c$, $\fBatteryLevel_\varphi(c,S_0)$ = 3}.
\item $\fGeneralBattery_\varphi(s)$ \emph{reflects the total amount of battery contained in the power pack used for recharging the battery of each robot. Each recharging action, that corresponds to the execution of the \aChargeBattery \ task, decreases the value of $\fGeneralBattery_\varphi(s)$ of a fixed quantity equal to the value of $\fBatteryRecharging_\varphi(s)$. In the initial situation $S_0$, $\fGeneralBattery_\varphi(S_0) = 30$.}
\item $\fBatteryRecharging_\varphi(s)$ \emph{is a data fluent that reflects the amount of each recharging action. In the initial situation $S_0$, $\fBatteryRecharging_\varphi(S_0) = 10$.}
\item $\fMoveStep_\varphi(s)$ \emph{indicates the fixed quantity of battery that will be consumed by a robot after the execution of a \aMove \ task. In the initial situation $S_0$, $\fMoveStep_\varphi(S_0) = 2$.}
\item $\fDebrisStep_\varphi(s)$ \emph{indicates the fixed quantity of battery that will be consumed by a robot after the execution of a \aRemoveDebris \ task. In the initial situation $S_0$, $\fDebrisStep_\varphi(S_0) = 3$.}
\item $\fNeigh_\varphi(loc1,loc2,s)$ \emph{expresses the neighborhood property between two locations $loc1$ and $loc2$. In $S_0$, the fluent holds for all that locations considered as adjacent in the scenario depicted in our case study. Hence, $\fNeigh_\varphi(loc00,loc01,S_0) = true$, $\fNeigh_\varphi(loc00,loc10,S_0) = true$, etc.}
\item $\fCovered_\varphi(loc,s)$ \emph{indicates all that locations that are directly covered by the radio range provided by the main antenna. In $S_0$, $\fCovered_\varphi(loc00,S_0) = true$, $\fCovered_\varphi(loc01,S_0) = true$, etc.}
\end{itemize}


\emph{Now, we show the complete list of the admissible workitems executable by the \indigolog PMS and precondition axioms for each task.}

\begin{itemize}[itemsep=1pt,parsep=1pt]
\item
    \emph{The task \aGo \ can be executed by the PMS if and only if it is associated to a couple of inputs, that represent two valid $Location\_type$, and to an expected output, that is again a $Location\_type$.}
    \begin{equation*}
    \begin{array}{l}
    \forall \ (from,to,id) \ s.t. \ Identifier(id) \ \land \ Location\_type(from) \ \land \\
    \qquad Location\_type(to) \Rightarrow workitem(\aGo,id,[from,to],[to]).
    \end{array}
    \end{equation*}
    \emph{A task \aGo \ with inputs \emph{from} and \emph{to} can be assigned to a service \emph{act} in situation $s$ if and only if \emph{act} is located in $from$ and is connected to the network.}
    \begin{equation*}
    \begin{array}{l}
    \Poss(\aAssign(act,id,\aGo,[from,to],[to]),s) \Leftrightarrow \\
    \quad workitem(\aGo,id,[from,to],[to])\ \land \ \fAt_{\varphi}(act,s) = from\ \land \\ \quad\isConnected(act,s).
    \end{array}
    \end{equation*}
            \end{itemize}
    \end{minipage}
}
    \vskip 0.5em \noindent\colorbox{light-gray}{\begin{minipage}{0.98\textwidth}
    \begin{itemize}[itemsep=1pt,parsep=1pt]
\item \emph{The task \aMove \ can be executed by the PMS if and only if it is associated to a couple of inputs, that represent two valid $Location\_type$, and to an expected output, that is again a $Location\_type$.}
    \begin{equation*}
    \begin{array}{l}
    \forall \ (from,to,id) \ s.t. \ Identifier(id) \ \land \ Location\_type(from) \ \land \\
    \qquad Location\_type(to) \Rightarrow workitem(\aMove,id,[from,to],[to]).
    \end{array}
    \end{equation*}
    \emph{A task \aMove \ with inputs \emph{from} and \emph{to} can be assigned to a service \emph{rb} in situation $s$ if and only if \emph{rb} is located in $from$, is connected to the network and has enough battery charge level for executing the task.}
    \begin{equation*}
    \begin{array}{l}
    \Poss(\aAssign(rb,id,\aMove,[from,to],[to]),s) \Leftrightarrow \\
    \quad workitem(\aMove,id,[from,to],[to])\ \land \ \fAtRobot_{\varphi}(rb,s) = from \ \land \\
    \quad \isRobotConnected(rb,s) \ \land \ \fBatteryLevel_{\varphi}(rb) >= \fMoveStep_\varphi(s).
    \end{array}
    \end{equation*}
\item \emph{The task \aTakePhoto \ can be executed by the PMS if and only if it is associated to an input that represents the $Location\_type$ where the pictures have to be taken, and to an expected output corresponding to a boolean value indicating the correct execution of the task.}
    \begin{equation*}
    \begin{array}{l}
    \forall \ (loc,id) \ s.t. \ Identifier(id) \ \land \ Location\_type(loc) \Rightarrow \\
    \qquad workitem(\aTakePhoto,id,[loc],[true]).
    \end{array}
    \end{equation*}
    \emph{A task \aTakePhoto \ with input \emph{loc} can be assigned to a service \emph{act} in situation $s$ if and only if \emph{act} is located in $loc$ and is connected to the network.}
    \begin{equation*}
    \begin{array}{l}
    \Poss(\aAssign(act,id,\aTakePhoto,[loc],[true]),s) \Leftrightarrow \\
    \quad workitem(\aTakePhoto,id,[loc],[true])\ \land \ \fAt_{\varphi}(act,s) = loc \ \land \\
    \quad \isConnected(act,s).
    \end{array}
    \end{equation*}

\item \emph{The task \aEvacuate \ can be executed by the PMS if and only if it is associated to an input that represents the $Location\_type$ where it is required to evacuate some people, and to an expected output corresponding to a boolean value indicating the correct execution of the task.}
    \begin{equation*}
    \begin{array}{l}
    \forall \ (loc,id) \ s.t. \ Identifier(id) \ \land \ Location\_type(loc) \Rightarrow \\
    \qquad workitem(\aEvacuate,id,[loc],[true]).
    \end{array}
    \end{equation*}
     \emph{A task \aEvacuate \ with input \emph{loc} can be assigned to a service \emph{act} in situation $s$ if and only if \emph{act} is located in $loc$ and is connected to the network. Moreover, it is required that $loc$ is not affected by debris or fire in situation $s$, and it has not already been evacuated.}
    \begin{equation*}
    \begin{array}{l}
    \Poss(\aAssign(act,id,\aEvacuate,[loc],[true]),s) \Leftrightarrow \\
    \quad workitem(\aEvacuate,id,[loc],[true])\ \land \ \fAt_{\varphi}(act,s) = loc\ \land \\
    \quad \fEvacuated(loc,s) = false \ \land \ \fStatus(loc,s) = ok \ \land \ \isConnected(act,s).
    \end{array}
    \end{equation*}
            \end{itemize}
    \end{minipage}
}
    \vskip 0.5em \noindent\colorbox{light-gray}{\begin{minipage}{0.98\textwidth}
    \begin{itemize}[itemsep=1pt,parsep=1pt]
\item \emph{The task \aUpdateStatus \ can be executed by the PMS if and only if it is associated to an input that represents the $Location\_type$ where it is required to update the status of the emergency, and to an expected output indicating that the final status of the location is good (i.e., value 'ok' for the expected output).}
    \begin{equation*}
    \begin{array}{l}
    \forall \ (loc,st,id) \ s.t. \ Identifier(id) \ \land \ Location\_type(loc) \ \land \\
    \qquad Status\_type(st) \Rightarrow workitem(\aUpdateStatus,id,[loc],[ok]).
    \end{array}
    \end{equation*}
    \emph{A task \aUpdateStatus \ with input \emph{loc} can be assigned to a service \emph{act} in situation $s$ if and only if \emph{act} is located in $loc$ and is connected to the network. Moreover, it is required that $loc$ is not affected by debris or fire in situation $s$.}
    \begin{equation*}
    \begin{array}{l}
    \Poss(\aAssign(act,id,\aUpdateStatus,[loc],[ok]),s) \Leftrightarrow \\
    \quad workitem(\aUpdateStatus,id,[loc],[ok])\ \land \ \fAt_{\varphi}(act,s) = loc \ \land \\
    \quad \fStatus_{\varphi}(loc,s) = ok \ \land \ \isConnected(act,s).
    \end{array}
    \end{equation*}

\item \emph{The task \aExtinguishFire \ can be executed by the PMS if and only if it is associated to an input that represents the $Location\_type$ where it is required to extinguish a fire, and to an expected output indicating that after the execution of the task, the fire will be extinguished (i.e., value 'ok' for the expected output).}
    \begin{equation*}
    \begin{array}{l}
    \forall \ (loc,st,id) \ s.t. \ Identifier(id) \ \land \ Location\_type(loc) \ \land \\
    \qquad Status\_type(st) \Rightarrow workitem(\aExtinguishFire,id,[loc],[ok]).
    \end{array}
    \end{equation*}
    \emph{A task \aExtinguishFire \ with input \emph{loc} can be assigned to a service \emph{act} in situation $s$ if and only if \emph{act} is located in $loc$ and is connected to the network. Moreover, it is required that in location $loc$ a fire broke out in situation $s$.}
    \begin{equation*}
    \begin{array}{l}
    \Poss(\aAssign(act,id,\aExtinguishFire,[loc],[ok]),s) \Leftrightarrow \\
    \quad workitem(\aExtinguishFire,id,[loc],[ok])\ \land \ \fAt_{\varphi}(act,s) = loc \ \land \\
    \quad \fStatus_{\varphi}(loc,s) = fire \ \land \ \isConnected(act,s).
    \end{array}
    \end{equation*}

\item \emph{The task \aRemoveDebris \ can be executed by the PMS if and only if it is associated to an input that represents the $Location\_type$ where it is required to remove some debris due to a rock slide, and to an expected output indicating that after the execution of the task the debris will be removed (i.e., value 'ok' for the expected output).}
    \begin{equation*}
    \begin{array}{l}
    \forall \ (loc,st,id) \ s.t. \ Identifier(id) \ \land \ Location\_type(loc)\ \land \\
    \qquad Status\_type(st) \Rightarrow workitem(\aRemoveDebris,id,[loc],[ok]).
    \end{array}
    \end{equation*}
    \emph{A task \aRemoveDebris \ with input \emph{loc} can be assigned to a service \emph{rb} in situation $s$ if and only if \emph{rb} is located in $loc$, is connected to the network and has enough battery charge level for executing the task. Moreover, it is required that in location $loc$ a rock slide has collapsed in situation $s$.}
    \begin{equation*}
    \begin{array}{l}
    \Poss(\aAssign(rb,id,\aRemoveDebris,[loc],[ok]),s) \Leftrightarrow \\
    \quad workitem(\aRemoveDebris,id,[loc],[ok])\ \land \ \fAtRobot_{\varphi}(rb,s) = loc\ \land \\
    \quad \fStatus_{\varphi}(loc,s) = debris \ \land \ \fBatteryLevel_{\varphi}(rb) >= \fDebrisStep_\varphi \ \land \\ \quad \isConnected(act,s).
    \end{array}
    \end{equation*}
    \end{itemize}
    \end{minipage}}
    \vskip 0.5em \noindent\colorbox{light-gray}{\begin{minipage}{0.98\textwidth}
    \begin{itemize}[itemsep=1pt,parsep=1pt]
\item \emph{The task \aChargeBattery \ can be executed by the PMS if and only if it is associated to an input that represents the robot that needs to be charged. The task has no expected outputs, meaning that we consider a recharging activity as if it always succeeds.}
    \begin{equation*}
    \begin{array}{l}
    \forall \ (rb,id) \ s.t. \ Identifier(id) \ \land \ Service(rb) \Rightarrow \\
    \qquad  workitem(\aChargeBattery,id,[rb],[]).
    \end{array}
    \end{equation*}
    \emph{A task \aChargeBattery \ with input \emph{rb} can be assigned to a service \emph{act} in situation $s$ if $act$ is connected to the network, $rb$ provides a battery (meaning that it is effectively a robot), and both \emph{rb} and \emph{act} are located in $loc$.}
    \begin{equation*}
    \begin{array}{l}
    \Poss(\aAssign(act,id,\aChargeBattery,[rb],[]),s) \Leftrightarrow \\
    \quad workitem(\aChargeBattery,id,[rb],[ok])\ \land \\
    \quad \fAt_{\varphi}(act,s) = \fAtRobot_{\varphi}(rb,s)\ \land \\
    \quad \fProvides(rb,battery) \ \land \ \isConnected(act,s).
    \end{array}
    \end{equation*}
    \end{itemize}
    \end{minipage}}

\subsection{Realizing the Framework}
\label{sec:approach-formalization-framework}

In Figure~\ref{fig:fig_approach_core}, we show how \smartpm has been concretely coded by the interpreter of \indigolog. The main procedure of the \indigolog\ program is $\proc{Main}()$, which involves three interrupts running at different priorities. Each interrupt is guarded by the fluent \fIsFinished; if it holds, it means that the process execution has been completed successfully.

\begin{equation} \label{eq:eq_finished}
\begin{array}{l}
\fIsFinished(do(a,s)) = true \equiv{} \\
\quad a=\finish \ \lor{} \ (\fIsFinished(s) = true).
\end{array}
\end{equation}

The procedure $\proc{Monitor}()$, which runs at higher priority, is in charge of monitoring changes in the environment and adapting accordingly. The first step in \proc{Monitor} checks whether fluent $\fRealityChanged$ holds
true, meaning that a task has completed its execution or that an exogenous (unexpected) action has occurred in the system.

\begin{equation} \label{eq:eq_reality_changed}
\begin{array}{l}
\fRealityChanged(do(a,s)) = true \equiv{} \\
\quad \exists \ t,c,id,\vec{i},\vec{p},\vec{q} \ \ s.t. \ a = \aRelease(c,id,t,\vec{i},\vec{p},\vec{q}) \ \lor \\
\quad exogenous(a) \ \lor\\
\quad \big(\fRealityChanged(s) = true \land a \neq \aResetReality \big).
\end{array}
\end{equation}

If some change in the contextual data is considered as $\Relevant$, the procedure $\proc{Adapt}()$ is launched for the synthesis of the recovery procedure. Both if $\Relevant$ holds or not, the $\proc{Monitor}()$ concludes by executing the action $\aResetReality$, which turns the fluent $\fRealityChanged$ to $false$. We give more details about the $\Relevant$ fluent and the working of the $\proc{Monitor}()$ procedure in the following sections.

At a lower priority, the system runs the actual \indigolog\ program representing the dynamic process to be executed, namely procedure $\proc{Process}()$. As shown in Fig.~\ref{fig:fig_approach_core}, the process is composed by three branches of tasks to be executed in parallel (cf. Section~\ref{subsec:approach-preliminaries-indigolog} for the list of \indigolog constructs), that correspond exactly to the emergency management process defined in Section~\ref{sec:introduction-case_study}. For example, the procedure $\proc{Branch1}()$ depicts three tasks to be executed in sequence. Specifically, $\proc{Branch1}()$ instructs first a selected service to reach location $loc33$ starting from location $loc00$, then to take pictures in that location (the expected output of this task is $true$, meaning that the pictures have been correctly collected), and finally to update the status of location $loc33$ (the 'ok' value means that nothing strange, such as a fire or a rock slide, has happened in $loc33$). Note that in Fig.~\ref{fig:fig_approach_core} we make use of a Prolog-like notation and lists of inputs/expected outputs/physical outputs are enclosed between squared brackets.

\begin{figure}[t]
\centering{
 \includegraphics[width=0.8\columnwidth]{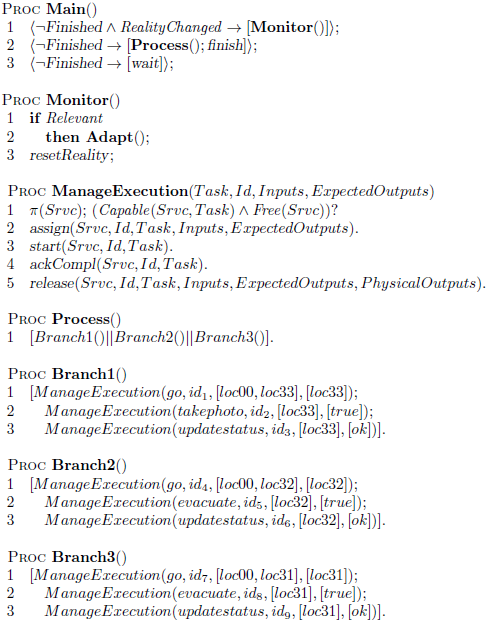}
 } \caption{The core procedures of \smartpm.}
 \label{fig:fig_approach_core}
\end{figure}

Each $\proc{Branch}()$ procedure relies, in turn, on procedure $\proc{ManageExecution}$.  Given a specific task $T$, the procedure $\proc{ManageExecution}$ first searches for a service $c$ that is \fFree\ in situation $s$ and is \fCapable\ to execute the task and, if such a service exists, the procedure will be in charge to manage task assignment, start signaling, acknowledgment of completion, and final release.

In the list of interrupts, at the lowest priority, there is also a third possibility for managing the progression of the process, that is activated only if the process is still not finished, but for some reason, it can not progress in its execution. The third interrupt, that consists just in waiting, reflects the fact that, for example, a certain task can not be assigned to any qualified service (i.e., the $\pi$ function is unable to find any service providing all the capabilities required by that task) or all qualified services are currently involved in the performance of other tasks. If the third interrupt is activated, the control passes back to the process designer, which can manually manage the situation (for example, by adding new services or by updating the capabilities of the existing services).


\section{Monitoring for Failures}
\label{sec:approach-monitor}

In this section, we turn our attention to the mechanism for automatically detecting failures. As described in Section~\ref{sec:approach-overview}, adaptation in \smartpm can be seen as reducing the gap between the \emph{expected reality}, the (idealized) model of reality that is used by the PMS to reason, and the \emph{physical reality}, the real world with the actual values of conditions and outcomes. In order to understand how the proposed technique works, we start by formalizing the concepts of \emph{physical} and \emph{expected reality}.
\begin{mydef}
A physical reality $\Phi(s)$ is the set of all data fluents $\fX_{\varphi,y}$ (where y ranges over \{1..m\}) defined in the SitCalc theory. Hence, $\Phi(s) = \bigcup_{y = {1..m}}\{\fX_{\varphi,y}\}$.
\end{mydef}
The physical reality $\Phi(s)$ captures exactly the value assumed by each \emph{data fluent} in the situation $s$. Such value reflects what is really happening in the real environment whilst the process is under execution. However, the PMS must guarantee that each task in the process is executed correctly, i.e., with an output that satisfies the process specification. For this purpose, the concept of \emph{expected reality} \ $\Psi(s)$ is needed. For some data fluents, the ones affected by a \emph{supposed} task effect (cf. Section~\ref{sec:approach-formalization}), we introduce a new \emph{expected fluent} $\fX_{\psi}$ that is meant to record the ``expected'' value of $X$ after the execution of a task $T$. The successor state axiom for this new fluent is straightforward:
\begin{equation} \label{eq:eq_expectedfluent}
\begin{array}{l}
\fX_\psi(\vec{o},do(a,s)) = p_j \equiv{} \\
\quad \big(\exists c,id,\vec{i},\vec{p},\vec{q} \ \ s.t.  \ a=\aRelease(c,id,T,\vec{i},\vec{p},\vec{q}) \ \land \\ \quad \vec{p} = [p_0,...,p_j,...,p_w]\big) \ \lor{} \\
\quad\big(\fX_\psi(\vec{o},s) = p_j\ \land \
\neg \exists c,id,\vec{i},\vec{p'},\vec{q} \ \ s.t.  \ a=\aRelease(c,id,T,\vec{i},\vec{p'},\vec{q}) \ \land \\
\quad \vec{p'} = [p'_{0},...,p'_{j},...,p'_{w}] \land (p'_{j} \neq p_j)  \big).
\end{array}
\end{equation}
The fluent states that, in the expected reality, a task is \emph{always} executed correctly (also when it is not), by forcing the value of $\fX_\psi$ to the value of the expected output $p_j$.

\vskip 0.5em \noindent\colorbox{light-gray}{\begin{minipage}{0.98\textwidth}
\begin{example}
\emph{When an actor \emph{act} terminates the execution of the task \aGo \ from a location \emph{from} to a location \emph{to}, the expected fluent $\fAt_\psi$ will assume the value \emph{to} without considering the real outcome of the task.}
\begin{equation} \label{eq:eq_atRobot}
\begin{array}{l}
\fAt_\psi(act,do(a,s)) = to \equiv{} \\
\quad \big(\exists \ q_j \ s.t. \ a=\aRelease(act,id,\aGo,[from,to],[to],[q_j])\big) \lor{} \\
\quad\big(\fAt_\psi(act,s) = to\ \land \\
\quad \neg \exists \ to', q_j \ s.t. \ a=\aRelease(act,id,\aGo,[from,to'],[to'],[q_j]) \land (to' \neq to)  \big).
\end{array}
\end{equation}
\end{example}
\end{minipage}
}\vskip 0.5em
Therefore, an expected fluent $\fX_\psi$ holds every time a task is completed with any real effect, by storing one of the expected outcomes of the task. More precisely, given a task $T$ whose execution returns a list of real outcomes $\vec{q} = [q_0,...,q_j,...,q_w]$ and a list of expected outcomes $\vec{p} = [p_0,...,p_j,...,p_w]$, there exists a data fluent $\fX_\varphi$ and a corresponding expected fluent $\fX_\psi$ that will assume respectively the values $q_j$ and $p_j$. We need to underline that some data fluents do not provide an associated expected fluent; for example, the fluent $\fBatteryLevel_\varphi$ records the battery charge level of each robot, and this value changes of a fixed quantity after each robot movement, whatever is the real final position of the robot. On the contrary, given a data fluent $\fX_{\varphi}$ that records one of the \emph{supposed} effects of a task, an expected fluent $\fX_{\psi}$ could be defined. 
For example, although the task \aMove \ provides a ``supposed'' outcome, there is no need to record the expected position of a robot. In fact, if a robot reaches a location different from the one expected, the important aspect is that this physical position allows in any way to guarantee the network connection.

\begin{mydef}
An expected reality $\Psi(s)$ is the set of all expected fluents $\fX_{\psi,y}$ (where y ranges over \{1..k\}, with $k \leq m$) defined in the SitCalc theory. Hence, $\Psi(s) = \bigcup_{y = {1..k}}\{\fX_{\psi,y}\}$.
\end{mydef}

In Figure~\ref{fig:fig_approach-execution_monitor}, the overall framework for detecting failures is depicted. The PMS starts by taking in input a process specification $\delta_0$ formalized in \indigolog, and builds the SitCalc theory (i.e., the set of first-order predicates and situation calculus fluents) representing the contextual environment in which the process has to be executed. Finally, the PMS instantiates the initial situation $S_0$, that indicates starting values for each data and expected fluent.

The execution of a process can be interrupted by the \emph{Monitor} module when a misalignment between the physical and the expected reality is discovered. Specifically, the Monitor blocks the execution of the main process $\delta$ \myi after the occurrence of a $\aRelease$ action in the current situation $s$, meaning that some service has completed the execution of a task, or \myii if some exogenous event has been catched in $s$. While in the first case there is the possibility that some physical and expected fluent has changed its value, an exogenous event may affect directly only some data fluent (i.e., the physical reality).

In this sense, our framework is able to capture - and to recover from - two different kinds of task failure. An \emph{internal failure} is related to the failure in the execution of a task, i.e., the task does not terminate, or it
is completed with an output that differs from the expected one. 
An \emph{external failure} is represented as an exogenous event $e$, given in input by the external environment, that forces a set of data fluents to assume a value imposed by the event itself. Such a new value could differ from the expected one, by generating a discrepancy between the two realities. 

A recovery procedure is needed if the two realities are different from each other, i.e., some tasks in the process failed their execution by returning an output $q$ whose value is different from the expected output $p$, or if some exogenous event has modified the physical reality in a undesirable way.


\begin{figure}[t]
\centering{
 \includegraphics[width=0.85\columnwidth]{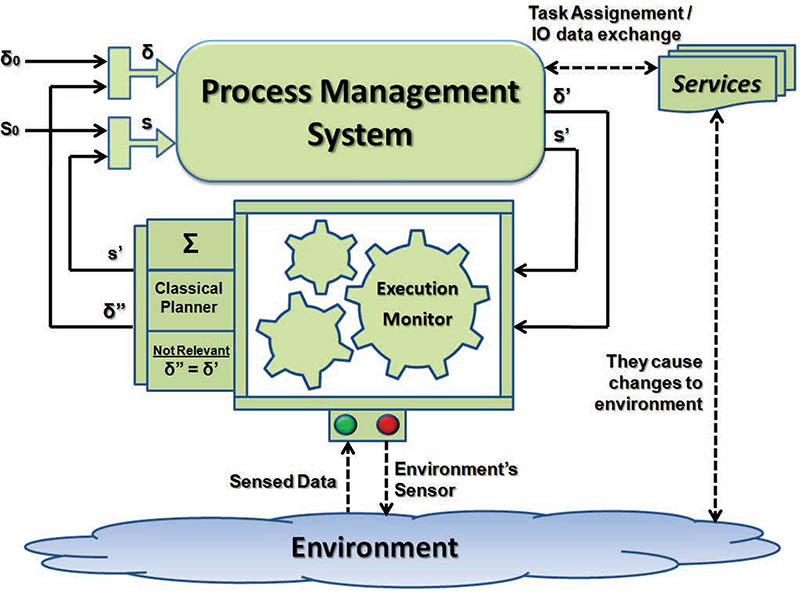}
 } \caption{Execution monitoring in \smartpm.}
 \label{fig:fig_approach-execution_monitor}
\end{figure}

The reduction of this gap requires sufficient knowledge of both kinds of realities. This knowledge, recorded in data and expected fluents, allows the PMS to sense deviations and to modify the process to ensure that, at the end, the above gap will be removed. Formally, a situation $s$ is known as \Relevant - candidate for adaptation - iff :

\begin{equation} \label{eq:eq_samestate}
\begin{array}{l}
\Relevant(\delta,s) \equiv{} \neg \SameState(\Phi(s),\Psi(s))
\end{array}
\end{equation}

Predicate $\SameState(\Phi(s),\Psi(s))$ holds iff the states denoted by $\Phi(s)$ and $\Psi(s)$ are the same\footnote{The evaluation of \SameState \ is performed only for those data fluents for which there exists a corresponding expected fluent. It is obvious that for data fluents considered as \emph{not relevant} for adaptation there is no need to monitor their evolution over situations.}.

\begin{mydef}
Given a situation $s$ and a set $\vec{F}$ of fluents , a $state(\vec{F}(s))$ is the set composed by the values - in situation $s$ - of each fluent $F_j$ that belongs to $\vec{F}$. Hence, $state(\vec{F}(s))$ = $\bigcup_{j = {1..m}}\{F_j\} \ s.t. \ F_j \in \vec{F}$.
\end{mydef}


Next, let us formalize how the monitor works. After the termination of a task, or after the occurrence of an exogenous event, the PMS infers the new situation $s'$ derived from the execution of a $\aRelease$ action (or from the arrival of an exogenous event) and passes it to the Monitor, together with the fragment of process $\delta'$ remaining to be executed.
In case of task completion in situation $s$, the realities $\Phi(s)$ and $\Psi(s)$ will evolve in $\Phi(s')$ and $\Psi(s')$, while in case of an exogenous event, only $\Phi(s)$ evolves in $\Phi(s')$, while $\Psi(s')$ remains equal to $\Psi(s)$. Finally, it is obvious that in case of exogenous events, $\delta' = \delta$ (i.e., an exogenous event does not reflects a task completion, but a change in the physical reality).

Now the Monitor component checks if $\Relevant(\delta',s')$ holds. If it does not hold, it means that no adaptation is required for the process $\delta'$, that can be still executed in situation $s'$. In such a case, the Monitor put $\delta''$ equal to $\delta'$ 
and carries on with the main process execution.

On the contrary, if $\Relevant(\delta',s')$ holds, adaptation of the process $\delta'$ is needed. Specifically, the purpose of \smartpm is to devise a recovery procedure $\delta_a$ that turns the \emph{wrong} physical reality $\Phi(s')$ in the \emph{correct} expected reality $\Psi(s')$. The \smartpm system proposes two different adaptation mechanisms for recovering the process from a failure or for adapting it to changing circumstances in the contextual environment due to an exogenous event. The first technique consists in using the lookahead search construct $\Sigma$ provided by \indigolog, while the second technique, which is the one currently deployed on the \smartpm prototype, is based on classical planning algorithms. Both the adaptation techniques are able to synthesize a linear process $\delta_a$, i.e., a process consisting of a sequence of tasks, such that $\delta_a = [t_{a_{0}};...;t_{a_{n}}]$. Such a recovery procedure $\delta_a$ will be inserted just before $\delta'$, by devising a new process $\delta'' = (\delta_a;\delta')$. Note that whenever a process needs to be adapted, every running task is interrupted, since the ``repair'' sequence of tasks $\delta_a$ is placed before them. Thus, active branches can only resume their execution after the repair sequence $\delta_a$ has been executed. This last requirement is fundamental to avoid the risk of introducing data inconsistencies during the repair phase.

The adapted process $\delta'' = (\delta_a;\delta')$ must guarantee some properties in order to ensure the correctness of the recovery mechanism:
\begin{itemize}[itemsep=1pt,parsep=1pt]
\item The recovery process $\delta_a$ has to be synthesized for being executed in situation $s'$. Moreover, after the execution of $\delta_a$, it should be guaranteed that $\Phi(s')$ is turned in $\Psi(s')$. We will discuss this  aspect in the following sections.
\item The execution of the recovery procedure $\delta_a$ corresponds to a new situation $s''$ and to new realities $\Phi(s'')$ and $\Psi(s'')$. In order to ensure that the remaining part of the main process $\delta'$ is still executable in $s''$, the recovery mechanism should guarantee that $\Psi(s')$ and $\Psi(s'')$ represent the same expected state of the world, i.e., the fluent $\SameState(\Psi(s'),\Psi(s''))$ must hold after the execution of $\delta_a$. Specifically, we want that after the execution of the recovery process $\delta_a$ in situation $s'$ (that results in a new situation $s''$), the remaining process $\delta'$ to be executed in $s''$ is equivalent to execute $\delta'$ in $s'$.
\end{itemize}

We formalize this second point by exploiting the concept of \emph{bisimulation}. In Computer Science, two systems are \emph{bisimilar} if they match each other's moves, i.e., one system simulates the other and vice-versa. In this sense, each of the systems can not be distinguished from the other by external an observer. For proving this, we define the the predicate $\SameConfig(\delta_\alpha,s_\alpha,s_\beta,\delta_\beta)$ as follows:
\begin{definition}
A predicate $SameConfig(\delta_\alpha,s_\alpha,\delta_\beta,s_\beta)$ is correct if for every $\delta_\alpha, s_\alpha,\delta_\beta,s_\beta$:
\begin{enumerate}[itemsep=1.2pt,parsep=1pt]
    \item $Final(\delta_\alpha,s_\alpha) \Leftrightarrow Final(\delta_\beta,s_\beta)$
    \item
    $\forall~a,\delta_\alpha \ s.t. \ Trans\big(\delta_\alpha,s_\alpha,\overline{\delta_\alpha},do(a,s_\alpha)\big) \Rightarrow$
    \\$\exists~\overline{\delta_\beta} \ s.t. \ Trans\big(\delta_\beta,s_\beta,\overline{\delta_\beta},do(a,s_\beta)\big)
    \wedge SameConfig\big(\overline{\delta_\alpha},do(a,s_\alpha),\overline{\delta_\beta},do(a,s_\beta)\big)$
    \item
    $\forall~a,\delta_\beta \ s.t. \ Trans\big(\delta_\beta,s_\beta,\overline{\delta_\beta},do(a,s_\beta)\big) \Rightarrow$
    \\$\exists~\overline{\delta_\alpha} \ s.t. \ Trans\big(\delta_\alpha,s_\alpha,\overline{\delta_\alpha},do(a,s_\alpha)\big)
    \wedge SameConfig\big(\overline{\delta_\beta},do(a,s_\beta),\overline{\delta_\alpha},do(a,s_\alpha)\big)$
\end{enumerate}
\end{definition}

Intuitively, a predicate $SameConfig(\delta_\alpha,s_\alpha,s_\beta,\delta_\beta)$ is said to be correct if $\delta_\alpha$ and $\delta_\beta$ are terminable either both or none of them. Furthermore, for each action $a$ performable by $\delta_\alpha$ in the situation $s_\alpha$, there exists the same action $a$ performable by $\delta_\beta$ in the situation $s_\beta$  (and viceversa). Moreover, the resulting configurations $(\overline{\delta_\alpha},do(a,s_\alpha))$ and $(\overline{\delta_\beta},do(a,s_\beta))$ must still satisfy $SameConfig$.


In our case, we can adopt a specific definition for \SameConfig, that we call $\SameConfig_{PM}(\delta_\alpha,s_\alpha,s_\beta,\delta_\beta)$:
\begin{equation}
\begin{array}{l}
\SameConfig_{PM}(\delta_\alpha,s_\alpha,s_\beta,\delta_\beta) \Leftrightarrow {} \\
\qquad \SameState(\Psi(s_\alpha),\Psi(s_\beta)) \wedge \delta_\alpha=\delta_\beta
\end{array}
\end{equation}

In other words, $\SameConfig_{PM}$ states that $\delta_\alpha$,$s_\alpha$ and $\delta_\beta$, $s_\beta$ are the same configuration if \myi all fluents have the same truth values in both $\Psi(s_\alpha)$ and $\Psi(s_\beta)$ ($\SameState$), and \myii $\delta_\beta$ is equal to $\delta_\alpha$ (i.e., both $\delta_\alpha$ and $\delta_\beta$ correspond to the remaining process to be executed $\delta'$). The following shows that $\SameConfig_{PM}$ is correct.

\begin{theorem}
\label{teo:bisimulation}
$\SameConfig_{PM}(\delta_\alpha,s_\alpha,s_\beta,\delta_\beta)$ is correct.
\end{theorem}

\begin{proof}
We show that \textsc{$SameConfig_{PM}$} is a bisimulation. Indeed:
\begin{itemize}
\item Since $\SameState(\Psi(s_\alpha),\Psi(s_\beta))$ requires all expected fluents to have the same values both in $s_\alpha$ and $s_\beta$, we have that $\big(Final(\delta_\alpha,s_\alpha) \Leftrightarrow Final(\delta_\beta,s_\beta)\big)$.

\item Since $\SameState(\Psi(s_\alpha),\Psi(s_\beta))$ requires all expected fluents to have the same values both in $s_\alpha$ and $s_\beta$, it follows that the PMS is allowed for the same process $\delta_\alpha$ to
assign the same tasks both in $s_\alpha$ and in $s_\beta$ and moreover for
each action $a$ and situation $s_\alpha$ and $s_\beta$ s.t.
$\SameState(\Psi(s_\alpha),\Psi(s_\beta))$, we have that $\SameState(\Psi(do(a,s_\alpha)),\Psi(do(a,s_\beta))$
hold. As a result, for each $a$ and $\overline{\delta_\alpha}$ such that
$Trans\big(\delta_\alpha,s_\alpha,\overline{\delta_\alpha},do(a,s_\alpha)\big)$ we have that
$Trans\big(\delta_\alpha,s_\beta,\overline{\delta_\alpha},do(a,s_\beta)\big)$ and
$\SameConfig_{PM}\big(\overline{\delta_\alpha},do(a,s_\alpha),\overline{\delta_\beta},do(a,s_\beta)\big)$.
Similarly for the other direction.
\end{itemize}
\end{proof}

We can now formalize a predicate $\Recovery$ that describes how our recovery mechanism works:

\begin{equation}
\begin{array}{l}
\Recovery(\delta',s',s'',\delta'') \Leftrightarrow {}\\ \qquad
\exists \delta_a \ s.t. \ \delta''=\delta_a;\delta' \wedge
Linear(\delta_a) \wedge {} \\ \qquad Do(\delta_a,s',s'') \wedge \neg \Relevant(\delta',s'') \wedge \\
\qquad \SameConfig_{PM}(\delta',s',s'',\delta')
\end{array}
\end{equation}

Intuitively, $\Recovery(\delta',s',s'',\delta'')$ holds if the program $\delta'$ that was intended to be executed in $s'$ is adapted in a new program $\delta''$. The adapted process $\delta''$ is a sequence composed by the linear recovery procedure $\delta_a$ and the remaining part of the process to be executed $\delta'$. The execution of $\delta_a$ in $s'$ results in a new situation $s''$, where $\SameState(\Phi(s''),\Psi(s''))$ (i.e., $\neg \Relevant(\delta',s'')$) and $\SameConfig_{PM}(\delta',s',s'',\delta'))$ hold.

The nice feature of $\Recovery$ is that it searches for a linear program that achieves a certain formula, namely
$\SameState(\Phi(s),\Psi(s))$. Moreover, restricting to sequential programs obtained by planning with no concurrency does not prevent any recoverable process from being adapted. In sum, we have reduced the synthesis of a recovery program to a
classical Planning problem in AI~\cite{TraversoBook2004}. As a result we can adopt a well-developed literature about planning for our aim. In particular, if services and input/output parameters are finite, then the recovery can be reduced to \emph{propositional} planning, which is known to be decidable in general.

\section{The \smartpm Adaptation Mechanisms}
\label{sec:approach-adaptation}

\subsection{The Built-in Adaptation Mechanism}
\label{subsec:approach-adaptation-built_in_adaptation}

The first adaptation technique we analyze for \smartpm is based on synthesizing a number of candidate recovery processes and on simulating them off-line. More precisely, when the fluent \Relevant \ holds, we need to find a recovery procedure which is able to align the physical reality $\Phi(s)$ with the expected reality $\Psi(s)$. A quick solution to this problem consists of devising the recovery plan on-line, by trying to execute performable actions in $\Phi(s)$ and senses what the next action should be, on the basis of a ``distance'' notion between the two realities. Such a solution does not require a reasoner to determine a lengthy course of action (formed perhaps of hundreds of tasks) before executing the first step in the world.

However, on the other hand, once an action has been executed in the world, there may be no way of backtracking it if it is later found out that it was performed incorrectly. This aspects assumes a great value if we think that we are executing tasks in an emergency management context, where delays or useless/incorrect activities may easily prevent the correctness of the whole procedure. As a result, an on-line execution of a program may fail where an off-line execution would succeed.


\bigskip

On Fig.~\ref{fig:fig_approach_built-in}, we show the fragment of the \indigolog code related to our built-in adaptation mechanism. Procedure $\proc{Adapt}()$ starts by invoking a basic action \aAdaptStart, whose effect is to make the fluent \fAdapting \ equal to $true$.

\begin{equation} \label{eq:eq_adapting}
\begin{array}{l}
\fAdapting(do(a,s)) = true \equiv{} \\
\quad a=\aAdaptStart \ \lor{} \\
\quad\big(\fAdapting(s) = true\ \land a \neq \aAdaptFinish \big).
\end{array}
\end{equation}

Then, $\proc{Adapt}()$ invokes the procedure $\proc{AdaptingProgram}()$ in order to build and execute the recovery program and, at the same time (cf. the prioritized concurrency in line 2), it waits until the recovery procedure has been completely performed. This will happen when $\proc{AdaptingProgram}()$ will terminate its execution, and the basic action \aAdaptFinish \ will turn the fluent \fAdapting \ to $false$.

\begin{figure}[t]
\centering{
 \includegraphics[width=0.95\columnwidth]{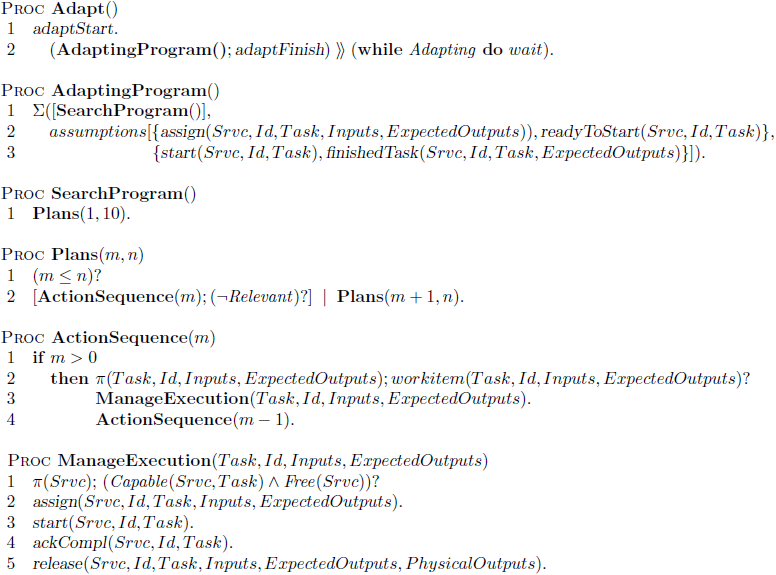}
 } \caption{The built-in adaptation mechanism of \smartpm.}
 \label{fig:fig_approach_built-in}
\end{figure}

Procedure $\proc{AdaptingProgram}()$ relies on the \indigolog \emph{search operator} $\Sigma$, that provides a form of lookahead planning. The idea is that given any program $\delta$, the program $\Sigma(\delta)$ executes online just like $\delta$ does off-line. In other words, before
taking any action, it first ensures using online reasoning that this step can be followed successfully by the rest of $\delta$. More precisely, according to~\cite{degiacomo:1999}, the semantics of the search operator is that:
\[
Trans(\Sigma(\delta),s,\Sigma(\delta'),s')\,\,\Leftrightarrow\,\,
Trans(\delta,s,\delta',s')\,\land\,\exists{s^*} \ s.t. \ Do(\delta',s',s^*).
\]
If $\delta$ is the entire program under consideration, $\Sigma(\delta)$ emulates complete off-line execution. In \smartpm, we use a specialized version of $\Sigma$ that relies on some assumptions on the performable actions. Each assumption is of form \ $\{\fontActions{actionPMS}(\vec{x}),\fontActions{actionService}(\vec{y})\}$, meaning that any action $\fontActions{actionPMS}$ executed by the \smartpm engine with input $\vec{x}$ will be eventually ``complemented'' by action $\fontActions{actionService}$ executed by a service with input $\vec{y}$. Vector of parameters $\vec{y}$ is a fully-deterministic transformation of $\vec{x}$. Here we are using the simple case where $\vec{y}$ is a subset to $\vec{x}$, but one can customize for specific tasks/actions.

Specifically, for \smartpm, we have coded two assumptions: the first is that the action $\aReady(Srvc,Id,Task)$ performed by a certain service $Srvc$ is expected to follow the PMS action
$\aAssign(Srvc,Id,Task,Inputs,ExpectedOutputs))$; the second concerns the PMS action $\aStart(Task,Id,Srvc)$, which is supposed to come before the action $\aFinish(Srvc,Id,Task,Inputs,PhysicalOutputs,ExpectedOutputs)$, executed by $Srvc$. (cf. also Fig.~\ref{fig:fig_approach-formalization_seqDiagram})


Finally, let us focus on the actual program in charge of building the recovery plan, namely procedure $\proc{SearchProgram}()$ that, in turn, invokes procedure $\proc{Plans}(m,n)$.
Generally speaking, such procedure will try to reach a situation in which \Relevant\ does not hold anymore (see line $2$).
For this aim, the $\proc{Plans}$ procedure builds recovery plans of growing length and simulates their execution starting from $\Phi(s)$. Recovery plans may consist of a variable number of tasks to be executed in sequence, that range from $m$ to $n$ tasks.

Let us suppose, for example, that $m = 1$, meaning that we are executing $\proc{Plans}(1,10)$ for the first time. As a consequence, the procedure $\proc{ActionSequence}(1)$ is invoked. It tries to generate all sequences composed by a single task. Then it simulates their execution in $\Phi(s)$ (cf. the invocation of $\proc{ManageExecution}$ in line 3) and checks if at least one of those sequences can turn the value of $\Phi(s')$ ($s'$ results from the simulated execution of the candidate recovery process) such that $\SameState(\Phi(s'),\Psi(s'))$ (cf. line 2 of the procedure $\proc{Plans}$). If it does not happen, the value of $m$ is increased of one unit (cf. line 2 of the procedure $\proc{Plans}$) and $\proc{Plans}(2,10)$ is invoked.

Now, $\proc{ActionSequence}(2)$ searches for sequences of two tasks $[t_1;t_2]$ to be executed in $\Phi(s)$. The execution of each candidate recovery plan happens off-line (i.e., when the main process has been stopped for waiting for the building of the recovery plan itself), but the built-in algorithm simulates an on-line execution of the plan, by applying the effects of the sequence directly on the current $\Phi(s)$. To be more precise, $\proc{ActionSequence}(2)$ first picks an admissible workitem (cf. line 2 of the procedure $\proc{ActionSequence}$), and then simulates its execution on $\Phi(s)$, by obtaining a new reality $\Phi(s')$. Now, starting from $\Phi(s')$, $\proc{ActionSequence}$ is invoked again for searching any task that is executable on $\Phi(s')$, by devising the new reality $\Phi(s'')$. Again, the control is passed back to $\proc{Plans}$, that verifies if $\SameState(\Phi(s''),\Psi(s''))$.

We need to underline that every expected reality $\Psi(s_\alpha)$ devised in any situation $s_\alpha$ (starting from the situation $s$) that comes from the simulated execution of recovery processes of growing length, will be always equal to $\Psi(s)$ (i.e., to the expected reality stored after a failure has been sensed). In fact, when the fluent \fAdapting \ holds, meaning that the system is searching and simulating the execution of some recovery procedure, the system is not allowed to change the values of expected fluents. This because the built-in algorithm tries to execute each task of the recovery procedure as if it returns its expected outcomes, with the purpose to turn $\Phi(s)$ into $\Psi(s)$. This means that $\SameState(\Psi(s),\Psi(s_\alpha))$ and, consequently, $\SameConfig_{PM}(\delta',s,s_\alpha,\delta'))$ hold. Moreover, if the $\Sigma$ search operator is able to find a recovery procedure $\delta_a$ composed by a sequence of tasks such that $Do(\delta_a,s,s_\alpha)$, then also $\Recovery(\delta',s,s_\alpha,(\delta_a;\delta'))$ is \emph{true}.

In general, the search technique is iterative deepening: if there exists no sequence whose length is less or equal to $m$ tasks, it tries with length-$(m+1)$ task sequences. Observe also that $\proc{ActionSequence}$ uses a simple breadth-first search mechanism, which specializes what proposed in~\cite{ReiterBook}.  The procedure tries to generate all admissible sequences of tasks composed by $m$ task. This keeps going deeper and deeper till reaching a sequences of 10 tasks or any task sequence that recovers. If no task sequence of at most 10 tasks exists, it is assumed that no recovery is possible. We have adopted 10 as bound to the length of the sequence as it is a reasonable assumption in our scenario.

It is worth highlighting that, because the monitor runs at a higher-priority level than the actual process, the solution plan found for the recovery program $\Sigma[\proc{SearchProgram}]$ would run at higher-priority than program $\proc{Process}$. So, the program $\proc{Process}$ cannot progress until the recovery has been finished and applied.
Consequently, after a sensed deviation, the program executed will be equivalent
to $(\Sigma[\proc{SearchProgram}()];\delta')$, where $\delta'$ is the program remaining
from procedure $\proc{Process}$.

\begin{theorem}
\label{teo:decidible} Let assume a domain in which services and
input and output parameters are finite. Then given a process
$\delta'$ and situations $s'$ and $s''$, it is decidable to compute
a recovery process $\delta''$ with the built-in approach just devised such that
$\Recovery(\delta',s',s'',\delta'')$ holds.
\end{theorem}

\begin{proof} In domains in which services and input and
output parameters are finite, also actions and fluents instantiated
with all possible parameters are finite. Hence we can phrase the
domain as a propositional one and the thesis follows from
decidability of propositional planning~\cite{TraversoBook2004}.
\end{proof}





\subsection{The Plan-based Adaptation Approach}
\label{subsec:approach-adaptation-plan_based_adaptation}

We now turn our attention to the plan-based adaptation mechanism currently working in the \smartpm system. Before starting the execution of the process $\delta$, the PMS builds the \textsc{PDDL} representation of each task defined in the SitCalc theory and sends it to an external planner. More in detail, \smartpm is able to build a PDDL planning domain starting from a \emph{domain theory} defined through \smartML, our declarative language used for representing contextual properties of a dynamic scenario. We show the syntax of \smartML in Section~\ref{sec:framework-smartpm_definition_tool-smartml}, but we anticipate that a \smartML specification can be easily converted in a SitCalc theory (cf. Section~\ref{sec:framework-smartpm_definition_tool-xml-to-indigolog}).

When a misalignment between $\Phi(s')$ and $\Psi(s')$ is sensed (we consider $s'$ as the situation where something wrong has happened and $\delta'$ as the faulty process), the $\proc{Monitor}()$ launches the $\proc{InvokePlanner}()$ procedure for recovering the faulty process $\delta'$. \proc{InvokePlanner} starts by building a PDDL planning problem that reflects the gap between the two realities. Specifically, it first determines the initial state $Init$ of the planning problem, by making it equal to $\Phi(s')$. In addition to the values of data fluents, $Init$ can also include the values of $\fProvides(c,b)$ and of $\fFree(c,s')$, for each capability $b$ and service $c$ defined in SitCalc theory. Such values will be used by the planner for scheduling correctly recovery tasks during the building of the plan.

\begin{figure}[t]
\centering{
 \includegraphics[width=0.8\columnwidth]{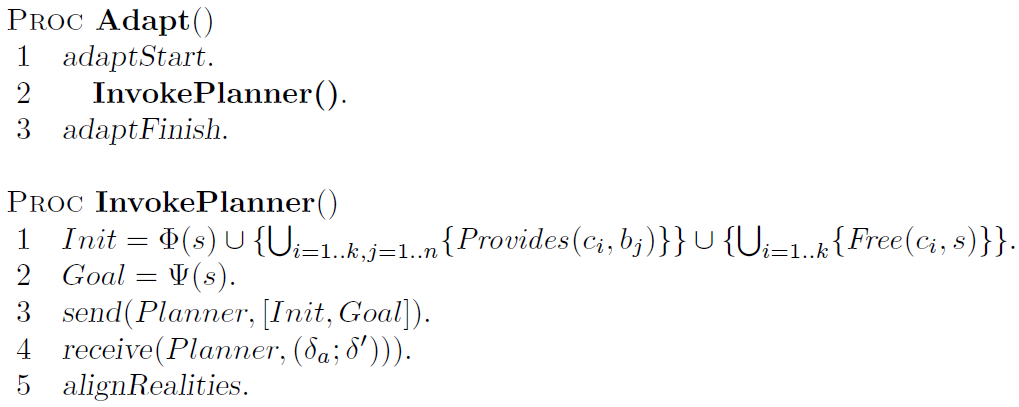}
 } \caption{The plan-based adaptation mechanism of \smartpm.}
 \label{fig:fig_approach_built-in}
\end{figure}

Then, \proc{InvokePlanner} builds the $goal$ of the planning problem. Since our target is to convert $\Phi(s')$ into $\Psi(s')$, it is clear that the \emph{goal} will be equal to $\Psi(s')$. There exists the possibility to add to the \emph{goal} some situation calculus abbreviation that is considered to be relevant for adaptation. For example, in our case study we need to guarantee that each service is connected to the network. Therefore, the $goal$ can be augmented with information stating that the execution of the synthesized recovery plan would not prevent the connection of services to the main network. Since our $goals$ are basically a conjunction of literals, the presence of SitCalc abbreviations does not prevent the reachability of the atomic goals contained in $\Psi(s')$ (unless the SitCalc abbreviation directly contradicts some expected fluents value, but this would be considered as a design-time error). The only effect deriving from the presence of SitCalc abbreviations within the $goal$ definition is that the searching for a recovery plan will be a bit more constrained.

\vskip 0.5em \noindent\colorbox{light-gray}{\begin{minipage}{0.98\textwidth}
\begin{example}
\label{example_plan_based}
\emph{Let us consider the example described in our case study. Specifically, suppose, for example, that the task \emph{go(loc00,loc33)} is assigned to actor \emph{act\emph{1}} (cf. Fig.~\ref{fig:fig_introduction-case_study-context_2}(a)), which reaches instead the location \emph{loc03} (cf. Fig.~\ref{fig:fig_introduction-case_study-context_2}(b)). This means that $act\emph{1}$ is now located in a different position than the desired one, and s/he is out of the optimal network range. Therefore, after the $\aRelease$ action has been executed, the fluent $\fAt_{\varphi}$ takes the value \emph{loc03}. But this output does not satisfy the expected outcome. The expected output \emph{loc33} is stored in the fluent $\fAt_{\psi}$; it generates a discrepancy between $\Phi(s')$ and $\Psi(s')$. This means that $\fRelevant(\delta,s')$ holds, and the main process $\delta'$ needs to be adapted.}
\end{example}
\end{minipage}
}\vskip 0.5em

Now, the \indigolog engine invokes an external planner by giving as inputs the $init$ and $goal$ sets just computed. We have to underline that the communication between the \indigolog engine and the planner is mediated by a \textbf{Synchronizer} component, that launches the planner and, when a recovery plan has been computed, translates the plan in a new \indigolog adapted process $\delta''=(\delta_a;\delta')$. We give more information about the \textbf{Synchronizer} in Chapter~\ref{ch:framework}.

The \receive \ command in line 4 of the procedure \proc{InvokePlanner} is blocking, i.e, the PMS waits until the recovery procedure has been devised. If the recovery procedure is empty, it means that no plan exists for the current planning problem, and the control passes back to the process designer, that can try to manage manually the exception. Otherwise, if a recovery process $\delta''=(\delta_a;\delta')$ is returned, the \indigolog engine can start to execute it. If we consider the example~\ref{example_plan_based}, the recovery plan needed for removing the gap between the two realities is shown in Fig.~\ref{fig:fig_introduction-case_study-context_3}(a).

The last command of the procedure \proc{InvokePlanner} is used for forcing $\Psi(s')$ to the current value of $\Phi(s')$. This because the execution of the recovery plan $\delta_a$ will turn the initial state (i.e., $\Phi(s')$) into the goal state (i.e., $\Psi(s')$). Moreover, after the execution of every recovery task of $\delta_a$ the situation $s'$ will evolve in a situation $s''$, where $\Psi(s'')$ is exactly the goal of the planning problem just reached. This means that $\SameConfig_{PM}(\delta',s',s'',\delta'))$ holds and, since the execution of the recovery plan guarantees that also $\SameState(\Phi(s''),\Psi(s''))$ holds, the predicate $\Recovery(\delta',s',s'',(\delta_a;\delta'))$ is \emph{true}.

\begin{theorem}
Let assume a domain in which services and input and output parameters are finite. Then, given a process $\delta'$ and situations $s'$ and $s''$, it is decidable to compute a recovery process $\delta''$ with the planning-based approach just devised such that $\Recovery(\delta',s',s'',\delta'')$ holds.
\end{theorem}
\begin{proof}
The decidability of the plan-based approach relies on the external planner used for the synthesis of the recovery procedure. In the \smartpm system, we have used the LPG-td planner~\cite{LPG} for synthesizing recovery plans. LPG-td is a state-based planner that is based on a stochastic local search in the space of particular ``action graphs'' derived from the planning problem specification. The basic search scheme of LPG-td is inspired to Walksat~\cite{WALKSAT}, an efficient procedure for solving SAT-problems. If a domain in which services and input and output parameters are finite, the LPG-td planner guarantees to terminate~\cite{LPG}.

\end{proof}

\subsection{The Continuous Planning Approach}
\label{subsec:approach-adaptation-continuous_planning}

In this section, we show a third approach for adapting a process from failures or exogenous events. It is based on \emph{Continuous Planning} techniques and, if compared with the two adaptation approaches employed on \smartpm and discussed above, it provides two interesting features in adapting a process: \myi it is a non-blocking technique, i.e., it does not stop directly any task in the main process during the computation of the recovery procedure and \myii it allows concurrent branches in the recovery procedure. This approach, developed on the same SitCalc Theory defined in Section~\ref{sec:approach-formalization}, requires to modify the working of the \proc{Monitor} procedure, and its implementation on the \indigolog platform is currently on going.

\emph{Continuous Planning}~\cite{CONTINUAL} refers to the process of planning in a world under continual change, where the planning problem is often a matter of adapting to the world when new information is sensed. A continuous planner is designed to persist indefinitely in the environment. Thus it is not a ``problem solver'' that is given a single goal and then plans and acts until the goal is achieved; rather, it lives through a series of ever-changing goal formulation, planning, and acting phases. Rather than thinking of the planner and execution monitor as separate processes, one of which passes its results to the other, we can think of them as a single process (cf. Fig.~\ref{fig:fig_approach-continuous_planning_approach}). 

\begin{figure}[t]
\centering{
 \includegraphics[width=0.8\columnwidth]{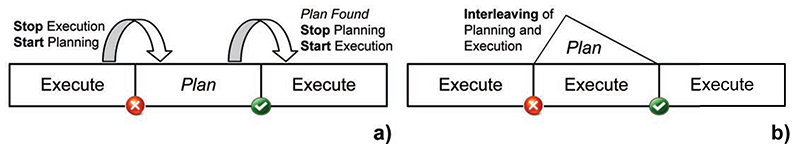}
 } \caption{Traditional ``plan then execute'' cycle (a) and the \emph{continuous planning} approach.}
 \label{fig:fig_approach-continuous_planning_approach}
\end{figure}

The proposed idea is to build the recovery procedure $\delta_a$ \emph{in parallel} with the execution of the main process $\delta$, avoiding to stop directly any task in the main process. Once ready, $\delta_a$ will be inserted as a new branch of $\delta$ and will be executed in concurrency with every other task. This means that the recovery branch will turn $\delta$ into a new process $\delta''$ = ($\delta||\delta_a$). Before to analyze how the whole approach works, we need to provide some further formal definition.

Since we are working in a domain in which services and input/output parameters are finite, each task defined in $\delta$ affects (or is affected by) only a finite number of fluents. This means that each task is interested only in that fragment of reality it contributes to modify.
\begin{mydef}
\label{affect}
A task T affects a data/expected fluent $\fX$ iff \ $\exists c,id,\vec{i},\vec{p},\vec{q},a \ \ s.t.  \ \ a= \\ \aRelease(c,id,T,\vec{i},\vec{p},\vec{q})$. We denote it with $T\triangleright{\fX}$.
\end{mydef}
\begin{mydef}
\label{affected}
A task T is affected by a data/expected fluent $\fX$ iff \ $\exists c,id,\vec{i},\vec{p},a \ \ s.t. \ \ a=  \aAssign(c,id,T,\vec{i},\vec{p})$. We denote it with $T\triangleleft{\fX}$.
\end{mydef}
The two latter definitions allow to state a new further definition of $\Phi(s)$ and $\Psi(s)$, whose range can be limited to a specific task $T$.
\begin{mydef}
\label{limitedPhysicalReality}
Given a specific task T, a T-limited physical reality \ $\Phi|_T(s)$ is the set of that data fluents $\fX_{\varphi,y}$ (where y ranges over \{1..m\}) such that $T\triangleright{\fX_{\varphi,y}}$ or $T\triangleleft{\fX_{\varphi,y}}$. We denote these fluents as $\fX_{\varphi|_T}$. Hence, $\Phi|_T(s) = \bigcup_{y = {1..m}}\{\fX_{\varphi,y|T}\}$ and $\Phi|_T(s) \subseteq \Phi(s)$.
\end{mydef}
\begin{mydef}
\label{limitedExpectedReality}
Given a specific task T, a T-limited expected reality \ $\Psi|_T(s)$ is the set of that expected fluents $\fX_{\psi,y}$ (where y ranges over \{1..m\}) such that $T\triangleright{\fX_{\psi,y}}$ or $T\triangleleft{\fX_{\psi,y}}$. We denote these fluents as $\fX_{\psi|_T}$. Hence, $\Psi|_T(s) = \bigcup_{y = {1..m}}\{\fX_{\psi,y|T}\}$ and $\Psi|_T(s) \subseteq \Psi(s)$.
\end{mydef}
From definitions \ref{limitedPhysicalReality} and \ref{limitedExpectedReality}, the following ones stem :
\begin{mydef}
Let $T_1, ..., T_n$ all tasks defined in the SitCalc theory. A physical reality $\Phi(s)$ is the union of all T-limited physical realities that hold in situation s : $\Phi(s) = \bigcup_{i = {1..n}} \Phi|_{T_i}(s)$.
\end{mydef}
\begin{mydef}
Let $T_1, ..., T_n$ all tasks defined in the SitCalc theory. An expected reality $\Psi(s)$ is the union of all T-limited expected realities that hold in situation s : $\Psi(s) = \bigcup_{i = {1..n}} \Psi|_{T_i}(s)$.
\end{mydef}
Now, the predicate \fRelevant \ can be easily refined in a way that focuses on a specific task $T$:
\begin{equation} \label{eq:eq_partialsamestate}
\begin{array}{l}
\fRelevant_T(\delta,s) \equiv{} \neg \SameState(\Phi|_T(s),\Psi|_T(s))
\end{array}
\end{equation}

\vskip 0.5em \noindent\colorbox{light-gray}{\begin{minipage}{0.98\textwidth}
\begin{example}
\label{example_continuous}
\emph{Let us consider again the example described in our case study. The task \emph{go(loc00,loc33)} was assigned to actor \emph{act\emph{1}}, which has reached the wrong location \emph{loc03} (cf. Fig.~\ref{fig:fig_introduction-case_study-context_2}(b)). Therefore, $act\emph{1}$ is now located in a different position than the desired one, and s/he is out of the optimal network range. After the $\aRelease$ action has been executed, the fluent $\fAt_{\varphi}$ takes the value \emph{loc03}, while the expected output \emph{loc33} is stored in the fluent $\fAt_{\psi}$. This generates a discrepancy between $\Phi|_{\aGo}(s)$ and $\Psi|_{\aGo}(s)$. This means that $\fRelevant_{\aGo}(\delta,s)$ holds, and the main process $\delta$ needs to be adapted.}
\end{example}
\end{minipage}
}\vskip 0.5em

In Fig.~\ref{fig:fig_approach-continuous_planning_proc} we show how we have concretely coded the continuous planning approach with \indigolog. Like in the plan-based adaptation approach, before starting the execution of the process $\delta$, the PMS builds the \textsc{PDDL} representation of each task defined in the SitCalc theory and sends it to a whatever external planner that implements the POP algorithm.

The main procedure involves four concurrent programs in priority. The interrupt at lowest priority, that consists just in waiting, is activated only if the process is still not finished, but for some reason, it can not progress in its execution (e.g., the PMS is unable to find any service providing all the capabilities required by a specific task to be executed). If the fourth interrupt is activated, the control passes back to the process designer, which can manually manage the situation.

Then, the system runs the actual \indigolog\ program representing the process to be executed (the procedure $\proc{Process}$). This procedure relies, in turn, on procedure $\proc{ManageExecution}$, which includes task assignment, start signaling, acknowledgment of completion, and final release (cf. Section~\ref{sec:approach-formalization-framework}). The monitor, which runs at higher priority, is in charge of monitoring changes in the environment and adapting accordingly. The first step in procedure \proc{Monitor} checks whether fluent $\fRealityChanged$ holds true, meaning that a service has terminated the execution of a task or an exogenous (unexpected) action has occurred in the system. Basically, the procedure $\proc{Monitor}$ is enabled when the physical or the expected reality (or both) change.

\begin{figure}[t]
\centering{
 \includegraphics[width=0.85\columnwidth]{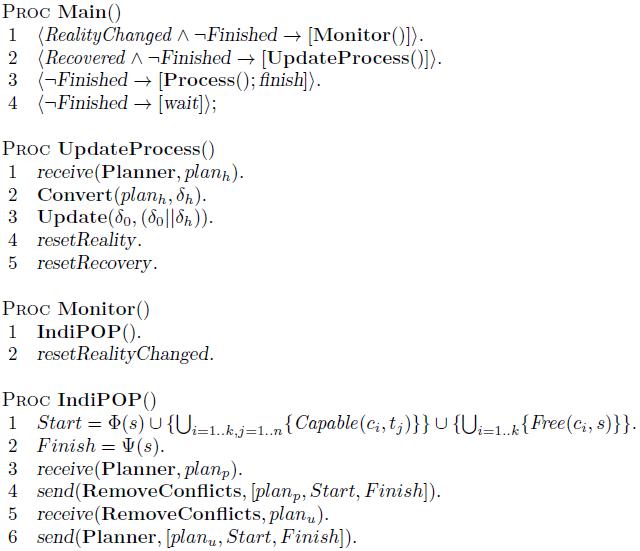}
 } \caption{The \indigolog main procedure customized for the continuous planning approach.}
 \label{fig:fig_approach-continuous_planning_proc}
\end{figure}

If it happens, the monitor calls the procedure \proc{IndiPOP}, whose purpose is to manage the execution of the external planner by updating its initial states and expected goals according with changes in the two realities. \proc{IndiPOP} first builds the two sets \emph{Start} (the initial state) and \emph{Finish} (the goal), by making them equal respectively to $\Phi(s)$ and $\Psi(s)$. As far as concerns the initial state, it will include, for each task $t$ and service $c$ defined in SitCalc theory, the values of $\fCapable(c,t)$ and of $\fFree(c,s)$ in addition to the values of data fluents.

Then \proc{IndiPOP} catches the partial plan $plan_p$ (that has the form of a set of partial ordering constraints between tasks; it is empty if no failure has happened yet) built till that moment by the external planner and updates it with the new sets Start and Finish. Such updating finds something about $plan_p$ that needs fixing in according with the new realities. Since $plan_p$ has been built working on old values of the two realities, it is possible that some ordering constraints between tasks are not valid anymore. This causes the generation of some conflicts, that need to be deleted by $plan_p$ through the external procedure $\proc{RemoveConflicts}$. Basically, \proc{IndiPOP} can be seen as a conflict-removal procedure that revises the partial recovery plan to the new realities. At this point, $plan_u$ (that is, $plan_p$ just updated, i.e., without conflicts) is sent back to the external planner together with the sets Start and Finish. The external planner can now restore its planning procedure.

Note that if the predicate \fRelevant(s) holds, meaning that a misalignment between the two realities exists, the PMS tries to continue with its execution. In particular, every $T_i$ whose T-limited expected reality $\Psi|_{T_i}(s)$ is different from the T-limited physical reality $\Phi|_{T_i}(s)$ could not anymore proceed with its execution. However, every task $T_j$ not affected by the deviation can advance without any obstacle. Once sent the sets of fluents composing the two realities to the external planner, the monitor resets the fluent $\fRealityChanged$ to \emph{false}, and the control passes to the process of interest (i.e., program \proc{Process}), that may again execute/advance.

When the external planner finds a recovery plan that can align physical and expected reality, the fluent \recovered\ is switched to \emph{true} and the procedure $\proc{UpdateProcess}$ is enabled. Now, after receiving the recovery process $\delta_a$ from the planner, the PMS updates the original process $\delta$ to a new process $\delta''$ that, respect to its predecessor, has a new branch to be executed in parallel; such branch is exactly $\delta_a$. It contains all that tasks able to repair the physical reality from the discrepancies (i.e. to unblock all that tasks stopped in $\delta$ because their preconditions did not hold). Note that when $\delta_a$ is merged with the original process $\delta$, the two realities are still different from each others. Therefore, the PMS makes them equal by forcing $\Psi(s)$ to the current value of $\Phi(s)$. This because the purpose of $\delta_a$, after that all recovery actions have been executed, is to turn the current $\Phi(s)$ into $\Psi(s')$, where $s'$ is that situation reached after the execution of recovery actions.

Let us now formalize the concept of $strongly\ consistency$ for a process $\delta$.
\begin{mydef}
Let $\delta$ a process composed by n tasks $T_1,..,T_n$. $\delta$ is \textbf{strongly consistent} iff:
\begin{itemize}
\item Given a specific task T and an input I, $\nexists c,c',\vec{p},\vec{p'},\vec{q},\vec{q'},a,a' \ s.t. \\a=\aRelease(c,T,I,\vec{p},\vec{q}) \ \land a'=\aRelease(c',T,I,\vec{p'},\vec{q'}) \land (\vec{p} \neq \vec{p'})$.
\item $\forall y\in{1..m}, \nexists (T_i,T_k)_{i \neq k} \ s.t. (T_i \triangleright{X_{\varphi,y}} \land T_k \triangleright{X_{\varphi,y}})$.
\end{itemize}
\end{mydef}
Intuitively, a process $\delta$ is strongly consistent if a specific task, executed on a given input, cannot return different values for its expected output; moreover, the above condition holds if do not exist two different tasks that affect the same fluent. For strongly consistent processes, we can state the concept of $goal$ :
\begin{mydef}
Given a strongly consistent process $\delta$, composed by $n$ tasks $T_1,, ... ,T_n$,  the goal of \ $\delta$ can be defined as the set of all expected fluents $\fX_{\psi,y}$ that are affected by $T_1,T_2, ... ,T_n$. Hence, \ $Goal(\delta) =\{ {\fX_{\psi,y}} \ s.t. \ \exists i_{1..n}.(T_i\triangleright{\fX_{\psi,y}})\}$.
\end{mydef}
After a recovery procedure $\delta_a$, $Goal(\delta)$ $\subseteq$ $Goal(\delta||\delta_a)$ , since the recovery procedure can introduce new tasks with respect to the original process $\delta$. Anyway, the original $Goal(\delta)$ is preserved also after the adaptation procedure.

\begin{figure}[t]
\centering{
 \includegraphics[width=0.95\columnwidth]{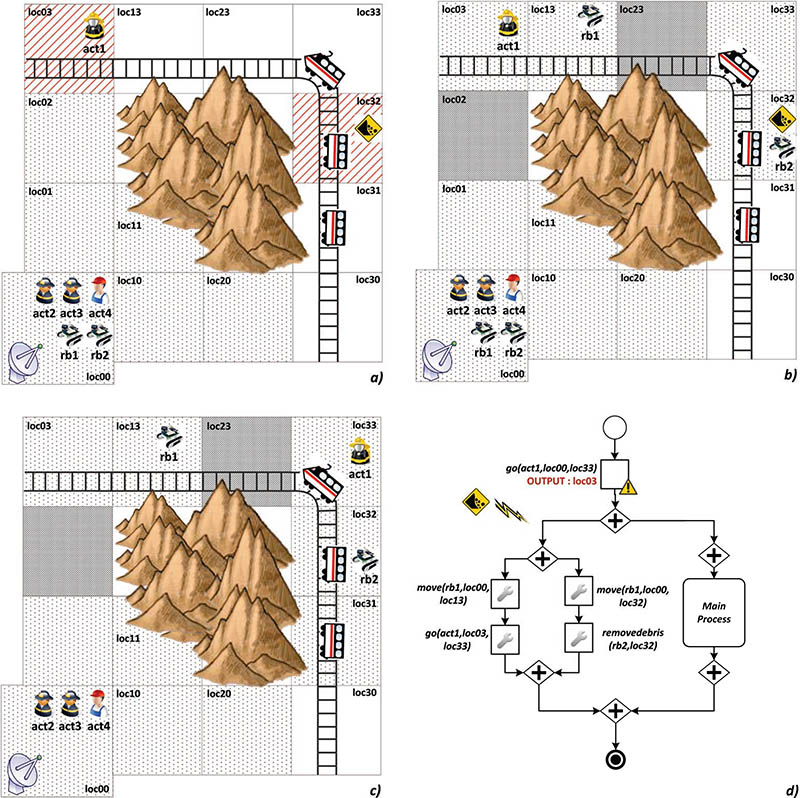}
 } \caption{An example of adaptation through the continuous planning approach.}
 \label{fig:fig_approach-continuous_planning_example}
\end{figure}

\begin{theorem}[Termination]
Let $\delta$ be a strongly consistent process composed by a finite number of tasks $T_1, ... ,T_n$. If $\delta$ does not contain while and iteration constructs, and the number of exogenous events is finite, then the core procedure of \indigolog\ PMS terminates.
\end{theorem}
\begin{proof}(\emph{Sketch}) \
We have to consider every case in which one of the two reality can change, since only in such a case the PMS updates the recovery plan under construction and restarts the external planner with new values of initial state and goals.
\begin{itemize}
\item \emph{Case 0 : No Failures.} We know that both $\Phi(s)$ and $\Psi(s)$ can change when a task $T$ ends its execution. As the number of tasks composing the original process $\delta$ is finite, if no internal failures or exogenous events occur (meaning that each task terminates its execution by returning the expected output), the process terminates without any need to compute a recovery procedure.
\item \emph{Case 1 : Internal Failures.} Again, since the number of tasks composing the original process $\delta$ is finite, also the number of \emph{internal failures} that could happen will be finite. In the worst case, in which $\delta = T_1 || T_2 || ... || T_n$, if every task fails during its execution, we will have exactly $n$ internal failures, meaning that $\Phi(s)$ and $\Psi(s)$ change $n$ times. In such a case, no task can more proceed in the execution, and the external planner can terminate (with a recovery plan $\delta_a$ or with a failure) without any more alteration of its initial situation and goals.
\item \emph{Case 2 : Exogenous Events.} As the number of exogenous events that affect the execution of $\delta$ is finite, when the very last exogenous event has occurred, the only way to change realities is either by the standard execution of $\delta$, or by internal failures, falling respectively into \emph{Case 0} or \emph{Case 1}.
\end{itemize}
\end{proof}
We want to underline that the termination cannot be guaranteed if $\delta$ contains loops or iteration, since potentially the two realities could indefinitely change. The same is true if the number of exogenous events is unbounded. In such a case, there exists the possibility that \proc{IndiPOP} updates continuously the external planner, that would not be able to terminate by returning the expected recovery plan.

\vskip 0.5em \noindent\colorbox{light-gray}{\begin{minipage}{0.98\textwidth}
\begin{example}
\emph{Let us suppose that during the synthesis of the recovery plan needed for dealing with the exception arose in the previous example~\ref{example_continuous}, an exogenous event $\rockSlide(loc32)$ is captured by the PMS. It aims at alerting about a rock slide collapsed in location \emph{loc32} (cf. Fig.~\ref{fig:fig_approach-continuous_planning_example}(a)).
Hence, we have different values for $\fStatus_\varphi(loc32,s')=debris$ and $\fStatus_\psi(loc32,s')=ok$. Again, the PMS invokes the external planner by obtaining the partial plan $plan_p$ built till that moment and verifies if it needs to be fixed according with new values of the two realities. If no conflicts are individuated, the PMS sends back $plan_p$ to the planner together with the information about the initial state and the goal, updated to situation $s'$. When the planner ends its computation, it returns a recovery process $\delta_a$ that includes two concurrent branches. The left branch of $\delta_a$ instructs robot \emph{rb1} to reach a position where actor \emph{act1} is again connected to the network, and can finally reach its original expected destination $loc33$ (cf. Fig.~\ref{fig:fig_approach-continuous_planning_example}(c)). In parallel, robot $rb2$ can reach location \emph{loc32} and remove debris. Note that $\delta_a$ can be executed in concurrency with $\delta$ (see the right-hand side of Figure~\ref{fig:fig_approach-continuous_planning_example}(d)) by preserving its original goal.}
\end{example}
\end{minipage}
}

\section{Conclusion}
\label{sec:approach-conclusion}

In this chapter we have presented the formal foundations of our general approach based on formalizing processes with \sitcalc and \indigolog, on detecting exceptions with execution monitoring and on adapting automatically a dynamic process through classical planning techniques. Such an approach is (i) practical, since it relies on well-established planning techniques, and (ii) does not require the definition of the adaptation strategy in the process itself (as most of the current approaches do).

We also gave details on three different techniques for adapting a faulty process instance. The built-in adaptation approach,  based on the $\Sigma$ search construct provided by \indigolog, allows to incorporate adaptation features directly into the PMS. However, the computation of a recovery plan with this approach is very inefficient, since it is based on a breadth first search algorithm that results in an exponential explosion in the size of the inputs/services.

The plan-based technique for adapting processes is instead very efficient, since it delegates the building of the plan to state-of-the-art external planners. Specifically, we used the LPG-td planner~\cite{LPG} for the synthesis of the recovery plans, and in Chapter~\ref{ch:validation} we clearly demonstrate the feasibility of this choice by showing some interesting experimental tests. The drawback of this approach lies in the intrinsic assumptions
of classical planning (determinism in the effects, model completeness, etc.), that could be too restrictive to address complex problems.

The continuous planning approach seems very promising for reducing the overall response time during adaptation. The strength of the approach lies in the ability to incorporate execution feedback directly into the plan, without blocking the execution of the main process. In fact, during the plan synthesis, if a positive event occurs (such as an external event or a task whose effects are to ``adjust'' the compromised situation), the system is able to take advantage of such an opportunity without the need of synthesizing a new plan. However, even though our intent is to make the planning process very responsive, there still remains a synchronization process between planning and execution, which can require a significant response time. Another issue concerns the augmented response time needed by the planner when more than one exception is treated at the same time. In such a case, the recovery plan could be constituted by $n$ recovery sub-processes independent one from another, and the overall time depends by the number of steps required for finding the longest one (in terms of recovery activities).


\chapter{The \smartpm System}
\label{ch:framework}


The \smartpm approach described in Chapter~\ref{ch:approach} has shown that a PMS supporting process adaptation at run-time requires an integrated approach that covers the modeling, execution and monitoring stages of process life-cycle. The combination of procedural and imperative models with declarative elements, along with the exploitation of techniques from the field of artificial intelligence (AI) such as planning algorithms and tools, is required for increasing the ability of a PMS of supporting dynamic processes.



To this end, in this chapter we aim at presenting the overall architecture of \smartpm and at giving technical details about its atomic components. In Section~\ref{sec:approach-architecture}, we show how software components of \smartpm have been organized into multiple logical layers, which correspond to the system's main features. 
Then, in Section~\ref{sec:framework-indigolog_platform} we presents the \indigolog platform we used for building our PMS and the execution monitor. In Section~\ref{sec:approach-smartpm_definition_tool} we describe \smartML, which combines a modeling formalism for representing the information of the
contextual scenario linked to a specific dynamic process, and a graphical tool (specifically, Eclipse BPMN\footnote{\url{http://www.eclipse.org/modeling/mdt/?project=bpmn2}}) for designing the control flow of the process. We also show how a dynamic process formalized through \smartML is automatically
translatable in situation calculus and \indigolog readable formats and is therefore ready for being executed by \smartpm.
Finally, in Section~\ref{sec:framework-builders} we give some high level detail on the translation algorithms used for building planning domains and problems starting from \smartML/\indigolog specifications, and in Section~\ref{sec:framework-screenshots}, we show some screenshots of the \smartpm system in action.


\section{System Architecture}
\label{sec:approach-architecture}

Our approach to integration of process execution and planning for providing automatic adaptation features at run-time relies on three main architectural layers as shown in Fig.~\ref{fig:fig_framework-smartpm_architecture}.

\begin{figure}[t]
\centering{
 \includegraphics[width=0.95\columnwidth]{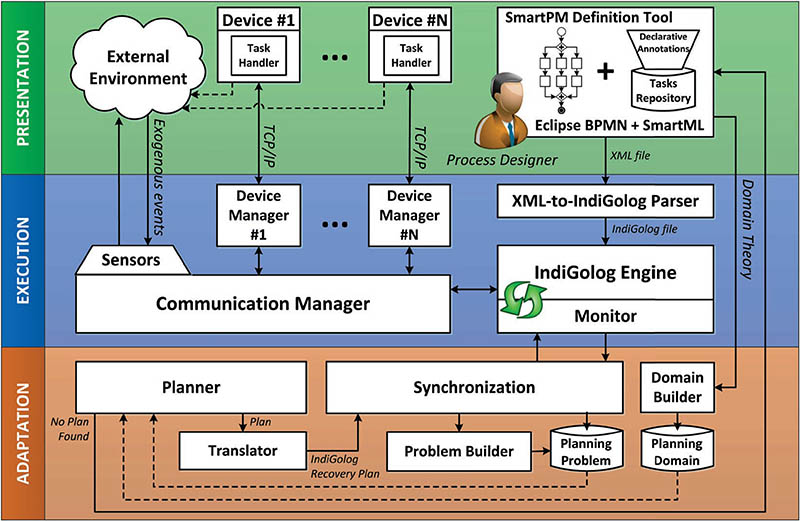}
 } \caption{The architecture of \smartpm.}
 \label{fig:fig_framework-smartpm_architecture}
\end{figure}

The \textbf{Presentation Layer} has a twofold purpose. On one hand, it allows a human process designer to define a dynamic process and all the contextual information concerning the scenario in which the process will be enacted. To this end, we provide a clarifying descriptive tool called \textbf{\emph{\smartpm Definition Tool}}. It is a GUI-based tool that can be used by a process designer to build a process specification defined according to the \smartpmML (a.k.a. \smartML) presented in Section~\ref{sec:framework-smartpm_definition_tool-smartml}. The language allows to specify a so-called \emph{Domain Theory}, i.e., to clearly represent the contextual properties reflecting a dynamic scenario. Moreover, through \smartML a process designer can build a repository of tasks defined in a declarative way, by explicitly providing tasks pre-conditions and effects based on the domain theory previously defined. Finally, the process designer can build graphically the control flow of a dynamic process through \emph{\textbf{Eclipse BPMN}}\footnote{\url{http://eclipse.org/bpmn2-modeler/}}, a graphical editor that allows to specify business processes using the BPMN 2.0 notation~\cite{BPMN20}. Note that the set of tasks composing the control flow of the dynamic process must be selected from the repository of tasks specified through \smartML. In Section~\ref{sec:approach-smartpm_definition_tool} we will show how to build a dynamic process through the \smartpm Definition Tool, that integrates the graphical features of Eclipse BPMN with the \smartML declarative language. The outcome of the process design activity will be a complete XML-encoded process specification to be passed to the \indigolog engine.


On the other side, the Presentation Layer includes every real world device that may interact with the \indigolog engine. In order to manage such an interaction, there is the need to installing and configuring a \emph{\textbf{Task Handler}} module on top of every device. The Task Handler is an interactive GUI-based software application that - during process execution - supports the visualization of assigned tasks and allows each participant to start task execution and notify task completion by selecting an appropriate outcome. 
A task is seen by the Task Handler as a black-box activity, and the only relevant information captured are the starting and the completion of the task itself (together with the respective outcomes). Finally, since our \smartpm system is mainly a proof-of-concept implementation of an adaptive PMS, we simulate the \textbf{\emph{External Environment}} as a software module that sends asynchronously exogenous events that may prevent the correct execution of the dynamic process.



The \textbf{Execution Layer} is in charge of managing and coordinating the execution of dynamic processes. It performs task assignments and it manages and stores information about services and tasks involved in process executions, tasks to be completed and variables/data modified during task executions. The Execution Layer is basically a customized version of the \indigolog platform, whose details are provided in Section~\ref{sec:framework-indigolog_platform}.

A dynamic process built through Eclipse BPMN and annotated with \smartML is taken as input from the \emph{\textbf{XML-to-\indigolog\ Parser}} component, which translates this specification in \sitcalc and \indigolog readable formats, in order to make the process executable by the \indigolog engine.

The \textbf{\emph{\indigolog Engine}} provides a proper execution engine that manages the process routing and decides which tasks are enabled for execution, by taking into account the control flow, the value of predicates and preconditions and effects of each task. Before a process starts its execution, the \indigolog engine builds its physical reality $\Phi_{s}$ by taking the initial context from the environment, and the expected reality $\Psi_{s}$, which initially is equal to $\Phi_{s}$.

Once a task is ready for being assigned, a component named \textbf{\emph{Communication Manager}} is in charge of assigning it to a proper service (which may be a human actor, a robot, a software application, etc.) that is available (i.e., free from any other task assignment) and that provides all the required capabilities for task execution. Every step of the task life cycle - ranging from the assignment to the release of a task - requires an interaction between the Communication Manager and the devices. For each real world device, the Communication Manager generates a separate \textbf{\emph{Device Manager}}, which is a software component that is able to interact with the Task Handlers deployed on the devices. Each device manager establishes a communications channel with the associated device by using TCP/IP stream sockets. Such an interaction is mainly intended for notifying the corresponding device of actions performed by the \indigolog engine as well as for notifying the engine of actions executed by the Task Handlers of the corresponding device. Finally, the Communication Manager allows also to catch exogenous events coming from the environment and, when the process terminates, to notify the process completion to the devices (again, through the Device Managers).



The \textbf{\emph{Monitor}} component, which interacts continually with the \indigolog engine, is in charge of monitoring contextual data in order to identify changes, modifications or events which may affect process execution, and notify them to the Adaptation layer. Specifically, at each execution step - i.e., when the ending of a task or an exogenous event has turned $\Phi_{s}$ into $\Phi_{s+1}$ - the monitor checks if the new situation $s+1$ can be classified as relevant (cf. Equation~\ref{eq:eq_samestate}). If this is the case, the monitor collects the physical reality $\Phi_{s+1}$, the expected reality $\Psi_{s+1}$ and sends them to the synchronization component. In a nutshell, the Monitor component decides whether adaptation is needed.

The \textbf{Adaptation Layer} is in charge of reacting to undesired or unforeseen events which may invalidate process execution. The \textbf{\emph{Synchronization}} component acts as unique entry point for incoming notifications from the Execution Layer, in order to adapt process execution through specific recovery/adaptation techniques. It enforces synchronization between the \indigolog engine, the monitor and the planner. Every time it receives from the monitor the two realities, it builds a corresponding planning problem in~\textsc{PDDL} (through the \textbf{\emph{Problem Builder}} component), by converting the physical reality $\Phi_{s}$ into the initial state and the expected reality $\Psi_{s}$ into the goal.

The \textbf{\emph{Planner}} component is invoked when the Synchronization component builds a new planning problem. In addition to the initial state and the goal, the Planner needs a specification for the planning domain too (that is, a PDDL specification with tasks and predicates). For this purpose, the \textbf{\emph{Domain Builder}} component translates the domain theory defined in \smartML in a PDDL planning domain readable by the Planner. In the \smartpm system, we synthesize our recovery plans through the LPG-td planner~\cite{LPG} (Local search for Planning Graphs). It is a planner based on local search and planning graphs that handles PDDL 2.2 domains and can produce good quality plans in terms of one or more criteria. In order to invoke the planner, the Synchronization component  needs to specify the value of three parameters indicating: \myi a file containing a set of PDDL operators (i.e., the planning domain, provided by the Domain Builder), \myii a file containing a planning problem (i.e., the initial state and goal of the problem, provided by the Problem Builder) formalized in PDDL and \myiii a running mode, which is either ``speed'' (for devising sub-optimal solutions) and ``quality'' (for devising optimal solutions). More details about the planner are given in Section~\ref{sec:validation-experiments}.

Finally, when a plan satisfying the goal is found, it is sent back to a \textbf{\emph{Translator}} component, that converts it in a readable format for the \indigolog engine and passes it back to the Synchronization component. The Synchronization component combines the faulty process instance $\delta'$ with the recovery plan $\delta_a$ just built, and obtains the adapted process $\delta'' = (\delta_a;\delta')$ to be executed by the \indigolog engine. If the Planner is not able to find any recovery plan for a specific deviation, the control is given back to the process manager.

\section{The \indigolog Platform}
\label{sec:framework-indigolog_platform}

The \indigolog platform is a logic-programming implementation of \indigolog that allows the incremental execution of high-level Golog-like programs. Part of this section is a summary of the work published in~\cite{Indigolog:2009}, that we customized to our needs in collaboration with the creators of \indigolog. Although most of the code is written in vanilla Prolog, the overall architecture is written in the well-known open source \swiprolog\footnote{Available at \url{http://www.swi-prolog.org/}}~\cite{SWIPROLOG}. \swiprolog provides flexible mechanisms for interfacing with other programming languages such as Java or C, allows the development of multi-threaded applications, and provides support for socket communication and constraint solving.

Generally speaking, the \indigolog implementation provides an incremental interpreter of high-level programs as well as a framework for dealing with the real execution of these programs on concrete platforms or devices. This amounts to handling the real execution of actions on concrete devices (e.g., a PDA), the collection of sensing outcome information (e.g., retrieving some sensor's output), and the detection of exogenous events happening in the world. To that end, the architecture is modularly divided into six parts, namely, (i) the top-level main cycle; (ii) the language semantics; (iii) the temporal projector; (vi) the communication manager; (v) the set of device managers; and finally (vi) the domain application. The first four modules are completely domain independent, whereas the last two are designed for specific domain(s). The architecture is depicted in Fig.~\ref{fig:fig_framework-indigolog_architecture}.

\begin{figure}[t]
\centering{
 \includegraphics[width=0.98\columnwidth]{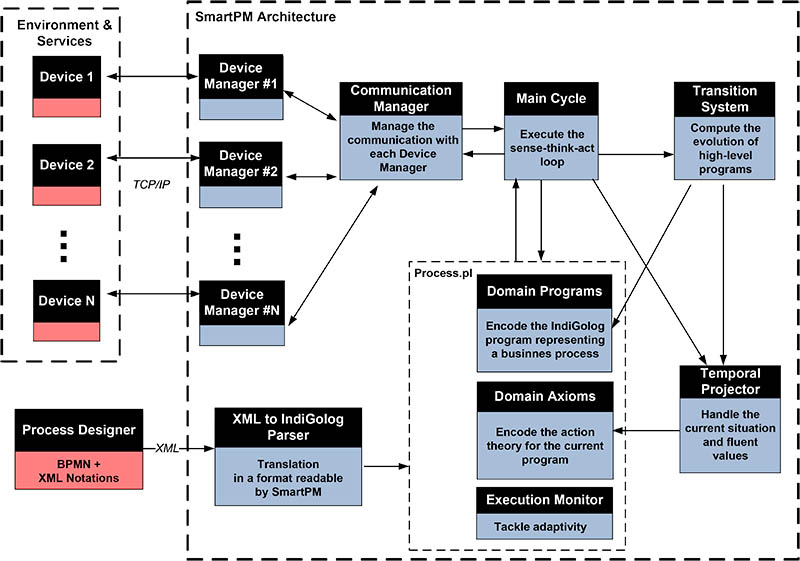}
 } \caption{The architecture of the \indigolog platform.}
 \label{fig:fig_framework-indigolog_architecture}
\end{figure}

\subsection{The Top-level Main Cycle and Language Semantics}
\label{sec:framework-indigolog_platform-topCycle}

The \indigolog\ platform codes the \textit{sense-think-act} loop well-known in the agent community \cite{kowalski1995}:
\begin{enumerate}[itemsep=1pt,parsep=1pt]
\item check for exogenous events that have occurred;
\item calculate the next program step; and
\item if the step involves an action, \textit{execute} the action.
\end{enumerate}
While executing actions, the platform keeps updated an history, which is the sequence of actions performed so far.

The main predicate of the main cycle is \texttt{indigo/2}; a goal of the form \texttt{indigo(E,H)} states that the high-level program \texttt{E} is to be executed online at history \texttt{H}.

The first thing the main cycle does is to assimilate all exogenous events that have occurred since the last execution step. After all exogenous actions have been assimilated and the history progressed as needed, the main cycle goes on to actual executing the high-level program \texttt{E}. First, if the current program to be executed is terminating in the current history, then the top-level goal \texttt{indigo/2} succeeds.
Otherwise, the interpreter checks whether the program can evolve a single step by relaying on predicate \texttt{trans/4} (explained below). If the program evolves without executing any action, then the history remains unchanged and we continue to execute the remaining program from the same history. If, however, the step involves performing an action, then this action is executed and incorporated into the current history, together with its sensing result (if any), before continuing the execution of the remaining program.

As mentioned above, the top-level loop relies on two central predicates, namely, \texttt{final/2} and \texttt{trans/4}. These predicates implement relations $Trans$ and $Final$, giving the single step semantics for each of the constructs in the language. It is convenient, however, to use an implementation of these predicates defined over histories instead of situations. Indeed, the constructs of the \indigolog\ interpreter never treat about situations but they are always assuming to work on the current situation.
So, for example, these are the corresponding clauses for sequence (represented
as a list), tests, nondeterministic choice of programs, and primitive actions:
\begin{verbatim}
final([E|L],H) :- final(E,H), final(L,H).
trans([E|L],H,E1,H1) :- final(E,H), trans(L,H,E1,H1).
trans([E|L],H,[E1|L],H1) :- trans(E,H,E1,H1).

final(ndet(E1,E2),H) :- final(E1,H) ; final(E2,H).
trans(ndet(E1,E2),H,E,H1) :- trans(E1,H,E,H1).
trans(ndet(E1,E2),H,E,H1) :- trans(E2,H,E,H1).

trans(?(P),H,[],H) :- eval(P,H,true).
trans(E,H,[],[E|H]) :- action(E), poss(E,P), eval(P,H,true).
/* Obs: no final/2 clauses for action and test programs */
\end{verbatim}
These Prolog\ clauses are almost directly ``lifted'' from the
corresponding axioms for $Trans$ and $Final$.
Predicates \texttt{action/1} and \texttt{poss/2} specify the actions of the
domain and their corresponding precondition axioms; both are defined in the
domain axiomatization (see below).
More importantly, \Prol{eval/3} is used to check the truth of a condition at
a certain history, and is provided by the temporal projector, described next.

The naive implementation of the search operator would deliberate from scratch
at every point of its incremental execution. It is clear, however, that one
could do better than that, and cache the successful plan obtained and avoid
planning in most cases:
\begin{verbatim}
final(search(E),H) :- final(E,H).
trans(search(E),H,path(E1,L),H1) :-
      trans(E,H,E1,H1), findpath(E1,H1,L).

/* findpath(E,H,L): solve (E,H) and store the path in list L  */
/* L = list of configurations (Ei,Hi) expected along the path */
findpath(E,H,[(E,H)]) :- final(E,H).
findpath(E,H,[(E,H)|L]) :- trans(E,H,E1,H1), findpath(E1,H1,L).

\end{verbatim}
So, when a search block is solved, the whole solution path found is
stored as the sequence of configurations that are expected. If the
actual configurations match, then steps are performed without any
reasoning (first \texttt{final/2} and \texttt{trans/4} clauses for
program \texttt{path(E,L)}). On the other hand, if the actual
configuration does not match the one expected next, for example,
because an exogenous action occurred and the history thus changed,
re-planning is performed to look for an alternative path (code not
shown).

\subsection{The Temporal Projector}
\label{sec:framework-indigolog_platform-temporal_projector}

The temporal projector is in charge of  maintaining the agent's beliefs about
the world and evaluating a formula relative to a history.
The projector module provides an implementation of predicate
\texttt{eval/3}: goal \texttt{eval(F,H,B)} states that formula
\texttt{F} has truth value \texttt{B}, usually \texttt{true} or
\texttt{false}, at history \texttt{H}.

Predicate \texttt{eval/3} is used to define \texttt{trans/4} and
\texttt{final/2}, as
 the legal evolutions of high-level programs may often depend on what things are
believed true or false.

We assume then that users provide definitions for each of the
following predicates for fluent $f$, action $a$, sensing result $r$,
formula $w$, and arbitrary value $v$:

\begin{description}
\item[\texttt{fun\_fluent(f)}]  \texttt{f} is a functional fluent;

\item[\texttt{rel\_fluent(f)}]  \texttt{f} is a functional fluent;

\item[\texttt{prim\_action(a)}]  \texttt{a} is a ground action;

\item[\texttt{init(f,v)}]  \texttt{v} is the value for fluent \texttt{f} in the starting situation;

\item[\texttt{poss(a,w)}]  it is possible to execute action \texttt{a} provided
formula \texttt{w} is known to be true;

\item[\texttt{causes\_val(a,f,v,w)}]  action \texttt{a} affects the value of
\texttt{f}
\end{description}

Formulas are represented in Prolog\ using the obvious names for the logical
operators and with all situations suppressed; histories are represented by lists
of the form $o(a,r)$ where $a$ represents an action and $r$ a sensing result.
We will not go over how formulas are recursively evaluated, but just
note that there exists a predicate \textit{(i)} $\kTrue(w,h)$ is the
main and top-level predicate and it tests if the formula $w$ is at
history $h$. Finally, the interface of the module is defined as
follows:

\begin{verbatim}
eval(F,H,true) :- kTrue(F,H).
eval(F,H,false) :- kTrue(neg(F),H).
\end{verbatim}

\subsection{The Communication Manager}
\label{subsec:framework-communication_manager}

The \indigolog system was meant for being used with concrete agent/robotic platforms, as well as with software/simulation environments. To this end, the online execution of \indigolog programs must be linked with the external world.
To this end, the \textit{communication manager} (CM) provides a complete interface with all the external devices, platforms, and real-world environments that the application needs to interact with.
In turn, each external device or platform that is expected to interact with the
application (e.g., a robot, a software module, or even a user interface) is assumed to
have a corresponding \textit{device manager}, a piece of software that is able to talk
to the actual device, instruct it to execute actions, as well as gather
information and events from it. The device manager understands the ``hardware''
of the corresponding device and provides a high-level interface to the CM.
It provides an interface for the execution of exogenous events (e.g., \aAssign, \aStart, etc.), the retrieval of sensing outcomes for actions, and the occurrence of exogenous events (e.g.,
\photoLost \ as well as \aFinish \ and \aReady).
Because actual devices are independent of the \indigolog\ application and
may be in remote locations, device managers are meant to run in different
processes and, possibly, in different machines; they communicate then with the
CM via TCP/IP sockets.
The CM, in contrasts, is part of the \indigolog agent architecture and is tightly coupled with the main cycle. Still, since the CM needs to be open to the external world regardless of any computation happening in the main cycle, the CM and the main cycle run in different (but interacting) threads, though in the same process and Prolog\ run-time engine.\footnote{\SWIProlog\ provides a clean and efficient way of programming multi-threaded Prolog\ applications.}
So, in a nutshell, the CM is responsible of executing actions in the real world and gathering information from it in the form of sensing outcome and exogenous events by communicating with the different device managers.

More concretely, given a domain high-level action (e.g.,
$\aAssign(Srvc,Id,Task,Inputs,ExOutputs)$), the CM is in charge of: \textit{(i)}
deciding which actual ``device'' should execute the action;
\textit{(ii)} ordering its execution by the device via its
corresponding device manager; and finally \textit{(iii)} collecting
the corresponding sensing outcome. To realize the execution of
actions, the CM provides an implementation of \texttt{exec/2} to the
top-level main cycle: \texttt{exec(A,S)} orders the execution of
action \texttt{A}, returning \texttt{S} as its sensing outcome.

When the system starts, the CM starts up all device managers required by the
 application and sets up communications channels to them using TCP/IP stream
 sockets. Recall that each real world device or environment has to have a
 corresponding device manager that understands it.

 After this initialization process, the CM enters into a \textit{passive mode}
 in which it asynchronously listens for messages arriving from the various
 devices managers. This passive mode should allow the top-level main cycle to
 execute without interruption until a message arrives from some device manager.

 In general, a message can be an exogenous event, a sensing outcome of some
 recently executed action, or a system message (e.g., a device being
 closed unexpectedly).  The incoming message should be read and handled in an
 appropriate way, and, in some cases, the top-level main cycle should be
 notified of the occurred event. 

\subsection{The Domain Application}
\label{sec:framework-indigolog_platform-domain_application}

From the user perspective, probably the most relevant aspect of the
architecture is the specification of the domain application.
Any domain application must provide:
\begin{enumerate}[itemsep=1pt,parsep=1pt]
\item An \emph{axiomatization of the dynamics of the world}. Such
axiomatization would depend on the temporal projector to be used.
\item One or more \textit{high-level agent programs} that will dictate the
different agent behaviors available. In general, these will be \indigolog\
 programs.
\item All the necessary \textit{execution information} to \textit{run} the
application in the external world. This amounts to specifying which external
devices the application relies on (e.g., the device manager for the ER1 robot),
and how high-level actions are actually executed on these devices (that is, by
which device each high-level action is to be executed).
Information on how to translate high-level symbolic actions and sensing results
into the device managers' low-level representations, and vice-versa, could also be provided.
\end{enumerate}

\section{The \smartpm Definition Tool}
\label{sec:approach-smartpm_definition_tool}

A process modeling language provides appropriate syntax and semantics to precisely specify business process requirements, in order to support automated process verification, validation, simulation and automation. One of the main obstacles in applying AI techniques to real problems is the difficulty to model the domains. Usually, this requires that people that have developed the AI system carry out the modeling phase since the representation depends very much on a deep knowledge of the internal working of the AI tools. On the contrary, during the realization of the \smartpm system we took care of the above aspect, and we worked on the definition of a language (named \smartML) that allow non-experts entering knowledge on processes through a user-friendly interface. Such knowledge, together with the business process to be enacted, is automatically translated in \sitcalc, \indigolog and PDDL.

In this section, we aim at describing the main components of the \smartpm Definition Tool. Specifically, we first present \smartML, our declarative language used for representing tasks and contextual data linked to a pervasive scenario, and \emph{Eclipse BPMN}, a graphical editor that allows to specify business processes using the BPMN 2.0 notation~\cite{BPMN20}. Then, we focus on the XML-to-IndiGolog Parser component, that translates a \smartML specification in a SitCalc Theory (which defines the initial situation and the set of data fluents and available actions with their pre- and post-conditions) and the BPMN control flow of the process in a \indigolog program corresponding to the process to be executed.

\subsection{The \smartML Modeling Language}
\label{sec:framework-smartpm_definition_tool-smartml}


The synthesis of a dynamic process requires a tight integration of process activities and contextual data in which the process is embedded in. The context is represented in the form of a \emph{Domain Theory} \textsf{D}, that involves capturing a set of tasks $t_i \in \textsf{T}$ (with $i \in 1..n$) and supporting information, such as the people/agents that may be involved in performing the process (roles or participants), the data and so forth.

Tasks are collected in a specific repository, and each task can be considered as a single step that consumes input data and produces output data. Data are represented through some ground atomic terms $v_{\emph{1}}[y_\emph{1}],v_{\emph{2}}[y_\emph{2}],...,v_{m}[y_m] \in \textsf{V}$ that range over a set of tuples (i.e., unordered sets of zero or more attributes) $y_\emph{1}, y_\emph{2},\allowbreak\dotsc\,y_m$ of \emph{data objects}, defined over some \emph{data types}. In short, a data object depicts an entity of interest. Some data types are pre-specified and used for representing the \emph{resource perspective} of the process.

For example, in our scenario we need to define data objects for representing participants (e.g., data type $Participant=\{act\emph{1},\,act\emph{2},\,act\emph{3},\,act\emph{4},\,rb\emph{1},\,rb\emph{2}\}$) and capabilities (e.g., data type $Capability=\{extinguisher,movement,\allowbreak\dotsc\,hatchet\}$).

\begin{footnotesize}
\begin{alltt}
\textbf{\underline{Resource Perspective :}}

Participant = \emph{\{act1,act2,act3,act4,rb1,rb2\}}
Capability = \emph{\{movement,hatchet,camera,gprs,extinguisher,battery,digger,powerpack\}}
\end{alltt}
\end{footnotesize}

The data types $Participant$ and $Capability$ are already pre-defined for being used in the framework, and the Process Designer is only required to provide values (i.e., to associate data objects) to the above types.

On the contrary, other data types need to be defined for describing the contextual scenario in which the process will be embedded. In our example, we may need of a data type $Location\_type = \{loc\emph{00},\,loc\emph{10},\,\allowbreak\dotsc\,\,loc\emph{33}\}$) for representing locations in the area. Moreover, a data type $Status\_type = \{ok,fire,debris\}$ may be defined for describing if a specific location is on fire or buried by debris.

\begin{footnotesize}
\begin{alltt}
\textbf{\underline{User-Defined Data Types :}}

Location_type = \emph{\{loc00,loc10,loc20,loc30,loc01,loc11,loc02,loc03,loc13,loc23,loc31,\\loc32,loc33\}}
Status_type = \emph{\{ok,fire,debris\}}
\end{alltt}
\end{footnotesize}

Under this representation, we consider possible values of a data type as constant symbols that univocally identify data objects in the scenario of interest. Each tuple $y_j$ may contain one or more data objects belonging to different data types. The domain $dom(v_{j}[y_j])$ over which a term is interpreted can be of various types:

\begin{itemize}[itemsep=1pt,parsep=1pt]
\item Boolean\_type: $dom(v_{j}[y_j])$ = \{$true,false$\};
\item Integer\_type: $dom(v_{j}[y_j]) = \{x...y\} \ s.t. \ x,y \in \mathbb{Z} \land (x <= y)$;
\item Functional: the domain contains a fixed number of data objects of a designated type.
\end{itemize}

The data type $Boolean\_type$ (and its respective ``objects'' \emph{true} and \emph{false}) and the type $Integer\_type$ are already pre-specified. However, since integer numbers form a countably infinite set, there is the need to set a lower and an upper bound to specify which finite subset of integers is relevant for the case to deal with. In the below example, we are considering the subset of the integers from 0 to 30.

\begin{footnotesize}
\begin{alltt}
\textbf{\underline{Pre-Defined Data Types :}}

Boolean_type = \emph{\{true,false\}}
Integer_type = \emph{\{0,1,2,3,4,...,30\}}
\end{alltt}
\end{footnotesize}

Terms can be used to express properties of domain objects (and relations over objects), and argument types of a term - taken from the set of data types previously defined\footnote{Predefined data types, like \emph{Boolean\_type} and \emph{Integer\_type}, can not be used as arguments of an atomic term.} - represent the finite domains over which the term is interpreted. Again, some terms are pre-specified for being used in the framework. For example, since each task has to be assigned to a participant that provides all of the skills required for executing that task, there is the need to consider the participants ``capabilities''. This can be done through a boolean term $provides[prt:Participant,cap:Capability]$ that is $true$ if the capability $cap$ is provided by $prt$ and $false$ otherwise. At the same way, a boolean term $requires[task:Task,cap:Capability]$ is needed for specifying which capabilities are required for executing a specific task\footnote{The special data type \emph{Task} will be defined later and can be used only as argument of the pre-defined term \emph{requires}.}.

\begin{footnotesize}
\begin{alltt}
\textbf{\underline{Pre-Defined Terms :}}

provides[prt:Participant,cap:Capability] = (bool:Boolean\_type)
requires[task:Task,cap:Capability] = (bool:Boolean\_type)
\end{alltt}
\end{footnotesize}

Moreover, we may need boolean terms for indicating if people have been evacuated from a location (e.g., $evacuated[loc:Location\_type] = (bool:Boolean\_type)$), integer terms for representing the battery charge level of each robot (e.g., $batteryLevel[prt:Participant] = (int:Integer\_type)$) or for indicating the number of pictures taken in a specific location (e.g., $photoTaken[loc:Location\_type] = (int:Integer\_type)$), and functional terms for recording the position of each actor (e.g., $at[prt:Participant] = (loc:Location\_type)$) and robot (e.g., $atRobot[prt:Participant] = (loc:Location\_type)$) in the area or for representing the updated situation of each location (e.g., $status[loc:Location\_type] = (st:Status\_type)$). Some terms may be used as \emph{constant values}, and in this case the set of arguments taken as input by the single term is empty. For example, the term $generalBattery[] = (int:Integer\_type)$ reflects the battery charge level stored in the power pack and used for recharging the battery of each robot (the term $batteryRecharging[] = (int:Integer\_type)$ indicates the amount of battery that is charged after each recharging action). The terms $moveStep[] = (int:Integer\_type)$ and $debrisStep[] = (int:Integer\_type)$ reflects the amount of battery consumed respectively after the robot has been moved from a location to another one and after having removed debris from a specific location.

Finally, terms can also be used for expressing static relations over objects. The term $neigh[loc\emph{1}:Location\_type,loc\emph{2}:Location\_type] = (bool:Boolean\_type)$ indicates all adjacent locations in the area (for example, $neigh[loc00,loc01] = true$), while the term $covered[loc:Location\_type]  = (bool:Boolean\_type)$ reflects the locations covered by the network provided by the fixed antenna.

For each term, the process designer has to decide which ones are \textbf{relevant for adaptation} and which ones have not to be considered for that. A term that is considered as relevant for adaptation will be continuously monitored by the PMS, and if its value becomes different from the one expected after a task execution, the PMS will provide adaptation features.

\begin{footnotesize}
\begin{alltt}

\textbf{\underline{Atomic Terms :}}

\textbf{\underline{Relevant for Adaptation :}}

\emph{at[prt:Participant] = (loc:Location_type)}
\emph{evacuated[loc:Location_type] = (bool:Boolean\_type)}
\emph{status[loc:Location_type] = (st:Status_type)}

\textbf{\underline{Not Relevant for Adaptation :}}

\emph{atRobot[prt:Participant] = (loc:Location_type)}
\emph{batteryLevel[prt:Participant] = (int:Integer\_type)}
\emph{photoTaken[loc:Location_type] = (int:Integer\_type)}
\emph{generalBattery[] = (int:Integer\_type)}
\emph{batteryRecharging[] = (int:Integer\_type)}
\emph{moveStep[] = (int:Integer\_type)}
\emph{debrisStep[] = (int:Integer\_type)}
\emph{neigh[loc1:Location\_type,loc2:Location\_type] = (bool:Boolean\_type)}
\emph{covered[loc:Location\_type]  = (bool:Boolean\_type)}
\end{alltt}
\end{footnotesize}

In addition to atomic terms, we allow the designer to define \emph{complex terms}. They are declared as basic atomic terms, with the additional specification of a well-formed first-order formula $\varphi$ that determines the truth value for the term. In our case study, we may need to express that an actor is connected to the network if s/he is in a covered location or if s/he is in a location adjacent to a location where a robot is located (and is thus connected through the robot):

\begin{footnotesize}
\begin{alltt}
isConnected(prt:Participant) \{
\quad \text{EXISTS}(l1:Location\_type,\,l2:Location\_type,\,rbt:Participant).((at(act)=l1) \text{AND}
\quad (Covered(l1) \, \text{OR} \, (atRobot(rbt)=l2 \, \text{AND} \, Neigh(l1,l2) \text{AND} 
\quad isRobotConnected(rbt))))\}.
\end{alltt}
\end{footnotesize}

The interpretation of complex terms derives from the corresponding first-order formula and is enacted at run-time by the \indigolog interpreter. Formulae can be negated (NOT) and existentially or universally quantified (EXISTS and FORALL). A complex term can not appear in task effects and can not involve recursion. On the contrary, complex terms can be used within a task precondition. The description of the complex term $isRobotConnected$ is given in the Appendix.

Concerning the definition of process tasks, the process designer is required to specify which tasks are applicable to the dynamic scenario under study. Those tasks will be stored in a specific \emph{tasks repository}, and can be used for composing the control flow of the process (cf. Section~\ref{sec:framework-smartpm_definition_tool-BPMN}) and for adaptation purposes.

\begin{footnotesize}
\begin{alltt}
\textbf{\underline{Tasks Repository :}}

Tasks = \emph{\{go, move, takephoto, evacuate, updatestatus, extinguishfire, chargebattery\}}
\end{alltt}
\end{footnotesize}

Each task is annotated with \emph{preconditions} and \emph{effects}. Preconditions are logical constraints defined as a conjunction of 
atomic terms, and they can be used to constrain the task assignment and must be satisfied before the task is applied, while effects establish the outcome of a task after its execution.
\begin{mydef}
\label{task_def_smartml}
A \emph{task} $t[x] \in$ \emph{\textsf{T}} is a tuple $t = (Act_t,x,Pre_t,Eff_t)$ that
consists of:
\begin{itemize}[itemsep=1pt,parsep=1pt]
\item the name $Act_t$ of the action involved in the enactment of the task (it often coincides with the task itself);
\item a tuple of data objects $x$ as input parameters;
\item a set of preconditions $Pre_t$, represented as the conjunction of $k$ atomic conditions defined over some specific terms, $Pre_t = \bigwedge_{l \in 1..k} pre_{t_{l}}$. Each $pre_{t_{l}}$ can be represented as \{$v_j[y_j] \ \textbf{op} \ \textbf{expr}$\}, where:
    \begin{itemize}
    \item $v_j[y_j] \in \emph{\textsf{V}}$ is an atomic term, with $y_j \subseteq x$, i.e., admissible data objects for $y_j$ need to be defined as task input parameters;
    \item An $\textbf{expr}$ can be a \underline{boolean value} (if $v_j$ is a boolean term); an \underline{input} \underline{parameter} identified by a data object (if $v_j$ is a functional term); an \underline{integer number} or an \underline{expression} involving integer numbers and/or terms, combined with the arithmetic operators \{+,-\} (if $v_j$ is a integer term);
    \item The condition $\textbf{op}$ can be expressed as the equality ($==$) between boolean terms or functional terms and an admissible \textbf{expr}. On the contrary, if $v_j$ is a integer term, it is possible to define the $\textbf{op}$ condition as an expression that make use of relational binary comparison operators ($<, >, =, \leq, \geq$) and involve integer numbers and/or integer terms in the \textbf{expr} field.
    \end{itemize}
\item a set of deterministic effects $Eff_t$, represented as the conjunction of $h$ atomic conditions defined over some specific terms, $Eff_t = \bigwedge_{l \in 1..h} eff_{t_{l}}$. Each $eff_{t_{l}}$ ($with \ l \in 1..h$) can be represented as \{$v_j[y_j] \ \textbf{op} \ \textbf{expr}$\}, where:
    \begin{itemize}
    \item $v_j[y_j] \in \emph{\textsf{V}}$ and $\textbf{expr}$ are defined as for preconditions.
    \item The condition $\textbf{op}$ may include assignment expressions to update the values of integer terms. A numeric effect consists of an assignment operator, the integer term to be updated and a integer number or a numeric expression. Assignment operators include \myi direct assignment ($=$), to assign to a integer term a value defined by an integer number; \myii relative assignments, which can be used to increase ($+$=) or decrease ($-$=) the value of a integer term (additive assignment effects). 
    \end{itemize}
\end{itemize}
\end{mydef}

Note that if no preconditions are specified, then the task is always executable. Moreover, the process designer is required to make explicit if a task effect can be considered as \emph{supposed} or \emph{automatic}. A \emph{supposed effect} indicates that the participant executing the task has to physically return an outcome for the supposed effect, that can be or not can be equal to the one declared during the task definition. If a supposed effect involves a relevant term, it is clear that the outcome returned assumes a great value for monitoring purposes. Otherwise, an effect can be flagged as \emph{automatic}, meaning that when the task terminates, the effect is automatically applied without the need to consider the task outcomes.

As an example, the task \aGo \ involves two input parameters $from$ and $to$ of type $Location\_type$, representing the starting and arrival locations. The participant that will be assigned for executing the task $Go$ is indicated with $PRT$. An instance of this task can be executed only if $PRT$ is currently at the starting location $from$ and provides the required capabilities for executing the task $go$. As a consequence of task execution, the actor moves from the starting to the arrival location, and this is reflected by assigning to the functional term $at[actor]$ the value $to$ in the effect. The fact that the effect of the task \aGo \ is $supposed$ \ means that after the execution of the task the participant will return the real outcome indicating her/his final position. Each task $t_i \in \textsf{T}$ together with its preconditions, effects and parameters can be represented as an XML annotation:


\begin{footnotesize}
\begin{alltt}
\textbf{\underline{Description of the task \aGo :}}

<task>
    <name>go</name>
    <parameters>
        <arg>from - Location_type</arg>
        <arg>to - Location_type</arg>
    </parameters>
    <precondition>at[PRT] == from AND isConnected[PRT] == true</precondition>
    <effects>
        <supposed>at[PRT] = to</supposed>
    </effects>
</task>
\end{alltt}
\end{footnotesize}

Let us observe now the task \aMove \. Like the task \aGo \, it involves two input parameters $from$ and $to$ of type $Location\_type$, that represent the starting and arrival locations given in input to the robot. We want to underline that the participant $PRT$ which will be assigned to execute a \aMove \ task will be for sure a robot. This can be inferred by analyzing the pre-defined terms \emph{provides} and \emph{requires}, that will be used by the PMS for understanding if a specific participant is capable to execute a specific task.

\begin{footnotesize}
\begin{alltt}
\textbf{\underline{Description of the task \aGo :}}

<task>
    <name>move</name>
    <parameters>
        <arg>from - Location_type</arg>
        <arg>to - Location_type</arg>
    </parameters>
    <precondition>atRobot[PRT] == from AND batteryLevel[PRT] >= moveStep[] AND
                  isRobotConnected[PRT] == true
    </precondition>
    <effects>
        <supposed>atRobot[PRT] = to</supposed>
        <automatic>batteryLevel[PRT] -= moveStep[]</automatic>
    </effects>
</task>
\end{alltt}
\end{footnotesize}

An instance of the task \aMove \ can be executed only if the robot $PRT$ is currently at the starting location $from$ and provides enough battery charge for executing the movement. As a consequence of task execution, the robot moves from the starting to the arrival location, and this is reflected by assigning to the functional term $atRobot[PRT]$ the value $to$ in the effect. This first effect of \aMove \ has been flagged as $supposed$, meaning that after the execution of the task the robot will return the real outcome indicating its final position. However, since $atRobot$ is not considered as a relevant term (see above), if the final position of the robot will differ with the one declared at design time, no adaptation is required. The second effect of \aMove \ is \emph{automatic}, and states that after the execution of a \aMove \ task the battery level of the robot has to be decreased of a fixed quantity, corresponding to moveStep[].

We can also represents exogenous events with \smartML. They reflect possible external events coming from the environment that can modify asynchronously atomic terms at run-time. In our case study, we can deal with three different exogenous events:

\begin{footnotesize}
\begin{alltt}
\textbf{\underline{Exogenous Events :}}

Ex\_events = \emph{\{\photoLost, \fireRisk, \rockSlide\}}
\end{alltt}
\end{footnotesize}

The definition of an exogenous event is similar to the classical task definition provided by \smartML. However, for defining an exogenous event there is no need to specify any precondition, and effects can only be considered as \emph{automatic} (i.e., they are automatically applied to the involved terms when the exogenous event is catched). For example, the exogenous event $\rockSlide(loc)$ alerts about a rock slide collapsed in location $loc$, and its effect concerns to modify the value of the atomic term $[loc]$ to the value '\emph{debris}'.

\begin{footnotesize}
\begin{alltt}
\textbf{\underline{Description of the exogenous event \rockSlide :}}

<ex-event>
    <name>rockSlide</name>
    <parameters>
        <arg>loc - Location_type</arg>
    </parameters>
    <effects>
        <automatic>status[loc] = debris</automatic>
    </effects>
</ex-event>
\end{alltt}
\end{footnotesize}

Finally, we want to underline that the \smartML syntax allows to represent planning domains and problems with the complexity of those describable in PDDL version 2.2~\cite{PDDL22}, that is characterized for enabling the representation of realistic planning domains. The translation algorithms are described in Section~\ref{sec:framework-builders}.

\subsection{Defining Processes in \smartpm through BPMN}
\label{sec:framework-smartpm_definition_tool-BPMN}



Starting from a domain theory $\textsf{D}$, the control flow of a dynamic process in \smartpm can be defined through the BPMN notation. The Business Process Modeling Notation (BPMN) was released to the public in May 2004 by the BPMI Notation Working Group and was adopted as OMG standard\footnote{\url{http://www.omg.org/spec/BPMN/2.0/}} for business process modeling in February 2006. BPMN provides a graphical notation for specifying business processes based on a flowcharting technique similar to activity diagrams from UML. In January 2001, the version 2.0 of the language has been released~\cite{BPMN20}. If compared to the previous specifications of the language, which provided only verbal descriptions of the graphic notations elements and modeling rules, BPMN 2.0 received a formal definition in the form of a meta-model, that defines the abstract syntax and semantics of the modeling constructs. The BPMN specification also provides a mapping between the graphic elements of the notation and the underlying constructs of BPEL (Business Process Execution Language)~\cite{BPMNtoBPEL}.

Formally, in \smartpm we define a dynamic process as as a directed graph consisting of tasks, gateways, events and transitions between them.
\begin{mydef}
Given a domain theory \emph{\textsf{D}} and a set of tasks \emph{\textsf{T}}, a \emph{dynamic process} $\emph{\textsf{P}}$ is a tuple (N,L) where:
\begin{itemize}[itemsep=1pt,parsep=1pt]
\item N  = $T \cup E \cup W \cup X$ is a finite set of nodes, such that :
\begin{itemize}[itemsep=1pt,parsep=1pt]
\item T is a set of task instances, i.e., occurrences of a specific task $t \in \emph{\textsf{T}}$ in the range of the dynamic process; 
\item E is a finite set of events, that consists of a single start event $\ocircle$ and a single end event $\odot$;
\item W = $W_{PS} \cup W_{PJ}$ is a finite set of parallel gateways, represented in the control flow with the $\diamond$ shape with a ``plus'' marker inside.
\item X = $X_{ES} \cup X_{EJ}$ is a finite set of exclusive gateways, represented in the control flow with the $\diamond$ shape with a ``X'' marker inside.
\end{itemize}
\item L = $L_T \cup L_E \cup L_{W_{PS}} \cup L_{W_{PJ}} \cup L_{X_{ES}} \cup L_{X_{EJ}}$ is a finite set of transitions connecting events, task instances and gateways:
\begin{itemize}[itemsep=1pt,parsep=1pt]
\item $L_T : T \rightarrow (T \cup W_{PS} \cup W_{PJ} \cup X_{ES} \cup X_{EJ} \cup {\odot})$
\item $L_E : \ocircle \rightarrow (T \cup W_{PS} \cup X_{ES} \cup {\odot})$
\item $L_{W_{PS}} : W_{PS} \rightarrow 2^T$
\item $L_{W_{PJ}} : W_{PJ} \rightarrow (T \cup W_{PS} \cup X_{ES} \cup {\odot})$
\item $L_{X_{ES}} : X_{ES} \rightarrow 2^T$
\item $L_{X_{EJ}} : X_{EJ} \rightarrow (T \cup X_{ES} \cup W_{PS} \cup {\odot})$
\end{itemize}
\end{itemize}
\end{mydef}

Note that the constructs used for defining a dynamic process are basically a subset of the ones definable through the BPMN notation. The intuitive meaning of these structures should be clear: an execution of the process starts at $\ocircle$ and ends at $\odot$; a \emph{task} is an atomic activity executed by the process; \emph{parallel splits} $W_{PS}$ open parallel parts of the process, whereas \emph{parallel joins} $W_{PJ}$ re-unite parallel branches; exclusive gateways are used to create alternative flows in a process where only one of the path can be taken on the basis of  a given condition. \emph{Transitions} are binary relations describing in which order the flow objects (tasks, events and gateways) have to be performed, and determine the \emph{control flow} of the dynamic process.

In the \smartpm system, we use the Eclipse BPMN editor\footnote{\url{http://eclipse.org/bpmn2-modeler/}} for defining the control flow of a dynamic process. The BPMN editor provides visual, graphical editing and creation of BPMN 2.0 business processes. We have customized the tool for our needs, by adding the possibility to link domain theories written through \smartML to the control flow of the process.

Once a process is ready for being executed (i.e., the process designer has completed the definition of the domain theory and of the process control flow), the last step before sending the dynamic process to the \indigolog engine for its enactment consists in instantiating the domain theory with a \emph{starting condition}, which reflects different assignment of values to the atomic terms. We assume complete information about the starting condition, since the logical framework described in Chapter~\ref{ch:approach} works under the \emph{closed-world assumption}~\cite{Reiter@NMR1987}. Basically, this means we force the process designer to instantiate every atomic term with an admissible value that represent what is known in the starting state about the dynamic scenario. Specifically, the starting condition is a conjunction $\{v_{\emph{1}}[y_\emph{1}] == val_\emph{1} \land v_{\emph{2}}[y_\emph{2}] = val_\emph{2}... \land v_{j}[y_j] == val_j\}$, where $val_j$ (with $j \in 1..m$) represents the j-th value assigned to the j-th atomic term.

\subsection{The XML-to-IndiGolog Parser}
\label{sec:framework-smartpm_definition_tool-xml-to-indigolog}

From a technical point of view, a dynamic process built through Eclipse BPMN and annotated with \smartML is basically a file saved in the XML format that is taken as input from the XML-to-IndiGolog Parser component, which translates the process specification in situation calculus and \indigolog readable formats, in order to make the process executable by the \indigolog engine.

We show the rules used for translating the \smartML specification of a dynamic process in a SitCalc Theory:
\begin{itemize}[itemsep=1pt,parsep=1pt]
\item Each data type defined in \smartML, together with its associated data objects, is converted in a corresponding situation calculus predicate. For example, the pre-defined data type \emph{Participant} corresponds to a predicate \emph{Service} in situation calculus. The same is valid for the user-defined data types $Location\_type$ and $Status\_type$, which are converted in the corresponding situation calculus versions, and for the all the tasks stored in the \smartML tasks repository. A situation calculus predicate $Task$ indicates which tasks will be considered by the PMS for the process execution and adaptation (if required).
\item The \smartML pre-defined terms, such as $provides$ and $requires$, have again a corresponding translation as situation calculus predicates.
\item Each atomic term will result in a situation calculus data fluent. For all those terms that are flagged as \emph{relevant}, an expected fluent is also provided. For example, the \emph{non relevant} term $atRobot[prt:Participant] = (loc:Location\_type)$ will be translated in a situation calculus data fluent $atRobot_{\varphi}(prt,s)$, where $prt$ is the participant (i.e., the robot) associated to the fluent and $s$ is a situation term. A non relevant data fluent is not monitored for adaptation purposes. On the contrary, the \emph{relevant} term $at[prt:Participant] = (loc:Location\_type)$ will be converted in a data fluent $at_{\varphi}(prt,s)$ and in an expected fluent $at_{\Psi}(prt,s)$, that will respectively record the physical and expected positions of $prt$ in every situation $s$, and are continuously monitored by the PMS.
\item Tasks definitions are crucial for expressing situation calculus preconditions axioms and successor state axioms for each fluent. Let us consider, for example, the definition of the task \aGo.

    \begin{footnotesize}
    \begin{alltt}
    <task>
        <name>go</name>
        <parameters>
            <arg>from - Location_type</arg>
            <arg>to - Location_type</arg>
        </parameters>
        <precondition>at[PRT] == from AND
                      isConnected[PRT] == true
        </precondition>
        <effects>
        <supposed>at[PRT] = to</supposed>
        </effects>
    </task>
    \end{alltt}
    \end{footnotesize}

First of all, the parser analyzes the structure of the XML element \texttt{<task>} for defining which admissible workitems including the task \aGo \ can be executed by the \indigolog PMS. By analyzing the structure of \aGo, it is clear that it will require two inputs (specifically, two \emph{Location\_type} elements) and it will return only an outcome, that is again a \emph{Location\_type} element. This will led the parser to generate the following code:

\begin{equation*}
\begin{array}{l}
\forall \ (from,to,id) \ s.t. \ Identifier(id) \ \land \ Location\_type(from) \land \\
\qquad Location\_type(to) \Rightarrow workitem(\aGo,id,[from,to],[to]).
\end{array}
\end{equation*}

Then, the parser builds the precondition axiom for the task \aGo \ by scanning the content of the XML element \texttt{<precondition>}:

\begin{equation*}
\begin{array}{l}
\Poss(\aAssign(prt,id,\aGo,[from,to],[to]),s) \Leftrightarrow \\
\quad service(prt) \ \land \ workitem(\aGo,id,[from,to],[to])\ \land \\
\quad \fAt_{\varphi}(prt,s) = from\ \land \isConnected(prt,s).
\end{array}
\end{equation*}
Basically, the precondition axiom for the task \aGo \ states that \aGo \ can be assigned to a service \emph{prt} in situation $s$ if and only if \emph{prt} is located in $from$ and is connected to the network. Moreover, the precondition axiom verifies that by associating at run-time a whatever identifier to the task \aGo \ with inputs \emph{from} and \emph{to} and expected effect equal to \emph{to}, it is possible to define an admissible workitem.

Finally, the parser analyzes all the elements nested in the XML element \texttt{<effects>}. For each \emph{supposed} effect, if the atomic term involved in the effect has been flagged as \emph{relevant}, then the parser generates a data fluent for capturing the physical outcome of the task and an expected fluent for recording the expected value after task execution. In the case of the task \aGo, the following fluents will be generated:
\begin{equation*}
\begin{array}{l}
\fAt_\varphi(prt,do(a,s)) = loc \equiv{} \\
\quad \big(\exists \ loc \ s.t. \ Location\_type(loc) \ \land \\
\quad a=\aRelease(prt,id,\aGo,[from,to],[to],[loc])\big) \ \lor{} \\
\quad\big(\fAt_\varphi(prt,s) = loc\ \land \\
\quad \neg \exists \ loc' \ s.t. \ Location\_type(loc') \ \land\\
\quad a=\aRelease(prt,id,\aGo,[from,to],[to],[loc']) \land (loc' \neq loc)  \big).
\end{array}
\end{equation*}
\begin{equation*}
\begin{array}{l}
\fAt_\Psi(prt,do(a,s)) = to \equiv{} \\
\quad \big(\exists \ loc \ s.t. \ Location\_type(loc) \ \land \\
\quad a=\aRelease(prt,id,\aGo,[from,to],[to],[loc])\big) \ \lor{} \\
\quad\big(\fAt_\Psi(prt,s) = to\ \land \\
\quad \neg \exists \ to',loc \ s.t. \ Location\_type(loc) \ \land \ Location\_type(to') \ \land \\
\quad a=\aRelease(prt,id,\aGo,[from,to],[to'],[loc]) \land (to' \neq to)  \big).
\end{array}
\end{equation*}
When $prt$ terminates the execution of the task \aGo, its final position $loc$ in situation $s$ is stored in a specific data fluent $\fAt_\varphi$, while the expected fluent $\fAt_\psi$ will assume the desired value \emph{to} without considering the real outcome of the task.

If a task description includes \emph{automatic effects} or \emph{supposed effects} involving atomic terms flagged as \emph{non relevant}, the parser will generate the successor state axioms only for the data fluents associated to the terms. For example, the definition of the task \aMove \ (that instructs a robot to move from a position to another) involves these two kinds of effects (cf. Section~\ref{sec:framework-smartpm_definition_tool-smartml}), meaning that the parser will generate the successor state axioms for $\fAtRobot_\varphi$ (cf. Equation~\ref{eq:eq_atRobot}) and for $\fBatteryLevel_\varphi$ (cf. Equation~\ref{eq:eq_BatteryLevel}).
\item Each exogenous event formalized through \smartML will be translated in a corresponding situation calculus exogenous event. Specifically, the parser analyzes the effects provided by the execution of the exogenous event and applies it to the involved data fluent. Remember that exogenous events can only change values of data fluents, by leaving the expected reality untouched. Sometimes the effects of an exogenous event affect a data fluent for which there exists another task whose execution can modify the value of the same fluent. In such a case, the parser needs to combine the effects provided by an exogenous event with the effects returned after a task execution, by generating a single successor state axiom for the data fluent involved. For example, let us consider the \smartML definition of the task \aTakePhoto \ and of the exogenous event \photoLost:

    \begin{footnotesize}
    \begin{alltt}
    <task>
        <name>takephoto</name>
        <parameters>
            <arg>loc - Location_type</arg>
        </parameters>
        <precondition>at[PRT] == loc AND isConnected[PRT] == true</precondition>
        <effects>
            <supposed>photoTaken[loc] = true</supposed>
        </effects>
    </task>

    <ex-event>
        <name>photoLost</name>
        <parameters>
        <arg>loc - Location_type</arg>
        </parameters>
        <effects>
        <automatic>photoTaken[loc] = false</automatic>
        </effects>
    </ex-event>
    \end{alltt}
    \end{footnotesize}
    The task \aTakePhoto \ is used for instructing a service (specifically, a human actor) $PRT$ to move in a location \emph{loc} in order to take some pictures. When $PRT$ terminates the execution of the task \aTakePhoto, the physical outcome of the task is stored in the atomic term $photoTaken[loc]$. The same happens when an exogenous event \photoLost \ is captured by the PMS, but in this last case the atomic term is asynchronously switched to the \emph{false} value. The parser, which has already generated a data fluent $\fPhotoTaken_{\varphi}$, is in charge to associate a single successor state axiom that reflect every possibility of changing the fluent:
    \begin{equation*}
    \begin{array}{l}
    \fPhotoTaken_\varphi(loc,do(a,s)) = true \equiv{} \\
    \quad \big(\exists \ q_j \ s.t. \ a=\aRelease(prt,id,\aTakePhoto,[loc],[true],[q_j]) \ \land (q_j = true)\big) \lor{} \\
    \quad\big(\fPhotoTaken_\varphi(loc,s) = true \ \land \\
    \quad \big((a\neq\photoLost(loc)) \ \lor \\
    \quad \neg \exists \ q_j' \ s.t. \ (a=\aRelease(prt,id,\aTakePhoto,[loc],[true],[q_j']) \land (q_j' = false))\big).
    \end{array}
    \end{equation*}

    Consider also that the effect of \aTakePhoto \ is a \emph{supposed} one, and the involved term $photoTaken$ has been flagged as $relevant$. Therefore, the parser generates an expected fluent $\fPhotoTaken_{\Psi}$ where recording the desired outcome of \aTakePhoto.

    \begin{equation*}
    \begin{array}{l}
    \fPhotoTaken_\Psi(loc,do(a,s)) = true \equiv{} \\
    \quad \exists \ q_j \ s.t. \ a=\aRelease(prt,id,\aTakePhoto,[loc],[true],[q_j]) \lor{} \\
    \quad\big(\fPhotoTaken_\Psi(loc,s) = true \ \land \\
    \quad \big(\neg \exists \ q_j' \ s.t. \ (a=\aRelease(prt,id,\aTakePhoto,[loc],[false],[q_j'])\big).
    \end{array}
    \end{equation*}

    Clearly the successor state axiom of $\fPhotoTaken_\Psi$ is not affected by exogenous events, since they are though to modify only the physical reality.
    \item Complex terms are basically first-order formula evaluated on the atomic terms. They are directly translated in situation calculus abbreviations.
    \item The initial condition used for instantiating any value of atomic terms in the domain theory corresponds with the initial situation $S_0$ in \indigolog.
\end{itemize}

The XML-to-IndiGolog Parser is also able to convert the control flow of a dynamic process, defined through the BPMN language, in a valid \indigolog based process. The conversion is straightforward, and concerns to map every possible BPMN construct used for formalizing processes in \smartpm (cf. Section \ref{sec:framework-smartpm_definition_tool-BPMN} for the list of admitted constructs) in a valid \indigolog construct, as shown in Table~\ref{tab:indigolog_constructs}. We omit here the full translation algorithm, but we refer the reader to~\cite{SMARTPM2011} for a similar complete mapping between \indigolog and WS-BPEL.

\section{Building the Planning Domain and the Planning Problem}
\label{sec:framework-builders}

In order to exploit our planning-based recovery mechanism, every task/annotation/property associated to a dynamic process needs to be translated in PDDL. A PDDL definition consists of two parts: the domain and the problem definition. The planning domain is built when a dynamic process is ready to be executed, i.e., the process designer has completed the definition of the control flow of the dynamic process with an associated domain theory represented in the \smartML language (cf. Section~\ref{sec:framework-smartpm_definition_tool-smartml}). Specifically, the \smartML domain theory is given as input to a software module named \emph{Domain Builder}, that is in charge to convert such specification in a planning domain that complies with the PDDL 2.2~\cite{PDDL22} language, characterized for enabling the representation of realistic planning domains.

Basically, the Domain Builder starts analyzing the definition of atomic/complex terms and data types as formalized in the previous sections, and by making explicit the \emph{actions} associated to each annotated task stored in the repository linked to the dynamic process under execution, together with the associated pre-conditions, effects and input parameters. Basically, the planning domain describes how terms may vary after a task execution, and reflects the contextual properties constraining the execution of tasks stored in a specific tasks repository.

In the following, we discuss how a domain theory defined through \smartML can be translated into a PDDL file representing the planning domain:
\begin{itemize}[itemsep=1pt,parsep=1pt]
\item each data type corresponds to an \emph{object type} in the planning domain;
\item boolean terms and complex terms have a straightforward representation as \emph{relational predicates} and \emph{derived predicates} (to model the dependency of given facts from other facts) in the planning domain;
\item integer terms correspond to PDDL \emph{numeric fluents}, and are used for modeling non-boolean resources (e.g., the battery level of a robot) in the planning domain;
\item functional terms do not have a direct representation in PDDL 2.2, but may be replaced as relational predicates. Since an object function $f:Object^n \to Object$ map tuples of objects with domain types $D^n$ to objects with co-domain type $U$, it may be coded in the planning domain as a relational predicate \emph{P} of type $(D^n,U)$;
\item a given task, together with the associated pre-conditions and effects and input parameters, is translated in a PDDL \emph{action schema}. An action schema describes how the relational predicates and/or numeric fluents may vary after the action execution. In the following, it is shown the PDDL representation of the task \aGo:

    \begin{footnotesize}\begin{alltt}
    (:action go
     :parameters (?x - service ?from - location_type ?to - location_type)
     :precondition (and (provides ?x movement) (free ?x)
                        (at ?x ?from) (isConnected ?x))
     :effect (and (not (at ?x ?from)) (at ?x ?to))
    )
    \end{alltt}\end{footnotesize}

    This task can be executed only if the participant denoted with \emph{x} is free and is currently located in his/her starting location \emph{from} and is connected to the network. The desired effect turns the value of the predicate $at(x,to)$ to \emph{true} and $at(x,from)$ to \emph{false}, meaning the actor moved in a new location. Let us note that respect the definition of the task \aGo \ given with \smartML, the PDDL version of the task requires to explicitly specify in the preconditions the capabilities required by a generic service $x$ for executing the specific task and the information about the availability of the service. Moreover, the list of arguments is augmented with the information related to the service $x$ that will execute the task. These information are crucial for allowing the planner to schedule correctly the task during the building of the plan.
\end{itemize}

When something relevant happens during the process execution, a new planning problem is built at run-time on the same planning domain defined above, through the description of an initial state and of a desired goal. As thoroughly discussed in Chapter~\ref{ch:approach}, process adaptation is required when it is sensed a misalignment between $\Phi(s)$ (the ``wrong'' physical reality) and $\Psi(s)$ (the ``safe'' expected reality). The \emph{Problem Builder} takes as input from the Synchronizer the two realities and converts them in a PDDL planning problem. Specifically:
\begin{itemize}[itemsep=1pt,parsep=1pt]
\item for each data type defined in the SitCalc Theory, all the possible object instances of that particular data type are explicitly instantiated as \emph{constant symbols} in the planning problem (e.g., the fact that $act1$, $act2$, $act3$, $act4$, $rb1$ and $rb2$ are \emph{Services}, $loc00$, ..., $loc33$ are \emph{Location\_type}, etc.);
\item Basically, the initial state of the planning problem corresponds to the physical reality that need to be adjusted, plus some information for helping the planner to build correctly the recovery procedure. The initial state is composed by the conjunction of all data fluents defined in the SitCalc Theory (that correspond to PDDL relational predicates or numeric fluents in the planning domain), and the information concerning which services are possibly free for executing the tasks in the recovery plan (i.e., if $s$ is the situation where the deviation has been sensed, for each service $srv$ we need to make explicit which fluent $\fFree(srv,s)$ holds). Moreover, for each service $srv$ and capability $c$ defined in the SitCalc theory, it is required the list of the predicates $provides(srv,c)$ equal to \emph{true}.
\item the \emph{goal state} of the planning problem is a logical expression over facts. In our approach, the goal state is built in order to reflect a safe state to be reached after the execution of a recovery procedure. Therefore, it is composed by the conjunction of all the expected fluents (and their corresponding values) in situation $s$.
   \end{itemize}

Appendix~\ref{appendix_A} shows the complete PDDL code that describes a planning domain and a planning problem for solving the exception depicted in Fig.~\ref{fig:fig_introduction-case_study-context_2} and~\ref{fig:fig_introduction-case_study-context_3}.

Finally, we want to briefly introduce the role of the \emph{Translator} component. Basically, it waits for a plan to be synthesized and, when the planner produces the recovery plan, the Translator converts it in a \indigolog procedure that is executable by the \indigolog PMS. The translation is straightforward, since the planner returns a sequence of workitems $w_1,..,w_n$ (i.e., the tasks to be executed together with their input/outputs and the information about the services that will execute them), corresponding to an \indigolog procedure $\delta_a=(\delta_{a_{1}};...;\delta_{a_{n}})$. This procedure is sent to the Synchronizer component, that will build the adapted process and sends it back to the PMS.



\section{\smartpm in Action}
\label{sec:framework-screenshots}

In this section, we show some screenshot of the \smartpm system while executing the process defined in our case study described in Section~\ref{sec:introduction-case_study}. In Fig.~\ref{fig:eclipse_bpmn} it is shown the main window of the Eclipse BPMN modeler, that we use for building the control flow of our dynamic processes.

\begin{figure}[t]
\centering{
 \includegraphics[width=0.95\columnwidth]{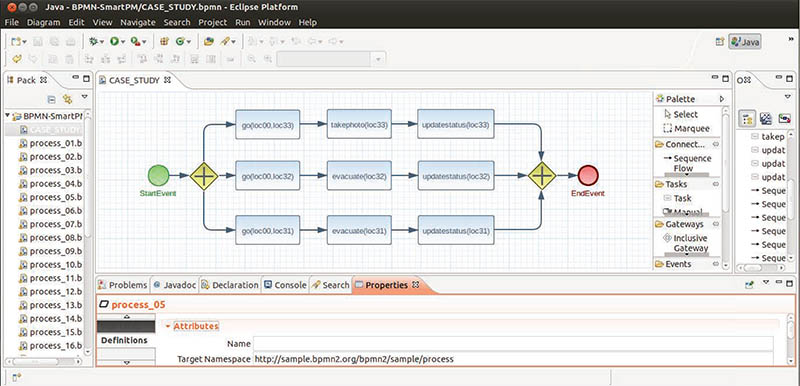}
 } \caption{The Eclipse BPMN modeler.}
 \label{fig:eclipse_bpmn}
\end{figure}

A Java module named the \emph{SmartML Editor} (cf. Fig.~\ref{fig:smartml_editor}(a)) is used for building domain theories based on the \smartML specification. The editor allows to load previously saved specifications to be attached to a BPMN process, and to build new specifications customized on the basis of the designer needs. An \emph{Edit} menu (cf. Fig.~\ref{fig:smartml_editor}(b)) provides some guided masks for driving the user in the definition of the various objects of the specification. Basically, the editor drives the user inputs in order to facilitate the creation of a \smartML valid domain theory. The editor integrates also a \emph{syntax checker} that verifies on the fly if the current specification is compliant with the \smartML formal one defined in Section~\ref{sec:framework-smartpm_definition_tool-smartml}.

\begin{figure}[h]
\centering{
 \includegraphics[width=0.91\columnwidth]{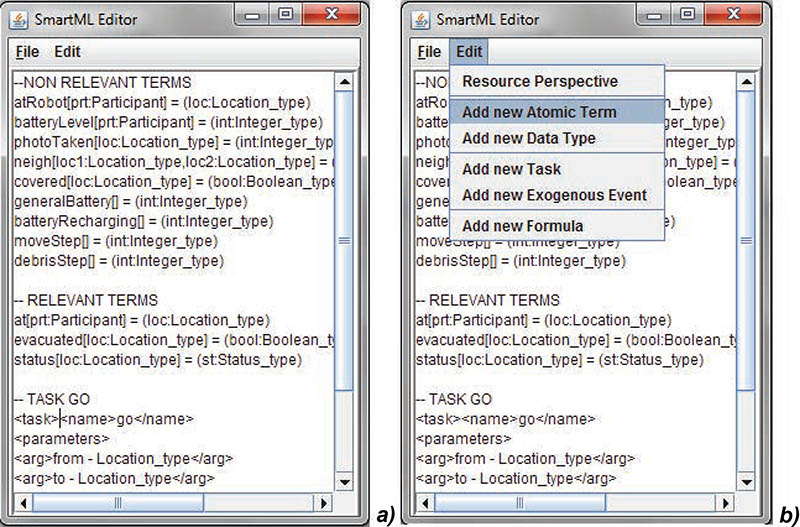}
 } \caption{The \smartML editor (a) and the menu for inserting new terms/tasks/objects/formulae to the specification (b).}
 \label{fig:smartml_editor}
\end{figure}

Once a process is ready for being executed, a new window is opened for allowing the process designer to instantiate the starting configuration for the atomic terms in the domain theory and for binding actual input/output parameters with data objects provided with \smartML. The \smartML specification plus the BPMN process is now passed to the XML-to-IndiGolog parser component, that build an \indigolog program and passes it to the \indigolog PMS for the execution. In parallel, the Domain Builder component builds the PDDL planning domain to be used for a possible future adaptation of the process.

\begin{figure}[h]
\centering{
 \includegraphics[width=0.99\columnwidth]{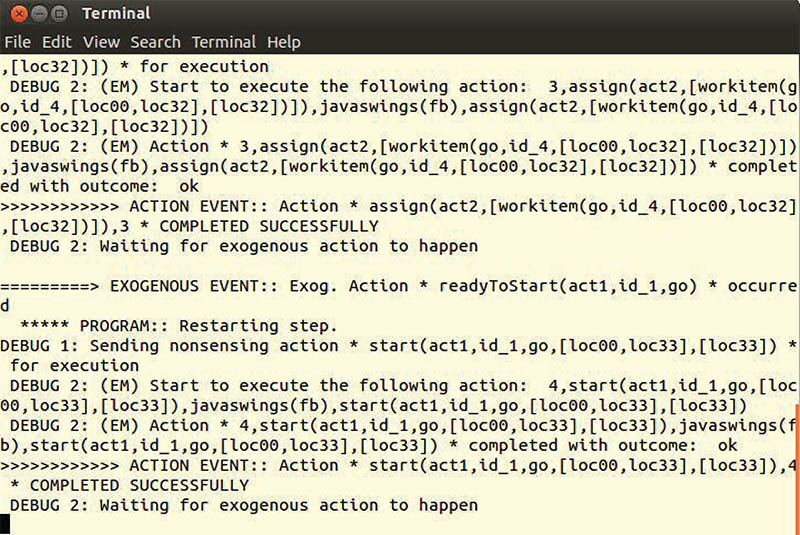}
 } \caption{The main window of the \indigolog PMS}
 \label{fig:indigolog}
\end{figure}

Fig.~\ref{fig:indigolog} depicts the main window of the \indigolog PMS showing the log of all actions exchanged between the PMS and services. In the screenshot, we can identify an \aAssign \ action that assigns to service $act2$ the task $\aGo(id\_4,[loc00,loc32],[loc32])$.

In Fig.~\ref{fig:dev_manager} we show some screenshots representing our Task Handler for executing tasks. The Task Handler is implemented in Java, by the use of standard Java 2D graphical libraries. Specifically, when the \indigolog PMS assigns a task to a service, this event is notified to the selected service through a popup window (cf. Fig.~\ref{fig:dev_manager}(a)). When a service is ready to start a task, it pushes the button \emph{Start It}, and a \aReady \ action is sent back to the \indigolog PMS. In Fig.~\ref{fig:indigolog} we can also see the presence of a row \texttt{========> EXOGENOUS EVENT}, that represents the fact that the PMS has captured the \aReady \ action sent by the service. We have to underline that the actions sent by services to PMS (i.e., the \aReady \ and the \aFinish \ actions) are recognized by the \indigolog \ PMS as ``good'' exogenous events. In response, the \indigolog PMS can command the service to start the task execution through a \aStart \ action (cf. the bottom part of Fig.~\ref{fig:indigolog}).

Now, let us suppose that service \emph{act1} has been instructed to start the task $\aGo(id\_1,[loc00,loc33],[loc33])$. Service $act1$ can choose one of the valid outcomes (i.e., a list of $Location\_type$ data objects) and pushes the button \emph{End Task} when the task is completed (cf. Fig.~\ref{fig:dev_manager}(b)). If the outcome provided by $act1$ is different from the one expected (meaning it reaches a different location respect the one desired), the \indigolog PMS senses the deviation, builds a planning problem that reflects the gap between physical and expected reality and launches the LPG-TD planner (cf. Fig.~\ref{fig:planner_screen}), which is in charge to synthesize the recovery plan. Note that the deviation we are analyzing is the same shown in our case study; i.e., we are supposing that \emph{act\emph{1}} has reached location \emph{loc03} rather than location \emph{loc33} (cf. Fig.~\ref{fig:fig_introduction-case_study-context_2}(b)).

\begin{figure}[t]
\centering{
 \includegraphics[width=0.9\columnwidth]{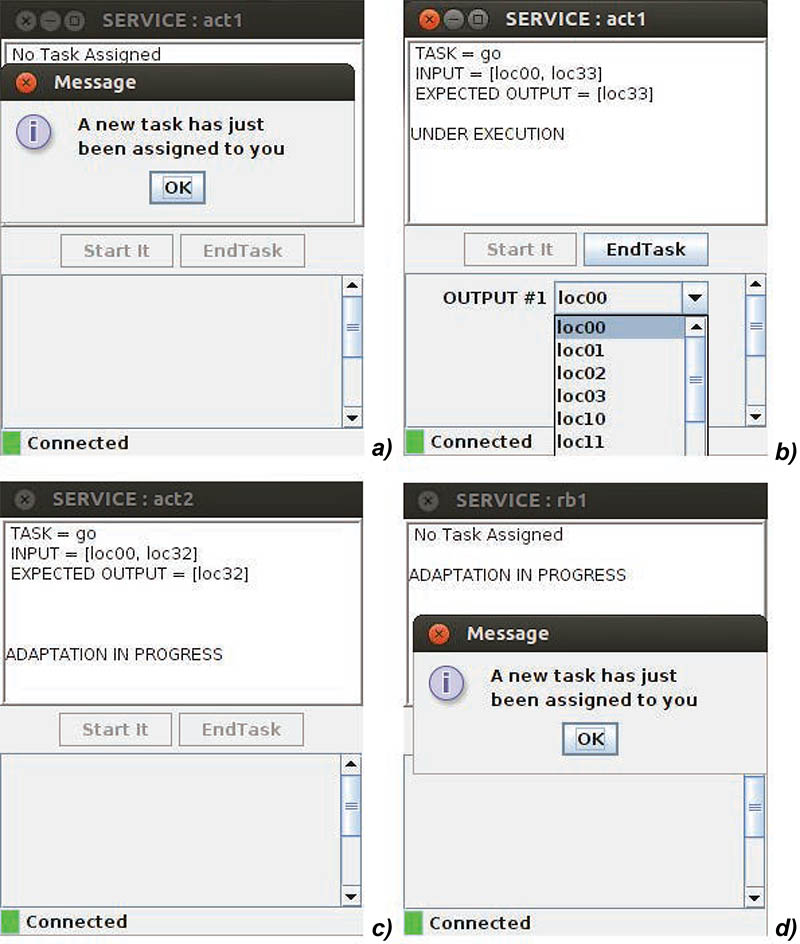}
 } \caption{The Work-list Handler of \smartpm.}
 \label{fig:dev_manager}
\end{figure}

During the synthesis of the recovery plan, every running task is interrupted for the time required to the planner for building the plan and to the \indigolog PMS for executing it (cf. Fig.~\ref{fig:dev_manager}(c)). When the planner finds a recovery procedure, it is passed back to the Synchronization component, that converts it in a executable \indigolog process and sends it to the \indigolog PMS for its enactment.

Since all the actors/robots need to be continually inter-connected to execute the process, the Planner finds a recovery procedure that first instructs the robots to move in specific positions for maintaining the network connection, and then re-assigns the task \emph{go(loc03,loc33)} to \emph{act1} (cf. Fig.~\ref{fig:fig_introduction-case_study-context_3}). In Fig.~\ref{fig:dev_manager}(d) it is shown the assignment of a recovery task to the service $rb1$. Specifically, service $rb1$ executes the first task of the recovery procedure dealing with the deviation (cf. Fig.~\ref{fig:fig_introduction-case_study-context_3}(a)).

\bigskip

We conclude this chapter by pointing out that the \smartpm System is a real working proof-of-concept prototype system that implements our approach to dynamic process adaptation shown in Section~\ref{ch:approach}. While other approaches and systems rely on pre-defines rules to specify the exact behaviors when special events are triggered, here we simply model (a subset
of) the running environment and the actions' effects, without considering any possible exceptional event. We argue that, in most of cases, modeling the environment, even in detail, is easier than modeling all possible exceptions.

\begin{figure}[t]
\centering{
 \includegraphics[width=0.95\columnwidth]{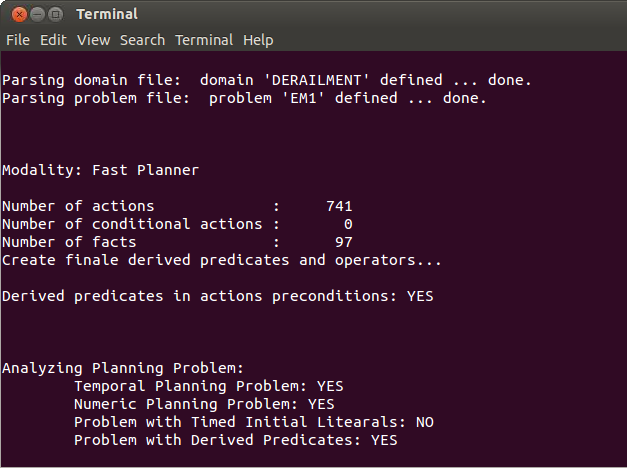}
 } \caption{The main window of the LPG-td planner.}
 \label{fig:planner_screen}
\end{figure}

\chapter{Validation}
\label{ch:validation}

This chapter reports on performance evaluation and system validation activities. Specifically, in Section~\ref{sec:validation-experiments} we first report on experimental evaluation results, in terms of time needed for automatically adapting the dynamic process taken from our case study when exceptions of growing complexity arise. Then, in Section~\ref{validation-effectiveness} we measure the effectiveness of \smartpm in finding recovery procedures by simulating the execution of thousands of processes instances having different structures of the control-flows.


\section{Performances of \smartpm in Computing Recovery Procedures}
\label{sec:validation-experiments}

In order to investigate the feasibility of the \smartpm approach, we performed some testing to learn the time amount needed for synthesizing a recovery plan for different adaptation problems. We made our tests by using the LPG-td planner\footnote{LPG-td was awarded at the 4th International Planning Competition\footnote{\url{http://ipc.icaps-conference.org/}} (IPC 2004) as the ``top performer in plan quality''.}~\cite{LPG}. Such a planner is based on a stochastic local search in the space of particular ``action graphs'' derived from the planning problem specification. The basic search scheme of LPG-td is inspired to Walksat~\cite{WALKSAT}, an efficient procedure for solving SAT-problems. More details on the search algorithm and heuristics devised for this planner can be found at~\cite{LPG,LPGth}.

We chose LPG-td as (i) it treats the full range of PDDL2.2\footnote{PDDLv2.2~\cite{PDDL22} is characterized for enabling the representation of realistic planning domains, which include (in particular) actions and goals involving numerical expressions, operators with universally quantified effects or existentially quantified preconditions, operators with disjunctive or implicative preconditions, derived predicates and plan metrics. However, currently, our formalism does not allow to represent conditional and universally quantified effects.}~\cite{PDDL22} and (ii) even if it is primarily thought as a satisficing planner, it is able to compute also quality plans under a pre-specified metric. In fact, LPG-td has been developed in two versions: a version tailored to computation speed, named LPG-td.speed, which produces \emph{sub-optimal plans}, and a version tailored for \emph{plan quality}, named LPG-td.quality. LPG-td.speed generates sub-optimal solutions that do not prove any guarantee other than the correctness of the solution. LPG-td.quality differs from LPG-td.speed basically for the fact that it does not stop when the first plan is found but continues until a stopping criterion is met. In our experiments, the optimization criteria was fixed as the minimum number of actions needed for the planner to reach the goal. Therefore, a quality plan uses the smallest number of actions needed for reaching a goal state.
It is important to underline that satisficing planning is easy (polynomial), while optimal planning is hard (NP-complete)~\cite{HELMERT}.

\begin {table}[tp]\scriptsize
\caption {Time performances of LPG-td for adaptation problems of growing complexity.}
\label{smartpm_validation_table}
\centering
\begin{tabular}{ccccccc}
\hline
\ Length of the \ & \ Problem \ & \ Avg. time needed for a \ & \ Avg. length of a \ & \ Avg. time needed for \\
\ recovery proc. \ & \ instances \ & \ sub-optimal sol. (sec) \ & \ sub-optimal sol.  \ & \ a quality sol. (sec)\\
\hline
\hline
1&29&6,769&3&7,768\\
2&36&7,213&3&16,865\\
3&32&7,846&4&24,123\\
4&25&8,128&5&37,017\\
5&21&8,598&8&39,484\\
6&17&8,736&9&52,421\\
7&13&9,188&13&73,526\\
8&12&9,953&14&81,414\\
\hline
\end{tabular}
\end{table}\setlength{\textfloatsep}{0.6cm}

The experimental setup was performed with the test case shown in our running example. We stored in the task repository 20 different emergency management tasks, annotated with 28 relational predicates, 2 derived predicates and 4 numeric fluents, in order to make the planner search space very challenging. Then, we provided 185 different planning problems of different complexity, by manipulating ad-hoc the values of the initial state and the goal in order to devise adaptation problems of growing complexity.

As shown in Table~\ref{smartpm_validation_table}, the column labeled as ``Length of the recovery procedure'' indicates the smallest number of actions needed for devising a plan of a specific length. Our purpose was to measure (in seconds) the computation time needed for finding a sub-optimal solution and a quality solution for problems that require a recovery procedure of growing complexity. The column labeled as ``Average length of a sub-optimal solution'' indicates the average number of actions that compose a sub-optimal solution for a problem of a given complexity. A sub-optimal solution is found in less time than a quality one, but generally it includes more tasks than the ones strictly needed. This means that when the complexity of the recovery procedure grows, the quality of a sub-optimal solution decreases. For example, as shown in table~\ref{smartpm_validation_table}, on 21 different planning problems requiring a recovery procedure of length 5, the LPG-td planner is able to find, on average, a sub-optimal plan in 8,598 seconds (with 3 more tasks, on average) and a quality plan (which consists exactly of the 5 tasks needed for the recovery) in 39,484 seconds, without the need of any domain expert intervention. Consequently, the approach is feasible for medium-sized dynamic processes used in practice\footnote{We did our tests by using an Intel U7300 CPU 1.30GHz Dual Core, 4GB RAM machine.}.


\section{Effectiveness of \smartpm in Adapting Processes}
\label{validation-effectiveness}

When developing a PMS with automatic adaptation features, an important aspect that needs to be analyzed concerns the \emph{effectiveness} of the PMS in executing process instances that have different structures. We define effectiveness as the \emph{ability of the PMS of executing every task included in the process control flow by adapting automatically the process itself if some failure arises, without the need of any manual intervention of the process designer at run-time}.

In order to measure the effectiveness of the \smartpm system in executing dynamic processes, we simulated the execution of 3600 process instances having different structure. Starting from tasks repositories with growing sizes, we generated process instances composed by sequences of tasks or by parallel branches. The simulation has concerned the testing of:
\begin{itemize}[itemsep=1pt,parsep=1pt]
\item 1200 process instances with a control flow composed by a sequence of tasks;
\item 1200 process instances with a control flow composed by 3 parallel branches with tasks to be executed in concurrency;
\item 1200 process instances with a control flow composed by 5 parallel branches with tasks to be executed in concurrency.
\end{itemize}
Our purpose was to measure the success rate in executing correctly the above process instances, by automatically adapting them if required. To this end, we developed a software module named the \smartpm \emph{Simulator}, which is able to build automatically \indigolog processes and SitCalc theories and simulating their execution on the basis of some customizable parameters:
\begin{itemize}[itemsep=1pt,parsep=1pt]
\item \textbf{Tasks repository size}: We allowed the generation of tasks repositories storing respectively 25, 50 or 75 tasks.
\item \textbf{Number of available services}: This parameter indicates the number of services available in the initial situation for tasks assignment. We fixed this value to 5 services.
\item \textbf{Maximum number of conditions in tasks preconditions/effects}: We allowed the generation of tasks having a maximum of 5 logical conditions in the precondition axioms, and a maximum of 5 different outcomes as effects.
\item \textbf{Number of available fluents}: We allowed the generation of SitCalc theories composed by 50 relevant data fluents (meaning that, for each data fluent, the \smartpm Simulator automatically builds the corresponding expected fluent for monitoring possible tasks failures). The generated fluents can assume only boolean values.
\item \textbf{Percentage of failures}: This parameter may assume two possible values (30\% or 70\%), and affects the percentage of tasks error during the process execution. For example, if a process instance is composed by 10 tasks, and the percentage of failures is equal to 70\%, this means that 7 tasks will terminate with some outcomes different from the ones expected, and process adaptation is required.
\item \textbf{Percentage of capabilities coverage}: This parameter can assume two possible values (30\% or 70\%), and affects the ability of each available service to execute the tasks stored in the repository. We assumed that each task is associated with a unique capability; therefore, if - for example - the tasks repository stores 75 tasks and the percentage of capabilities coverage is equal to 30\%, this means that each available service will be able to execute at least 22 tasks in the repository. Moreover, we assume that for each task in the control flow there exists at least one service that is able to execute that task.
\item \textbf{Number of tasks in the process control-flow}: Given a specific structure of the control flow, a fixed percentage of capabilities coverage and a fixed percentage of failures, for each possible size of the tasks repository the \smartpm \emph{Simulator} generated 100 process instances with control flows composed respectively by 5, 8, 10, 12, 14, 15, 18, 20, 22, 25 tasks. Tasks composing the control flow are randomly picked from the tasks repository. For control flows with $n$ concurrent branches, tasks are organized in a way that allows every branch to contain at least one task.
\end{itemize}


Test results are shown respectively in Figures~\ref{fig:fig_validation-sequence_5},~\ref{fig:fig_validation-parallel_3} and~\ref{fig:fig_validation-parallel_5}. Each figure reports a diagram related to the execution of 1200 process instances having \myi a control flow composed by a sequence of processes (cf. Fig.~\ref{fig:fig_validation-sequence_5}) \myii a control flow composed by 3 parallel branches (cf. Fig.~\ref{fig:fig_validation-parallel_3}) and \myiii a control flow composed by 5 parallel branches (cf. Fig.~\ref{fig:fig_validation-parallel_5}). For each diagram, on the y-axis it is reported the size of the tasks repository, and on the x-axis is indicated the success rate (i.e., the effectiveness) in executing process instances with specific characteristics.

For example, let us consider the diagram in Fig.~\ref{fig:fig_validation-sequence_5}, that shows the effectiveness of \smartpm in executing processes with a structure composed by a sequence of tasks. When the size of the tasks repository is fixed to 25 tasks, the 4 colored bars denote respectively the effectiveness of the system in executing 400 process instances (each bar reflects the enactment of 100 instances) with a \emph{percentage of capabilities coverage} and a \emph{percentage of failures} varying from 30\% to 70\%. For example, the blue bar represents the effectiveness of the system in executing 100 process instances (with a variable dimension of the number of tasks composing the control flow, see above) with a percentage of failures of 30\% and a percentage of capabilities coverage equal to 30\%. According to the results shown in the diagram, the percentage of success is equal to 80\%, meaning that 80 processes between the 100 that was executed have been correctly enacted and terminated, whereas for the other 20\% the system did not found any recovery plan for dealing with the exception, meaning that a manual support to adaptation was required. The other bars, instead, indicate the effectiveness of the system by increasing the percentage of failures and the percentage of capabilities coverage. The analysis of the performed tests put in evidence some interesting aspects, that are valid for each of the 3 process structures we tested and for every possible size of the tasks repository:
\begin{itemize}[itemsep=1pt,parsep=1pt]
\item When the percentage of failures increases, the effectiveness of the \smartpm system decreases.
\item Given a fixed percentage of failures, to increase the percentage of capabilities coverage means that the effectiveness of the \smartpm system increases, because there are more possibilities for a task in the repository to be selected by an available service.
\item In general, to have a large tasks repository helps to increase the effectiveness of the \smartpm system, since the planner has more possibilities to select tasks for building a recovery procedure.
\end{itemize}

Moreover, we also compared the obtained data with the purpose to understand if the structure of process instances may affect the effectiveness of \smartpm. To this end, for each possible combination of the percentage of failures with the percentage of capabilities coverage, we reorganized the collected data by generating 4 further diagrams that show the effectiveness of \smartpm after executing process instances whose control flows are composed by sequences of tasks or by parallel branches. The test results are shown respectively in Figures~\ref{fig:fig_validation_compare_1},~\ref{fig:fig_validation_compare_2} and~\ref{fig:fig_validation_compare_3} and~\ref{fig:fig_validation_compare_4}. For example, in Fig.~\ref{fig:fig_validation_compare_2} we are comparing the execution of process instances with a percentage of failures equal to 30\% and a percentage of capabilities coverage fixed to 70\%. For instance, when the size of the tasks repository is equal to 75, the effectiveness of \smartpm in executing 100 process instances composed by a sequence of tasks is equal to 85\%. The effectiveness decreases if the instances have tasks organized in 3 parallel branches (84\%) and in 5 parallel branches (79\%). In general, if the process control flow contains parallel branches, the effectiveness of the \smartpm system decreases as the number of parallel branches increases, since possibly more services are involved at the same time for tasks execution, by letting only few services available for process adaptation.

To sum, the execution of 3600 process instances with different structures was a valid test for measuring the effectiveness of \smartpm, that was able to complete correctly (i.e., without any intervention of domain experts at run-time) 2537 process instances, corresponding to an effectiveness of about 70,5\%. Finally, we want to underline that the effectiveness of a PMS depends also by the ability of the process designer in formalizing the domain theory associated to a dynamic process. For example, by testing our case study (cf. Section~\ref{sec:validation-experiments}) and by generating randomly an elevate number of possible failures, the planner was always able to find a proper recovery procedure and to adapt the process.



\begin{sidewaysfigure}
\centering{
 \includegraphics[width=0.95\columnwidth]{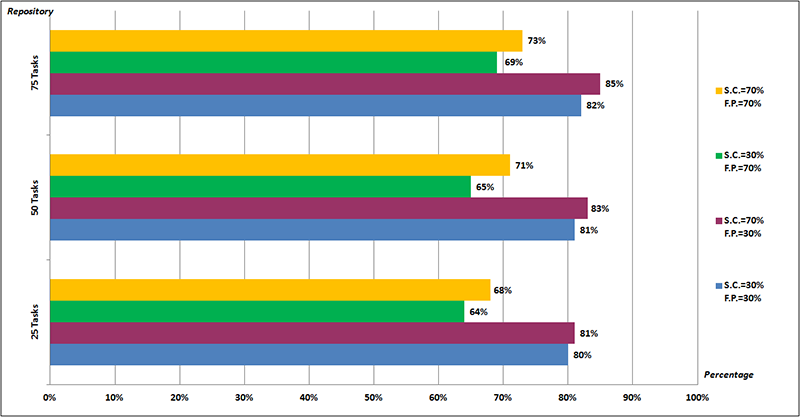}
 } \caption{Effectiveness of \smartpm when executing processes composed by a sequence of tasks.}
 \label{fig:fig_validation-sequence_5}
\end{sidewaysfigure}

\begin{sidewaysfigure}
\centering{
 \includegraphics[width=0.95\columnwidth]{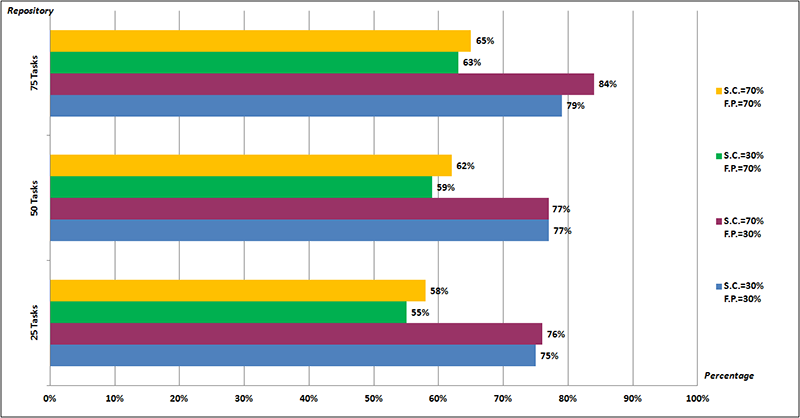}
 } \caption{Effectiveness of \smartpm when executing processes composed by 3 parallel branches.}
 \label{fig:fig_validation-parallel_3}
\end{sidewaysfigure}

\begin{sidewaysfigure}
\centering{
 \includegraphics[width=0.95\columnwidth]{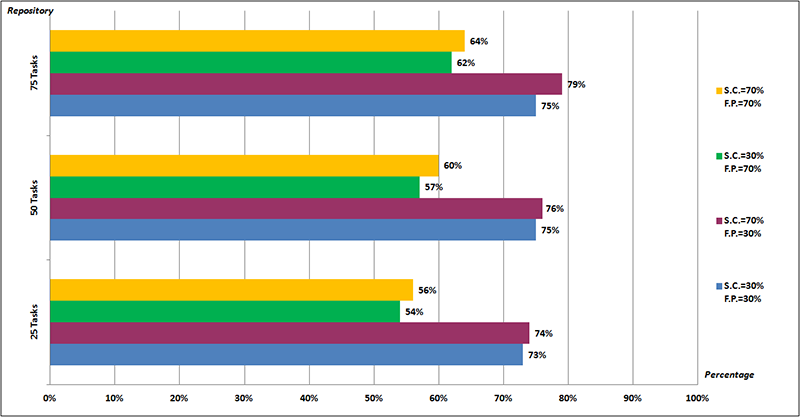}
 } \caption{Effectiveness of \smartpm when executing processes composed by 5 parallel branches.}
 \label{fig:fig_validation-parallel_5}
\end{sidewaysfigure}

\begin{figure}[t]
\centering{
 \includegraphics[width=0.8\columnwidth]{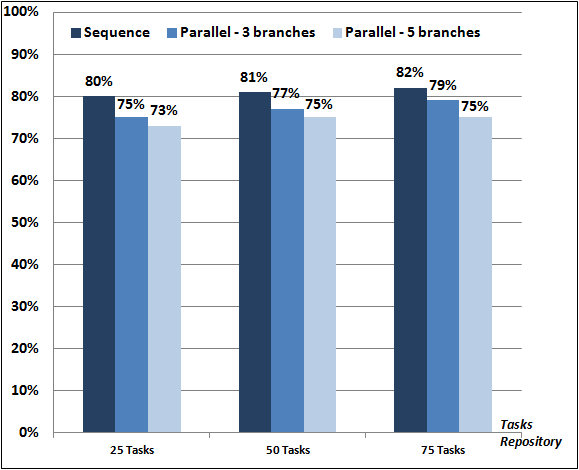}
 } \caption{Measure of the \smartpm effectiveness by comparing process instances having different structures, with a percentage of capabilities coverage equal to 30\% and a percentage of failures equal to 30\%.}
 \label{fig:fig_validation_compare_1}
\end{figure}

\begin{figure}[t]
\centering{
 \includegraphics[width=0.8\columnwidth]{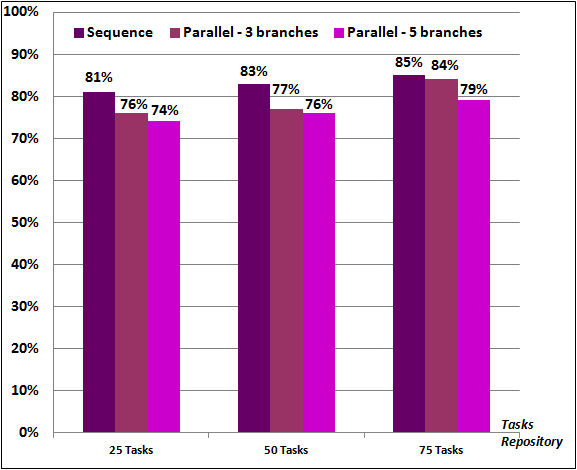}
 } \caption{Measure of the \smartpm effectiveness by comparing process instances having different structures, with a percentage of capabilities coverage equal to 70\% and a percentage of failures equal to 30\%.}
 \label{fig:fig_validation_compare_2}
\end{figure}

\begin{figure}[t]
\centering{
 \includegraphics[width=0.8\columnwidth]{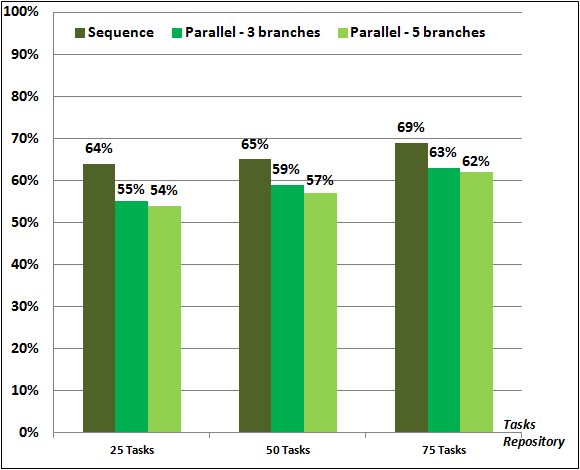}
 } \caption{Measure of the \smartpm effectiveness by comparing process instances having different structures, with a percentage of capabilities coverage equal to 30\% and a percentage of failures equal to 70\%.}
 \label{fig:fig_validation_compare_3}
\end{figure}

\begin{figure}[t]
\centering{
 \includegraphics[width=0.8\columnwidth]{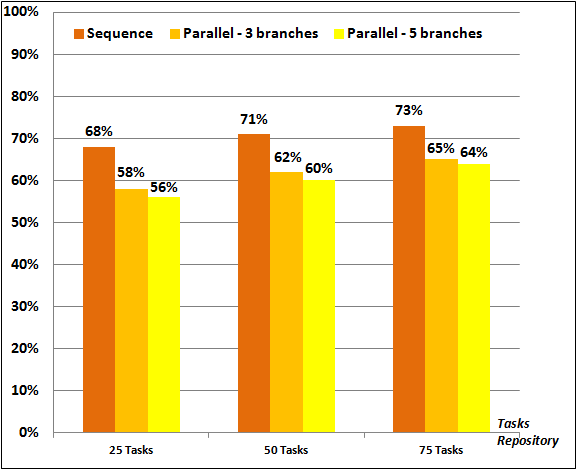}
 } \caption{Measure of the \smartpm effectiveness by comparing process instances having different structures, with a percentage of capabilities coverage equal to 70\% and a percentage of failures equal to 70\%.}
 \label{fig:fig_validation_compare_4}
\end{figure}

\chapter{Automatic Generation of Process Templates}
\label{ch:templates}

Current workflow technology is based on the idea that, in general, there exists an underlying fixed process that can be used to automate the work. Once identified, a process is formalized into a \emph{process model} and automated through a Process Management System that can execute it repeatedly. Conventional Business Process Management solutions require us to pre-define a detailed model of the process (control flow, data and resources) which captures every possible case (i.e., process instance) to be executed at run-time~\cite{WeskeBook2007}. This approach works for processes where procedures are well known, repeatable (the work is done in the same way every time; i.e., there is enough similarity in executing each process instance) and can be planned in advance with some level of detail. As it emerges from the discussion conducted in Chapter~\ref{ch:introduction}, the need to deal with \emph{knowledge-intensive processes} and \emph{dynamic processes} and provide support for flexible process management has emerged as a leading research topic in the BPM domain~\cite{ReichertBook2012}.

In this chapter, we present an approach that allows us to automatically synthesize a \emph{library} of \emph{process templates} starting from a representation of the contextual domain in which the process is embedded in and from an extensive repertoire of tasks defined for such a context. A template depicts the best-practice procedure drawn up with whatever contextual information available at the time; it describes a recommended control flow for the process that does not only work in a specific state of the world, but can be enacted in a range of states satisfying the context conditions. In order to build process templates, we make use of \emph{partial-order planning algorithms} (aka POP~\cite{Weld@AImag1994,TraversoBook2004}), which guarantee some interesting properties in the construction of the template:
\begin{itemize}[itemsep=1pt,parsep=1pt]
\item \emph{Correctness}. Tasks composing the template are contextually
    selected from a specific repository and partially ordered in a way consistent with the context conditions to ensure that the template's objectives are achieved. Hence, a template is proven \emph{correct} relative to the initial state of knowledge about the context.
\item \emph{Sound concurrency}. A process template has the property of \emph{sound concurrency} in the execution of its activities. In fact, concurrent activities of a process template are proven to be effectively \emph{independent} one from another (i.e., concurrent tasks cannot affect the same data). This means more flexibility during process execution. At runtime, the most appropriate execution path can be selected from those allowed by the design time process template definition, without the risk of interference between concurrent tasks.
\item \emph{Executability in partially known environments}. The use of classical plan-space algorithms for building the template requires complete knowledge of the starting state (i.e., it is not admitted that any fact is unknown). Once synthesized, a template can be executed in several different starting states, since it (usually) requires a fragment of the knowledge of the starting state to successfully achieve its objectives. We identify the \emph{weakest preconditions} of process templates, and all the states satisfying such preconditions are good candidates for executing them. In this sense, templates can be enacted in \emph{partially known environments}.
\end{itemize}
We exploit the idea behind POP of representing flexible plans that enables deferring decisions. Instead of committing prematurely to a complete, totally ordered sequence of actions, plans are represented as a partially ordered set, and only the required ordering decisions are recorded. A process template is generated on the basis of such a partially ordered
set of activities, and we are able to identify what knowledge about the starting state is required for successful template execution. Moreover, we build step-by-step a library of process template specifications and support efficient retrieval of appropriate templates in partially known environments.

The rest of the chapter is organized as follows. Section~\ref{sec:templates-case_study} introduces a running example, derived from the main case study presented in Section~\ref{sec:introduction-case_study}. Section~\ref{sec:templates-preliminaries} discusses some preliminary notions around the use of partial-order planning techniques. Section~\ref{sec:templates-template} introduces the concept of process template. Section~\ref{sec:templates-approach} presents the general approach used for synthesizing a library of process templates, starting from a planning domain and a planning problem. Section~\ref{sec:templates-algorithms} gives some technical details about the algorithms used for computing process templates. Section~\ref{sec:templates-experiments} reports on experimental evaluation results. Section~\ref{sec:templates-related_work} discusses related work and Section~\ref{conclusion} concludes by discussing benefits, limitations and future developments of the approach.

\section{Case Study}
\label{sec:templates-case_study}

Let us consider the emergency management scenario described in Fig.~\ref{fig:fig_templates-case_study}(a). It concerns a train derailment and depicts a map of the area (as a 4x4 grid of locations) where the disaster happened. For the sake of simplicity, we consider the same contextual information introduced in Section~\ref{sec:introduction-case_study}. Therefore, we have a derailed train composed of a locomotive (located in \emph{loc33}) and two coaches (located in \emph{loc32} and \emph{loc31} respectively), and a response team sent to the derailment scene. The team is composed of four actors and two robots, initially located in \emph{loc00}. If compared with the scenario introduced in Section~\ref{sec:introduction-case_study}, in this case we do not consider any information related to the network connection between actors and robots. We also remember that each process participant provides a set of specific capabilities. For example, actor $act\emph{1}$ is able to extinguish fire and take pictures, while $act\emph{2}$ and $act\emph{3}$ can evacuate people from train coaches. The two robots, instead, may remove debris from specific locations. Each robot has a battery and each action consumes a given amount of battery charge. When the battery of a robot is discharged, actor $act\emph{4}$ can charge it. Fig.~\ref{fig:fig_templates-case_study}(b) summarizes the above and shows the initial battery charge level of each robot.

\begin{figure}[t]
\centering{
 \includegraphics[width=0.9\columnwidth]{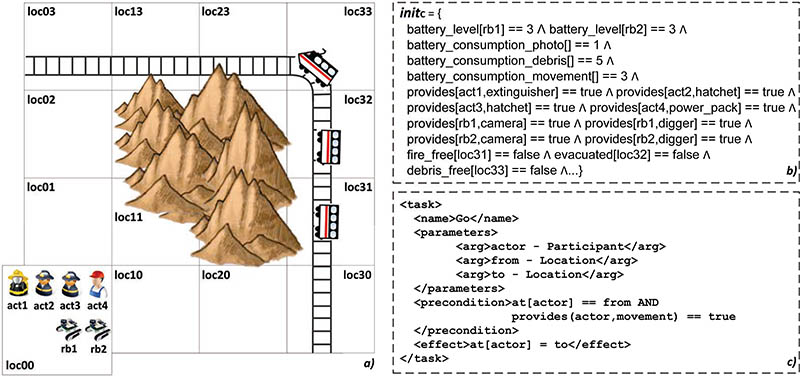}
 } \caption{Area and context of the intervention.}
 \label{fig:fig_templates-case_study}
\end{figure}

Suppose now that the goal of an incident response plan defined for such a context is to evacuate people from the coach located in \emph{loc32}, to extinguish fire in the coach in \emph{loc31} and finally to remove debris from \emph{loc33}. If compared with the case study described in Section~\ref{sec:introduction-case_study}, now we do not provide any pre-built incident response plan for dealing with the disaster scenario. In fact, since the process may be different every time it is defined because it strictly depends on the actual contextual information (the positions of actors/robots, the location of every coach, the battery level of robots, etc.), it is unrealistic to assume that the process designer can pre-define all the possible process models for dealing with this environment. Moreover, if contextual data describing the environment are known, the synthesis of a process dealing with such environment is not straightforward, as the correctness of the process model is highly constrained by the values (or combination of values) of contextual data.

A simple approach to solving our problem is to build a process as a sequence of activities, e.g., the sequence of actions shown in Fig.~\ref{fig:fig_templates-sequence}. However, this solution is highly ``inefficient'', since as many actions are independent, and they could be executed concurrently to reduce intervention time; e.g., a robot could removing debris in parallel with the extinguishing of the fire in \emph{loc31}. But, at the same time, a process designer that has to design such a process may find difficult to organize activities for concurrent execution, since each action, for its executability, depends on the values of contextual data (e.g., a robot needs enough battery charge for moving into a location and removing debris). Also dependencies between actions play a key role in the definition of the process model (e.g., in order to evacuate people at \emph{loc32}, a robot must have removed the debris beforehand). Finally, a process designer usually tends to represent more contextual information than that strictly needed for defining a process. For example, the execution of the process in Fig.~\ref{fig:fig_templates-sequence} does not involve actor $act\emph{3}$, meaning that any information concerning $act\emph{3}$ (e.g., its capabilities, its location, etc.) is not required for synthesizing and executing the process. In order to overcome the above issues, we propose a solution that involves exploiting \emph{partial-order planning} for generating a library of \emph{process templates} for different contextual cases. Our templates provide \emph{sound concurrency} in the execution of their activities and are executable in \emph{partially known environments}.

\begin{figure}[t]
\centering{
 \includegraphics[width=0.9\columnwidth]{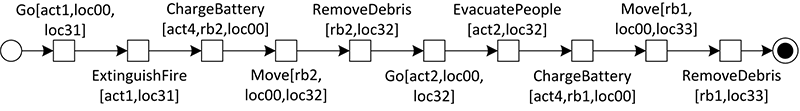}
 } \caption{A process dealing with the scenario of Fig.~\ref{fig:fig_templates-case_study}(a).}
 \label{fig:fig_templates-sequence}
\end{figure}



\section{Partial-Order Planning}
\label{sec:templates-preliminaries}

Planning systems are problem-solving algorithms that operate on explicit representations of states and actions. The standard
representation language for classical planners is known as the Planning Domain Definition Language (PDDL(cf.~\cite{PDDL})); it allows one to formulate a problem $\textsf{PR}$ through the description of the initial state of the world $init_{\textsf{PR}}$, the description of the desired goal condition $goal_{\textsf{PR}}$ and a set of possible actions. An action definition defines the conditions under which an action can be executed, called \emph{preconditions}, and its \emph{effects} on the state of the world. The set of all action definitions $\Omega$ is the \emph{domain} $\textsf{PD}$ of the planning problem. Each action $a \in \Omega$ has a precondition list and an effect list, denoted respectively as $Pre_a$ and $Eff_a$. A planner that works on such inputs generates a sequence of actions (the \emph{plan}) that corresponds to a path from the initial state to a state meeting the goal condition. In this chapter, we represent planning domains and planning problems making use of PDDL version 2.1\footnote{PDDL 2.1 is characterized for enabling the representation of realistic planning domains, which include (in particular) actions with (linear) continuous numeric effects and effects dependent on the durations of the actions, actions and goals involving numerical expressions, operators with universally quantified effects or existentially quantified preconditions. However, currently, our formalism does not allow to handle negative preconditions, disjunctive preconditions and conditional effects.}(cf.~\cite{Fox@JAIR2003}).

In the literature, there exists a wide range of different planning techniques, that are characterized by the specific assumptions made. In this chapter, we make use of \emph{plan-space planning algorithms}. They differ from classical \emph{state-space planning algorithms}, that explore only strictly linear sequences of actions directly connected to the start or goal, by devising totally ordered plans. A plan space is an implicit directed graph whose nodes are \emph{partially specified plans} and whose edges correspond to refinement operations intended to further complete a partial plan, i.e., to achieve an open goal or to remove a possible inconsistency. In order to demonstrate our approach, we focus on Partial-Order Planning (POP) algorithms~\cite{TraversoBook2004,Weld@AImag1994}, a specific type of plan-space planning algorithms. POP algorithms take as input a planning problem defined in PDDL and search the space of partial plans without committing to a totally ordered sequence of actions. They work back from the goal, by adding actions to the plan to achieve each subgoal. A tutorial introduction to POP algorithms can be found in~\cite{Weld@AImag1994}.

Basically, a \textbf{partial plan} is a three-tuple $\textsf{P} = (A,O,CL)$, where
$A \subseteq \Omega$ is a set of (ground) actions, $O$ is a set of \emph{ordering constraints} over $A$, and $CL$ is a set of \emph{causal links} over $A$. Ordering constraints $O$ are of the form $a \prec b$, which is read as ``$a$ before $b$'' and means that action $a$ must be executed sometime before action $b$, but not necessarily immediately before. Causal links $CL$ may be represented as $c \xrightarrow{p} d$, which is read as ``$c$ achieves $p$ for $d$'' and means that $p$ is an effect of action $c$ and a precondition for action $d$. It also asserts that $p$ must remain $true$ from the time of action $c$ to the time of action $d$. In other words, the plan may not be extended by adding a new action that conflicts with the causal link and makes $p$ false between $c$ and $d$.
Consequently, a precondition without a causal link requires further refinement to the plan to establish it, and is considered to be an \emph{open condition} in the partial plan. Loosely speaking, the open conditions are preconditions of actions in the partial plan which have not yet been achieved in the current partial plan. More formally, an open condition is of the form $(p,a)$, where $p \in Pre_a$ and $a \in A$, and there is no causal link $b \xrightarrow{p} a$ (where $b$ is any action of the partial plan $\textsf{P}$).

A classical POP algorithm starts with a null partial plan $\textsf{P}$ and keeps refining it until a solution plan is found. The null partial plan contains two dummy actions $a_0$ $\prec$ $a_\infty$ where the preconditions of $a_\infty$ correspond to the top level goals $goal_{\textsf{PR}}$ of the problem, and the effects of $a_0$ correspond to the conditions in the initial state $init_{\textsf{PR}}$. Intuitively, a refinement operation avoids adding to the partial plan any constraints that are not strictly needed for addressing the refinement objective. This is called the \emph{least commitment principle}~\cite{Weld@AImag1994}. The main advantage of the least-commitment philosophy is that decisions about action ordering are postponed until a decision is forced; constraints are not added to a partial plan unless strictly needed, thus guaranteing flexibility in the execution of the plan and by possibly permitting actions to run concurrently. A \textbf{consistent} plan is defined as a plan with no cycles in the ordering constraints and no conflicts with the causal links. A consistent plan with no open conditions is a \textbf{solution}~\cite{TraversoBook2004}.

\section{Process Templates}
\label{sec:templates-template}

Our approach for the generation of a process template requires to explicitly model the contextual knowledge in which the dynamic process is embedded through some declarative rules (some pre-defined at design time, some known just before the synthesis of the template) and logical constraints expressed in terms of task preconditions and effects. Such information are given as input to an external partial-order planner~\cite{TraversoBook2004,Weld@AImag1994} that will be in charge to build a \emph{process template}, i.e., a graph of activities reflecting the dynamic process required for solving the specific contextual problem. The language we are going to show can be seen as a dialect of \smartML (cf. Section~\ref{sec:framework-smartpm_definition_tool-smartml}), that we customized for describing process templates.

The synthesis of a dynamic process requires a tight integration of process activities and contextual data in which the process is embedded in. The context is represented in the form of a \emph{Domain Theory} \textsf{D}, that involves capturing a set of tasks $t_i \in \textsf{T}$ (with $i \in 1..n$) and supporting information, such as the people/agents that may be involved in performing the process (roles or participants), the data and so forth. Tasks are collected in a specific repository, and each task can be considered as a single step that consumes input data and produces output data. Data are represented through some ground atomic terms $v_{\emph{1}}[y_\emph{1}],v_{\emph{2}}[y_\emph{2}],...,v_{m}[y_m] \in \textsf{V}$ that range over a set of tuples (i.e., unordered sets of zero or more attributes) $y_\emph{1}, y_\emph{2},\allowbreak\dotsc\,y_m$ of \emph{data objects}, defined over some \emph{data types}. In short, a data object depicts an entity of interest.   Under this representation, we consider possible values of a data type as constant symbols that univocally identify data objects in the scenario of interest.

\vskip 0.5em \noindent\colorbox{light-gray}{\begin{minipage}{0.98\textwidth}
\begin{example}
\emph{In our scenario we need to define data objects for representing participants (e.g., data type $Participant=\{act\emph{1},\,act\emph{2},\,act\emph{3},\,act\emph{4},\,rb\emph{1},\,rb\emph{2}\}$), capabilities (e.g., data type $Capability=\{extinguisher,movement,\allowbreak\dotsc\,hatchet\}$) and locations in the area (e.g., data type $Location = \{loc\emph{00},\,loc\emph{10},\,\allowbreak\dotsc\,\,loc\emph{33}\}$)}.
\end{example}
\end{minipage}
}\vskip 0.5em

Each tuple $y_j$ may contain one or more data objects belonging to different data types.
The domain $dom(v_{j}[y_j])$ over which a term is interpreted can be of various types: (i) \emph{Boolean}: $dom(v_{j}[y_j])$ = \{$true,false$\}, (ii) \emph{Integer}: $dom(v_{j}[y_j]) = \mathbb{Z}$, (iii) \emph{Functional}: the domain contains a fixed number of data objects of a designated type.
Terms can be used to express properties of data objects (and relations over objects) and argument types of a term (taken from the set of data types previously defined) represent the finite domains over which the term is interpreted.

\vskip 0.5em \noindent\colorbox{light-gray}{\begin{minipage}{0.98\textwidth}
\begin{example}
\emph{In our example, we may need boolean terms for expressing the presence of a fire in a location (e.g., $fire\_free[loc : Location] = (bool:Boolean)$), integer terms for representing the battery charge level of each robot (e.g., $battery\_level[prt:Participant] \in \mathbb{Z}$) or functional terms for recording the position of each actor in the area (e.g., $at[prt:Participant] = (loc : Location)$).}
\end{example}
\end{minipage}
}\vskip 0.5em

Moreover, since each task has to be assigned to a participant that provides all of the skills required for executing that task, there is the need to consider the participants ``capabilities''. This can be done through a boolean term $provides[prt:Participant,cap:Capability]$ that is $true$ if the capability $cap$ is provided by $prt$ and $false$ otherwise. 

Each task is annotated with \emph{preconditions} and \emph{effects}. Preconditions are logical constraints defined as a conjunction of 
atomic terms, and they can be used to constrain the task assignment and must be satisfied before
the task is applied, while effects establish the outcome of a task after its execution. Note that, as shown in Fig.~\ref{fig:fig_templates-anatomy}(a), our approach treats each task as a ``black box'' and no assumption is made about its internal behavior (we consider the task execution as an instantaneous activity).
\begin{mydef}
\label{taskdef}
A \emph{task} $t[x] \in$ \emph{\textsf{T}} 
consists of:
\begin{itemize}[itemsep=1pt,parsep=1pt]
\item the name 
of the action involved in the enactment of the task (it often coincides with the task itself);
\item a tuple of data objects $x$ as input parameters;
\item a set of preconditions $Pre_t$, represented as the conjunction of $k$ atomic conditions defined over some specific terms, $Pre_t = \bigwedge_{l \in 1..k} pre_{t_{l}}$. Each $pre_{t_{l}}$ can be represented as \{$v_j[y_j] \ \textbf{op} \ \textbf{expr}$\}, where:
    \begin{itemize}[itemsep=1pt,parsep=1pt]
    \item $v_j[y_j] \in \emph{\textsf{V}}$ is an atomic term, with $y_j \subseteq x$, i.e., admissible data objects for $y_j$ need to be defined as task input parameters;
    \item An $\textbf{expr}$ can be a \underline{boolean value} (if $v_j$ is a boolean term); an \underline{input} \underline{parameter} identified by a data object (if $v_j$ is a functional term); an \underline{integer number} or an \underline{expression} involving integer numbers and/or terms, combined with the arithmetic operators \{+,-\} (if $v_j$ is a integer term);
    \item $\textbf{op} \in \{<,>,==,\leq,\geq\}$ is a relational operator. The condition $\textbf{op}$ can be expressed as the equality ($==$) between boolean terms or functional terms and an admissible \textbf{expr}. On the contrary, if $v_j$ is a integer term, it is possible to define the $\textbf{op}$ condition as an expression that make use of relational binary comparison operators ($<, >, =, \leq, \geq$) and involve integer numbers and/or integer terms in the \textbf{expr} field.
    \end{itemize}
\item a set of deterministic effects $Eff_t$, represented as the conjunction of $h$ atomic conditions defined over some specific terms, $Eff_t = \bigwedge_{l \in 1..h} eff_{t_{l}}$. Each $eff_{t_{l}}$ ($with \ l \in 1..h$) can be represented as \{$v_j[y_j] \ \textbf{op} \ \textbf{expr}$\}, where:
    \begin{itemize}[itemsep=1pt,parsep=1pt]
    \item $v_j[y_j] \in \emph{\textsf{V}}$ and $\textbf{expr}$ are defined as for preconditions.
    \item $\textbf{op} \in \{$=$,$+=$,$-=$\}$ is used for assigning  ($=$) to a term a value consistent with the \textbf{expr} field or for incrementing ($+=$) or decrementing ($-=$) an integer term by that value.
    \end{itemize}
\end{itemize}
\end{mydef}
Note that if no preconditions are specified, then the task is always executable. As we will see in Section~\ref{sec:templates-approach}, the use of classical partial-order planning techniques for synthesizing process templates imposes some limitation in the expressiveness of the language used for defining the Domain Theory \textsf{D}. Specifically, negative preconditions are not admitted (e.g., the use of the NOT operator is forbidden and all the atomic conditions that require to evaluate if a boolean term is equal to \emph{false} will be ignored) and we assume that all effects are deterministic.

\vskip 0.5em \noindent\colorbox{light-gray}{\begin{minipage}{0.98\textwidth}
\begin{example}
\emph{The task $Go$ (cf. Fig.~\ref{fig:fig_templates-case_study}(c)) involves two input parameters $from$ and $to$ of type $Location$, representing the starting and arrival locations, and an input parameter $actor$ of type $Participant$ representing the first responder that will execute the task. An instance of this task can be executed only if $actor$ is currently at the starting location $from$ and provides the required capabilities for executing the task $Go$. As a consequence of task execution, the actor moves from the starting to the arrival location, and this is reflected by assigning to the functional term $at[actor]$ the value $to$ in the effect.}
\end{example}
\end{minipage}
}\vskip 0.5em


Modeling a business process involves representing how a business pursues its objectives/goals. Objectives are represented in terms of a process goal to be satisfied. The goal may vary depending on the specific $Process \ Case$ $\textsf{C}$ to be handled. A case $\textsf{C}$ reflects an instantiation of the domain theory $\textsf{D}$ with a starting condition $init_\textsf{C}$ and a goal condition $goal_\textsf{C}$.
Both conditions are conjunctions of ground atomic terms. We do not assume complete information about the starting condition; this means we allow a process designer to instantiate only the ground atomic terms necessary for representing what is known about the starting state, i.e.,
 $init_\textsf{C} = \{v_{\emph{1}}[y_\emph{1}] == val_\emph{1} \land 
... \land v_{j}[y_j] == val_j\}$, where $val_j$ (with $j \in 1..m$) represents the j-th value assigned to the j-th atomic term.
Fig.~\ref{fig:fig_templates-case_study}(b) shows a portion of $init_\textsf{C}$ concerning the scenario depicted in  Fig.~\ref{fig:fig_templates-case_study}(a). The goal is a 
condition represented as a conjunction of some specific terms we want to make true through the execution of the process.

\vskip 0.5em \noindent\colorbox{light-gray}{\begin{minipage}{0.98\textwidth}
\begin{example}
\emph{For example, in the scenario shown in Section~\ref{sec:templates-case_study}, the goal has to be represented as : $goal_\textsf{C} = \{fire\_free[loc\emph{31}] == true \land evacuated[loc\emph{32}] == true \land \ debris\_free[loc\emph{33}] == true\}$.}
\end{example}
\end{minipage}
}\vskip 0.5em

The syntax of goal conditions is the same as for tasks preconditions.
A state is a complete assignment of values to atomic terms in \textsf{V}.
Given a case $\textsf{C}$, an intermediate state $state_{\textsf{C}_{i}}$ is the result of $i$ tasks performed so far, and atomic terms in \textsf{V} may be thought of as ``properties'' of the world whose values may vary across states.
\begin{mydef}
A task $t$ can be performed in a given $state_{\textsf{\emph{C}}_{i}}$ (and in this case we say that $t$ is \textbf{executable} in $state_{\textsf{\emph{C}}_{i}}$) iff $state_{\textsf{\emph{C}}_{i}} \vdash Pre_{t}$, i.e. $state_{\textsf{\emph{C}}_{i}}$ \textbf{satisfies} the preconditions $Pre_t$ for the task $t$.
\end{mydef}
Moreover, if executed, the effects $Eff_t$ of $t$ modify some atomic terms in \textsf{V} and change 
$state_{\textsf{C}_{i}}$ into a new state $state_{\textsf{C}_{i+1}} = update(state_{\textsf{C}_{i}},Eff_{t})$.
The $update$ function returns the new state obtained by applying effects $Eff_{t}$ on the current state $state_{\textsf{C}_{i}}$. 
Starting from a domain theory $\textsf{D}$, a \emph{Process Template} captures a partially ordered set of tasks, whose successful execution (i.e., without exceptions) leads from $init_\textsf{C}$ to $goal_\textsf{C}$. Formally, we define a template as a directed graph consisting of tasks, gateways, events and transitions between them.
\begin{mydef}
Given a domain theory \emph{\textsf{D}}, a set of tasks \emph{\textsf{T}} and a case \emph{\textsf{C}}, a Process Template $\emph{\textsf{PT}}$ is a tuple (N,L) where:
\begin{itemize}[itemsep=1pt,parsep=1pt]
\item N  = $T \cup E \cup W$ is a finite set of nodes, such that :
\begin{itemize}[itemsep=1pt,parsep=1pt]
\item T is a set of tasks instances, i.e., occurrences of a specific task $t \in \emph{\textsf{T}}$ in the range of the process template; 
\item E is a finite set of events, that consists of a single start event $\ocircle$ and a single end event $\odot$;
\item W = $W_{PS} \cup W_{PJ}$ is a finite set of parallel gateways, represented in the control flow with the $\diamond$ shape with a ``plus'' marker inside.
\end{itemize}
\item L = $L_T \cup L_E \cup L_{W_{PS}} \cup L_{W_{PJ}}$ is a finite set of transitions connecting events, task instances and gateways:\\
\begin{itemize*}[itemjoin=\qquad]
\item $L_T : T \rightarrow (T \cup W_{PS} \cup W_{PJ} \cup {\odot})$
\item $L_E : \ocircle \rightarrow (T \cup W_{PS} \cup {\odot})$\\
\item $L_{W_{PS}} : W_{PS} \rightarrow 2^T$
\ \ \ \ \ \ \ \ \ \ \ \ \ \ \ \ \ \ \item $L_{W_{PJ}} : W_{PJ} \rightarrow (T \cup W_{PS} \cup {\odot})$
\end{itemize*}
\end{itemize}
\end{mydef}
\begin{figure}[t]
\centering{
 \includegraphics[width=0.8\columnwidth]{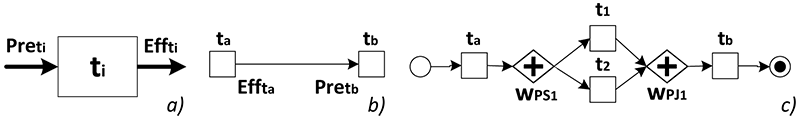}
 } \caption{Task anatomy (a), causality (b) and concurrency (c) in a process model.}
 \label{fig:fig_templates-anatomy}
\end{figure}
Note that the constructs used for defining a template are basically a subset of the ones definable through the BPMN notation. Intuitively, an execution of the process starts at $\ocircle$ and ends at $\odot$; a \emph{task} is an atomic activity executed by the process; \emph{parallel splits} $W_{PS}$ open parallel parts of the process, whereas \emph{parallel joins} $W_{PJ}$ re-unite parallel branches. \emph{Transitions} are binary relations describing in which order the flow objects (tasks, events and gateways) have to be performed, and determine the \emph{control flow} of the template. A transition $l \in L$ is usually represented as $p \rightarrow q$, where $(p,q) \in N$. This represents the fact that there is a transition from the flow object $p$ to the flow object $q$. For $n \in N$, $IN(n)/OUT(n)$ denotes the set of incoming/outgoing transitions of $n$, with the following restrictions :
\begin{itemize}[itemsep=1pt,parsep=1pt]
\item Only one outgoing/incoming flow may be associated with $\ocircle$ and $\odot$ respectively, i.e., $IN(\ocircle) = 0$, $OUT(\ocircle) = 1$, $IN(\odot) = 1$, $OUT(\odot) = 0$
\item Each parallel split $w_{PS} \in W_{PS}$ accepts one incoming flow and more outgoing flows, i.e., $IN(W_{PS}) = 1$, $OUT(W_{PS}) > 1$
\item Each parallel join $w_{PJ} \in W_{PJ}$ accepts more incoming flows and one outgoing flow, i.e., $IN(W_{PJ}) > 1$, $OUT(W_{PJ}) = 1$;
\item Every task $t \in T$ is connected exactly to one incoming/outgoing flow, i.e., $IN(t) = 1$, $OUT(t) = 1$.
\end{itemize}
For example, in Fig.~\ref{fig:fig_templates-anatomy}(b) we have a relation of \emph{causality} between tasks $t_a$ and $t_b$, stating that $t_a$ must take place before $t_b$ happens as $t_a$ achieves some of $t_b$'s preconditions. 
An important feature provided by a process template is \emph{concurrency}, i.e., several tasks can occur concurrently. In Fig.~\ref{fig:fig_templates-anatomy}(c) an example of concurrency between $t_1$ and $t_2$ is shown. In order to represent two or more concurrent tasks in a template, the process designer makes use of the parallel gateways, that indicate points of the template in which tasks can be carried out concurrently. A parallel gateway may act as a \emph{divergence element} (parallel split $W_{PS}$) or \emph{convergent element} (parallel join $W_{PJ}$). As a point of divergence, the diamond shape is used when many tasks have to be carried out at the same time and in any order, which indicates that all transitions that exit this shape will be enabled together. As a point of convergence, the diamond shape is used to synchronize paths that exit a divergence element. This means that a process template is a \emph{graph of tasks} (i.e., not a sequence) that imposes a \emph{partial order} on their execution.
A \emph{linearization} of a process template is any linear ordering of the tasks that is consistent with the ordering constraints of the template itself~\cite{Godefroid:1996}; i.e., a linearization of a partial order is a potential \emph{execution path} of the template from the start event $\ocircle$ to the end event $\odot$. For example, the template in Fig.~\ref{fig:fig_templates-anatomy}(c) has two possible execution paths $r_1 = [\ocircle;t_a;t_1;t_2;t_b;\odot]$ and $r_2 = [\ocircle;t_a;t_2;t_1;t_b;\odot]$.
\begin{mydef}
Given a process template $\emph{\textsf{PT}}$ and an initial state $state_{\emph{\textsf{C}}_{0}} \vdash init_{\emph{\textsf{C}}}$, a state $state_{\emph{\textsf{C}}_{i}}$ is said to be \textbf{reachable} with respect to $\emph{\textsf{PT}}$ iff there exists an execution path $r = [\ocircle;t_1;t_2;...t_k;\odot]$ of $\emph{\textsf{PT}}$ and a task $t_i$ (with i $\in$ 1..k) such that
$state_{\emph{\textsf{C}}_{i}} = update(update(\ldots update(state_{\emph{\textsf{C}}_{0}},Eff_{t_{1}}) \ldots ,Eff_{t_{i-1}}),Eff_{t_{i}})$.
\end{mydef}
\begin{mydef}
A task $t_1$ \textbf{affects} the execution of a task $t_2$ ($t_1 \triangleright t_2$) iff there exists a reachable state $state_{\emph{\textsf{C}}_{i}}$ of $\emph{\textsf{PT}}$ (for some initial state $state_{\emph{\textsf{C}}_{0}}$) such that:
\begin{itemize}[itemsep=1pt,parsep=1pt]
\item $state_{\emph{\textsf{C}}_{i}} \vdash Pre_{t_{2}}$
\item $update(state_{\emph{\textsf{C}}_{i}},Eff_{t_{1}}) \nvdash Pre_{t_{2}}$
\end{itemize}
\end{mydef}
This means that $Eff_{t_{1}}$ modify some terms in \textsf{V} that are required as preconditions for making $t_2$ executable in $state_{\textsf{C}_{i}}$. 
\begin{mydef}
Given a process template $\emph{\textsf{PT}}$, a case $\emph{\textsf{C}}$ and an initial state $state_{\emph{\textsf{C}}_{0}} \vdash init_{\emph{\textsf{C}}}$, an execution path $r = [\ocircle;t_1;t_2;...t_k;\odot]$ (where $k = |T|$) of $\emph{\textsf{PT}}$ is said to be \textbf{executable} in $\emph{\textsf{C}}$ iff:\\
\begin{itemize*}[itemjoin=\qquad]
\item $state_{\emph{\textsf{C}}_{0}} \vdash Pre_{t_{1}}$\\
\item for $1 \leq i \leq k-1$, $update(state_{\emph{\textsf{C}}_{i-1}},Eff_{t_{i}}) \vdash Pre_{t_{i+1}}$\\
\item $update(state_{\emph{\textsf{C}}_{k-1}},Eff_{t_{k}}) = state_{\emph{\textsf{C}}_{k}} \vdash goal_\emph{\textsf{C}}$
\end{itemize*}
\end{mydef}
\begin{mydef}
\label{executability}
A process template $\emph{\textsf{PT}}$ is said to be \textbf{executable} in a case $\emph{\textsf{C}}$ iff any execution path of $\emph{\textsf{PT}}$ is executable in $\emph{\textsf{C}}$.
\end{mydef}
The concept of execution path of a template helps in defining formally the \emph{independence} property between concurrent tasks:
\begin{mydef}
Given a process template $\emph{\textsf{PT}}$, a task $t_x$ is said to be \textbf{concurrent} with a task $t_z$ iff there exist two execution paths $r_1$ and $r_2$ of $\emph{\textsf{PT}}$ such that $r_1 = [\ocircle;t_1;t_2;...;t_x;...;t_z;...;\odot]$ and $r_2 = [\ocircle;t_1;t_2;...;t_z;...;t_x;...;\odot]$.
\end{mydef}
\begin{mydef}
Two concurrent tasks $t1$ and $t2$ are said to be \textbf{independent} ($t1 \parallel t2$) iff $t1 \ntriangleright t2$ and $t2 \ntriangleright t1$; that is, $t1$ does not affect $t2$ and vice versa.
\end{mydef}


\section{On Synthesizing a Library of Process Templates}
\label{sec:templates-approach}

Our approach is focussed on the development and use of a \emph{library} of \emph{process templates}.  These are reusable processes that achieve specified goals of interest in a range of starting states, i.e., any starting state that satisfies the template's required preconditions. We claim that in many cases, process template development can be partially automated through the use of AI planning tools.

Specifically, we focus on the use of a POP-based tool that can synthesize complex concurrent process models that are hard for humans to develop correctly (it is difficult for human designers to ensure that concurrent tasks never interfere with each other). The process designer's role is to specify the domain and context in which the template may be executed. Our POP-based tool can then be used to synthesize some candidate process models for the template. If the tool fails to generate a process model or the generated processes are of insufficient quality (e.g., they are too time consuming, unreliable, or lack concurrency), the designer can refine the domain theory and case to obtain better solutions.  Once a satisfactory template has been obtained, it is added to the library. The POP-based tool automatically identifies the required preconditions for the template to achieve its goal, meaning the template can be reused whenever a case that matches the template's preconditions arises.

In a dynamic process domain, there is a wide range of cases/contexts to handle.  New cases often arise and the requirements for the system frequently evolve. The designer develops and maintains the template library over time, in order to have templates that handle effectively most the cases that arise. The library also stores the templates specifications, i.e., their process domains, goals, and initial conditions/cases. New cases are often variants of existing cases and the designer will be able to adapt existing domain and case specifications to generate templates for the new cases using the tool. In the following, we describe an architecture and methodology for developing such a library-based approach.

\subsection{The General Framework}
\label{subsec:templates-approach-framework}
Our approach to the definition of a process template (cf.~Fig.~\ref{fig:fig_templates-approach})
requires a fundamental shift in how one thinks about modeling business processes. Instead of defining a process model ``by hand'', here, the process designer has to address his/her efforts to specifying the Domain Theory $\textsf{D}$ and the Case $\textsf{C}$ to be handled. In particular, the process designer has to ``guess'' the starting condition $init_{\textsf{C}}$, by instantiating only some atomic terms, that ones needed for depicting the context the user has in mind. This means that $init_{\textsf{C}}$ can be partially specified, i.e, not all terms need to be instantiated with some value. Also the goal condition $goal_\textsf{C}$ is required, since it reflects the target state after having executed the template.

\vskip 0.5em \noindent\colorbox{light-gray}{\begin{minipage}{0.98\textwidth}
\begin{example}
\emph{Let us consider the scenario depicted in Section~\ref{sec:templates-case_study}, represented with a Domain Theory $\textsf{D}_1$ and a goal condition $goal_{\textsf{C}_1} = \{fire\_free[loc\emph{31}] == true \land evacuated[loc\emph{32}] == true \land \ debris\_free[loc\emph{33}] == true\}$. Since the process designer may be interested in an emergency process that involves the fewest participants, s/he can start by modeling a starting condition $init_{\textsf{C}_1}$ with information involving only actors \emph{act1} and \emph{act2} and the robot \emph{rb1}, while terms involving \emph{act3}, \emph{act4} and \emph{rb2} are not explicitly instantiated in $init_{\textsf{C}_1}$.}
\end{example}
\end{minipage}
}\vskip 0.5em

A specific module named \texttt{PC2PR} is in charge of converting the Domain Theory $\textsf{D}$ and the Case $\textsf{C}$ just defined into the corresponding Planning Domain \textsf{PD} and Planning Problem \textsf{PR} specified in PDDL version 2.1 (cf.~\cite{Fox@JAIR2003}). Basically, \texttt{PC2PR} implements a function $f_{\texttt{PC2PR}} : (\textsf{D},init_\textsf{C},goal_\textsf{C}) \rightarrow (\textsf{PD},init_{\textsf{PR}},goal_{\textsf{PR}})$.
Since the use of classical partial-order algorithms for synthesizing the template requires the initial state of \textsf{PR} to be a complete state, we make the closed world assumption~\cite{Reiter@NMR1987} and
assume that every atomic term $v_{j}[y_{j}]$ that is not explicitly specified in $init_\textsf{C}$ is assumed to be false (if $v_{j}[y_{j}]$ is a boolean term) or ``not assigned'' (if $v_{j}[y_{j}]$ is a integer or a functional term) in $init_\textsf{PR}$. Technical details of the algorithm employed in \texttt{PC2PR} are shown in Section~\ref{subsec:templates-algorithms-PC2PR}.

\begin{figure}[t]
\centering{
 \includegraphics[width=0.99\columnwidth]{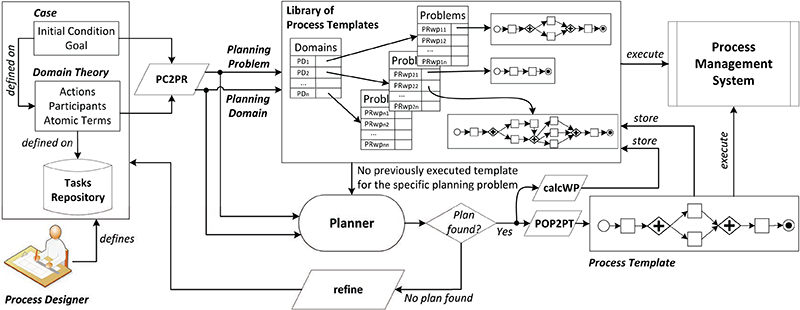}
 } \caption{Overview of the general approach.}
 \label{fig:fig_templates-approach}
\end{figure}

At the heart of our approach lies a library of process templates built for specific planning domains and problems/cases.
If library templates exist for the current values of \textsf{PD} and \textsf{PR}, we can retrieve an appropriate template and allow to execute it through an external PMS. However, if no template exists for the current values of \textsf{PD} and \textsf{PR}, we can invoke an external POP planner on these same inputs. The planner will try to synthesize a plan fulfilling the goal condition $goal_{\textsf{PR}}$.
If the planner is unable to find a plan, this suggests there are some missing elements in the definition of the Domain Theory $\textsf{D}$ or in the Case $\textsf{C}$. Hence, to address this particular case, one can try to \emph{refine} the case $\textsf{C}$ and add information so that it becomes possible to generate a plan. There are many ways to strengthen a problem description, such as adding to the starting condition $init_{\textsf{C}}$ some terms initially ignored (e.g., to specify the position of every participant), or adding new objects in $\textsf{D}$ or new activities in \textsf{T} (e.g., if a task for extinguish fire is missing).
Our approach assumes that one specifies the context step-by-step, and requires the process designer to contribute to the system.

\vskip 0.5em \noindent\colorbox{light-gray}{\begin{minipage}{0.98\textwidth}
\begin{example}
\emph{If the planner is invoked with $init_{\textsf{PR}_1}$ (devised by applying $f_{\texttt{PC2PR}}$ on the triple $\textsf{D}_1$, $init_{\textsf{C}_1}$, $goal_{\textsf{C}_1}$), it will not be able to find any plan for the specific problem. This is because \emph{rb1} does not have enough battery charge for moving and removing debris. The designer can try to add new information to the problem description by instantiating in $init_{\textsf{C}_1}$ all those atomic terms related to actor \emph{act4}, the only one able to charge robot batteries, and devises a new starting condition $init_{\textsf{C}_2}$ (and, consequently, a new initial planning state $init_{\textsf{PR}_2})$. A planner invoked with $init_{\textsf{PR}_2}$ is finally able to find a consistent plan $\textsf{P}_1$ satisfying $goal_{\textsf{PR}_1}$.}
\end{example}
\end{minipage}
}\vskip 0.5em

When the POP planner is able to find a partially ordered plan $\textsf{P}$ consistent with the actual contextual information, three further steps are required. First of all, there is the need to translate the plan just found into a template \textsf{PT} that preserves the ordering constraints imposed by the plan.
A \textbf{solution plan} is a three-tuple $\textsf{P} = (A,O,CL)$ that specifies the causal relationships for the actions $a_i \in A$, but without specifying an exact order for executing them. Since the set of actions $A \in \textsf{P}$ and the set of ordering constraints $O$ over $A$ must be explicitly expressed as nodes and transitions for the template's control flow (as well as their intrinsic ordering),  
we developed a module \texttt{POP2PT}
implementing a function $f_{\texttt{POP2PT}} : \textsf{P} \rightarrow \textsf{PT}$
that takes as input \textsf{P} and converts it into a template \textsf{PT}. A detailed description of the module implementing \texttt{POP2PT} is presented in Section~\ref{subsec:templates-algorithms-POP2PT}.

\vskip 0.5em \noindent\colorbox{light-gray}{\begin{minipage}{0.98\textwidth}
\begin{example}
\emph{By applying $f_{\texttt{POP2PT}}$ to $\textsf{P}_1$, we devise the template $\textsf{PT}_1$ in~Fig.~\ref{fig:fig_templates-template}(a). Dashed arrows are causal links that imply an ordering constraint between pairs of tasks. For example, the ordering constraint between \emph{Go[act1,loc00,loc31]} and \emph{ExtinguishFire[act1,loc31]} is derived from the fact that \emph{Go} has the effect \emph{at[act1]=loc31} that is needed by \emph{ExtinguishFire} as precondition (i.e., \emph{act1} has to be located in \emph{loc31} for extinguish the fire in those location).}
\end{example}
\end{minipage}
}\vskip 0.5em

Secondly, our approach aims to $infer$ the weakest preconditions $w_\textsf{PT}$ under which the process template will achieve its objectives, i.e., to identify the least amount of information about the starting state that is required for the template to achieve its goal. The module we use for inferring $w_\textsf{PT}$ is called \texttt{calcWP} and works by analyzing the set of causal links $CL$ computed by the POP planner, to see which logical facts $f_k$ are involved in causal links that originate from the dummy start action $a_0$ and end in some $a_k \in A$. More formally:\vspace{-0.5em}
\begin{equation} \label{eq_infer}
\forall (cl_k, f_k, a_k) \ s.t. \ cl_k = (a_0 \xrightarrow{f_k} a_k) \in CL, then \ f_k \in w_\textsf{PT}.\vspace{-0.5em}
\end{equation}
This exploits the fact that the effects of $a_0 \in A$ reflect the ``true facts'' in the starting state $init_\textsf{PR}$ that are actually used by the plan and ensure that it is executable and achieves the goal. A first result is that all atomic terms not explicitly instantiated in $init_\textsf{C}$ will correspond to facts assumed to be $false$ in $init_\textsf{PR}$. As the closed world assumption suggests, false facts do not appear in $init_\textsf{PR}$ and are not eligible for being used as effects of $a_0$. Secondly, by going through causal links as specified by~(\ref{eq_infer}), it is possible to identify all logical facts that - although being instantiated in $init_\textsf{PR}$ - are not required by the plan \textsf{P}. Basically, $w_\textsf{PT}$ is the conjunction of those facts strictly required for executing the plan \textsf{P} (and, consequently, the devised template \textsf{PT}), and is used for devising a new problem $\textsf{PR}_{wp}$ = \{$w_\textsf{PT},goal_{\textsf{PR}}$\}.
\begin{figure}[t]
\centering{
 \includegraphics[width=0.80 \columnwidth]{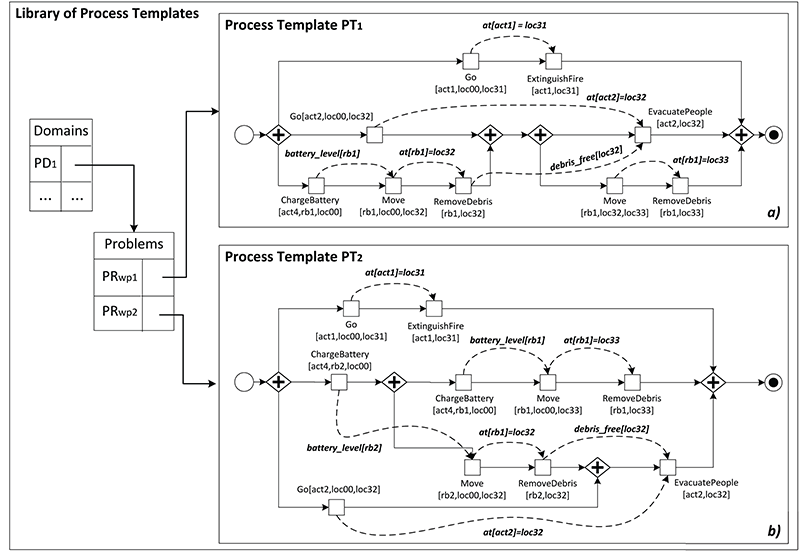}
 } \caption{Templates dealing with the scenario in Fig.~\ref{fig:fig_templates-case_study}.}
 \label{fig:fig_templates-template}
\end{figure}

\vskip 0.5em \noindent\colorbox{light-gray}{\begin{minipage}{0.98\textwidth}
\begin{example}
\emph{If we invoke \texttt{calcWP} on the causal links devised from $\textsf{P}_1$, we may easily infer $w_{\textsf{PT}_1}$. This means that for executing $\textsf{PT}_1$ (cf. Fig.~\ref{fig:fig_templates-template}(a)) we simply need to know the positions and capabilities of actors \emph{act1}, \emph{act2}, \emph{act4} and robot \emph{rb1}; the other contextual information is not strictly needed for a correct execution of the template.}
\end{example}
\end{minipage}
}\vskip 0.5em

Thirdly, after the process template \textsf{PT} has been synthesized starting from \textsf{P}, it can be stored in our library together with information about the planning domain \textsf{PD} and abstracted problem $\textsf{PR}_{wp}$. Specifically, for every different planning domain \textsf{PD} devised through our approach, there is a pointer to a list of different abstracted planning problems $\textsf{PR}_{wp}$ used for obtaining consistent plans in previous executions of our tool, together with the devised process templates. When a process designer defines a new Domain Theory $\textsf{D}_{new}$ and a Case $\textsf{C}_{new}$, the system checks if the corresponding planning domain $\textsf{PD}_{new}$ and problem $\textsf{PR}_{new}$ (obtained by applying $f_{\texttt{PC2PR}}$ to $\textsf{D}_{new}$ and $\textsf{C}_{new}$) are already present in our library. If the library contains a planning domain $\textsf{PD}$ and an abstracted planning problem $\textsf{PR}_{wp}$ (together with the associated template $\textsf{PT}_{lib}$) such that $\textsf{PD}_{new}$ = \textsf{PD} and $goal_{\textsf{PR}}$ = $goal_{\textsf{PR}_{new}}$ and with $init_{\textsf{PR}_{new}} \vdash w_\textsf{PT}$, then $\textsf{PT}_{lib}$ is executable respect to $\textsf{PR}_{new}$ (and therefore with respect to $\textsf{C}_{new}$). This makes our templates reusable in a variety of different situations, in which we don't have complete information about the starting state. At this point, the process designer may decide to execute through an external PMS the template $\textsf{PT}_{lib}$ just found, or to refine $\textsf{D}_{new}$ and $\textsf{C}_{new}$ if $\textsf{PT}_{lib}$ does not fit with the designer expectations (for example, if the designer wants a template with a higher degree of concurrency).

\vskip 0.5em \noindent\colorbox{light-gray}{\begin{minipage}{0.98\textwidth}
\begin{example}
\emph{Let us suppose that the template shown in Fig.~\ref{fig:fig_templates-template}(a) does not satisfies at all the process designer, since s/he could add one further robot \emph{rb2} to the scenario in order to increase the degree of parallelism in the tasks execution. It follows that a new starting condition $init_{\textsf{C}_3}$ including also contextual information about \emph{rb2} can be defined. The associated initial planning state $init_{\textsf{PR}_3}$, together with the original goal condition $goal_{\textsf{PR}_1}$ and the planning domain $\textsf{PD}_1$ are first used for verifying if a previously executed template is already stored the library. The library returns the template $\textsf{PT}_1$ shown in Fig.~\ref{fig:fig_templates-template}(a), since its weakest preconditions $w_{\textsf{PT}_1}$ are satisfied by $init_{\textsf{PR}_3}$ (i.e., $init_{\textsf{PR}_3} \vdash w_{\textsf{PT}_1}$), and goal condition and planning domain are the same as before. Even if the template in Fig.~\ref{fig:fig_templates-template}(a) is executable with $init_{\textsf{PR}_3}$, the designer may try to search for another plan that (maybe) could exploit the presence of the new robot \emph{rb2}. The planner builds a new plan starting from $init_{\textsf{PR}_3}$, and the associated template $\textsf{PT}_2$ is shown in Fig.~\ref{fig:fig_templates-template}(b). $\textsf{PT}_2$ requires the presence of one more robot (i.e., robot \emph{rb2}) and more contextual information for being executed (so its weakest preconditions $w_{\textsf{PT}_2}$ are ``richer'' than $w_{\textsf{PT}_1}$), but it provides an higher degree of concurrency in the execution of its tasks. This means that the process designer can choose which template is the best for her/his purposes: one with less concurrency in the tasks enactment but with the fewest participants (cf. Fig.~\ref{fig:fig_templates-template}(a)), or one with more concurrency but requiring more resources for being executed (cf. Fig.~\ref{fig:fig_templates-template}(b)).}
\end{example}
\end{minipage}
}\vskip 0.5em


Despite the fact that a template is executable ``as is'', it can be seen as an ``intermediate version'' of a completely defined process. In fact, the present POP-based tool cannot be used to synthesize templates involving loops or branching on conditions, and the designer may develop these manually by customizing the template to the specifics of the situation. In the future, one could experiment with more complex AI-planning systems to handle such cases.

\subsection{Properties}
A process template \textsf{PT} guarantees some interesting properties, such as the \emph{correctness} of the template with respect to the information available in the starting state, and the property of \emph{sound concurrency}, meaning that concurrent activities of a template are proven to be effectively independent one from another (i.e., they can not affect the same data).
\begin{mydef}
\label{def_correctness}
A process template \textsf{\emph{PT}} is \textbf{correct} respect to a domain theory \textsf{\emph{D}} and a case $\textsf{\emph{C}}$ iff it is executable in $\textsf{\emph{C}}$.
\end{mydef}
\begin{theorem}
Given a solution plan \emph{\textsf{P}}, a process template \emph{\textsf{PT}} synthesized for \emph{\textsf{P}} using our approach is \textbf{correct} for any process case $\textsf{\emph{C}}$ that satisfies the weakest preconditions $wp_{\emph{\textsf{PT}}}$ inferred from \emph{\textsf{P}}.
\end{theorem}
The proof of Theorem 1 is straightforward. By definition, a sound planner generates a \emph{consistent} plan~\cite{Weld@AImag1994} that leads from an initial state to a goal. Since we represent the domain theory $\textsf{D}$ and the case $\textsf{C}$ respectively as PDDL planning domain and problem, the planner synthesizes a plan (i.e., a process template) that is correct respect to Definition~\ref{def_correctness}. Note that we could also show that the process template guarantees the achievement of the goal.

A second property we can prove on templates is \emph{sound concurrency}. Despite the fact that in a process designed through the rules imposed by data patterns~\cite{data-patterns} and workflow patterns~\cite{workflow-patterns} the concurrent execution of two or more tasks should guarantee the consistency of data accessed by the concurrent tasks, in practice this is often not true. In fact, in complex environments there isn't a clear correlation between a change in the context and corresponding process changes, making difficult to design by hand a process template where concurrent tasks are also independent. On the contrary, all concurrent tasks of a template synthesized with our approach are proven to be \emph{independent} one from another.
\begin{theorem}
Given a process template \emph{\textsf{PT}} synthesized with our approach, all concurrent tasks are \textbf{independent}.
\end{theorem}
\begin{proof}
By contradiction, let us suppose that a process template \textsf{PT} has two concurrent tasks $t_1$ and $t_2$ such that $t_1 \nparallel t_2$.
Hence, $t_1$ (or $t_2$) has some effect affecting the precondition of $t_2$ (or of $t_1$). This means that $t_1 \triangleright t_2$ or $t_2 \triangleright t_1$. Since \textsf{PT} has been synthesized as result of a POP planner, this dependency between $t_1$ and $t_2$ would be represented with a causal link $t_1 \xrightarrow{e} t_2$ (or $t_2 \xrightarrow{e} t_1$), where $e$ is an effect of task $t_1$ and a precondition for task $t_2$ (or vice-versa). This causal link requires an ordering between $t_1$ and $t_2$, meaning they need to be executed (and represented in the process template) in sequence. But this means that $t_1$ and $t_2$ are not concurrent tasks, by contradicting the original hypothesis.
\end{proof} 

\section{Translation Algorithms}
\label{sec:templates-algorithms}

This section is focussed primarily on presenting technical details concerning the translation algorithms introduced in Section~\ref{sec:templates-approach}. Specifically, in the following, we analyze:
\begin{itemize}[itemsep=1pt,parsep=1pt]
\item the module \texttt{PC2PR} used for translating a domain theory \textsf{D} and a process case \textsf{C} into a planning domain \textsf{PD} and a planning problem \textsf{PR};
\item the module \texttt{POP2PT} used for converting a partially ordered plan \textsf{P} into a process template \textsf{PT}.
\end{itemize}

\subsection{Representing Domain Theories and Process Cases in PDDL}
\label{subsec:templates-algorithms-PC2PR}

To obtain a process template that handles a Case $\textsf{C}$, a corresponding PDDL planning problem definition \textsf{PR} has to be specified. This can be done by mapping $init_\textsf{C}$ to $init_{\textsf{PR}}$ and $goal_\textsf{C}$ to
$goal_{\textsf{PR}}$. The planning domain \textsf{PD} is built starting from the definition of ground atomic terms and data types as shown in Section~\ref{sec:templates-template}, and by making explicit the \emph{actions} associated with each annotated task $t \in \textsf{T}$, together with their pre-conditions, effects and input parameters. Basically, the
planning domain describes how predicates and functions change after an action's execution, and specifies the contextual properties constraining the execution of tasks stored in the tasks repository.

Our framework provides a software module named \texttt{PC2PR} in charge of performing such a translation, which makes use of PDDL version 2.1 (cf.~\cite{Fox@JAIR2003}). In the following, we discuss how the domain theory $\textsf{D}$ can be translated into a PDDL file representing the planning domain \textsf{PD}:
\begin{itemize}[itemsep=1pt,parsep=1pt]
\item the \emph{name} and the \emph{domain} of a data type correspond to an \emph{object type} in the planning domain;
\item boolean terms have a straightforward representation as \emph{relational predicates} (templates for logical facts) in the planning domain;
\item integer terms correspond to PDDL \emph{numeric fluents}, and are used for modeling non-boolean resources (e.g., the battery charge level of a robot) in the planning domain;
\item functional terms do not have a direct representation in PDDL v2.1, but may be represented as relational predicates. Since a functional term is a function $f:Object^n \to Object$ that maps tuples of objects with domain types $D^n$ to objects with co-domain type $U$, it may be encoded in the planning domain as a relational predicate \emph{P} of type $(D^n,U)$;
\item a given task, together with the associated pre-conditions, effects and input parameters, is translated into a PDDL \emph{action schema}. An action schema describes how the relational predicates and/or numeric fluents change after the action's execution. For example, given the following XML specification of the task $Go \in \textsf{T}$ (with respect to the language provided in Section~\ref{sec:templates-template}):
    \begin{scriptsize}
    \begin{verbatim}
    <task>
    <name>Go</name>
    <parameters>
        <arg>prt - Participant</arg>
        <arg>from - Location</arg>
        <arg>to - Location</arg>
    </parameters>
    <precondition>at[prt] == from AND
                  provides[prt,movement] == true
    </precondition>
    <effect>at[prt] = to</effect>
    </task>
    \end{verbatim}
    \end{scriptsize}

    the module \texttt{PC2PR} produces the following PDDL representation:

    \begin{scriptsize}
    \begin{alltt}
    (:action go
    :parameters (?prt - participant ?from - location
                 ?to - location)
    :precondition (and ((at ?prt ?from)
                        (provides ?prt movement)))
    :effect (and (forall (?loc - location) (not (at ?prt ?loc)))
                 (at ?prt ?to))
    \end{alltt}
    \end{scriptsize}
This task can be executed only if the actor denoted by \emph{prt}
     is not currently located in the target location \emph{to} (and is
     located in his/her starting location \emph{from}) and is able to
     move into the area. The desired effect turns the value of the
     predicate $at(prt,to)$ to $true$ and $at(prt,from)$ to $false$,
     meaning the actor moved to the new location.
\end{itemize}
The planning problem \textsf{PR} is built by translating $init_{\textsf{C}}$ into $init_{\textsf{PR}}$ and $goal_{\textsf{C}}$ into $goal_{\textsf{PR}}$ :
\begin{itemize}[itemsep=1pt,parsep=1pt]
\item for each data type defined in the planning domain, all the
  possible object instances of that particular data type are
  explicitly instantiated as \emph{constant symbols} in the initial
  state of the planning problem (e.g., the fact that $act1$, $act2$, $act3$,
  $act4$, $rb1$ and $rb2$ are \emph{Participants}, and that $loc00$, ..., $loc33$ are \emph{Locations});
\item a representation of the \emph{initial state} of the planning problem is needed. Basically, the initial state of the planning problem $init_{\textsf{PR}}$ is composed of a conjunction of relational predicates (representing functional and boolean terms in $init_{\textsf{C}}$) and the initial value of each numeric fluent (e.g., the value of the battery charge level for each robot), corresponding to the values of integer terms in $init_{\textsf{C}}$;
\item the \emph{goal condition} of the planning problem is a logical
  expression over facts, which partially specifies the state to be
  reached after the execution of the process template. Again, it is a
  condition represented as a conjunction of relational predicates and
  numeric fluent atoms representing the specific boolean, functional
  and integer terms we want to make true through the correct execution of the process template (as defined in $goal_\textsf{C}$).
\end{itemize}

\subsection{Translating a Partially Ordered Plan \textsf{P} into a \\Process Template \textsf{PT}}
\label{subsec:templates-algorithms-POP2PT}

Once the plan $\textsf{P}$ has been synthesized, it needs to be translated in a process template \textsf{PT}. As explained in Section~\ref{sec:templates-preliminaries} , a \textbf{solution plan} is a three-tuple $\textsf{P} = (A,O,CL)$, where $A$ is the set of actions appearing in the plan, $O$ and $CL$ are respectively the set of ordering constraints and of causal links over $A$.
Since the set of actions $A$ composing the plan and the set of ordering constraints $O$ over $A$ require to be explicitly expressed as nodes and transitions of the template's control flow (as well as their intrinsic ordering), we have implemented a module named \texttt{POP2PT}\footnote{The software implementing POP2PT is available at \url{http://www.dis.uniroma1.it/~marrella/documents/phd_thesis/templates/POP2PT.zip}} that takes as input a solution plan $\textsf{P}$ and converts it into a process template $\textsf{PT}$.

We provide two algorithms - named respectively ``\texttt{$Find_{PREC/NEXT}$}'' (cf. Algorithm~\ref{alg_1}) and ``\texttt{$Build_\textsf{PT}$}'' (cf. Algorithm~\ref{alg_2}) - to be executed sequentially for automatically computing process template $\textsf{PT}$. For each planning action $a_i \in A$, Algorithm~\ref{alg_1} is in charge of detecting which actions ``directly'' precede and follow $a_i$ in the plan. This information will be crucial for instantiating the transitions between the flow objects of the process template. However, this knowledge is not directly available in $O$. In fact, an ordering constraint $a \prec b$ between two actions $a \in A$ and $b \in A$ indicates that $a$ must be executed sometime before action $b$, but not necessarily immediately before. Algorithm~\ref{alg_1}, for each action $a_i \in A$, builds two sets containing the actions that immediately precede $a_i$ (the set $PREC(a_i)$) and immediately follow $a_i$ (the set $NEXT(a_i)$).
\begin{mydef}
Given an ordering constraint $(a \prec b) \in O$ between two actions $a \in A$ and $b \in A$, we say that \textbf{a directly precedes b} and \textbf{b directly follows a} iff no further action $c \in A$ exists such that $a \prec c$ and $c \prec b$.
\end{mydef}
Basically, for each ordering constraint $o_k = (a_i \prec a_j)\footnote{Before the execution of the algorithm, all the ordering constraints involving the dummy start action $a_0$ and the dummy end action $a_\infty$ are removed from $O$.}
\in \emph{O}$, the algorithm $Find_{PREC/NEXT}$ works as follows:
\begin{itemize}[itemsep=1pt,parsep=1pt]
\item if there does not exist any planning action $a_h \neq a_j$ that precedes $a_i$ - i.e., such that $(a_h \prec a_i) \in \emph{O}$ - then the only predecessor of $a_i$ is the dummy start action $a_0$. Therefore, $a_0$ is added to the set of predecessors of $a_i$ (i.e., $PREC(a_i) = PREC(a_i) \cup \{a_0\}$) and $a_i$ is added to the set of successors of $a_0$ (i.e., $NEXT(a_0) = NEXT(a_0) \cup \{a_i\}$).
\item if there does not exist any planning action $a_h \neq a_i$ that follows $a_j$ - i.e., such that $(a_j \prec {a_h}) \in \emph{O}$ - then the only successor of $a_j$ is the dummy end action $a_\infty$. Therefore, $a_j$ is added to the set of predecessors of $a_\infty$ (i.e., $PREC(a_\infty) = PREC(a_\infty) \cup \{a_j\}$) and $a_\infty$ is added to the set of successors of $a_j$ (i.e., $NEXT(a_j) = NEXT(a_j) \cup \{a_\infty\}$).
\item if there does not exist any planning action $a_h \neq a_i$ that precedes $a_j$ - i.e., such that $(a_h \prec a_j) \in \emph{O}$ - then $a_i$ directly precedes $a_j$ (and $a_j$ directly follows $a_i$), meaning that $PREC(a_j) = PREC(a_j) \cup \{a_i\}$ and $NEXT(a_i) = NEXT(a_i) \cup \{a_j\}$.
\item if there exists a planning action $a_h \neq a_i$ that precedes $a_j$ - i.e., such that $(a_h \prec a_j) \in \emph{O}$ - but there does not exist any finite sequence of actions $a_1, a_2, ... , a_n$ such that $(a_i \prec ... \prec a_1 \prec a_2 \prec ... \prec a_n \prec ... \prec a_h)$, then $a_i$ directly precedes $a_j$ (and $a_j$ directly follows $a_i$), meaning that $PREC(a_j) = PREC(a_j) \cup \{a_i\}$ and $NEXT(a_i) = NEXT(a_i) \cup \{a_j\}$.
\end{itemize}
\begin{figure}[t]
\centering{
 \includegraphics[width=0.9\columnwidth]{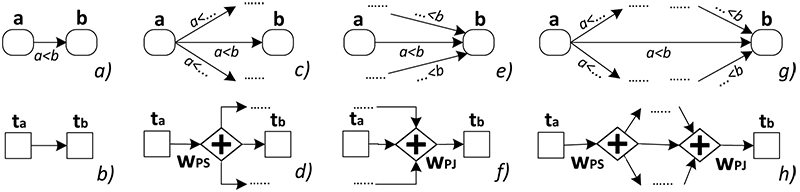}
 } \caption{Overview of the working of the Algorithm \texttt{$Find_{PREC/NEXT}$}.}
 \label{fig:fig_templates-alg}
\end{figure}
Starting from the two sets just computed, the algorithm \texttt{$Build_\textsf{PT}$} is in charge of building the process template $\textsf{PT}$, by instantiating the tasks, gateways, events and transitions between them. For every action $a \in A$, a corresponding task instance $t_a \in T$ (basically, a task instance is an occurrence of a specific task $t \in \textsf{T}$) is generated (cf. the function \emph{taskify} in Algorithm~\ref{alg_2}). Transitions between tasks depend on the number of predecessors/successors contained in $PREC(a)$/$NEXT(a)$. In particular, if an action $b \in A$ is such that $b \in PREC(a)$ or $b \in NEXT(a)$, it is clear that there must be some kind of transition between $t_a$ and $t_b$ in $\textsf{PT}$:
\begin{itemize}[itemsep=1pt,parsep=1pt]
\item if $b$ is the only successor of $a$, and $a$ is the only predecessor of $b$ (cf. Fig.~\ref{fig:fig_templates-alg}(a)), then $(t_a \rightarrow t_b) \in L_T$ (cf. Fig.~\ref{fig:fig_templates-alg}(b));
\item if $a$ is the only predecessor of $b$, but $b$ is not the only successor of $a$ (cf. Fig.~\ref{fig:fig_templates-alg}(c)), then a parallel split $w_{PS} \in W_{PS}$ is needed between $t_a$ and $t_b$. Hence $(t_a \rightarrow w_{PS}) \in L_T$ and $(w_{PS} \rightarrow t_b) \in L_{W_{PS}}$ (cf. Fig.~\ref{fig:fig_templates-alg}(d));
\item if $b$ is the only successor of $a$, but $a$ is not the only predecessor of $b$ (cf. Fig.~\ref{fig:fig_templates-alg}(e)), then a parallel join $w_{PJ} \in W_{PJ}$ is needed between $t_a$ and $t_b$. Hence $(t_a \rightarrow w_{PJ}) \in L_T$ and $(w_{PJ} \rightarrow t_b) \in L_{W_{PJ}}$ (cf. Fig.~\ref{fig:fig_templates-alg}(f));
\item if $t_a$ is not the only predecessor of $t_b$ and $t_b$ is not the only successor of $t_a$ (cf. Fig.~\ref{fig:fig_templates-alg}(g)), then a parallel split $w_{PS} \in W_{PS}$ and a parallel join $w_{PJ} \in W_{PJ}$ are needed between $t_a$ and $t_b$. Hence $(t_a \rightarrow w_{PS}) \in L_T$, $(w_{PS} \rightarrow w_{PJ}) \in L_{W_{PS}}$ and $(w_{PJ} \rightarrow t_b) \in L_{W_{PJ}}$ (cf. Fig.~\ref{fig:fig_templates-alg}(h)).
\end{itemize}
Finally, if an action $a \in A$ has no predecessors/successors (i.e., the set $PREC(a)/NEXT(a)$ is empty), this means that $t_a$ must be connected with the start event $\ocircle$  and end event $\odot$.

\begin{algorithm}
 \caption{\texttt{$Find_{PREC/NEXT}$} - Find actions predecessors and successors}\label{alg_1}
\DontPrintSemicolon
\SetKwInOut{Init}{Init}
 \KwData{\\
        \begin{itemize}[itemsep=0pt,parsep=0pt]
        \item \emph{A} : the set of actions appearing in the final plan $\textsf{P}$, including dummy actions $a_0$ and $a_\infty$.\\
        \item \emph{O} : the set containing the ordering constraints returned by the planner, in the form $o_k = (a_i \prec a_j)$, with $(a_i,a_j) \in A$. Ordering constraints that involve $a_0$ and $a_\infty$ are removed.\\
        \item $NEXT(a_i)$ : a set containing the list of successors of the i-th action $a_i \in A$.\\
		\item $PREC(a_i)$ : a set containing the list of predecessors of the i-th action $a_i \in A$.\\
        \item $k$ : an integer number, used as counter for the ordering constraints.
        \item $lenght(S)$ : returns the size of a set $S$.\\
		\item $add(a,S)$ : inserts an action $a$ into a set $S$.
        \end{itemize}
        }
 \KwResult{for each $a_i \in A$, it returns $NEXT(a_i)$ and $PREC(a_i)$}
 \Init{\\
 \emph{k} = 0\\
 \For{\textbf{\emph{each}} $a_i \in A$}{
 NEXT($a_i$) = $\emptyset$\\
 PREC($a_i$) = $\emptyset$
 }
 }
 \Begin{
 \While{$k < lenght(O)$}{
  \emph{take the k-th ordering constraint,} $o_k = (a_i \prec a_j)$ \emph{from O}\\
  \emph{start scanning the O set}\\
  \If{$\nexists \ (a_h \in A) : a_h \neq a_j \land (a_h \prec a_i \in O)$}
  {
   $add(a_i,NEXT(a_0))$;\\
   $add(a_0,PREC(a_i))$;
   }
   \If{$\nexists \ (a_h \in A) : a_h \neq a_i \land (a_j \prec a_h \in O)$}
  {
  $add(a_\infty,NEXT(a_j))$;\\
  $add(a_j,PREC(a_\infty))$;
  }
  \eIf{$\nexists \ (a_h \in A) : a_h \neq a_i \land (a_h \prec a_j \in O) \ \textbf{\emph{OR}}$ \ $\exists \ (a_h \in A) : a_h \neq a_i \land (a_h \prec a_j \in O) \ \textbf{\emph{AND}}$ $\nexists$ \emph{any finite sequence of actions} $a_1,a_2,...,a_n$ \emph{such that} $(a_i \prec ... \prec a_1 \prec a_2 \prec ... \prec a_n \prec ... \prec a_h)$}
      {
       $add(a_j,NEXT(a_i))$;\\
	   $add(a_i,PREC(a_j))$;
        }
  {\textbf{do nothing}, \emph{because it means that} $a_i \prec a_h \prec a_j$. \emph{The ordering constraint} $a_i \prec a_h$ \emph{will be considered in a future iteration of the algorithm}.}
  \emph{k++}\\
 }
 }
\end{algorithm}

\begin{algorithm}
\caption{\texttt{$Build_\textsf{PT}$} - Build Process Template}\label{alg_2}
\begin{scriptsize}
\DontPrintSemicolon
 \KwData{\\
        \begin{itemize}[itemsep=0pt,parsep=0pt]
        \item \emph{A} : the set of actions appearing in the final plan, including $a_0$ and $a_\infty$.\\
        \item $L$ = $L_T \cup L_E \cup L_{W_{PS}} \cup L_{W_{PJ}}$ is a finite set of transitions connecting events, task instances and gateways. Initially it is empty.\\
        \item $NEXT(a_i)/PREC(a_i)$ : a set with the list of successors/predecessors of $a_i$.\\
        \item $lenght(S)$ : returns the size of a set $S$.\\
		\item $insert(x,S)$ : given an element $x$ and a set $S$, if $x \notin S$ inserts $x$ into $S$.
        \item $taskify(a)$ : given a planning action $a \in A$, it generates a corresponding task instance $t_a \in T$.
        \end{itemize}
        }
 \Begin{
    \For{\textbf{\emph{each}} $a_i \in A$}{
    $t_{a_{i}}$ = $taskify(a_i)$\\
    \eIf{$lenght(PREC(a_i)) > 1$}
    {
    \For{\textbf{\emph{each}} $a_q \in PREC(a_i)$}{
    $t_{a_{q}}$ = $taskify(a_q)$\\
        \eIf{$lenght(NEXT(a_q)) > 1$}
        {$insert(\{w\_split_q \rightarrow w\_join_i\},L_{W_{PS}})$}
        {$insert(\{t_{a_{q}} \rightarrow w\_join_i\},L_T$)}
        }
    $insert(\{w\_join_i \rightarrow t_{a_{i}}\},L_{W_{PJ}})$
    }
    {
    \eIf{$lenght(NEXT(a_q)) > 1$ (let $a_q$ be the only predecessor of $a_i$)}
        {\eIf{$lenght(PREC(a_q)) > 0$}
        {$insert(\{w\_split_q \rightarrow t_{a_{i}}\},L_{W_{PS}})$}
        {$insert(\{\ocircle \rightarrow w\_split_q\},L_E)$\\
         $insert(\{w\_split_q \rightarrow t_{a_{i}}\},L_{W_{PS}})$}
        }
        {
        \eIf{$lenght(PREC(a_q)) > 0$}
        {$insert(\{t_{a_{q}} \rightarrow t_{a_{i}}\},L_T)$}
        {$insert(\{\ocircle \rightarrow t_{a_{i}}\},L_E)$}
        }
    }

    \eIf{$lenght(NEXT(a_i)) > 1$}
    {
    \For{\textbf{\emph{each}} $a_j \in NEXT(a_i)$}{
    $t_{a_{j}}$ = $taskify(a_j)$\\
        \eIf{$lenght(PREC(a_j)) > 1$}
        {$insert(\{w\_split_q \rightarrow w\_join_i\},L_{W_{PS}})$}
        {$insert(\{w\_split_i \rightarrow t_{a_{j}}\},L_{W_{PS}})$}
        }
    $insert(\{t_{a_{i}} \rightarrow w\_split_i\},L_T)$
    }
    {
    \eIf{$lenght(PREC(a_j)) > 1$ (let $a_j$ be the only successor of $a_i$)}
        {
        \eIf{$lenght(NEXT(a_j)) > 0$}
        {$insert(\{t_{a_{i}} \rightarrow w\_join_j\},L_T)$}
        {$insert(\{t_{a_{i}} \rightarrow w\_join_j\},L_T)$\\
         $insert(\{w\_join_j \rightarrow \odot\},L_{W_{PJ}})$}
        }
        {
        \eIf{$lenght(NEXT(a_j)) > 0$}
        {$insert(\{t_{a_{i}} \rightarrow t_{a_{j}}\},L_T)$}
        {$insert(\{t_{a_{i}} \rightarrow \odot\},L_T)$}
        }
        }

    }
    }
    \end{scriptsize}
\end{algorithm}

\section{Experiments}
\label{sec:templates-experiments}

To show the feasibility of the approach, we ran some experiments and measured the time required for
synthesizing a partially ordered plan for some variants of our running example described in Section~\ref{sec:templates-case_study}. We ran our tests using POPF2~\cite{POPF2_2010}, which is a temporal planner that handles PDDL 2.1~\cite{Fox@JAIR2003} and preserves the benefits of partial-order plan construction in terms of producing makespan-efficient, flexible plans.  Search in POPF2 is based around the idea of expanding a partial-order plan in a forward direction; steps added to the plan are ordered after a subset of those in the partial plan, rather than after every step in the plan so far.  We conducted the experiments on a Intel U7300 CPU 1.30GHz Dual Core, 4GB RAM machine.

The experimental setup was run on variants of the test case shown in our running example. We represented 7 planning actions in \textsf{PD} (corresponding to 7 different emergency management tasks stored in the tasks repository \textsf{T}), annotated with 7 relational predicates and 6 numeric fluents, in order to make the planner search space sufficiently challenging. Then, we defined 18 different planning problems of varying complexity by manipulating the number of facts in the goal. 
As well, we examined how irrelevant domain knowledge affects the performance of the planner. Starting from a planning problem \textsf{PR} with an initial state $init_{\textsf{PR}}$ completely specified and with a goal condition $goal_{\textsf{PR}}$ expressed as the conjunction of $n$ facts, we manipulated the specification of the initial state $init_{\textsf{PR}}$ to reduce the number of known facts. In our experiments, the number of facts in goal condition ranges from 1 single fact to a conjunction of 6 logical facts (that make the contextual problem harder). As shown in Table 1, for a given goal condition composed of $n$ facts, our purpose was to measure the computation time needed for finding a sub-optimal solution for problems specified with starting states with a decreasing amount of knowledge. The column labeled as ``Knowledge in $init_{\textsf{PR}}$'' makes explicit which information is removed from the initial state of the planning problem with a given goal condition $goal_{\textsf{PR}}$. For example, if we consider our running scenario from Section~\ref{sec:templates-case_study}, whose goal condition is composed of 3 facts and characterized by a complete specification of the starting state, the time needed for finding a solution plan is of 0.13 seconds. After removing from the initial state all the information concerning the actor $act\emph{3}$, the time required for computing the plan decreases to 0.11 seconds. In general, for a given goal condition, removing ``irrelevant information'' from the initial state reduces the search space and the computation required for synthesizing the plan. It is also interesting to note that a sub-optimal solution includes more actions than those strictly required for fulfilling the goal condition. In particular, when the number of facts in a goal condition increases, the quality of the sub-optimal solution produced typically decreases. Based on our experiments, the approach seems quite feasible for medium-sized dynamic processes as used in practice.
\begin {table}[tp]\scriptsize
\caption {Time performances of POPF2.}
\label{validation_table}
\centering
\begin{tabular}{ccc}
\hline
Facts in $goal_{\textsf{PR}}$ & Knowledge in $init_{\textsf{PR}}$& Time for a sub-opt. sol.\\
\hline
\hline
\multirow{3}{*}{1} & complete state & 0.17 \\
                   & No information about act1 & 0.15 \\
                   & No information about act1 and act3 & 0.12 \\
\hline
\multirow{3}{*}{2} & complete state & 0.12 \\
                   & No information about act3 & 0.10 \\
                   & No information about act3 and rb1 & 0.08 \\
\hline
\multirow{3}{*}{3} & complete state & 0.13 \\
                   & No information about act3 & 0.11 \\
                   & No information about act3 and rb2 & 0.09 \\
\hline
\multirow{3}{*}{4} & complete state & 0.21 \\
                   & No information about act3 & 0.20 \\
                   & No information about act3 and rb1 & 0.10 \\
\hline
\multirow{3}{*}{5} & complete state & 0.17 \\
                   & No information about act3 & 0.16 \\
                   & No information about act3 and rb1 & 0.10 \\
\hline
\multirow{3}{*}{6} & complete state & 1.56 \\
                   & No information about act3 & 1.19 \\
                   & No information about act3 and act1 & 1.13 \\
\hline
\hline
\end{tabular}
\end{table} 

\section{Related Work}
\label{sec:templates-related_work}

\emph{Process modeling} is the first and most important step in the BPM lifecycle~\cite{WeskeBook2007}, which intends to provide a high-level specification of a business process that is independent from implementation and serves as a basis for process analysis, automation, and verification. 
The task of defining a model is often performed with the aid of tools that provide a graphical representation, but without any automatic generation of the process model. However in recent years, numerous AI planning-based approaches have been devised for the latter, and the closest to our approach are~\cite{Weske@ADBIS2004,Moreno@2007,Ferreira@IJCIS2006}.
\cite{Weske@ADBIS2004} presents the basic idea behind the use of planning techniques for generating a process schema, but no implementation seems to be provided, and the direct use of the PDDL language for annotating tasks and specifying the
domain theory requires a deep understanding of AI planning technology.

In~\cite{Moreno@2007}, the authors exploit the IPSS planner~\cite{Moreno2006} for modeling processes in SHAMASH~\cite{SHAMASH}, a knowledge-based system that uses a rule-based approach. To automate the process model generation, they first translate the semantic representation of SHAMASH into the IPSS language. Then, IPSS produces a parallel plan of activities that is finally translated back into SHAMASH and is presented graphically to the user. This work proposes the scheduling of parallel activities  (that implicitly handle time and resource constraints), meta-modeling that deals with planning explicitly, and suggests that learning could be used for process optimization.
However, the emphasis here is on supporting processes for which one has complete knowledge. This assumption does not always hold for dynamic processes, where some contextual information may not be available at the time of process model synthesis. The work of~\cite{Ferreira@IJCIS2006} proposes a new life cycle for workflow management based on the continuous interplay between learning and planning. The approach is based on the use of machine learning algorithms for inferring pre-conditions and effects of activities, and generate a partially-ordered execution plan (i.e., a process model with concurrent branches) that complies to these rules. An interesting result concerns the possibility of producing process models even though the activities may not be accurately described. In such cases, the authors use a best-effort planner that is always able to create a plan, even though the plan may be incorrect. By refining the preconditions and effects of planning actions, the planner will be able to produce several candidate plans, and after a finite number of refinements, the best candidate plan (i.e., the one with the lowest number of unsatisfied preconditions) may be chosen and translated into a process model. Unfortunately, this approach has not been very successful, as the best plan typically moves further and further away from the correct solution.

\section{Conclusion}
\label{conclusion}
In this chapter, we developed a technique based on partial-order planning algorithms and declarative specifications of process tasks for synthesizing a library of concurrent process templates to be enacted in partially specified contextual scenarios. Characteristic of these processes is the role of contextual data acting as a driver for process modeling. We are currently working on a complete implementation and thorough validation of the whole approach, including the formalization of metrics for evaluating process templates' quality. A future direction for this work is to generate hierarchical process templates, with high-level templates achieving more general goals that can invoke simpler templates to achieve some of their subgoals. We also plan to address expressiveness limitations, such as handling preferences, representing negative preconditions (negative literals in goals and preconditions are not supported by most POP planners, including POPF2) and tasks with context-dependent effects.

\chapter{Recovering Dynamic Processes in YAWL}
\label{ch:planlets}

The need to automatically adapt processes in response to exceptions, events and contextual changes
has emerged as a leading issue, both in dynamic and pervasive scenarios and in processes with a varying degree of structuring (cf. Section~\ref{sec:introduction-spectrum}). In order to demonstrate the general validity of the approach for automatic adaptation proposed in Chapter~\ref{ch:approach}, a further direction of the author's research activity was to integrate such an approach into YAWL~\cite{YAWLBook2009}, that is among the most well-known PMSs coming from academia. To this end, our contribution has been twofold:

\begin{itemize}[itemsep=1pt,parsep=1pt,topsep=1pt]
\item In~\cite{BPMDEMO2011} we contextualize and demonstrate our adaptation approach in the service-oriented environment provided by the YAWL system. We leverage the \emph{Flexibility as a Service} approach~\cite{faas} to obtain a complete integration between YAWL and \smartpm by designing and implementing a so-called YAWL Custom Service~\cite{YAWLimpl}, named the \smartpm Service. At design time, the process designer is thus able to associate atomic tasks in the YAWL specification with the \smartpm Service. Atomic tasks defined in YAWL may be decomposed into complex sub-processes defined for the \smartpm environment. At runtime the service is then able to check-out a work-item and execute the corresponding subprocess. When a subprocess is executed in \smartpm, the automatic adaptivity features are exploited in order to react to exogenous events and to progress the subprocess. Upon completion of a subprocess, control is passed back to the YAWL environment along with any data that was produced.
\item In~\cite{CoopIS2012} we rely on the modeling capabilities provided by the YAWL language and on the process enactment and exception handling capabilities provided by the YAWL environment, and we propose significant extensions at both the modeling and architectural level. Specifically, we introduce and define \textsc{Planlets}, as self-contained YAWL specifications where process tasks are annotated at design-time with pre-conditions, desired effects and post-conditions. In the presence of an exception, this approach allows delegating to an external planner the automatic run-time synthesis of a suitable recovery procedure by contextually selecting the compensation tasks from a specific repository linked to the \textsc{Planlet} under execution.
\end{itemize}

For describing the two approaches, we have considered again a motivating example drawn from an emergency management setting (cf. Section~\ref{sec:planlets-running_example}). It clearly represents an application scenario for which it is unrealistic to assume that all exceptional situations, as well as required exception handlers, can be anticipated at design-time and thus be incorporated a priori into the pre-specified process model.

The rest of the chapter is organized as follows. In Section~\ref{sec:planlets-YAWL_architecture} we briefly focus on the exception handling approach implemented in the YAWL system, as presented in~\cite{YAWLBook2009}.
In Section~\ref{sec:planlets-YAWL_SmartPM} we show how the YAWL environment (and its imperative modeling approach) can be complemented with our \smartpm execution environment that exploits a declarative modeling approach to deal with environmental changes and exceptions in process executions.
In Section~\ref{sec:planlets-approach} we provide an in-depth discussion and concrete design and implementation proposal of how the YAWL architecture can be extended to support \textsc{Planlets} and to integrate planning capabilities.
Finally, Section~\ref{sec:planlets-conclusions} discusses limitation and future developments of the approach.

\section{Running Example}
\label{sec:planlets-running_example}

\begin{figure}[t]
\centering{
 \includegraphics[width=0.99\columnwidth]{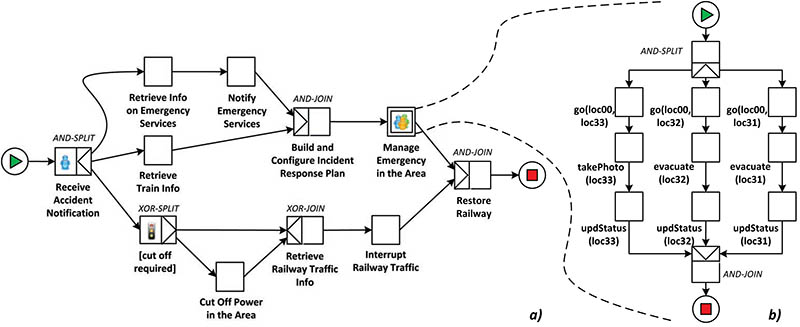}
 } \caption{The YAWL process defined for a train derailment scenario (a), in which the composite task ``Manage Emergency in the Area'' is a dynamic process (b).}
\label{fig:fig_planlets-yawl_process}
\end{figure}

As an application scenario, we consider again an emergency management process defined for train derailments. The corresponding YAWL process to be executed is shown in Figure~\ref{fig:fig_planlets-yawl_process}(a). The process starts when the railway traffic control center receives an accident notification from the train driver and collects some information about the derailment, including the train ID code, the GPS location and the number of coaches and passengers. Then it could be required to cut off the power in the area and to interrupt the railway traffic near the derailment scene. In parallel, after having collected additional information about the train (e.g., security equipment) and about emergency services available in the area, an emergency response team can be sent to the derailment scene.

Collected information is used for defining and configuring at run-time an incident response plan, defined by a contextually and dynamically selected set of activities to be executed on the field by first responders. Such activities are abstracted into the composite task\footnote{A composite task is a container for another YAWL sub-net, with its own set of elements.} ``Manage Emergency in the Area'' (cf. Fig.~\ref{fig:fig_planlets-yawl_process}(b)). The subnet is composed by three parallel branches with tasks that instruct first responders to act for evacuating people from train coaches, to take pictures and to assess the gravity of the accident.

Note that the YAWL sub-process defined in Figure~\ref{fig:fig_planlets-yawl_process}(b) corresponds to the same process used as case study in Section~\ref{sec:introduction-case_study}. For the sake of simplicity, we also consider the same contextual information and scenario shown in Figure~\ref{fig:fig_introduction-case_study-context_1}(b) and described in Section~\ref{sec:introduction-case_study}.

\section{The YAWL Architecture}
\label{sec:planlets-YAWL_architecture}

The YAWL system, like the majority of classical PMSs, allows to define stable and well-understood processes and offer adequate support as long as the processes are structured and do not require much flexibility. The exception handling capabilities provided by YAWL were designed and implemented starting from the
the conceptual framework for workflow exception handling presented in~\cite{WorkflowExceptionPatterns2006}. In order to understand how exceptions are detected and handled in YAWL we refer to the architecture\footnote{The picture refers to the architecture defined in~\cite{YAWLBook2009}.} in Figure~\ref{fig:fig_planlets-yawl_architecture}.

\begin{figure}[t]
\centering
\includegraphics[width=0.9\columnwidth]{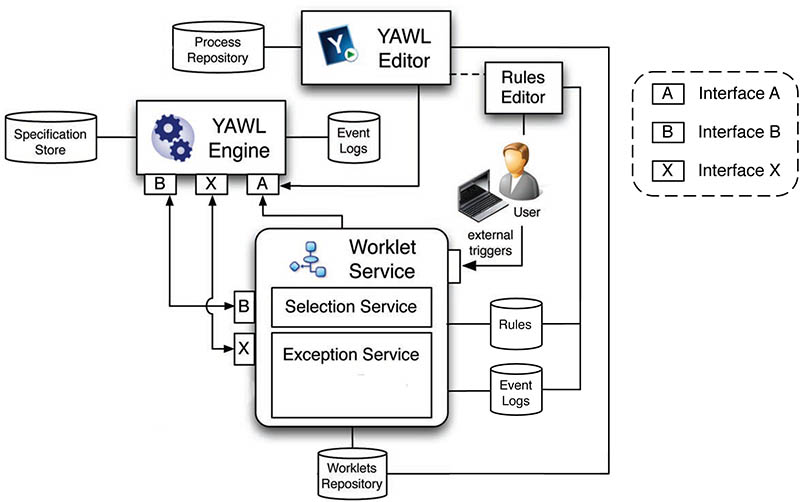}
\caption{The YAWL architecture.}
\label{fig:fig_planlets-yawl_architecture}
\end{figure}

At design-time, the process designer identifies the possible events that may result in deviations from the expected process execution at run-time, and defines both the detection and exception handling strategies. To date, YAWL\footnote{We refer to the final release of YAWL 2.1.} is able to deal with eight different types of exceptions, i.e., case-level and workitem-level pre/post-execution constraints violations, case-level and workitem-level externally triggered exceptions, timeouts and unavailability of resources.

For each exception that can be anticipated, it is possible to define an exception handling process, named \emph{exlet}, which includes a number of exception handling primitives (for removing, suspending, continuing, completing, failing and restarting a workitem/case) and one or more compensatory processes in the form of \emph{worklets} (i.e., self-contained YAWL specifications executed as a replacement for a workitem or as compensatory processes). Exlets are linked to specifications by defining specific \emph{rules} (through the \emph{Rules Editor} graphical tool), in the shape of Ripple Down Rules specified as \texttt{if} \emph{condition} \texttt{then} \emph{conclusion}, where the \emph{condition} defines the exception triggering condition and the \emph{conclusion} defines the exlet.

At run-time, exceptions are detected and managed by the \emph{Exception Service}, a sub-service of the \emph{Worklet Service}.  The \emph{Exception Service} is notified by the engine via \emph{Interface X} of exception triggering events (which include timeouts, resource unavailabilities and notifications fired when a case/workitem begins/ends in order to check pre/post-execution constraints) that may result in an exception\footnote{Externally triggered exceptions are instead notified by the environment, specifically by a client involved in process execution.}. For each event notification, the service determines whether an exception has occurred and, if so, it executes the corresponding exlet. Exception handling primitives are directly executed invoking the corresponding methods (for removing, suspending, etc. a workitem/case) provided by the engine-side of \emph{Interface X}. If the exlet includes a compensation worklet, the \emph{Exception Service} first retrieves it from the repository, then loads it into the engine via \emph{Interface A} and finally starts it via \emph{Interface B}. The worklet is then executed by the engine as a new separate case, possibly in parallel with the parent case if it was not suspended by the exlet. 

\section{Making YAWL and SmartPM interoperate}
\label{sec:planlets-YAWL_SmartPM}



%

The service-oriented approach that characterizes the architecture of YAWL makes the system easily extendable and provides direct support for implementing the \emph{Flexibility as a Service} approach~\cite{faas}. In YAWL the engine manages running cases, but is not directly responsible for task executions. Resources, entities and systems able to execute tasks are abstracted as \emph{services} that interact with the YAWL engine via a set of interfaces. Specifically, the interaction between the engine and the Custom Services mainly occurs through Interface B (cf. Figure~\ref{sec:planlets-YAWL_architecture}).

At design-time each atomic task in a YAWL specification can be associated with a so-called Custom Service~\cite{YAWLimpl} (e.g., the Default Worklist Handler/Resource Service, the Worklet Service, the Declare Service, etc.) that at run-time is responsible for task execution according to its internal logic. To this end, we designed and implemented a Custom Service, named \smartpm Service, that allows the YAWL execution environment to delegate the execution of sub-processes and activities to the \smartpm execution environment which is able to automatically adapt a process to deal with emerging changes and exceptions. As shown in Figure~\ref{fig:fig_planlets-yawl_smartpm}, the \smartpm Service allows the interaction between the two environments.

\begin{figure}[t]
\centering
\includegraphics[width=0.9\columnwidth]{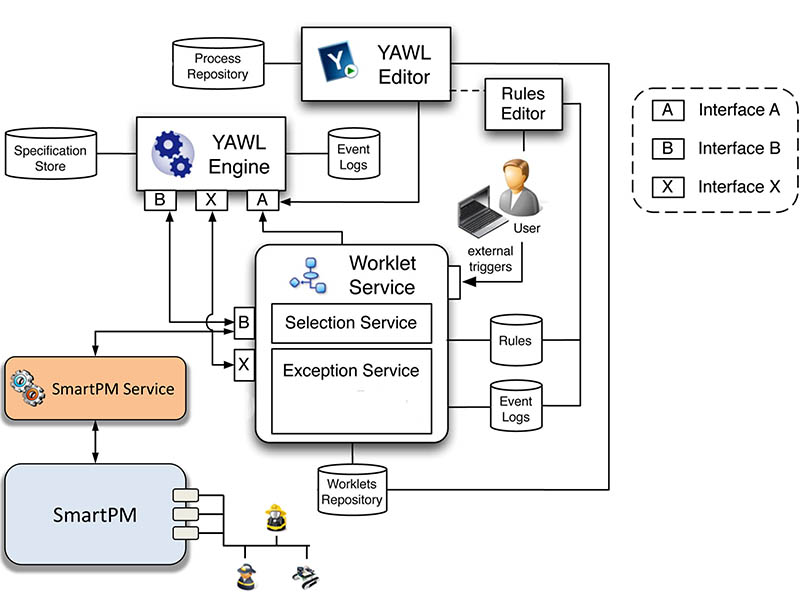}
\caption{The integration between YAWL and \smartpm.}
\label{fig:fig_planlets-yawl_smartpm}
\end{figure}

The \smartpm Service enables the decomposition of YAWL tasks into processes to be executed by \smartpm. At design-time, the process designer is thus able to associate atomic tasks in the YAWL specification with the \smartpm Service. The service requires as input variable a process to be performed by \smartpm, defined according to the formalism introduced in Chapter~\ref{ch:approach}. If the \smartpm process is already available at design-time, it can be directly associated with the input variable required by the service. However, a process to be delegated to \smartpm can also be built starting from an available template, which is configured and finalized by exploiting data produced as output by other tasks in the main YAWL process. In this case, the executable \smartpm process has to be produced as output by a YAWL task that precedes the task associated with the \smartpm Service. The domain-dependent configuration of a \smartpm process can be done either manually by a process designer or automatically by a dedicated service.

In the example shown in Figure~\ref{fig:fig_planlets-yawl_process}(a), we have an high-level process for managing railway accidents, which has been modeled in and is currently managed by YAWL. One of the tasks, specifically the one where a team of operators has to act on the field for evacuating people from train coaches (cf. the complex activity ``Manage Emergency in the Area'' in Figure~\ref{fig:fig_planlets-yawl_process}(a)), represents a step in the process model referring to a complex dynamic process (cf. Figure~\ref{fig:fig_planlets-yawl_process}(b)). As thoroughly described in Section~\ref{sec:introduction-case_study}, we assume that operators are equipped with mobile devices, and aspects like context-awareness and managing frequent exogenous events (e.g., bad connections between devices and operators) is crucial. For this reason, the high dynamism of the operating environment requires that such activities are executed by the \smartpm environment, which provides automatic adaptation features. Therefore, every time the complex activity ``Manage Emergency in the Area'' gets enabled during process execution, the enactment of its corresponding sub-process model is delegated to the \smartpm environment through the \smartpm service.

As any other Custom Service, the \smartpm service implements the service-side of Interface B and is thus able to receive notifications from the YAWL engine when a new work-item is created and delegated to the service for execution. Invoking the specific methods provided by the engine-side of Interface B, the \smartpm Service is then able to check-out a work-item (i.e., notify the engine that the work-item is going to be executed) and execute the corresponding \smartpm input process. Once the execution of the subprocess has been completed, the service can check-in the work-item (i.e., notify the engine of execution completion), again via Interface B, along with the corresponding output. Specifically, the \smartpm Service produces as output, among other possible values, a boolean one that indicates whether the input subprocess was successfully completed or not. Control is then passed back to the YAWL environment, which can continue to carry out the main process.

A screen-cast of a demonstration that shows the interaction between the two environments is available at \url{http://www.dis.uniroma1.it/~marrella/public/DemoBPM2011.zip}. The application presented in the demo is currently a proof of concept and aimed at demonstrating the feasibility and validity of the approach. In the demo application some tasks managed by YAWL and the devices of the operators are simulated, whereas the integration among the two systems and the communication between \smartpm and autonomous devices is real.
The next step towards the definition and implementation of a full-fledged solution has been achieved in~\cite{CoopIS2012} and is described in the following section, where we discusses how concretely the adaptation approach shown in Chapter~\ref{ch:approach} can be built on top of YAWL.


\section{The Planlets Approach}
\label{sec:planlets-approach}

In this section we propose a solution that builds on top of YAWL~\cite{YAWLBook2009}, and consists in annotating at design-time a YAWL specification with additional information which allows process instances to be automatically recovered. In particular, we assume the tasks of a YAWL process specification
to be annotated with \emph{pre-conditions}, \emph{desired effects} and \emph{post-conditions}. Failures arise either when associated pre-conditions for a task are not satisfied at the time the task is to be started, or when post-conditions do not hold after the execution of the task. Effects represent the changes that a successful task execution imposes on the \emph{state} of the world, reflecting the current value of the contextual properties that constraint the process under execution. Hence, the process designer just states \emph{what} conditions have to be satisfied, without having to anticipate \emph{how} these can be fulfilled.
In order to formalize the concept, we introduce the definition of \textsc{Planlet}:
\begin{mydef}[Planlet]
Let $YN$ be a YAWL net, $T$ be the tasks defined in $YN$, and $V$ be the set of variables defined in $YN$.
Let $Expr(V)$ be the set of expressions over the variables in $V$.
A \textsc{Planlet} is a tuple $(YN,Pre,Post,Eff)$ where
\begin{itemize}[itemsep=1pt,parsep=1pt,topsep=1pt]
\item $Pre : T \rightarrow Expr(V)$ returns an expression representing the pre-conditions of tasks in $T$;
\item $Post : T \rightarrow Expr(V)$ returns an expression representing the post-conditions of tasks in $T$;
\item $Eff : T \rightarrow Expr(V)$ returns an expression representing the effects of tasks $T$.
\end{itemize}
\end{mydef}
The role of pre/post-conditions and effects for a YAWL task is twofold:
\begin{enumerate}[itemsep=1pt,parsep=1pt]
\item pre-conditions and post-conditions enable run-time process execution monitoring and exception detection: they are checked respectively before and after task executions, and the violation of a pre-condition or post-condition results in an exception to be handled;
\item along with the input/output parameters consumed/produced by the task, pre-conditions and effects provide a complete specification of the task: this allows the task to be represented as an action in a planning domain description and therefore used for solving a planning problem built to handle an exception.
\end{enumerate}
At design-time, the annotated tasks are stored in a repository linked to the \textsc{Planlet} specification, which may contain also other annotated tasks deriving from previous executions on the same contextual domain. At run-time, while instances of the YAWL specification are carried on, tasks become enabled. Every time a task $t \in T$ becomes enabled, expression $Pre(t)$ is evaluated; similarly, upon the completion of $t$, expression $Post(t)$ is evaluated.
If an evaluation returns false upon enablement or completion of a task, the system is in an \emph{invalid state} and, hence, the YAWL specification instance needs to be adapted to come back into the ``right track''. In order to do that, the case execution is suspended, and a recovery procedure is automatically synthesized.


To provide more details, let us assume that the current \textsc{Planlet} is $\delta_0$ = ($\delta_1;\delta_2$) in which $\delta_1$ is the part of the \textsc{Planlet} already executed and $\delta_2$ is the part of the \textsc{Planlet} which remains to be executed when an exception is identified. The adapted \textsc{Planlet} is $\delta'_0$ = $(\delta_1;\delta_h;\delta_2)$. However, whenever a \textsc{Planlet} needs to be adapted, every running task is interrupted, since the ``repair'' sequence of tasks $\delta_h = [t_1, \ldots, t_n]$ is placed before them. Thus, active branches can only resume their execution after the repair sequence has been executed. This last requirement is fundamental to avoid the risk of introducing data inconsistencies during a repair.

The automatic synthesis of the recovery procedure $\delta_h$ is enacted on-the-fly by an external planner. A planner solves the problem to find a sequence of actions that move a system state from the initial one to a target goal, using a predefined set of admissible actions. Each action is associated the set of pre-conditions $Pre(t)$ in order for that step to be chosen, as well as the effects $Eff(t)$ obtained as result of the action's execution.

Along with defining the set of admissible actions, it is also crucial to define how the state is represented, since pre-conditions and effects of actions are given in term of the chosen state representation. The actions' set and the state definition are often referred to as \emph{planning domain}. The standard representation language of planners to define actions and state is the Planning Domain Definition Language (PDDL)~\cite{PDDL}. In the context of adaptation of instances of YAWL specifications, each task specification is associated with a different action in the planning domain; the task's pre-conditions and effects are translated in PDDL and associated to the corresponding action.
In addition to the so-created planning domain, when an exception arises, the invalid state and the pre-condition (or post-condition) violated is given in input to a planner, which can try to build a plan. If the plan exists, the planner is eventually going to return it. In this case, the plan is converted into a sequence of YAWL tasks which are assigned to qualifying participants. When the converted plan is carried out, the original suspended process is restored for execution.

\vskip 0.5em \noindent\colorbox{light-gray}{\begin{minipage}{0.98\textwidth}
\begin{example}
\emph{Let us consider the example introduced in Section~\ref{sec:planlets-running_example}. The composite activity ``Manage Emergency in the Area'' may be modeled as a self-contained \textsc{Planlet} specification (cf. Fig.~\ref{fig:fig_planlets-yawl_process}(b)), linked to a repository containing a set of emergency management (annotated) tasks, that range from the simple activity of taking pictures to the more complex extinguishment of a fire. An explicit representation of contextual information (the connection of each actor to the network, the map of the area, the battery level of each robot etc.) is needed for preserving the correct \textsc{Planlet} execution.
Suppose now that the task \emph{go(loc00,loc33)} is assigned to actor \emph{act1}, which reaches instead the location \emph{loc03} (cf. Fig.~\ref{fig:fig_introduction-case_study-context_2}(b)). This means that the actor \emph{act1} is not more connected to the network and s/he is in a position different than the desired one, resulting in a \emph{post-condition failure}. The planner builds a planning problem by taking as initial state the invalid state of the \textsc{Planlet}, and as goal a state where all actors/robots are inter-connected to the network and \emph{act1} is in the desired location \emph{loc33}. The recovery plan is automatically synthesized by contextually selecting tasks from the repository linked to the \textsc{Planlet}. Suppose, for example, that the two robots \emph{rb1} and \emph{rb2} have an empty battery. In such a case, the planner devises on-the-fly a possible solution, composed by a sequence of 5 tasks [\emph{chargeBattery(act4,rb1),move(rb1,loc00,loc03),go(act1,loc03,loc33),chargeBattery(act4,rb2),\\move(rb2,loc00,loc33)}] that change the state of the world as shown in Fig.~\ref{fig:fig_introduction-case_study-context_3}(d). Note that the recovery plan is synthesized by taking care of the skills of process participants, and their availability for task assignment and execution. Hence, each task composing the plan is already associated to the participant that will execute it.}
\end{example}
\end{minipage}
}\vskip 0.5em

\subsection{Incorporating Planlets into YAWL}
\label{subsec:planlets-approach-architecture}


\begin{figure}[t]
\centering
\includegraphics[width=0.95\columnwidth]{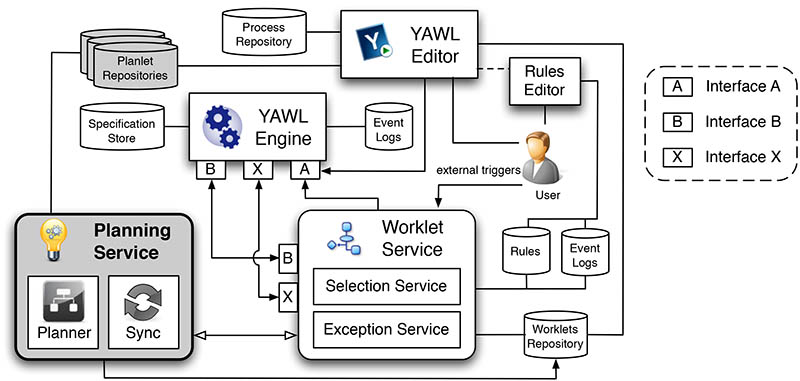}
\caption{The YAWL architecture extended with the \ps.}
\label{fig:architecture}
\end{figure}

The architectural extension and integration we designed takes advantage of YAWL's exception detection capabilities and leverages the flexibility of the exlet-based handling techniques. From an architectural perspective, as shown in Fig.~\ref{fig:architecture}, planning capabilities are provided by a \ps that implements the planning logic and algorithm. In order to define the role of the \ps and clarify how it interacts with existing YAWL architectural components and services, we follow the process and exception handling life-cycle, from process design, enactment and monitoring to exception detection, handling and (possibly) resolution. At design time, the process designer builds one or more \textsc{Planlet} \emph{Repositories} (or modifies the existing ones), by inserting/deleting annotated tasks and by (possibly) modifying the contextual domain linked to each repository. Tasks involved in a \textsc{Planlet} specification are selected from a specific \textsc{Planlet} repository, since they are thought to be enacted in a specific contextual domain. Before to execute a \textsc{Planlet}, the process designer instantiates the initial values for the properties of the contextual domain. As shown in Section~\ref{subsec:planlets-approach-yawl_model}, tasks pre- and post-conditions are automatically translated in YAWL pre- and post-constraints.
%
%
In order
to delegate the exception handling to the \ps, we introduce the possibility of mapping a \emph{compensation activity} to the \ps. By defining this mapping instead of explicitly selecting a compensation worklet, the process designer configures the \emph{Exception Service} so that the generation of the compensation worklet is delegated to the \emph{Planning Service}. Fig.~\ref{fig:planneractivation} shows an excerpt of the rule file
defined for detecting and handling a workitem-level pre-execution (or post-execution) constraint violation. Lines 1-4 define the exception triggering condition (a pre- or a post-condition failure), while lines 5-12 define the exception handling exlet (which consists in suspending the current case, performing some compensation activities and then resuming the suspended case). In our extended version, the mapping of a compensation task to the \ps is identified by a \texttt{<target>} element containing the \texttt{PlanningService} value (line 9), in order to enact planning capabilities.

\begin{figure*}[t!]
\centering
\includegraphics[width=0.75\textwidth]{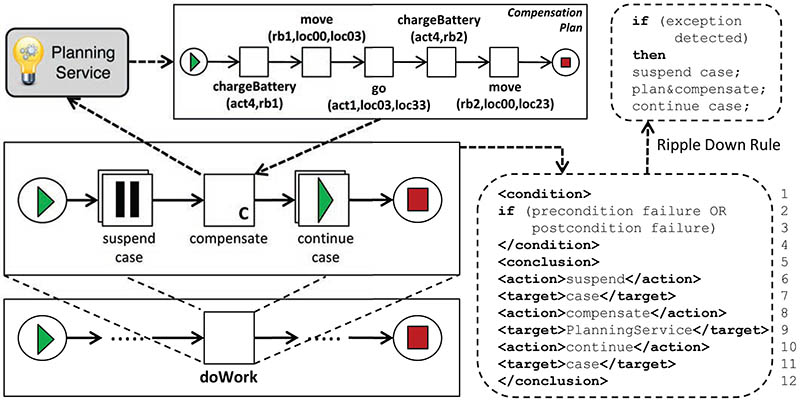}
\caption{\ps activation hierarchy for exception handling.}
\label{fig:planneractivation}
\end{figure*}

\vskip 0.5em \noindent\colorbox{light-gray}{\begin{minipage}{0.98\textwidth}
\begin{example}
\emph{If we consider our running example, the compensation plan devised in Fig.~\ref{fig:planneractivation} corresponds to the one needed for re-establishing the network connection between actors/robots and for instructing actor act1 to move in the desired location.}
\end{example}
\end{minipage}
}\vskip 1em

\noindent\textbf{Planning Service Activation.} When the \emph{Exception Service} activates the \ps, it provides as input all case data associated with the running case, along with the detected violation over pre- and post-conditions. Based on this information,
and on the specifications of available tasks, stored in the repository linked to the \textsc{Planlet} under execution, the \emph{Synchronization} component of the \ps is able to build the planning domain and to define a planning problem, and submit them to the \emph{Planner} module in charge of synthesizing a recovery plan. If the \emph{Planner} is able to successfully synthesize a compensation plan, it stores it as an executable specification (i.e., a worklet) in the \emph{Worklets Repository} and notifies the \emph{Exception Service}. The \emph{Exception Service} is then able to enact the execution of the compensation worklet as if it was manually selected at design time, by loading the specification into the engine and launching it as a separate case.
When the execution completes, output data produced by the worklet are mapped back to the parent case and subsequent actions in the exlet are executed. Following the exlet defined in Fig.~\ref{fig:planneractivation}, as the compensation worklet synthesized by the \emph{Planner} is supposed to recover from the constraint violation, the suspended case can then be resumed and executed. If no valid plan can be found by the \emph{Planner}, a notification alert is sent to an administrator, who is charge of handling the unsolved exception, e.g., manually building a compensation process or just canceling the process case.



\subsection{Annotating YAWL Specifications in Planlets}
\label{subsec:planlets-approach-yawl_model}

A main step of our approach in YAWL consists in enriching the process model with a specification of process tasks, in terms of pre-conditions, desired effects and post-conditions, and with an explicit representation of the contextual domain needed for the correct process enactment.

In YAWL, each atomic task $t$ can be linked to a decomposition.
Decompositions can have a number of input and output \emph{parameters}, each identified by a \emph{name} and characterized by a \emph{type} dictating valid values it may store, and define the so-called YAWL Service that will be responsible for task execution. As process data are represented through net-level variables, inbound and outbound mappings define how data is transferred from net variables to task variables and vice-versa. We propose to extend task specifications at the decomposition level, with the possibility of defining pre-conditions, post-conditions and effects as logical formulae and expressions over task parameters.
\vskip 0.5em
\noindent{\textbf{Defining and representing finite domain types.}}
The definition of a \textsc{Planlet} requires the specification of the data types that characterize the information
manipulated by process instances and define the domains over which predicates and functions are interpreted.
In order to have a compact and finite representation of a process state, given by the values assumed by process variables at a given point in the execution, all data types
must correspond to \emph{finite} domains over which variables of that type can range; this requirement is imposed by the planning-based approach we propose.
Examples of such domains are finite integer intervals or sets of strings, and other enumerated domains. As YAWL applies strong data typing and all data types are defined using XML Schemas, this can be easily achieved by defining data types as XML Schemas and using restrictions (e.g., via the \texttt{enumeration} constraint) to limit the content of an XML element to a set of acceptable values.

In our example, we need to define data types for representing actors, robots\footnote{Although emergency operators and robots can be considered as \emph{resources} or \emph{services} able to execute tasks and can be represented in the organizational model provided by YAWL, we also need to explicitly represent them in the process because we need to define predicates and functions over these domains.} and locations in the area.
The following listing provides the (simplified) definition of the data type $Loc = \{loc00,\,loc10,\,\allowbreak\dotsc\,,\,loc33\}$.

\begin{footnotesize}
\begin{alltt}
<xs:simpleType name="Loc">
  <xs:restriction base="xs:string">
    <xs:enumeration value="loc00"/>
      ...
    <xs:enumeration value="loc33"/>
  </xs:restriction>
</xs:simpleType>
\end{alltt}
\end{footnotesize}

Under this representation, we consider possible values as constant symbols
that univocally identify objects in the domain of interest; different variables of the same type having the same value denote the same object.
\vskip 0.5em
\noindent\textbf{Defining and representing predicates and functions.}
The annotation of tasks with preconditions, postconditions and effects requires the definition of the set of predicates, numeric functions and object functions to be used when extending task definitions. Predicates and functions have to be completely specified at design time. Predicates can be used to express properties of domain objects and relations over objects. A predicate
consists of a predicate \emph{symbol} $P$ and a set of \emph{typed parameters} or \emph{arguments}\footnote{Predicates with no arguments, i.e., with arity 0, are allowed and can be considered as propositions; they are directly represented as boolean variables.}. Argument types (taken from the set of data types previously defined) represent the finite domains over which predicates are interpreted. A generic predicate declaration is of the form:
\begin{align*}
P(arg_1:type_1,\,arg_2:type_2,\,\dotsc\,,\,arg_n:type_n)
\end{align*}
In our example, we may need predicates for expressing the presence of a fire in a location or whether a location is covered by the network signal provided by the main antenna, or relations, such as the adjacency between locations, i.e.,
\begin{align*}
Fire(loc:Loc)\qquad Covered(loc:Loc)\qquad Adjacent(loc1:Loc,\,loc2:Loc)
\end{align*}
In addition to basic predicates, we allow the designer to define \emph{derived} predicates. They are declared as basic predicates, with the additional specification of a well-formed formula $\varphi$
that determines the truth value for the predicate :
\begin{align*}
P(arg_1:type_1,\,\dotsc\,,\,arg_n:type_n)\, \{\, \varphi\, \}
\end{align*}
In our domain, we may need to express that an actor is connected to the network if s/he is in a covered location or if s/he is in a location adjacent to a location where a robot is located (and is thus connected through the robot); assuming we have defined the data types $Robot=\{rb1,\,rb2\}$ and $Actor=\{act1,\,act2,\,act3,\,act4\}$, we have:
\begin{flalign*}
& Connected(act:Actor) \, \{ \, \text{EXISTS}(l1:Loc,\,l2:Loc,\,rbt:Robot) \, (&\\
& (at(act)=l1) \, \text{AND} \, (Covered(l1) \, \text{OR} \, (atRobot(rbt)=l2 \, \text{AND} \, Adjacent(l1,l2))))) \}&
\end{flalign*}

\noindent\textbf{Numeric and Object Functions.} Numeric and object functions allow to represent and handle numeric values and domain objects as functions of other objects. Function declarations
are of the form:
\begin{align*}
f(arg_1:type_1,\,arg_2:type_2,\,\dotsc\,,\,arg_n:type_n):DataType
\end{align*}
and consist of a function \emph{symbol} $f$, a set of \emph{typed parameters}\footnote{Numeric functions with no arguments are allowed, and can be considered as state variables rather than constants, as their value may change during process executions; they are represented as integer or float/double variables.}, and a \emph{return type}.
Numeric functions have as return type an integer or a real number, whereas object functions have a return type taken from the set of data types defined in the net specification. The arguments of functions range over \emph{finite} domains, and for object functions the same requirement holds for result types (i.e., object functions must have finite co-domains).
In our example, we need to keep track of the battery level of the robots.
This can be represented through the numeric function
{$$batteryLevel(robot:Robot):Integer\qquad \text{and}\qquad batteryStep:Integer$$}
Similarly, we can represent the position of actors and robots
by defining the following functional predicates that map actors and robots to their location:
\begin{align*}
at(actor:Actor):Loc\qquad \text{and}\qquad atRobot(robot:Robot):Loc
\end{align*}
\vskip 0.5em
\noindent{\textbf{State variables representation.}}
The use of predicates and functions requires that at run-time we represent the corresponding logical interpretations, as state variables that hold \emph{(a)}~the truth value of the defined predicates over domain objects, and \emph{(b)}~the values of the defined functions with respect to different argument assignments. The interpretations are used to evaluate pre- and post-conditions, and are modified as a result of task executions. As a consequence of the declaration of a predicate or function, two new data types are automatically generated and added to the XML data types definitions for the net:
\begin{enumerate}[itemsep=1pt,parsep=1pt]
\item[T1.] a complex data type that is able to represent the name of the predicate or function and
\begin{itemize}[itemsep=0.5pt,parsep=0.5pt]
\item for predicates, all argument assignments for which the predicate holds\footnote{Basically, a set containing all object tuples for which the predicate is true.} (i.e., the current interpretation $P^{\mathcal{I}}$ for the predicate);
\item for functions, all argument assignments for which the function is defined, along with the corresponding value\footnote{Basically, a map where object tuples are mapped to objects.} (i.e., the current interpretation $f^{\mathcal{I}}$ for the function);
\end{itemize}
\item[T2.] a complex data type that is able to represent a predicate or function instance, in terms of the name of the predicate or function, the set of arguments and their assignment, and the truth value or numeric/object value of the predicate or function with respect to the specific assignment; different parameters of this type can be defined for process tasks, to be used for representing the effects that they can have on the predicate or function interpretation.
\end{enumerate}
For each predicate and function, a single net-level state variable of type T1 is defined and it can be initialized so as to contain all values for the objects for which the predicate is true or the function is defined in the initial state. Derived predicates are not explicitly represented through net-level state variables, as their interpretation can be always derived from the corresponding formula, and they can not appear in task effects (but task effects can act on the basic predicates and functions that appear in the formula, thus indirectly modifying the truth value for the derived predicate).\\
\vskip 0.2em
\noindent{\textbf{Initial interpretation.}}
The initial interpretation over which a process instance is executed is given by an assignment of values to the state variables that represent truth values for predicates and values functions. The initial interpretation is thus defined by a set of grounded predicates (initial facts) and initialization values for functions.\\

The following listings provide an example of a possible initialization for the variables of type T1 for the $Adjacent$ predicate and the $batteryLevel$ function, representing the grounded predicates (facts) $Adjacent(loc00,\,loc01)$, $Adjacent(loc00,\,loc10)$ and so on, and the assignments $batteryLevel(rb1) = 2$ and $batteryLevel(rb2) = 4$.\\

\begin{minipage}[t]{0.5\textwidth}
\begin{footnotesize}
\begin{alltt}
<PredicateAdjacent>
  <name>Adjacent</name>
  <interpretation>
    <loc1>loc00</loc1>
    <loc2>loc01</loc2>
  </interpretation>
  <interpretation>
    <loc1>loc00</loc1>
    <loc2>loc10</loc2>
  </interpretation>
    ...
</PredicateAdjacent>
\end{alltt}
\end{footnotesize}
\end{minipage}
\begin{minipage}[t]{0.5\textwidth}
\begin{footnotesize}
\begin{alltt}
<FunctionBatteryLevel>
  <name>BatteryLevel</name>
  <interpretation>
    <robot>rb1</robot>
    <value>2</value>
  </interpretation>
  <interpretation>
    <robot>rb2</robot>
    <value>4</value>
  </interpretation>
</FunctionBatteryLevel>
\end{alltt}
\end{footnotesize}
\end{minipage}
%
\vskip 0.8em
\noindent{\textbf{Pre-conditions, post-conditions and effects.}}
Pre-conditions, post-conditions and effects are defined at design time as logical annotations associated with tasks in a \textsc{Planlet}.
Basically, we assume a first-order predicate logic with numeric and object functions, with the restriction that free variables are not allowed and thus all variable symbols must be task parameter names or occur in the scope of a quantifier. The language is clearly
inspired from PDDL, although we prefer an infix notation for the operators.

Pre and postconditions can be defined as an atomic formula as well as the conjunction (AND) or disjunction (OR) of formulae,
and formulae can be negated (NOT) and existentially or universally quantified (EXISTS and FORALL).


Atomic formulae are (basic or derived) predicates defined over argument terms, where a term can be a task parameter or a constant;
the equality ($==$) or negated equality (!=) predicates between terms are supported. In addition, it is possible to define conditions as expressions that make use of relational binary comparison operators ($<, >, =, \leq, \geq$) and involve numeric expression. Numeric expressions are constructed, using arithmetic operators ($+, -, *, /$), from primitive numeric expressions, which include integers and real numbers, and numeric functions defined over argument terms. Numeric expressions are not allowed to appear as terms, i.e., as arguments to predicates or values of task parameters (otherwise, predicates would be defined over arguments with \emph{infinite} ranges).
Object functions are defined over argument terms, and can be used as terms only in equality and inequality atomic formulae.

Task effects basically define the changes that a successful task execution imposes on the current state of the world, as represented in the process execution context.
A task effect can be defined as an atomic formula as well as the conjunction of formulae
, and atomic formulae can be negated.
As for preconditions, atomic formulae are predicates defined over argument terms (task parameters or constants), but derived predicates and equality are not allowed. In addition to boolean predicates, task effects may include assignment expressions to update the values of numeric functions. A numeric effect consists of an assignment operator, the numeric function to be updated and a numeric expression. Assignment operators include \myi direct assignment ($=$), to assign to a numeric function a numerical value defined by a number or by a numeric expression; \myii relative assignments, which can be used to increase ($+$=) or decrease ($-$=) the value of a numeric function (additive assignment effects), as well as to scale-up ($*$=) or scale-down (/=) the value of a numeric function (scaling assignment effects). Similarly, it is possible to update the value of object functions with an assignment operator, restricted to direct assignment (=), to assign to an object function a typed object identified by a term.
\vskip 0.5em
\noindent\textbf{Marking up task specifications.}
Task pre-conditions, post-conditions and effects
are represented in task specifications via the \texttt{<precondition>}, \texttt{<effect>} and \texttt{<postcondition>} markup elements. In our example, consider the task labeled as \texttt{go}, which require that an actor moves from a location to another in the area.
It defines two input parameters $from$ and $to$ of type $Loc$, representing the starting and arrival locations, and an input parameter $actor$ of type $Actor$ representing the emergency operator that executes the task. An instance of this task can be executed only if s/he is currently at the starting location and is connected to the network. As a consequence (effect) of task execution, the actor moves from the starting to the arrival location, but we need, as post-condition, to verify whether the arrival location has been reached and the actor is still connected to the network. We can thus define the following annotations:

\begin{footnotesize}
\begin{alltt}
<precondition>at(actor) == from AND Connected(actor)</precondition>
<effect>at(actor) = to</effect>
<postcondition>at(actor) == to AND Connected(actor)</postcondition>
\end{alltt}
\end{footnotesize}

The designer can distinguish between: \myi \emph{direct effects}, i.e., effects that always take place after an execution, and therefore the corresponding changes on the state variables are automatically performed when the task completes (e.g., if an effect of the form $BatteryLevel(robot)$ += $5$ is marked as automatic, after task execution the value for $BatteryLevel(robot)$ is directly increased by 5); and \myii \emph{supposed} effects, i.e., effects that define changes that are assumed to be performed only when the task is considered as an action in a planning domain. Supposed effects can be interpreted as the effects that a task is supposed to have, but the actual produced changes are defined at run-time as a result of the concrete execution, such as the actual truth value of a predicate or the actual value for a direct assignment.

\vskip 0.5em \noindent\colorbox{light-gray}{\begin{minipage}{0.98\textwidth}
\begin{example}
\emph{In our example, \emph{at(actor) = to} is a supposed effect, as the actual value for \emph{at(actor)} is produced as a task output and may be different from the desired one (i.e., the value of the \emph{to} variable prescribed in the effect). If the designer needs to verify whether a task execution has produced the intended effect, s/he has to define a corresponding post-condition (i.e., the \emph{at(actor) == to}).}
\end{example}
\end{minipage}
}\vskip 0.5em

Direct effects can be directly represented by generating an outbound mapping with an XQuery expression that adds/removes a tuple to/from the state variable representing predicate's interpretation (for positive/negative predicates), or updates the value for a tuple in the state variable representing function's interpretation (for assignment effects). In supposed effects, the actual values are produced by workitem executions, and all predicates and functions that appear in the effect expression have to be represented as task variables, so as to allow to specify (according to task's execution logic) the truth value for predicates or the value for functions. To this end, we represent each predicate and function that appears in the supposed effects as task parameters of type T2, where the predicate/function name is given and fixed, the values for the argument variables (i.e., the grounding) are defined by the inbound mappings for task parameters
and the predicate/function actual value will be defined as a result of task execution. For these variables, outbound mappings are then generated, including XQuery expressions to update net-level state variables as for direct effects.

\vskip 0.5em \noindent\colorbox{light-gray}{\begin{minipage}{0.98\textwidth}
\begin{example}
\emph{
Fig.~\ref{fig:fig_planlets-screenshot} shows an example of how a variable can be used to represent an effect and how the actual value for \emph{at(act1)} can be produced as output; we show that the output value for $\emph{at(act1)}$ is produced by a sensor (i.e., a GPS device) supporting the worklist handler. The produced value, in the example 'loc33', is then used to update the net variable representing the \emph{at} interpretation to reflect that \emph{at(act1) $\mapsto$ loc33}.
}
\end{example}
\end{minipage}
}\vskip 0.5em

\begin{figure}[t]
\centering
\includegraphics[width=0.8\columnwidth]{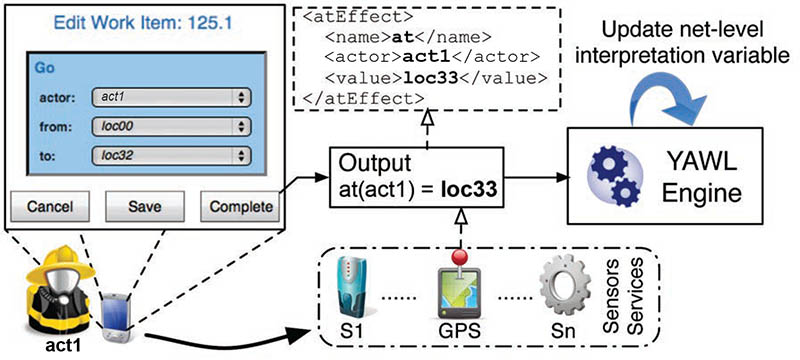}
\caption{A task effect represented as a variable assignment.}
\label{fig:fig_planlets-screenshot}
\end{figure}
\noindent{\textbf{State model and exceptions}}
In a \textsc{Planlet}, a process state $S$ is given by the token marking $m_S$ (as defined in~\cite{YAWLBook2009} for YAWL nets) and the logical interpretation $\mathcal{I_S}$ that assigns truth values to predicates and values to functions.
The initial state over which a process instance is executed is given by the initial marking and an assignment of values to the state variables that represent the initial interpretation for predicates and functions.
When a task $t$
becomes enabled in a state $S$ (as determined by $m_S$), its execution can start only if the task precondition formula $\varphi_{pre}$ is true in $\mathcal{I}_S$, i.e.,  $\mathcal{I}_S \models \varphi_{pre}$. A task execution changes the interpretation according to actual task effects (which for a successful execution are given by the corresponding effects expression $expr_{\emph{\text{eff}}}$) and leads to a new state $S'$ where $m_{S'}$ is the produced marking and $\mathcal{I}_{S'}$ is the new interpretation. A completed task is considered as successfully executed if its postcondition formula $\varphi_{post}$ is true in $\mathcal{I}_{S'}$, i.e., $\mathcal{I}_{S'} \models \varphi_{post}$.
At run time all task executions are thus preceded and followed by the verification of  whether $\mathcal{I} \models \varphi_{pre}$ and $\mathcal{I} \models \varphi_{post}$\footnote{As $\varphi_{pre}$ and $\varphi_{post}$ are \emph{closed} formulae, their truth values can be considered as the answers to the corresponding boolean queries, given the interpretation $\mathcal{I}$.}.
In this model, an exception occurs in a given state with an interpretation $\mathcal{I}$ if a task is enabled but $\mathcal{I} \not\models \varphi_{pre}$ or if a task has completed but $\mathcal{I} \not\models \varphi_{post}$.
\vskip 0.5em
\noindent{\textbf{From pre-/post-conditions to pre-/post-execution constraints}}
As part of its exception handling mechanism, YAWL supports the definition of workitem-level pre- and post-execution constraints, as rules with conditions that \myi are checked when the workitem becomes enabled and when it is completed,
and \myii if violated, they trigger an exception and the execution of an exception handling process (i.e., a YAWL exlet). Conditions are defined over case variables as strings of operands and arithmetic, comparison and logical operators; conditional expressions may also take the form of boolean XQuery expressions~\cite{YAWLBook2009}.
In our approach, we leverage on this built-in feature and map the evaluation of pre- and post-conditions to the evaluation of pre and post-execution constraints, by automatically translating $\varphi_{pre}$ and $\varphi_{post}$ formulae for each task into YAWL conditional expressions. While arithmetic, comparison and logical operators in our annotation language directly map to the operators supported by YAWL, predicates and functions can be resolved by appropriate XQuery expressions, according to a standard evaluation algorithm for boolean queries\footnote{we recall that no free variables are allowed and all formulae are closed}; basically:
\begin{itemize}[itemsep=1pt,parsep=1pt]
\item all variables in the formula that correspond to task parameters (for which an inbound or outbound mapping from or to a net variable is defined as an XQuery expression) assume the values as defined by the mappings 
\item each function, which becomes defined over a tuple of ground terms, is replaced by an XQuery expression that extract from function's interpretation state variable the value corresponding to the tuple (i.e., $f(c_1,\, \dotsc\,,\,c_k)$ is replaced by the value $f^{\mathcal{I}}(c_1,\, \dotsc\,,\,c_k)$);
\item each predicate, which becomes defined over a tuple of ground terms, is replaced by a boolean XQuery expressions that verify whether the tuple is contained in the predicate interpretation state variable ($\mathcal{I} \models P(c_1,\, \dotsc\,,\,c_k)$ if $(c_1,\,\dotsc \,,\, c_n) \in P^{\mathcal{I}}$)
\item existentially and universally quantified formulae are replaced by boolean XQuery expressions that verify whether there exist an assignment for the quantified typed variable that makes the formula true, or whether the formula is true for all possible assignments (among the finite possible values defined in the data type specification)
\item derived predicates in preconditions are replaced by the corresponding formulae, which are evaluated as per the previous items
\end{itemize}
\vskip 0.5em
\noindent{\textbf{Representing Planlet Annotations in PDDL.}}
\label{YAWLPDDL}
In order to exploit our planning-based recovery mechanism, every task/annotation/property associated to a \textsc{Planlet} needs to be translated in PDDL. A PDDL definition consists of two parts: the domain and the problem definition. The planning domain is built starting by the definition of basic/derived predicates, object/numeric functions and data types as shown in the previous sections, and by making explicit the \emph{actions} associated to each annotated task stored in the repository linked to the \textsc{Planlet} under execution, together with the associated pre-conditions, effects and input parameters. Basically, the planning domain describes how predicates and functions may vary after an action execution, and reflects the contextual properties constraining the execution of tasks stored in a specific \textsc{Planlet} repository. Our annotation syntax allows to represent planning domains and problems with the complexity of those describable in PDDL version 2.2~\cite{PDDL22}, that is characterized for enabling the representation of realistic planning domains.

In the following, we discuss how our annotations are translated into a PDDL file representing the planning domain:
\begin{itemize}[itemsep=1pt,parsep=1pt]
\item the \emph{name} and the \emph{domain} of a data type corresponds to an \emph{object type} in the planning domain;
\item basic and derived predicates have a straightforward representation as \emph{relational predicates} (templates for logical facts) and \emph{derived predicates} (to model the dependency of given facts from other facts) in the planning domain;
\item numeric functions correspond to PDDL \emph{numeric fluents}, and are used for modeling non-boolean resources (e.g., the battery level of a robot) in the planning domain;
\item object functions do not have a direct representation in PDDLv2.2, but may be replaced as relational predicates. Since an object function $f:Object^n \to Object$ map tuples of objects with domain types $D^n$ to objects with co-domain type $U$, it may be coded in the planning domain as a relational predicate \emph{P} of type $(D^n,U)$;
\item a given YAWL task, together with the associated pre-conditions and effects and input parameters, is translated in a PDDL \emph{action schema}. An action schema describes how the relational predicates and/or numeric fluents may vary after the action execution. In the following, it is shown the PDDL representation of the task \emph{go}:
    \begin{footnotesize}\begin{alltt}
    (:action go
    :parameters (?x - actor ?from - loc ?to - loc)
    :precondition (and (not (at ?x ?to)) (at ?x ?from) (connected ?x))
    :effect (and (not (at ?x ?from)) (at ?x ?to)))\end{alltt}\end{footnotesize}
    This task can be executed only if the actor denoted with \emph{x} is not currently located in the target location \emph{to} (and, consequently, is located in his/her starting location \emph{from}) and is connected to the network. The desired effect turns the value of the predicate $at(x,to)$ to TRUE and $at(x,from)$ to FALSE, meaning the actor moved in a new location.
\end{itemize}
When an exception arises, on a same planning domain a new planning problem is built at run-time, through the description of an initial state (that corresponds to the invalid state $s$ of the process) and the description of the desired goal (a safe state $s'$, derived from the violated pre- or post-condition).
\begin{itemize}[itemsep=1pt,parsep=1pt]
\item for each data type defined in the planning domain, all the possible object instances of that particular data type are explicitly instantiated as \emph{constant symbols} in the planning problem (e.g., the fact that $act1$, $act2$, $act3$, $act4$ are \emph{Actors}, $rb1$ and $rb2$ are \emph{Robots}, $loc00$, ..., $loc33$ are \emph{Locations});
\item a representation of the \emph{initial state} of the planning environment is needed. Basically, the initial state of the planning problem corresponds to an invalid state (i.e., a state that needs to be fixed after a pre- or post-condition violation during the process execution). It is composed by a conjunction of relational predicates, derived predicates (e.g., the information about which actors/robots are currently connected to the network) and by the current value of each numeric fluent (e.g., the value of the battery level for each robot);
\item the \emph{goal state} of the planning problem is a logical expression over facts. In our approach, the goal state is built in order to reflect a safe state to be reached after the execution of a recovery procedure. Suppose that $t$ is the task whose pre-conditions $Pre(t)$ (or post-conditions $Post(t)$) are not verified. The safe state $s'$ corresponding to the goal state of the planning problem is generated starting from the invalid state $s$, by substituting the wrong facts that led to the exception with the content of the pre-conditions (or post-conditions) violated.
    \vskip 0.5em \noindent\colorbox{light-gray}{\begin{minipage}{0.9\textwidth}
    \begin{example}
    \emph{In the case of our running example, actor \emph{act1} reaches a wrong location \emph{loc03} after the execution of the task \emph{go}, and s/he is no longer connected to the network, resulting in a post-condition failure (cf. Figure~\ref{fig:fig_introduction-case_study-context_2}(b)). Therefore, whereas the initial state of the planning problem generated for dealing with such an exception corresponds exactly to the invalid state of the process (hence, with the value of the predicate \emph{Connected(act1)=FALSE} and \emph{at(act1)=loc03}), the goal state is composed by all those predicates/functions not affected by the task \emph{go} (that must remain unchanged after the recovery procedure), and by the post-conditions just violated, i.e., \emph{Connected(act1)=TRUE} and \emph{at(act1)=loc33}, which is the desired location.}
    \end{example}
    \end{minipage}
    }\vskip 0.5em
   \end{itemize}

\section{Conclusion}
\label{sec:planlets-conclusions}

In this chapter, we have customized our general approach for automatic adaptation of dynamic processes presented in Chapter~\ref{ch:approach} by discussing its deployment on top of existing PMSs. 

Specifically, we have first shown how the YAWL environment can be complemented with the \smartpm execution environment by leveraging the ``Flexibility as a Service'' approach~\cite{faas}. A screen-cast of a demonstration that shows the interaction between the two environments is available at \url{http://www.dis.uniroma1.it/~marrella/public/DemoBPM2011.zip}.

Then, we have presented a concrete design and implementation proposal of how the YAWL architecture can be extended to integrate planning capabilities. For this aim, we have proposed the approach of \planlets, self-contained YAWL specifications with recovery features, based on modeling of pre- and post-conditions of tasks and the use of planning techniques. The feasibility of the \textsc{Planlet} approach has been investigated by performing some testing to learn the time amount needed for synthesizing a recovery plan for different adaptation problems. We made our tests by using the LPG-td planner~\cite{LPG,LPGth}, and the obtained results can be compared with the ones shown in Section~\ref{sec:validation-experiments}.

Currently, the implementation of the \planlets approach is ongoing, in collaboration with researchers from the Queensland University of Technology (QUT, Australia).

\chapter{Conclusion}
\label{ch:conclusion}

The research activity outlined in this thesis has been devoted to define a general approach, a concrete architecture and a prototype PMS for automatic adaptation of dynamic processes, on the basis of a declarative specification of process tasks and relying on well-established planning techniques. Our purpose was to demonstrate that the combination of procedural and imperative models with declarative elements, along with the exploitation of techniques from the field of artificial intelligence (AI) such as planning algorithms and tools, can increase the ability of existing PMSs of supporting dynamic processes.

To this end, we developed a prototype PMS named \smartpm, which is specifically tailored for supporting collaborative work of process participants during pervasive scenarios. 
The adaptation mechanism deployed on \smartpm is based on execution monitoring for detecting failures at run-time, which does not require the definition of the adaptation strategy in the process itself (as most of the current approaches do), and on automatic planning techniques for the synthesis of the recovery procedure.


In order to exploit the automatic adaptation features provided by \smartpm, some extra modeling effort is required at design time. In fact, our approach requires that processes are defined at best partly textually (for the formalization of the domain theory associated to the dynamic process) and partly graphically (for the definition of the process control flow). However, we think the overhead is compensated at run-time by the automation of adaptation procedures. While, in general, such modelling effort may seem significant, in practice it is comparable to the effort needed to encode the adaptation logic using alternative methodologies like happens, for example, in rule-based approaches.


The assumptions of classical planning (determinism in the action effects, the closed-world assumption for the instantiation of the initial situation, etc.) we used for modeling dynamic processes has a twofold consequence. On the one hand, we can exploit the good performance of classical planners (e.g., LPG-td~\cite{LPG}, POPF2~\cite{POPF2_2010}) to solve real-world problems with a realistic complexity. In fact, test results reported in Section~\ref{sec:validation-experiments} show that the time overhead introduced by a planner for synthesizing recovery procedures with variable length for adaptation problems of growing complexity (in terms of size and parameters of the corresponding planning problems) is in the order of tens of seconds, which makes the approach feasible for medium-sized dynamic processes used in practice. On the other hand, classical planning imposes some restrictions for addressing more expressive problems, including incomplete information, preferences and multiple task effects. Future works will include an extension of our approach dealing with the above aspects and by considering also classical business scenarios, with the purpose to maintain the planning process very responsive.

A second future work concerns to integrate the approach for building process templates (cf. Chapter~\ref{ch:templates}) in the \smartpm system.
In fact, when considering specific classes of emergency management processes (e.g., earthquakes, floods, building collapses, etc.), it is possible to identify recurring and predefined activities that have to be performed according to general guidelines and emergency action plans. In order to balance between the need to create ad-hoc process
specifications and to exploit guidelines and recurrent activity patterns, it is possible to take
advantage of our approach for generating process templates. Processes can be designed on-demand from provided templates, which are tailored and configured to create an appropriate process specification. This kind of approach would allow to speed up process definition activities, as well as to facilitate the reuse of process fragments. 

Finally, even if our approach is able to adapt a process instance at run-time, it does not allow to evolve the original process model on the basis of exceptions captured. Therefore, a third future work concerns to avoid to consider all deviations from the process as errors, but as a natural and valuable part of the work activity, which provides the opportunity for learning and thus evolving the process model for future instantiations.


\appendix
\chapter{The Full Code of the Example}
\label{appendix_A}

This appendix is focussed primarily for describing the full code of the example related to our case study. We show, in sequence, the \smartML specification linked to the dynamic process depicted in Section~\ref{sec:introduction-case_study}, the \indigolog program to be executed by the PMS and the PDDL Planning Domain and Problem built for dealing with the exception detected in Fig.~\ref{fig:fig_introduction-case_study-context_2}. We underline that the \indigolog program reflecting the approach shown in Chapter~\ref{ch:approach} is written with the \texttt{Prolog} language.


\section*{The \smartML specification of the domain theory}

\begin{footnotesize}
\begin{alltt}
TYPES = \{Participant, Capability, Location_type, Status_type, Boolean\_type, Integer\_type\}

    \emph{\textbf{Participant} = \{act1,act2,act3,act4,rb1,rb2\}}
    \emph{\textbf{Capability} = \{movement,hatchet,camera,gprs,extinguisher,battery,digger,powerpack\}}
    \emph{\textbf{Location_type} = \{loc00,loc10,loc20,loc30,loc01,loc11,loc02,loc03,loc13,loc23,loc31,
          loc32,loc33\}}
    \emph{\textbf{Status_type} = \{ok,fire,debris\}}
    \emph{\textbf{Boolean\_type} = \{true,false\}}
    \emph{\textbf{Integer\_type} = \{0..30\}}

PRE-DEFINED TERMS = \{provides, requires\}

    \emph{\textbf{provides[prt:Participant]} = (bool:Boolean\_type)}
    \emph{\textbf{requires[prt:Participant]} = (bool:Boolean\_type)}

NON RELEVANT TERMS = \{atRobot, batteryLevel, photoTaken, neigh, covered, generalBattery,\\batteryRecharging, moveStep, debrisStep\}

    \emph{\textbf{atRobot[prt:Participant]} = (loc:Location_type)}
    \emph{\textbf{batteryLevel[prt:Participant]} = (int:Integer\_type)}
    \emph{\textbf{photoTaken[loc:Location_type]} = (int:Integer\_type)}
    \emph{\textbf{neigh[loc1:Location_type,loc2:Location_type]} = (bool:Boolean\_type)}
    \emph{\textbf{covered[loc:Location_type]} = (bool:Boolean\_type)}
    \emph{\textbf{generalBattery[]} = (int:Integer\_type)}
    \emph{\textbf{batteryRecharging[]} = (int:Integer\_type)}
    \emph{\textbf{moveStep[]} = (int:Integer\_type)}
    \emph{\textbf{debrisStep[]} = (int:Integer\_type)}

RELEVANT TERMS = \{evacuated, at, status\}

   \emph{\textbf{at[prt:Participant]} = (loc:Location_type)}
   \emph{\textbf{evacuated[loc:Location_type]} = (bool:Boolean\_type)}
   \emph{\textbf{status[loc:Location_type]} = (st:Status_type)}

COMPLEX TERMS = \{isConnected, isRobotConnected\}

isConnected(prt:Participant) \{
\quad \text{EXISTS}(l1:Location\_type,\,l2:Location\_type,\,rbt:Participant).((at(act)=l1) \text{AND}
\quad (covered(l1) \, \text{OR} \, (atRobot(rbt)=l2 \, \text{AND} \, neigh(l1,l2)=true \text{AND} isRobotConnected(rbt))))\}.

isRobotConnected(rbt:Participant) \{
\quad \text{EXISTS}(l1:Location\_type,\,l2:Location\_type).((atRobot(rbt)=l1) \text{AND}
\quad (covered(l1) \text{OR}
\quad (neigh(l1,l2) \text{AND} Covered(l2)) \text{OR}
\quad (\text{EXISTS}(l3:Location\_type,\,l4:Location\_type,
\quad \quad \quad \quad l5:Location\_type,\,rbt2:Participant).atRobot(rbt2)=l3 \text{AND}
\quad \quad ((neigh(l1,l3) AND covered(l3)) OR
\quad \quad  (neigh(l1,l5) AND neigh(l3,l5) AND covered(l3)) OR
\quad \quad  (neigh(l3,l4) AND covered(l4) AND neigh(l1,l3)) OR
\quad \quad  (neigh(l3,l4) AND covered(l4) AND neigh(l1,l5) AND neigh(l3,l5))))))\}.


EXOGENOUS EVENTS = \{\photoLost, \fireRisk, \rockSlide\}

<ex-events>
    <ex-event>
        <name>rockSlide</name>
        <parameters>
            <arg>loc - Location_type</arg>
        </parameters>
        <effects>
            <automatic>status[loc] = debris</automatic>
        </effects>
    </ex-event>
    <ex-event>
        <name>fireRisk</name>
        <parameters>
            <arg>loc - Location_type</arg>
        </parameters>
        <effects>
            <automatic>status[loc] = fire</automatic>
        </effects>
    </ex-event>
    <ex-event>
        <name>photoLost</name>
        <parameters>
        <arg>loc - Location_type</arg>
        </parameters>
        <effects>
        <automatic>photoTaken[loc] = false</automatic>
        </effects>
    </ex-event>
</ex-events>

TASKS REPOSITORY = \{go, move, takephoto, evacuate, updatestatus, 
extinguishfire, chargebattery\}

<tasks>
    <task>
        <name>go</name>
        <parameters>
            <arg>from - Location_type</arg>
            <arg>to - Location_type</arg>
        </parameters>
        <precondition>at[PRT] == from AND isConnected[PRT] == true</precondition>
        <effects>
            <supposed>at[PRT] = to</supposed>
        </effects>
    </task>
    <task>
        <name>move</name>
        <parameters>
            <arg>from - Location_type</arg>
            <arg>to - Location_type</arg>
        </parameters>
        <precondition>atRobot[PRT] == from AND batteryLevel[PRT] >= moveStep[] AND
                      isRobotConnected[PRT] == true
        </precondition>
        <effects>
            <supposed>atRobot[PRT] = to</supposed>
            <automatic>batteryLevel[PRT] -= moveStep[]</automatic>
        </effects>
    </task>
    <task>
        <name>takephoto</name>
        <parameters>
            <arg>loc - Location_type</arg>
        </parameters>
        <precondition>at[PRT] == loc AND isConnected[PRT] == true</precondition>
        <effects>
            <supposed>photoTaken[loc] = true</supposed>
        </effects>
    </task>
    <task>
        <name>evacuate</name>
        <parameters>
            <arg>loc - Location_type</arg>
        </parameters>
        <precondition>at[PRT] == loc AND isConnected[PRT] == true AND
                      status[loc] == ok AND evacuated[loc] == false
        </precondition>
        <effects>
            <supposed>evacuated[loc] = true</supposed>
        </effects>
    </task>
    <task>
        <name>updatestatus</name>
        <parameters>
            <arg>loc - Location_type</arg>
        </parameters>
        <precondition>at[PRT] == loc AND isConnected[PRT] == true AND
                      status[loc] == ok
        </precondition>
        <effects>
            <supposed>status[loc] = ok</supposed>
        </effects>
    </task>
    <task>
        <name>removedebris</name>
        <parameters>
            <arg>loc - Location_type</arg>
        </parameters>
        <precondition>atRobot[PRT] == loc AND isRobotConnected[PRT] == true AND
                      status[loc] == debris AND batteryLevel[PRT] >= debrisStep[]
        </precondition>
        <effects>
            <supposed>status[loc] = ok</supposed>
            <automatic>batteryLevel[PRT] -= debrisStep[]</automatic>
        </effects>
    </task>
    <task>
        <name>extinguishfire</name>
        <parameters>
            <arg>loc - Location_type</arg>
        </parameters>
        <precondition>at[PRT] == loc AND isConnected[PRT] == true AND
                      status[loc] == fire
        </precondition>
        <effects>
            <supposed>status[loc] = ok</supposed>
        </effects>
    </task>
    <task>
        <name>chargebattery</name>
        <parameters>
            <arg>rb - Participant</arg>
        </parameters>
        <precondition>provides[rb,battery] == true AND at[PRT] == atRobot[rb] AND
                      isConnected[PRT] == true AND generalBattery[] >= batteryRecharging[]
        </precondition>
        <effects>
            <automatic>batteryLevel[rb] += batteryRecharging[]</automatic>
            <automatic>generalBattery[] -= batteryRecharging[]</automatic>
        </effects>
    </task>
</tasks>
\end{alltt}
\end{footnotesize}

\section*{The \indigolog Program}

\begin{footnotesize}
\begin{alltt}
\%\%\%\%\%\%\%\%\%\%\%\%\%\%\%\%\%\%\%\%
\% FILE: aPMS.pl
\% DESCRIPTION : The most recent version of the IndiGolog adaptive Process Management System.
\% VERSION : 1.0
\%\%\%\%\%\%\%\%\%\%\%\%\%\%\%\%\%\%\%\%

/* SOME PRE-BUILT COMMANDS AND MACROS*/
:- dynamic controller/1.

cache(_):-fail.

causes_true(_,_,_) :- false.
causes_false(_,_,_) :- false.

actionNum(X,X).

square(X,Y) :- Y is X * X.
member(ELEM,[HEAD|_]) :- ELEM=HEAD.
member(ELEM,[_|TAIL]) :- member(ELEM,TAIL).
listEqual(L1,L2) :- subset(L1,L2),subset(L2,L1).

/* THE LIST OF AVAILABLE SERVICES */
service(S) :- domain(S,[act1,act2,act3,act4,rb1,rb2]).

/* THE LIST OF CAPABILITIES RELEVANT FOR THE PROCESS OF INTEREST */
capability(B) :- domain(B,[movement,hatchet,camera,gprs,extinguisher,battery,digger,powerpack]).

/* THE CAPABILITIES PROVIDED BY EACH SERVICE*/
provides(act1,movement).
provides(act1,gprs).
provides(act1,extinguisher).
provides(act1,camera).
provides(act2,movement).
provides(act2,gprs).
provides(act2,hatchet).
provides(act3,movement).
provides(act3,gprs).
provides(act3,hatchet).
provides(act4,movement).
provides(act4,powerpack).
provides(act4,gprs).
provides(rb1,battery).
provides(rb1,digger).
provides(rb2,battery).
provides(rb2,digger).

/* PRE-DEFINED DATA TYPES */

boolean_type(Q) :- domain(Q,[true,false]).

integer_type(N) :- domain(N,[0,1,2,3,4,5,6,7,8,9,10,11,12,13,14,15,16,17,18,19,20,
21,22,23,24,25,26,27,28,29,30]).

/* USER-DEFINED DATA TYPES */

status_type(S) :- domain(S,[ok,fire,debris]).

location_type(L) :- domain(L,[loc00,loc01,loc02,loc03,loc10,loc11,loc13,loc20,loc23,loc30,
loc31,loc32,loc33]).

/* TASKS DEFINED IN THE PROCESS SPECIFICATION */
task(T) :- domain(T,[chargebattery,go,move,evacuate,removedebris,takephoto,
updatestatus,extinguishfire]).

/* THE LIST OF ADMISSIBLE IDENTIFIERS */
task_identifiers([id_1,id_2,id_3,id_4,id_5,id_6,id_7,id_8,id_9,id_10,
id_11,id_12,id_13,id_14,id_15,id_16,id_17,id_18,id_19,id_20,id_21,
id_22,id_23,id_24,id_25,id_26,id_27,id_28,id_29,id_30,id_adapt]).
id(D) :- domain(D,task_identifiers).

/* THE CAPABILITIES REQUIRED BY EACH TASK*/
requires(go,movement).
requires(evacuate,hatchet).
requires(takephoto,camera).
requires(updatestatus,gprs).
requires(extinguishfire,extinguisher).
requires(move,battery).
requires(removedebris,digger).
requires(chargebattery,powerpack).

/* THE LIST OF ADMISSIBLE WORKITEMS */
workitem(go,ID,[FROM,TO],[TO]) :- id(ID),location_type(FROM),location_type(TO).
workitem(takephoto,ID,[LOC],[true]) :- id(ID),location_type(LOC),boolean_type(true).
workitem(evacuate,ID,[LOC],[true]) :- id(ID),location_type(LOC),boolean_type(true).
workitem(move,ID,[FROM,TO],[TO]) :- id(ID),location_type(FROM),location_type(TO).
workitem(chargebattery,ID,[RB],[]) :- id(ID),service(RB).
workitem(removedebris,ID,[LOC],[ST]) :- id(ID),location_type(LOC),status_type(ST).
workitem(extinguishfire,ID,[LOC],[ST]) :- id(ID),location_type(LOC),status_type(ST).
workitem(updatestatus,ID,[LOC],[ST]) :- id(ID),location_type(LOC),status_type(ST).

/* TASKS PRECONDITIONS */

prim_action(assign(SRVC,ID,TASK,_INPUTS,_EXOUTPUTS)) :- service(SRVC),id(ID),task(TASK).

poss(assign(SRVC,[workitem(go,ID,[FROM,TO],[TO])]),
    and(isConnected(SRVC),at(SRVC)=FROM)) :-
        service(SRVC),id(ID),location_type(FROM),location_type(TO).

poss(assign(SRVC,[workitem(takephoto,ID,[LOC],[true])]),
    and(isConnected(SRVC),at(SRVC)=LOC)) :-
        service(SRVC),id(ID),location_type(LOC),boolean_type(true).

poss(assign(SRVC,[workitem(evacuate,ID,[LOC],[true])]),
    and(isConnected(SRVC),and(at(SRVC)=LOC,
    and(evacuated(LOC)=false,status(LOC)=ok)))) :-
        service(SRVC),id(ID),location_type(LOC),boolean_type(true).

poss(assign(SRVC,[workitem(move,ID,[FROM,TO],[TO])]),
    and(isRobotConnected(SRVC),
    and(atRobot(SRVC)=FROM,batteryLevel(SRVC)>=moveStep))) :-
        service(SRVC),id(ID),location_type(FROM),location_type(TO).

poss(assign(SRVC,[workitem(chargebattery,ID,[RB],[])]),
    and(isConnected(SRVC),(and(at(SRVC)=atRobot(RB),
    and(provides(RB,battery),generalBattery>=batteryRecharging))))) :-
        service(SRVC),id(ID),service(RB).

poss(assign(SRVC,[workitem(removedebris,ID,[LOC],[ST])]),
    and(isRobotConnected(SRVC),and(atRobot(SRVC)=LOC,
    and(status(LOC)=debris,batteryLevel(SRVC)>=debrisStep)))) :-
        service(SRVC),id(ID),location_type(LOC),status_type(ST).

poss(assign(SRVC,[workitem(extinguishfire,ID,[LOC],[ST])]),
    and(isConnected(SRVC),and(at(SRVC)=LOC,status(LOC)=fire))) :-
        service(SRVC),id(ID),location_type(LOC),status_type(ST).

poss(assign(SRVC,[workitem(updatestatus,ID,[LOC],[ST])]),
    and(isConnected(SRVC),and(at(SRVC)=LOC,status(LOC)=ok))) :-
        service(SRVC),id(ID),location_type(LOC),status_type(ST).

/* BASIC ACTIONS AND ENGINE FLUENTS FOR MANAGING THE TASK LIFE-CYCLE */

rel_fluent(assigned(SRVC,ID,TASK)) :- service(SRVC),id(ID),task(TASK).
causes_val(assign(SRVC,ID,TASK,_INPUTS,_EXOUTPUTS),assigned(SRVC,ID,TASK),true,true).
causes_val(release(SRVC,ID,TASK,_INPUTS,_EXOUTPUTS,_PHOUTPUTS),assigned(SRVC,ID,TASK),false,true).

rel_fluent(reserved(SRVC,ID,TASK)) :- task(TASK), service(SRVC), id(ID).
causes_val(readyToStart(SRVC,ID,TASK),reserved(SRVC,ID,TASK),true,true).
causes_val(finishedTask(SRVC,ID,TASK,_),reserved(SRVC,ID,TASK),false,true).

prim_action(start(SRVC,ID,TASK)) :- service(SRVC), id(ID), task(TASK).
poss(start(SRVC,ID,TASK),reserved(SRVC,ID,TASK)) :- service(SRVC),id(ID),task(TASK).

prim_action(ackCompl(SRVC,ID,TASK)) :- task(TASK),service(SRVC),id(ID).
poss(ackCompl(SRVC,ID,TASK),neg(reserved(SRVC,ID,TASK))).

prim_action(release(SRVC,ID,TASK,_INPUTS,_EXOUTPUTS,_PHOUTPUTS)) :- service(SRVC),id(ID),task(TASK).
poss(release(_SRVC,_ID,_TASK,_INPUTS,_EXOUTPUTS,_PHOUTPUTS), true).

rel_fluent(free(SRVC)) :- service(SRVC).
causes_val(release(SRVC,_ID,_TASK,_INPUTS,_EXOUTPUTS,_PHOUTPUTS),free(SRVC),true,true).
causes_val(assign(SRVC,_ID,_TASK,_INPUTS,_EXOUTPUTS),free(SRVC),false,true).

/* IF THIS FLUENT HOLDS, THE SYSTEM SWITCHES FROM USING THE PLAN-BASED ADAPTATION MECHANISM
TO USE THE BUILT-IN ADAPTATION APPROACH */
rel_fluent(built_in_adaptation).

/*USER-DEFINED STATIC FLUENTS*/
fun_fluent(neigh(LOC1,LOC2)) :- location_type(LOC1),location_type(LOC2).
fun_fluent(covered(LOC)) :- location_type(LOC).

/*FLUENT AT*/
fun_fluent(at(SRVC)) :- service(SRVC).

causes_val(release(SRVC,ID,go,_I,_E,[OUTPUT]),at(SRVC),OUTPUT,
    assigned(SRVC,[workitem(go,ID,[FROM,TO],[TO])])) :-
	   service(SRVC),id(ID),location_type(OUTPUT),location_type(FROM),location_type(TO).

fun_fluent(at_exp(SRVC)) :- service(SRVC).

causes_val(release(SRVC,ID,go,_I,_E,[OUTPUT]),at_exp(SRVC),TO,
    and(neg(adapting),assigned(SRVC,[workitem(go,ID,[FROM,TO],[TO])]))) :-
	   service(SRVC),id(ID),location_type(OUTPUT),location_type(FROM),location_type(TO).

/*FLUENT AT_ROBOT*/
fun_fluent(atRobot(SRVC)) :- service(SRVC).

causes_val(release(SRVC,ID,move,_I,_E,[OUTPUT]),atRobot(SRVC),OUTPUT,
    assigned(SRVC,[workitem(move,ID,[FROM,TO],[TO])])) :-
        service(SRVC),id(ID),location_type(OUTPUT),location_type(FROM),location_type(TO).


/*FLUENT EVACUATED*/
fun_fluent(evacuated(P)) :- location_type(P).

causes_val(release(SRVC,ID,evacuate,_I,_E,[OUTPUT]),evacuated(LOC),OUTPUT,
   assigned(SRVC,[workitem(evacuate,ID,[LOC],[true])])) :-
        service(SRVC),id(ID),location_type(LOC),boolean_type(true),boolean_type(OUTPUT).

fun_fluent(evacuated_exp(P)) :- location_type(P).

causes_val(release(SRVC,ID,evacuate,_I,_E,[OUTPUT]),evacuated_exp(LOC),true,
   and(neg(adapting),assigned(SRVC,[workitem(evacuate,ID,[LOC],[true])]))) :-
        service(SRVC),id(ID),location_type(LOC),boolean_type(true),boolean_type(OUTPUT).

/*FLUENT STATUS*/
fun_fluent(status(P)) :- location_type(P).

causes_val(release(SRVC,ID,extinguishfire,_I,_E,[OUTPUT]),status(LOC),OUTPUT,
    assigned(SRVC,[workitem(extinguishfire,ID,[LOC],[ok])])) :-
            location_type(LOC),status_type(ok),status_type(OUTPUT).

causes_val(release(SRVC,ID,removedebris,_I,_E,[OUTPUT]),status(LOC),OUTPUT,
    assigned(SRVC,[workitem(removedebris,ID,[LOC],[ok])])) :-
        location_type(LOC),status_type(ok),status_type(OUTPUT).

causes_val(release(SRVC,ID,updatestatus,_I,_E,[OUTPUT]),status(LOC),OUTPUT,
    assigned(SRVC,[workitem(updatestatus,ID,[LOC],[ok])])) :-
        location_type(LOC),status_type(ok),status_type(OUTPUT).

fun_fluent(status_exp(P)) :- location_type(P).

causes_val(release(SRVC,ID,extinguishfire,_I,_E,[OUTPUT]),status_exp(LOC),ok,
    and(assigned(SRVC,[workitem(extinguishfire,ID,[LOC],[ok])]),neg(adapting))) :-
        location_type(LOC),status_type(ok),status_type(OUTPUT).

causes_val(release(SRVC,ID,removedebris,_I,_E,[OUTPUT]),status_exp(LOC),ok,
    and(assigned(SRVC,[workitem(removedebris,ID,[LOC],[ok])]),neg(adapting))) :-
        location_type(LOC),status_type(ok),status_type(OUTPUT).

causes_val(release(SRVC,ID,updatestatus,_I,_E,[OUTPUT]),status_exp(LOC),ok,
    and(assigned(SRVC,[workitem(updatestatus,ID,[LOC],[ok])]),neg(adapting))) :-
	   location_type(LOC),status_type(ok),status_type(OUTPUT).

/*FLUENT BATTERYLEVEL*/
fun_fluent(batteryLevel(SRVC)) :- service(SRVC).

causes_val(release(SRVC,ID,move,_I,_E,[OUTPUT]),batteryLevel(SRVC),L,
    and(assigned(SRVC,[workitem(move,ID,[FROM,TO],[TO])]),
                            L is batteryLevel(SRVC)-moveStep)) :-
        service(SRVC),location_type(OUTPUT),location_type(FROM),location_type(TO).

causes_val(release(SRVC,ID,removedebris,_I,_E,[OUTPUT]),batteryLevel(SRVC),L,
    and(assigned(SRVC,[workitem(removedebris,ID,[LOC],[ok])]),
                                L is batteryLevel(SRVC)-debrisStep)) :-
        service(SRVC),location_type(LOC),status_type(ok),status_type(OUTPUT).

causes_val(release(SRVC,ID,chargebattery,_I,_E,[]),batteryLevel(RB),L,
    and(assigned(SRVC,[workitem(chargebattery,ID,[RB],[])]),
                            L is batteryLevel(RB)+batteryRecharging)) :-
        service(SRVC),service(RB).

/*FLUENT PHOTO TAKEN*/
fun_fluent(photoTaken(P)) :- location_type(P).

causes_val(release(SRVC,ID,takephoto,_I,_E,[OUTPUT]),photoTaken(LOC),OUTPUT,
  	assigned(SRVC,[workitem(takephoto,ID,[LOC],[true])])) :-
        	service(SRVC),id(ID),location_type(LOC),boolean_type(true),boolean_type(OUTPUT).

fun_fluent(photoTaken_exp(P)) :- location_type(P).

causes_val(release(SRVC,ID,takephoto,_I,_E,[OUTPUT]),photoTaken_exp(LOC),true,
  	and(neg(adapting),assigned(SRVC,[workitem(takephoto,ID,[LOC],[true])]))) :-
        	service(SRVC),id(ID),location_type(LOC),boolean_type(true),boolean_type(OUTPUT).

/*FLUENT GENERAL BATTERY*/
fun_fluent(generalBattery).

causes_val(release(SRVC,ID,chargebattery,_I,_E,[]),generalBattery,L,
	and(assigned(SRVC,[workitem(chargebattery,ID,[RB],[])]),L is generalBattery-batteryRecharging)) :-
        	service(SRVC),id(ID),service(RB).

/*FLUENT BATTERY RECHARGING*/
fun_fluent(batteryRecharging).

/*FLUENT MOVE STEP*/
fun_fluent(moveStep).

/*FLUENT DEBRIS STEP*/
fun_fluent(debrisStep).


/*ABBREVIATIONS*/

proc(isConnected(SRVC),
	and(provides(SRVC,movement),	
	or(covered(at(SRVC)),
    some(rb,and(service(rb),and(provides(rb,battery),
        and(or(neigh(at(SRVC),atRobot(rb)),
        at(SRVC)=atRobot(rb)),isRobotConnected(rb))))))
	)).


proc(isRobotConnected(RBT),
  and(service(RBT),and(provides(RBT,battery),
    or(covered(atRobot(RBT)),
    or(some(loc,and(location_type(loc),and(covered(loc),neigh(atRobot(RBT),loc)))),
    or(some(rb,and(service(rb),
        and(provides(rb,battery),and(neg(rb=RBT),and(neigh(atRobot(RBT),atRobot(rb)),
	or(covered(atRobot(rb)),
        some(lc2,and(location_type(lc2),and(covered(lc2),neigh(atRobot(rb),lc2)))))))))),
       		some(rb,and(service(rb),and(provides(rb,battery),and(neg(rb=RBT),
	    		some(lc,and(location_type(lc),and(neigh(atRobot(RBT),lc),and(neigh(atRobot(rb),lc),
			or(covered(atRobot(rb)),
                some(lc2,and(location_type(lc2),and(covered(lc2),neigh(atRobot(rb),lc2)))))
)))))))))
				    )
				  )
	    )	
   )
).

/* INITIAL STATE */

initially(reserved(SRVC,ID,TASK),false) :- task(TASK), service(SRVC), id(ID).
initially(assigned(SRVC,ID,TASK),false) :- task(TASK), service(SRVC), id(ID).
initially(finished,false).
initially(adapting,false).
initially(built_in_adaptation,false).

initially(free(act1),true).
initially(free(act2),true).
initially(free(act3),false).
initially(free(act4),false).
initially(free(rb1),true).
initially(free(rb2),true).

initially(batteryRecharging,10).
initially(moveStep,2).
initially(debrisStep,3).
initially(generalBattery,30).

initially(at(act1),loc00).
initially(at_exp(act1),loc00).
initially(at(act2),loc00).
initially(at_exp(act2),loc00).
initially(at(act3),loc00).
initially(at_exp(act3),loc00).
initially(at(act4),loc00).
initially(at_exp(act4),loc00).

initially(atRobot(rb1),loc00).
initially(atRobot(rb2),loc00).

initially(evacuated(loc00),false).
initially(evacuated_exp(loc00),false).
initially(evacuated(loc01),false).
initially(evacuated_exp(loc01),false).
initially(evacuated(loc02),false).
initially(evacuated_exp(loc02),false).
initially(evacuated(loc03),false).
initially(evacuated_exp(loc03),false).
initially(evacuated(loc10),false).
initially(evacuated_exp(loc10),false).
initially(evacuated(loc11),false).
initially(evacuated_exp(loc11),false).
initially(evacuated(loc13),false).
initially(evacuated_exp(loc13),false).
initially(evacuated(loc20),false).
initially(evacuated_exp(loc20),false).
initially(evacuated(loc23),false).
initially(evacuated_exp(loc23),false).
initially(evacuated(loc30),false).
initially(evacuated_exp(loc30),false).
initially(evacuated(loc31),false).
initially(evacuated_exp(loc31),false).
initially(evacuated(loc32),false).
initially(evacuated_exp(loc32),false).
initially(evacuated(loc33),false).
initially(evacuated_exp(loc33),false).

initially(photoTaken(loc00),false).
initially(photoTaken_exp(loc00),false).
initially(photoTaken(loc01),false).
initially(photoTaken_exp(loc01),false).
initially(photoTaken(loc02),false).
initially(photoTaken_exp(loc02),false).
initially(photoTaken(loc03),false).
initially(photoTaken_exp(loc03),false).
initially(photoTaken(loc10),false).
initially(photoTaken_exp(loc10),false).
initially(photoTaken(loc11),false).
initially(photoTaken_exp(loc11),false).
initially(photoTaken(loc13),false).
initially(photoTaken_exp(loc13),false).
initially(photoTaken(loc20),false).
initially(photoTaken_exp(loc20),false).
initially(photoTaken(loc23),false).
initially(photoTaken_exp(loc23),false).
initially(photoTaken(loc30),false).
initially(photoTaken_exp(loc30),false).
initially(photoTaken(loc31),false).
initially(photoTaken_exp(loc31),false).
initially(photoTaken(loc32),false).
initially(photoTaken_exp(loc32),false).
initially(photoTaken(loc33),false).
initially(photoTaken_exp(loc33),false).

initially(status(loc00),ok).
initially(status_exp(loc00),ok).
initially(status(loc01),ok).
initially(status_exp(loc01),ok).
initially(status(loc02),ok).
initially(status_exp(loc02),ok).
initially(status(loc03),ok).
initially(status_exp(loc03),ok).
initially(status(loc10),ok).
initially(status_exp(loc10),ok).
initially(status(loc11),ok).
initially(status_exp(loc11),ok).
initially(status(loc13),ok).
initially(status_exp(loc13),ok).
initially(status(loc20),ok).
initially(status_exp(loc20),ok).
initially(status(loc23),ok).
initially(status_exp(loc23),ok).
initially(status(loc30),ok).
initially(status_exp(loc30),ok).
initially(status(loc31),ok).
initially(status_exp(loc31),ok).
initially(status(loc32),ok).
initially(status_exp(loc32),ok).
initially(status(loc33),ok).
initially(status_exp(loc33),ok).

initially(batteryLevel(act1),0).
initially(batteryLevel(act2),0).
initially(batteryLevel(act3),0).
initially(batteryLevel(act4),0).
initially(batteryLevel(rb1),15).
initially(batteryLevel(rb2),15).

initially(covered(loc00),true).
initially(covered(loc10),true).
initially(covered(loc20),true).
initially(covered(loc11),true).
initially(covered(loc01),true).
initially(covered(loc02),true).

initially(neigh(loc00,loc10),true).
initially(neigh(loc00,loc11),true).
initially(neigh(loc00,loc01),true).
initially(neigh(loc11,loc10),true).
initially(neigh(loc11,loc01),true).
initially(neigh(loc11,loc00),true).
initially(neigh(loc11,loc20),true).
initially(neigh(loc11,loc02),true).
initially(neigh(loc10,loc20),true).
initially(neigh(loc10,loc00),true).
initially(neigh(loc10,loc11),true).
initially(neigh(loc10,loc01),true).
initially(neigh(loc01,loc02),true).
initially(neigh(loc01,loc11),true).
initially(neigh(loc01,loc10),true).
initially(neigh(loc01,loc00),true).
initially(neigh(loc02,loc03),true).
initially(neigh(loc02,loc13),true).
initially(neigh(loc02,loc01),true).
initially(neigh(loc02,loc11),true).
initially(neigh(loc03,loc02),true).
initially(neigh(loc03,loc13),true).
initially(neigh(loc13,loc03),true).
initially(neigh(loc13,loc23),true).
initially(neigh(loc13,loc02),true).
initially(neigh(loc23,loc13),true).
initially(neigh(loc23,loc33),true).
initially(neigh(loc23,loc32),true).
initially(neigh(loc33,loc23),true).
initially(neigh(loc33,loc32),true).
initially(neigh(loc32,loc33),true).
initially(neigh(loc32,loc23),true).
initially(neigh(loc32,loc31),true).
initially(neigh(loc31,loc32),true).
initially(neigh(loc31,loc20),true).
initially(neigh(loc31,loc30),true).
initially(neigh(loc30,loc31),true).
initially(neigh(loc30,loc20),true).
initially(neigh(loc20,loc30),true).
initially(neigh(loc20,loc31),true).
initially(neigh(loc20,loc10),true).
initially(neigh(loc20,loc11),true).

/* EXOGENOUS EVENTS USED FOR SWITCHING A TASK FROM A STATE TO ANOTHER */
prim_action(A) :- exog_action(A).
poss(A,true) :- exog_action(A).

exog_action(readyToStart(SRVC,ID,TASK)) :- task(TASK), service(SRVC), id(ID).
exog_action(finishedTask(SRVC,ID,TASK,_V)) :- task(TASK), service(SRVC), id(ID).

/* INTERNAL EXOGENOUS ACTION FOR LOADING THE RECOVERY PROCEDURE */
exog_action(planReady(_PLAN)).

/* DOMAIN-DEPENDENT EXOGENOUS EVENTS */

exog_action(photoLost(LOC)) :- location_type(LOC).
causes_val(photoLost(LOC),photoTaken(LOC),false,true).

exog_action(fireRisk(LOC)) :- location_type(LOC).
causes_val(fireRisk(LOC),status(LOC),fire,true).

exog_action(rockSlide(LOC)) :- location_type(LOC).
causes_val(rockSlide(LOC),status(LOC),debris,true).

/* INDIGOLOG INTERNAL ACTIONS */

prim_action(initPMS).
poss(initPMS, true).

prim_action(endPMS).
poss(endPMS, true).

prim_action(finish).
poss(finish,true).
rel_fluent(finished).
causes_val(finish,finished,true,true).

rel_fluent(realityChanged).
initially(realityChanged,false).
causes_val(release(_,_),realityChanged,true,neg(adapting)).

prim_action(resetReality).
poss(resetReality,true).
causes_val(resetReality,realityChanged,false,true).

prim_action(adaptFinish).
poss(adaptFinish,true).

prim_action(adaptStart).
poss(adaptStart,true).

prim_action(adaptFound).
poss(adaptFound,true).

prim_action(adaptNotFound).
poss(adaptNotFound,true).

/*COMMANDS FOR INVOKING THE PLANNER*/
prim_action(invokePlanner).
poss(invokePlanner,true).

rel_fluent(adaptationPlanReady).
initially(adaptationPlanReady,true).
causes_val(invokePlanner,adaptationPlanReady,false,true).
causes_val(planReady(_P),adaptationPlanReady,true,true).

fun_fluent(recoveryPlan).
initially(recoveryPlan,[]).
causes_val(planReady(P),recoveryPlan,P,true).

prim_action(stop).
poss(stop,false).

rel_fluent(adapting).
causes_val(adaptStart,adapting,true,true).
causes_val(adaptFound,adapting,false,true).

/* DEFINITION OF THE RELEVANT FLUENT */
proc(relevant,or(some(srvc,and(service(srvc),and(provides(srvc,movement),
    or(neg(at(srvc)=at_exp(srvc)),neg(isConnected(srvc)))))),
    some(poi,and(location_type(poi),or(neg(evacuated_exp(poi)=evacuated(poi)),
    or(neg(photoTaken_exp(poi)=photoTaken(poi)),neg(status_exp(poi)=status(poi)))))))).

/* THE MONITOR PROCEDURE */
proc(monitor,[?(writeln('Monitor')),ndet([?(neg(relevant)),?(writeln('NonRelevant'))],
					 [?(relevant),?(writeln('Relevant')),
					 [adaptStart,createPlanningProblem,?(report_message(user, 'About to adapt...')),
                     if(built_in_adaptation,pconc([adaptingProgram,adaptFinish],
                        while(adapting,[?(writeln('waiting')),wait])),
					 [invokePlanner,manageRecoveryProcess(recoveryPlan),adaptFound,adaptFinish])
					 ]]),
					 resetReality]).

/* THE BUILT-IN ADAPTATION PROCEDURE */
proc(adaptingProgram,searchn([?(true),adapt,[adaptFound,
        ?(report_message(user, 'Adaptation program found!'))]],			
        [assumptions([[assign(N,D,T,I,E)]),readyToStart(N,D,T)],
                              [start(N,D,T),finishedTask(N,D,T,E)]])])).

proc(adapt,plans(0,10)).

proc(plans(M,N),[if(M=N,[adaptNotFound],[?(M<(N+1)),
    ndet([actionSequence(M),?(neg(relevant))],
    [?(SUCCM is M+1),plans(SUCCM,N)])])]).

proc(actionSequence(N),ndet([?(N=0)],[?(N>0),
	pi([n,t,i,e],[ ?(and(service(n),and(free(n),capable(n,[workitem(t,id_adapt,i,e)])))),
	assign(n,id_adapt,t,i,e),
	start(n,id_adapt,t),
	ackCompl(n,id_adapt,t),
	release(n,id_adapt,t,i,e),
	?(PRECN is N-1), actionSequence(PRECN)])
])).

proc(capable(SRVC,TASK),and(findall(CAP,requires(TASK,CAP),D),
    and(findall(CAP,provides(SRVC,CAP),C),subset(D,C)))).

/* THE MANAGE EXECUTION PROCEDURE */

proc(manageExecution(X),
    [atomic([pi(n,[?and(free(n),capable(n,X))),
    assign(n,X)])]),pi(n,[?(assigned(n,_d,X)=true),executionHelp(n,X)]),
    atomic([pi(n,[?(assigned(n,_d,X)=true),[release(n,X),printALL]])])]).

proc(executionHelp(_N,[]),[]).
proc(executionHelp(N,[workitem(T,D,I,E)|TAIL]),
[start(N,D,T), ackCompl(N,D,T), executionHelp(N,TAIL)]).


proc(manageRecoveryProcess([]),[]).
proc(manageRecoveryProcess(N),[?(N=[recoveryTask(TASKNAME,SERVICE,INPUTS)|TAIL]),
					manageRecoveryTask(TASKNAME,SERVICE,INPUTS),manageRecoveryProcess(TAIL)]).

proc(manageRecoveryTask(TASKNAME,SERVICE,INPUTS),
[pi(o,[?(and(A=workitem(TASKNAME,_D,INPUTS,o),listelem(A))),
manageExecution(SERVICE,[workitem(TASKNAME,id_adapt,INPUTS,o)])])]).

proc(manageExecution(S,X),[assign(S,X),pi(n,[?(assigned(n,X)=true),
executionHelp(n,X)]),atomic([pi(n,[?(assigned(n,X)=true),[release(n,X),printALL]])])]).

proc(main, mainControl(N)) :- controller(N), !.
proc(main, mainControl(5)).

proc(mainControl(5), prioritized_interrupts([
interrupt(and(neg(finished),neg(adaptationPlanReady)),
    [?(writeln('>>>>>>>>>>>> Waiting for a recovery plan...')), wait]),
interrupt(and(neg(finished),realityChanged), monitor),
interrupt(true, [process,finish]),
interrupt(neg(finished), wait)])).

/* THE INDIGOLOG MAIN PROCESS */

proc(process,[initPMS,
rrobin([
[manageExecution([workitem(go,id_1,[loc00,loc33],[loc33])]),
manageExecution([workitem(takephoto,id_2,[loc33],[true])]),
 manageExecution([workitem(updatestatus,id_3,[loc33],[ok])])],
rrobin([
[manageExecution([workitem(go,id_4,[loc00,loc32],[loc32])]),
manageExecution([workitem(evacuate,id_5,[loc32],[true])]),
 manageExecution([workitem(updatestatus,id_6,[loc32],[ok])])],
[manageExecution([workitem(go,id_7,[loc00,loc31],[loc31])]),
manageExecution([workitem(evacuate,id_8,[loc31],[true])]),
 manageExecution([workitem(updatestatus,id_9,[loc31],[ok])])]
])]),
endPMS]).

/* CREATION OF THE PLANNING PROBLEM WHEN A DEVIATION IS SENSED */

proc(createPlanningProblem,[writeFile('Planners/LPG-TD/problem.pddl')]).

proc(writeFile(File),[?(open(File, write, Stream)),
?(writeln(Stream,('(define (problem EM1) (:domain Derailment)'))),
?(writeln(Stream,('(:objects'))),
?(writeln(Stream,('act1 - service'))),
?(writeln(Stream,('act2 - service'))),
?(writeln(Stream,('act3 - service'))),
?(writeln(Stream,('act4 - service'))),
?(writeln(Stream,('rb1 - service'))),
?(writeln(Stream,('rb2 - service'))),
?(writeln(Stream,('ok - status_type'))),
?(writeln(Stream,('fire - status_type'))),
?(writeln(Stream,('debris - status_type'))),
?(writeln(Stream,('movement - capability'))),
?(writeln(Stream,('hatchet - capability'))),	
?(writeln(Stream,('camera - capability'))),	
?(writeln(Stream,('gprs - capability'))),	
?(writeln(Stream,('extinguisher - capability'))),		
?(writeln(Stream,('battery - capability'))),	
?(writeln(Stream,('digger - capability'))),	
?(writeln(Stream,('powerpack - capability'))),
?(writeln(Stream,('loc00 - location_type'))),
?(writeln(Stream,('loc01 - location_type'))),
?(writeln(Stream,('loc02 - location_type'))),
?(writeln(Stream,('loc03 - location_type'))),
?(writeln(Stream,('loc10 - location_type'))),
?(writeln(Stream,('loc11 - location_type'))),
?(writeln(Stream,('loc13 - location_type'))),
?(writeln(Stream,('loc20 - location_type'))),
?(writeln(Stream,('loc23 - location_type'))),
?(writeln(Stream,('loc30 - location_type'))),
?(writeln(Stream,('loc31 - location_type'))),
?(writeln(Stream,('loc32 - location_type'))),
?(writeln(Stream,('loc33 - location_type'))),	
?(writeln(Stream,(')'))),
?(writeln(Stream,('(:init'))),
if(free(act1)=true,?(writeln(Stream,('(free act1)'))),[]),
if(free(act2)=true,?(writeln(Stream,('(free act2)'))),[]),
if(free(act1)=true,?(writeln(Stream,('(free act3)'))),[]),
if(free(act2)=true,?(writeln(Stream,('(free act4)'))),[]),
if(free(rb1)=true,?(writeln(Stream,('(free rb1)'))),[]),
if(free(rb2)=true,?(writeln(Stream,('(free rb2)'))),[]),
?(writeln(Stream,('(provides act1 movement)'))),
?(writeln(Stream,('(provides act1 gprs)'))),
?(writeln(Stream,('(provides act1 camera)'))),
?(writeln(Stream,('(provides act1 extinguisher)'))),
?(writeln(Stream,('(provides act2 movement)'))),
?(writeln(Stream,('(provides act2 gprs)'))),
?(writeln(Stream,('(provides act2 camera)'))),
?(writeln(Stream,('(provides act2 hatchet)'))),
?(writeln(Stream,('(provides act3 movement)'))),
?(writeln(Stream,('(provides act3 gprs)'))),
?(writeln(Stream,('(provides act3 camera)'))),
?(writeln(Stream,('(provides act3 hatchet)'))),
?(writeln(Stream,('(provides act4 movement)'))),	
?(writeln(Stream,('(provides act4 powerpack)'))),	
?(writeln(Stream,('(provides act4 gprs)'))),	
?(writeln(Stream,('(provides rb1 battery)'))),	
?(writeln(Stream,('(provides rb1 digger)'))),		
?(writeln(Stream,('(provides rb2 battery)'))),	
?(writeln(Stream,('(provides rb2 digger)'))),		
?(writeln(Stream,('(covered loc00)'))),		
?(writeln(Stream,('(covered loc10)'))),		
?(writeln(Stream,('(covered loc20)'))),		
?(writeln(Stream,('(covered loc01)'))),			
?(writeln(Stream,('(covered loc11)'))),		
?(writeln(Stream,('(covered loc02)'))),		
?(writeln(Stream,('(neigh loc00 loc10)'))),		
?(writeln(Stream,('(neigh loc00 loc11)'))),	
?(writeln(Stream,('(neigh loc00 loc01)'))),		
?(writeln(Stream,('(neigh loc11 loc10)'))),		
?(writeln(Stream,('(neigh loc11 loc01)'))),		
?(writeln(Stream,('(neigh loc11 loc00)'))),	
?(writeln(Stream,('(neigh loc11 loc20)'))),		
?(writeln(Stream,('(neigh loc11 loc02)'))),		
?(writeln(Stream,('(neigh loc10 loc20)'))),			
?(writeln(Stream,('(neigh loc10 loc00)'))),		
?(writeln(Stream,('(neigh loc10 loc11)'))),	
?(writeln(Stream,('(neigh loc10 loc01)'))),		
?(writeln(Stream,('(neigh loc01 loc02)'))),			
?(writeln(Stream,('(neigh loc01 loc11)'))),
?(writeln(Stream,('(neigh loc01 loc10)'))),	
?(writeln(Stream,('(neigh loc01 loc00)'))),	
?(writeln(Stream,('(neigh loc02 loc03)'))),		
?(writeln(Stream,('(neigh loc02 loc13)'))),		
?(writeln(Stream,('(neigh loc02 loc01)'))),		
?(writeln(Stream,('(neigh loc02 loc11)'))),		
?(writeln(Stream,('(neigh loc03 loc02)'))),			
?(writeln(Stream,('(neigh loc03 loc13)'))),	
?(writeln(Stream,('(neigh loc13 loc03)'))),	
?(writeln(Stream,('(neigh loc13 loc23)'))),	
?(writeln(Stream,('(neigh loc13 loc02)'))),		
?(writeln(Stream,('(neigh loc23 loc13)'))),		
?(writeln(Stream,('(neigh loc23 loc33)'))),	
?(writeln(Stream,('(neigh loc23 loc32)'))),		
?(writeln(Stream,('(neigh loc33 loc23)'))),	
?(writeln(Stream,('(neigh loc33 loc32)'))),	
?(writeln(Stream,('(neigh loc32 loc33)'))),	
?(writeln(Stream,('(neigh loc32 loc23)'))),	
?(writeln(Stream,('(neigh loc32 loc31)'))),	
?(writeln(Stream,('(neigh loc31 loc32)'))),	
?(writeln(Stream,('(neigh loc31 loc20)'))),		
?(writeln(Stream,('(neigh loc31 loc30)'))),
?(writeln(Stream,('(neigh loc30 loc31)'))),	
?(writeln(Stream,('(neigh loc30 loc20)'))),
?(writeln(Stream,('(neigh loc20 loc30)'))),
?(writeln(Stream,('(neigh loc20 loc31)'))),
?(writeln(Stream,('(neigh loc20 loc10)'))),	
?(writeln(Stream,('(neigh loc20 loc11)'))),
?(write(Stream,('('))),?(write(Stream,('at act1 '))),
?(write(Stream,(at(act1)))),?(writeln(Stream,(')'))),
?(write(Stream,('('))),?(write(Stream,('at act2 '))),
?(write(Stream,(at(act2)))),?(writeln(Stream,(')'))),
?(write(Stream,('('))),?(write(Stream,('at act3 '))),
?(write(Stream,(at(act3)))),?(writeln(Stream,(')'))),
?(write(Stream,('('))),?(write(Stream,('at act4 '))),
?(write(Stream,(at(act4)))),?(writeln(Stream,(')'))),
?(write(Stream,('('))),?(write(Stream,('atRobot rb1 '))),
?(write(Stream,(atRobot(rb1)))),?(writeln(Stream,(')'))),
?(write(Stream,('('))),?(write(Stream,('atRobot rb2 '))),
?(write(Stream,(atRobot(rb2)))),?(writeln(Stream,(')'))),
?(write(Stream,('(= '))),?(write(Stream,('(batteryLevel rb1) '))),
?(write(Stream,(batteryLevel(rb1)))),?(writeln(Stream,(')'))),
?(write(Stream,('(= '))),?(write(Stream,('(batteryLevel rb2) '))),
?(write(Stream,(batteryLevel(rb2)))),?(writeln(Stream,(')'))),
?(write(Stream,('(= '))),?(write(Stream,('(batteryRecharging) '))),
?(write(Stream,(batteryRecharging))),?(writeln(Stream,(')'))),
?(write(Stream,('(= '))),?(write(Stream,('(generalBattery) '))),
?(write(Stream,(generalBattery))),?(writeln(Stream,(')'))),
?(write(Stream,('(= '))),?(write(Stream,('(debrisStep) '))),
?(write(Stream,(debrisStep))),?(writeln(Stream,(')'))),
?(write(Stream,('(= '))),?(write(Stream,('(moveStep) '))),
?(write(Stream,(moveStep))),?(writeln(Stream,(')'))),
if(photoTaken(loc00)=true,?(writeln(Stream,('(photoTaken loc00)'))),[]),
if(photoTaken(loc01)=true,?(writeln(Stream,('(photoTaken loc01)'))),[]),
if(photoTaken(loc02)=true,?(writeln(Stream,('(photoTaken loc02)'))),[]),
if(photoTaken(loc03)=true,?(writeln(Stream,('(photoTaken loc03)'))),[]),
if(photoTaken(loc10)=true,?(writeln(Stream,('(photoTaken loc10)'))),[]),
if(photoTaken(loc11)=true,?(writeln(Stream,('(photoTaken loc11)'))),[]),
if(photoTaken(loc13)=true,?(writeln(Stream,('(photoTaken loc13)'))),[]),
if(photoTaken(loc20)=true,?(writeln(Stream,('(photoTaken loc20)'))),[]),
if(photoTaken(loc23)=true,?(writeln(Stream,('(photoTaken loc23)'))),[]),
if(photoTaken(loc30)=true,?(writeln(Stream,('(photoTaken loc30)'))),[]),
if(photoTaken(loc31)=true,?(writeln(Stream,('(photoTaken loc31)'))),[]),
if(photoTaken(loc32)=true,?(writeln(Stream,('(photoTaken loc32)'))),[]),
if(photoTaken(loc33)=true,?(writeln(Stream,('(photoTaken loc33)'))),[]),
?(write(Stream,('('))),?(write(Stream,('status loc00 '))),
?(write(Stream,(status(loc00)))),?(writeln(Stream,(')'))),
?(write(Stream,('('))),?(write(Stream,('status loc01 '))),
?(write(Stream,(status(loc01)))),?(writeln(Stream,(')'))),
?(write(Stream,('('))),?(write(Stream,('status loc02 '))),
?(write(Stream,(status(loc02)))),?(writeln(Stream,(')'))),
?(write(Stream,('('))),?(write(Stream,('status loc03 '))),
?(write(Stream,(status(loc03)))),?(writeln(Stream,(')'))),
?(write(Stream,('('))),?(write(Stream,('status loc10 '))),
?(write(Stream,(status(loc10)))),?(writeln(Stream,(')'))),
?(write(Stream,('('))),?(write(Stream,('status loc11 '))),
?(write(Stream,(status(loc11)))),?(writeln(Stream,(')'))),
?(write(Stream,('('))),?(write(Stream,('status loc13 '))),
?(write(Stream,(status(loc13)))),?(writeln(Stream,(')'))),
?(write(Stream,('('))),?(write(Stream,('status loc20 '))),
?(write(Stream,(status(loc20)))),?(writeln(Stream,(')'))),
?(write(Stream,('('))),?(write(Stream,('status loc23 '))),
?(write(Stream,(status(loc23)))),?(writeln(Stream,(')'))),
?(write(Stream,('('))),?(write(Stream,('status loc30 '))),
?(write(Stream,(status(loc30)))),?(writeln(Stream,(')'))),
?(write(Stream,('('))),?(write(Stream,('status loc31 '))),
?(write(Stream,(status(loc31)))),?(writeln(Stream,(')'))),
?(write(Stream,('('))),?(write(Stream,('status loc32 '))),
?(write(Stream,(status(loc32)))),?(writeln(Stream,(')'))),
?(write(Stream,('('))),?(write(Stream,('status loc33 '))),
?(write(Stream,(status(loc33)))),?(writeln(Stream,(')'))),
if(evacuated(loc00)=true,?(writeln(Stream,('(evacuated loc00)'))),[]),
if(evacuated(loc01)=true,?(writeln(Stream,('(evacuated loc01)'))),[]),
if(evacuated(loc02)=true,?(writeln(Stream,('(evacuated loc02)'))),[]),
if(evacuated(loc03)=true,?(writeln(Stream,('(evacuated loc03)'))),[]),
if(evacuated(loc10)=true,?(writeln(Stream,('(evacuated loc10)'))),[]),
if(evacuated(loc11)=true,?(writeln(Stream,('(evacuated loc11)'))),[]),
if(evacuated(loc13)=true,?(writeln(Stream,('(evacuated loc13)'))),[]),
if(evacuated(loc20)=true,?(writeln(Stream,('(evacuated loc20)'))),[]),
if(evacuated(loc23)=true,?(writeln(Stream,('(evacuated loc23)'))),[]),
if(evacuated(loc30)=true,?(writeln(Stream,('(evacuated loc30)'))),[]),
if(evacuated(loc31)=true,?(writeln(Stream,('(evacuated loc31)'))),[]),
if(evacuated(loc32)=true,?(writeln(Stream,('(evacuated loc32)'))),[]),
if(evacuated(loc33)=true,?(writeln(Stream,('(evacuated loc33)'))),[]),
?(writeln(Stream,(')'))),
?(writeln(Stream,('(:goal'))),
?(writeln(Stream,('(and'))),
?(write(Stream,('('))),?(write(Stream,('at act1 '))),
?(write(Stream,(at_exp(act1)))),?(writeln(Stream,(')'))),
?(write(Stream,('('))),?(write(Stream,('at act2 '))),
?(write(Stream,(at_exp(act2)))),?(writeln(Stream,(')'))),
?(write(Stream,('('))),?(write(Stream,('at act3 '))),
?(write(Stream,(at_exp(act3)))),?(writeln(Stream,(')'))),
?(write(Stream,('('))),?(write(Stream,('at act4 '))),
?(write(Stream,(at_exp(act4)))),?(writeln(Stream,(')'))),
?(write(Stream,('('))),?(write(Stream,('status loc00 '))),
?(write(Stream,(status_exp(loc00)))),?(writeln(Stream,(')'))),
?(write(Stream,('('))),?(write(Stream,('status loc01 '))),
?(write(Stream,(status_exp(loc01)))),?(writeln(Stream,(')'))),
?(write(Stream,('('))),?(write(Stream,('status loc02 '))),
?(write(Stream,(status_exp(loc02)))),?(writeln(Stream,(')'))),
?(write(Stream,('('))),?(write(Stream,('status loc03 '))),
?(write(Stream,(status_exp(loc03)))),?(writeln(Stream,(')'))),
?(write(Stream,('('))),?(write(Stream,('status loc10 '))),
?(write(Stream,(status_exp(loc10)))),?(writeln(Stream,(')'))),
?(write(Stream,('('))),?(write(Stream,('status loc11 '))),
?(write(Stream,(status_exp(loc11)))),?(writeln(Stream,(')'))),
?(write(Stream,('('))),?(write(Stream,('status loc13 '))),
?(write(Stream,(status_exp(loc13)))),?(writeln(Stream,(')'))),
?(write(Stream,('('))),?(write(Stream,('status loc20 '))),
?(write(Stream,(status_exp(loc20)))),?(writeln(Stream,(')'))),
?(write(Stream,('('))),?(write(Stream,('status loc23 '))),
?(write(Stream,(status_exp(loc23)))),?(writeln(Stream,(')'))),
?(write(Stream,('('))),?(write(Stream,('status loc30 '))),
?(write(Stream,(status_exp(loc30)))),?(writeln(Stream,(')'))),
?(write(Stream,('('))),?(write(Stream,('status loc31 '))),
?(write(Stream,(status_exp(loc31)))),?(writeln(Stream,(')'))),
?(write(Stream,('('))),?(write(Stream,('status loc32 '))),
?(write(Stream,(status_exp(loc32)))),?(writeln(Stream,(')'))),
?(write(Stream,('('))),?(write(Stream,('status loc33 '))),
?(write(Stream,(status_exp(loc33)))),?(writeln(Stream,(')'))),
if(evacuated_exp(loc00)=true,?(writeln(Stream,('(evacuated loc00)'))),[]),
if(evacuated_exp(loc01)=true,?(writeln(Stream,('(evacuated loc01)'))),[]),
if(evacuated_exp(loc02)=true,?(writeln(Stream,('(evacuated loc02)'))),[]),
if(evacuated_exp(loc03)=true,?(writeln(Stream,('(evacuated loc03)'))),[]),
if(evacuated_exp(loc10)=true,?(writeln(Stream,('(evacuated loc10)'))),[]),
if(evacuated_exp(loc11)=true,?(writeln(Stream,('(evacuated loc11)'))),[]),
if(evacuated_exp(loc13)=true,?(writeln(Stream,('(evacuated loc13)'))),[]),
if(evacuated_exp(loc20)=true,?(writeln(Stream,('(evacuated loc20)'))),[]),
if(evacuated_exp(loc23)=true,?(writeln(Stream,('(evacuated loc23)'))),[]),
if(evacuated_exp(loc30)=true,?(writeln(Stream,('(evacuated loc30)'))),[]),
if(evacuated_exp(loc31)=true,?(writeln(Stream,('(evacuated loc31)'))),[]),
if(evacuated_exp(loc32)=true,?(writeln(Stream,('(evacuated loc32)'))),[]),
if(evacuated_exp(loc33)=true,?(writeln(Stream,('(evacuated loc33)'))),[]),
if(photoTaken_exp(loc00)=true,?(writeln(Stream,('(photoTaken loc00)'))),[]),
if(photoTaken_exp(loc01)=true,?(writeln(Stream,('(photoTaken loc01)'))),[]),
if(photoTaken_exp(loc02)=true,?(writeln(Stream,('(photoTaken loc02)'))),[]),
if(photoTaken_exp(loc03)=true,?(writeln(Stream,('(photoTaken loc03)'))),[]),
if(photoTaken_exp(loc10)=true,?(writeln(Stream,('(photoTaken loc10)'))),[]),
if(photoTaken_exp(loc11)=true,?(writeln(Stream,('(photoTaken loc11)'))),[]),
if(photoTaken_exp(loc13)=true,?(writeln(Stream,('(photoTaken loc13)'))),[]),
if(photoTaken_exp(loc20)=true,?(writeln(Stream,('(photoTaken loc20)'))),[]),
if(photoTaken_exp(loc23)=true,?(writeln(Stream,('(photoTaken loc23)'))),[]),
if(photoTaken_exp(loc30)=true,?(writeln(Stream,('(photoTaken loc30)'))),[]),
if(photoTaken_exp(loc31)=true,?(writeln(Stream,('(photoTaken loc31)'))),[]),
if(photoTaken_exp(loc32)=true,?(writeln(Stream,('(photoTaken loc32)'))),[]),
if(photoTaken_exp(loc33)=true,?(writeln(Stream,('(photoTaken loc33)'))),[]),
?(writeln(Stream,('(isConnected act1)'))),
?(writeln(Stream,('(isConnected act2)'))),
?(writeln(Stream,('(isConnected act3)'))),
?(writeln(Stream,('(isConnected act4)'))),
?(writeln(Stream,('))'))),
?(writeln(Stream,('(:metric minimize (total-time))'))),
?(writeln(Stream,(')'))),
?(nl(Stream)),?(close(Stream))]).

\end{alltt}
\end{footnotesize}

\section*{PDDL Planning Domain}

\begin{footnotesize}
\begin{alltt}
(define (domain Derailment)
(:requirements :derived-predicates :typing :fluents :equality)
(:types service capability location_type status_type)

(:predicates
(free ?x - service)
(provides ?x - service ?c - capability)
(neigh ?y1 - location_type ?y2 - location_type)
(covered ?y - location_type)
(at ?x - service ?y - location_type)
(atRobot ?x - service ?y - location_type)
(status ?y - location_type ?s - status_type)
(evacuated ?y - location_type)
(photoTaken ?l - location_type)
(isRobotConnected ?x - service)
(isConnected ?x - service)
)

(:functions
(batteryLevel ?x - service)
(moveStep)
(debrisStep)
(generalBattery)
(batteryRecharging)
)

(:action go
:parameters (?x - service ?from - location_type ?to - location_type)
:precondition (and (provides ?x movement) (free ?x) (at ?x ?from) (isConnected ?x))
:effect (and (not (at ?x ?from)) (at ?x ?to))
)

(:action move
:parameters (?x - service ?from - location_type ?to - location_type)
:precondition (and (provides ?x battery) (free ?x) (atRobot ?x ?from) (>= (batteryLevel ?x)
                   (moveStep)) (isRobotConnected ?x))
:effect (and (not (atRobot ?x ?from)) (atRobot ?x ?to) (decrease (batteryLevel ?x) (moveStep)))
)

(:action extinguishfire
:parameters (?x - service ?y - location_type)
:precondition (and (provides ?x extinguisher) (free ?x) (at ?x ?y)
                   (status ?y fire) (isConnected ?x))
:effect (status ?y ok)
)

(:action evacuate
:parameters (?x - service ?y - location_type)
:precondition (and (provides ?x hatchet) (free ?x) (at ?x ?y) (status ?y ok)
                   (not (evacuated ?y)) (isConnected ?x))
:effect (evacuated ?y)
)

(:action removedebris
:parameters (?x - service ?y - location_type)
:precondition (and (provides ?x digger) (free ?x) (atRobot ?x ?y) (status ?y debris)
              (>= (batteryLevel ?x) (debrisStep)) (isRobotConnected ?x))
:effect (and (status ?y ok) (decrease (batteryLevel ?x) (debrisStep)))
)

(:action chargebattery
:parameters (?x1 - service ?x2 - service ?y - location_type)
:precondition (and (provides ?x1 powerpack) (free ?x1) (at ?x1 ?y) (provides ?x2 battery)
                   (atRobot ?x2 ?y)  (>= (generalBattery) (batteryRecharging)) (isConnected ?x1))
:effect (and (decrease (generalBattery) (batteryRecharging)) 
             (increase (batteryLevel ?x2) (batteryRecharging)))
)

(:action updatestatus
:parameters (?x - service ?y - location_type)
:precondition (and (provides ?x gprs) (free ?x) (at ?x ?y) (status ?y ok) (isConnected ?x))
:effect (status ?y ok)
)

(:action takephoto
:parameters (?x - service ?y - location_type)
:precondition (and (provides ?x camera) (free ?x) (at ?x ?y) (isConnected ?x))
:effect (photoTaken ?y)
)

(:derived (isConnected ?x - service)
	(and (provides ?x movement)
	(or (exists (?y - location_type) (and (at ?x ?y) (covered ?y)))	
            (exists (?r - service ?y - location_type ?z - location_type)
                    (and (provides ?r battery) (at ?x ?y) (atRobot ?r ?z)
						 (or (neigh ?y ?z) (= ?y ?z)) (isRobotConnected ?r)))
									      )
									  )
)

(:derived (isRobotConnected ?x - service) (and (provides ?x battery)
	  (or (exists (?y - location_type)  (and (atRobot ?x ?y) (covered ?y)))
	      (exists (?y - location_type ?z - location_type) 
                (and (atRobot ?x ?y) (covered ?z) (neigh ?y ?z)))
	      (exists (?r - service ?y - location_type ?z - location_type ?k - location_type)
                  (and (provides ?r battery) (not (= ?x ?r))
		      (atRobot ?x ?y) (atRobot ?r ?z) (neigh ?y ?z) (or (covered ?z) 
                    (and (covered ?k) (neigh ?k ?z)))))
	      (exists (?r - service ?y - location_type ?z - location_type 
                        ?k - location_type ?c - location_type)
                  (and (provides ?r battery)
	              (not (= ?x ?r)) (atRobot ?x ?y) (atRobot ?r ?z) (neigh ?y ?k) (neigh ?z ?k)
                       (or (covered ?z) (and (covered ?k) (neigh ?k ?z))) ))
	  )
)
)
)
\end{alltt}
\end{footnotesize}

\section*{PDDL Planning Problem}

\begin{footnotesize}
\begin{alltt}
(define (problem EM1) (:domain Derailment)
(:objects
act1 - service
act2 - service
act3 - service
act4 - service
rb1 - service
rb2 - service
ok - status_type
fire - status_type
debris - status_type
movement - capability
hatchet - capability
camera - capability
gprs - capability
extinguisher - capability
battery - capability
digger - capability
powerpack - capability
loc00 - location_type
loc01 - location_type
loc02 - location_type
loc03 - location_type
loc10 - location_type
loc11 - location_type
loc13 - location_type
loc20 - location_type
loc23 - location_type
loc30 - location_type
loc31 - location_type
loc32 - location_type
loc33 - location_type
)
(:init
(free act1)
(free act2)
(free act3)
(free act4)
(free rb1)
(free rb2)
(provides act1 movement)
(provides act1 gprs)
(provides act1 camera)
(provides act1 extinguisher)
(provides act2 movement)
(provides act2 gprs)
(provides act2 camera)
(provides act2 hatchet)
(provides act3 movement)
(provides act3 gprs)
(provides act3 camera)
(provides act3 hatchet)
(provides act4 movement)
(provides act4 powerpack)
(provides act4 gprs)
(provides rb1 battery)
(provides rb1 digger)
(provides rb2 battery)
(provides rb2 digger)
(covered loc00)
(covered loc10)
(covered loc20)
(covered loc01)
(covered loc11)
(covered loc02)
(neigh loc00 loc10)
(neigh loc00 loc11)
(neigh loc00 loc01)
(neigh loc11 loc10)
(neigh loc11 loc01)
(neigh loc11 loc00)
(neigh loc11 loc20)
(neigh loc11 loc02)
(neigh loc10 loc20)
(neigh loc10 loc00)
(neigh loc10 loc11)
(neigh loc10 loc01)
(neigh loc01 loc02)
(neigh loc01 loc11)
(neigh loc01 loc10)
(neigh loc01 loc00)
(neigh loc02 loc03)
(neigh loc02 loc13)
(neigh loc02 loc01)
(neigh loc02 loc11)
(neigh loc03 loc02)
(neigh loc03 loc13)
(neigh loc13 loc03)
(neigh loc13 loc23)
(neigh loc13 loc02)
(neigh loc23 loc13)
(neigh loc23 loc33)
(neigh loc23 loc32)
(neigh loc33 loc23)
(neigh loc33 loc32)
(neigh loc32 loc33)
(neigh loc32 loc23)
(neigh loc32 loc31)
(neigh loc31 loc32)
(neigh loc31 loc20)
(neigh loc31 loc30)
(neigh loc30 loc31)
(neigh loc30 loc20)
(neigh loc20 loc30)
(neigh loc20 loc31)
(neigh loc20 loc10)
(neigh loc20 loc11)
(at act1 loc03)
(at act2 loc00)
(at act3 loc00)
(at act4 loc00)
(atRobot rb1 loc00)
(atRobot rb2 loc00)
(= (batteryLevel rb1) 15)
(= (batteryLevel rb2) 15)
(= (batteryRecharging) 10)
(= (generalBattery) 30)
(= (debrisStep) 3)
(= (moveStep) 2)
(status loc00 ok)
(status loc01 ok)
(status loc02 ok)
(status loc03 ok)
(status loc10 ok)
(status loc11 ok)
(status loc13 ok)
(status loc20 ok)
(status loc23 ok)
(status loc30 ok)
(status loc31 ok)
(status loc32 ok)
(status loc33 ok)
)
(:goal
(and
(at act1 loc33)
(at act2 loc00)
(at act3 loc00)
(at act4 loc00)
(status loc00 ok)
(status loc01 ok)
(status loc02 ok)
(status loc03 ok)
(status loc10 ok)
(status loc11 ok)
(status loc13 ok)
(status loc20 ok)
(status loc23 ok)
(status loc30 ok)
(status loc31 ok)
(status loc32 ok)
(status loc33 ok)
(isConnected act1)
(isConnected act2)
(isConnected act3)
(isConnected act4)
))
(:metric minimize (total-time))
)
\end{alltt}
\end{footnotesize} 


\backmatter
\cleardoublepage
\bibliographystyle{sapthesis} 
\bibliography{myPapers,state_of_the_art,sitcalc,WfMS_AI,planning} 
\end{sloppypar}
\end{document}